\RequirePackage{etex}
\documentclass[11pt]{article}
\usepackage[top=3cm,bottom=3cm,left=2.5cm,right=2.5cm]{geometry}
\usepackage[round]{natbib}
\renewcommand{\bibname}{References}

\usepackage[normalem]{ulem}
\usepackage[none]{hyphenat}
\usepackage{breakcites}
\usepackage[utf8]{inputenc} %
\usepackage[T1]{fontenc}    %
\usepackage[colorlinks=True,citecolor=blue,urlcolor=blue,pagebackref=true,backref=true,linkcolor=blue]{hyperref}
\renewcommand*{\backrefalt}[4]{%
    \ifcase #1 \footnotesize{(Not cited.)}%
    \or        \footnotesize{(Cited on page~#2.)}%
    \else      \footnotesize{(Cited on pages~#2.)}%
    \fi}
\usepackage{booktabs}       %
\usepackage{amsfonts}       %
\usepackage{nicefrac}       %
\usepackage{microtype}      %
\usepackage{xcolor}         %
\usepackage{caption}
\usepackage{subcaption}
\usepackage{float}
\usepackage{tikz}
\usetikzlibrary{positioning}
\usetikzlibrary{arrows.meta}
\usepackage{amsmath,amsthm}
\usepackage{amssymb}
\usepackage{mathtools}
\usepackage{soul}
\usepackage{bm}
\usepackage{enumitem}
\usepackage{multirow}
\usepackage{wrapfig}
\usepackage{scalerel}
\allowdisplaybreaks
\usepackage{dsfont}
\usepackage[noend]{algpseudocode}
\usepackage{xcolor}
\usepackage{thmtools,thm-restate}
\usepackage{cleveref}  
\usepackage{autonum}
\usepackage{tocloft} %
\usepackage[linesnumbered,ruled,vlined]{algorithm2e}

\SetCommentSty{mycommfont}
\SetKwInput{KwInput}{Input}                %
\SetKwInput{KwOutput}{Output}  
\SetKwInput{KwInitialization}{Initialization}

\newtheorem{remark}{Remark}

\newtheorem{definition}{Definition}

\newtheorem{assumption}{Assumption}
\newtheorem{theorem}{Theorem}
\newtheorem{corollary}{Corollary}
\newtheorem{lemma}{Lemma}
\newtheorem{claim}{Claim}
\newtheorem{proposition}{Proposition}

\newtheorem{example}{Example}

\usepackage[noblocks]{authblk}
\usepackage{graphicx}
\usepackage{comment}
\DeclareMathAlphabet{\mathbsf}{OT1}{cmss}{bx}{n}%
\DeclareMathAlphabet{\mathssf}{OT1}{cmss}{m}{sl}%

\DeclareMathOperator*{\argmin}{arg\,min}

\newcommand{\rv}[1]{\mathssf{#1}}   %
\newcommand{\rvx}{\mathssf{x}}	%
\newcommand{\rvv}{\mathssf{v}}	%
\newcommand{\rvu}{\mathssf{u}}	%
\newcommand{\rvy}{\mathssf{y}}	%
\newcommand{\rvz}{\mathssf{z}}	%
\newcommand{\rva}{\mathssf{a}}	%

\newcommand{\rvb}[1]{\mathbsf{#1}} %
\newcommand{\rvbx}{\mathbsf{x}} %
\newcommand{\rvbv}{\mathbsf{v}} %
\newcommand{\crvbv}{\wbar{\rvbv}} %
\newcommand{\rvbu}{\mathbsf{u}} %
\newcommand{\rvbw}{\mathbsf{w}} %
\newcommand{\rvby}{\mathbsf{y}} %
\newcommand{\rvbz}{\mathbsf{z}} %
\newcommand{\rvba}{\mathbsf{a}} %
\newcommand{\trvbx}{\wtil{\rvbx}} %
\newcommand{\tx}{\wtil{x}} %
\newcommand{\cx}{\wbar{x}} %
\newcommand{\cv}{\wbar{v}} %
\newcommand{\tilp}{\wtil{p}}
\newcommand{\barp}{\wbar{p}}

\newcommand{\svb}[1]{\boldsymbol{#1}} %
\newcommand{\svbx}{\boldsymbol{x}} %
\newcommand{\svbv}{\boldsymbol{v}} %
\newcommand{\csvbv}{\wbar{\svbv}} %
\newcommand{\svbw}{\boldsymbol{w}} %
\newcommand{\svbu}{\boldsymbol{u}} %
\newcommand{\svby}{\boldsymbol{y}} %
\newcommand{\svbz}{\boldsymbol{z}} %
\newcommand{\svba}{\boldsymbol{a}} %
\newcommand{\tsvbx}{\wtil{\boldsymbol{x}}} %
\newcommand{\TrueParameterMatrix}{\Theta^{\star}} %
\newcommand{\ExtendedTrueParameterMatrix}{\underline{\Theta}^{\star}} %

\newcommand{\ParameterMatrix}{\Theta} %
\newcommand{\tParameterMatrix}{\wtil{\Theta}} %
\newcommand{\bParameterMatrix}{\wbar{\Theta}} %
\newcommand{\ExtendedParameterMatrix}{\underline{\Theta}} %
\newcommand{\tExtendedParameterMatrix}{\wtil{\underline{\Theta}}} %
\newcommand{\EstimatedParameterMatrix}{{\what{\Theta}}} %
\newcommand{\ExtendedEstimatedParameterMatrix}{\what{\underline{\Theta}}} %

\newcommand{\ParameterRowt}[1][t]{\Theta_{#1}} %
\newcommand{\tParameterRowt}[1][t]{\wtil{\Theta}_{#1}} %
\newcommand{\tParameterRowttt}[1][t]{\wtil{\Theta}_{#1, -#1}} %
\newcommand{\ParameterRowttt}[1][t]{\Theta_{#1,-#1}} %
\newcommand{\TrueParameterRowt}[1][t]{\Theta_{#1}^{\star}} %
\newcommand{\TrueParameterRowttt}[1][t]{\Theta_{#1,-#1}^{\star}} %
\newcommand{\TrueParameterRowtttTop}[1][t]{\Theta_{#1,-#1}^{\star\top}} %
\newcommand{\EstimatedParameterRowt}[1][t]{\what{\Theta}_{#1}} %
\newcommand{\EstimatedParameterRowttt}[1][t]{\what{\Theta}_{#1,-#1}}
\newcommand{\ExtendedParameterRowT}[1][t]{\underline{\Theta}_{#1}} %
\newcommand{\tExtendedParameterRowT}[1][t]{\wtil{\underline{\Theta}}_{#1}} %
\newcommand{\ExtendedTrueParameterRowT}[1][t]{\underline{\Theta}_{#1}^{\star}} %
\newcommand{\ExtendedEstimatedParameterRowt}[1][t]{\what{\underline{\Theta}}_{t}} %

\newcommand{\TrueParameterTU}[1][tu]{\Theta_{#1}^{\star}} %
\newcommand{\EstimatedParameterTU}[1][tu]{\what{\Theta}_{#1}} %
\newcommand{\tParameterTU}[1][tu]{\wtil{\Theta}_{#1}} %
\newcommand{\ParameterTU}[1][tu]{\Theta_{#1}} %

\newcommand{\ExternalField}{\theta} %
\newcommand{\ExternalFieldt}[1][t]{\theta_#1} %
\newcommand{\ExternalFieldI}[1][i]{\theta^{(#1)}} %
\newcommand{\ExternalFieldtI}[1][i]{\theta_{t}^{(#1)}} %
\newcommand{\ExternalFieldIt}[1][t]{\theta_{#1}^{(i)}} %

\newcommand{\TrueExternalField}{\theta^{\star}} %
\newcommand{\TrueExternalFieldt}[1][t]{\theta^{\star}_#1} %
\newcommand{\TrueExternalFieldI}[1][i]{\theta^{\star(#1)}} %
\newcommand{\TrueExternalFieldtI}[1][i]{\theta^{\star(#1)}_t} %

\newcommand{\tExternalFieldI}[1][i]{\wtil{\theta}^{(#1)}} %
\newcommand{\tExternalFieldtI}[1][i]{\wtil{\theta}_{t}^{(#1)}} %

\newcommand{\bExternalField}{\wbar{\theta}} %
\newcommand{\EstimatedExternalFieldI}[1][i]{\what{\theta}^{(#1)}} %
\newcommand{\ParameterSet}{{\Lambda}} %

\newcommand{\xmax}{x_{\max}} %

\newcommand{\TrueJointDist}{f_{\rvbx|\rvbz}\bigparenth{\svbx| \svbz; \TrueExternalField(\svbz), \TrueParameterMatrix}} %
\newcommand{\TrueJointDistfun}[1][]{f_{\rvbx|\rvbz}\bigparenth{\cdot| \svbz^{#1}; \TrueExternalField(\svbz^{#1}), \TrueParameterMatrix}} %
\newcommand{\JointDist}{f_{\rvbx|\rvbz}\bigparenth{\svbx| \svbz; \ExternalField(\svbz), \ParameterMatrix}}
\newcommand{\JointDistfun}{f_{\rvbx|\rvbz}\bigparenth{\cdot| \svbz; \ExternalField(\svbz), \ParameterMatrix}}
\newcommand{\TrueJointDistfunT}{f_{\rvbx|\rvbz}\bigparenth{\cdot| \svbz^{(i)}; \TrueExternalField(\svbz^{(i)}), \TrueParameterMatrix}} %

\newcommand{\TrueConditionalDistIt}{f_{\rvx_t|\rvbx_{-t}, \rvbz}\bigparenth{x_t^{(i)} | \svbx_{-t}^{(i)}, \svbz^{(i)};  \TrueExternalFieldt(\svbz^{(i)}), \TrueParameterRowt}}
\newcommand{\ConditionalDistT}{f_{\rvx_t | \rvbx_{-t}, \rvbz}(x_t| \svbx_{-t}, \svbz; \ExternalFieldt(\svbz), \ParameterRowt)} %

\newcommand{\ConditioningParameter}[1][tv]{\Upsilon_{#1}} %

\newcommand{\ConditioningField}{\upsilon}
\newcommand{\ConditioningFieldU}[1][t]{\upsilon_{#1}}

\newcommand{\ConditioningMatrix}{\Upsilon}

\newcommand{\cX}{\mathcal{X}} %
\newcommand{\cA}{\mathcal{A}} %
\newcommand{\cZ}{\mathcal{Z}} %
\newcommand{\cR}{\mathcal{R}} %
\newcommand{\cE}{\mathcal{E}} %
\newcommand{\cG}{\mathcal{G}} %
\newcommand{\cN}{\mathcal{N}} %
\newcommand{\cC}{\mathcal{C}} %
\newcommand{\cV}{\mathcal{V}} %
\newcommand{\cU}{\mathcal{U}} %
\newcommand{\cY}{\mathcal{Y}} %
\newcommand{\ball}{\mathcal{B}} %
\newcommand{\metric}{\mathcal{M}} %
\newcommand{\loss}{\mathcal{L}} %
\newcommand{\tmetric}{\wtil{\mathcal{M}}} %

\newcommand{\sumu}[1][u]{\sum_{#1 \in [p]}}
\newcommand{\sump}[1][t]{\sum_{#1 \in [p]}}
\newcommand{\sumset}[1][t]{\sum_{#1 \in \set}}
\newcommand{\sumn}[1][i]{\sum_{#1 \in [n]}}
\newcommand{\maxp}[1][t]{\max_{#1 \in [p]}}

\newcommand{\uOm}{\underline{\Omega}}
\newcommand{\uOmT}[1][t]{\underline{\Omega}_{#1}}
\newcommand{\DeltatIp}{\normalbrackets{\Delta_t^{(i)}}\tp}
\newcommand{\DeltatI}{\Delta_t^{(i)}}

\newcommand{\DeltaItu}[1][u]{\Delta_{t#1}^{(i)}}

\newcommand{\Om}{\Omega}
\newcommand{\Omt}{\Omega_t}
\newcommand{\Omttt}[1][t]{\Omega_{#1,-#1}} %
\newcommand{\Omtu}[1][u]{\Omega_{t#1}}
\newcommand{\om}{{\omega}}
\newcommand{\omt}[1][t]{\omega_{#1}}
\newcommand{\omI}[1][i]{\omega^{(#1)}}
\newcommand{\omtI}[1][i]{\omega_t^{(#1)}}
\newcommand{\omIt}[1][t]{\omega_{#1}^{(i)}}

\newcommand{\directionalGradientFull}{\partial_{\uOm}\loss(\ExtendedParameterMatrix)}
\newcommand{\directionalGradient}{\partial_{\uOmT}\loss_t(\ExtendedParameterRowT)}
\newcommand{\tdirectionalGradient}{\partial_{\uOmT}\loss_t(\tExtendedParameterRowT)}
\newcommand{\directionalGradientTrue}{\partial_{\uOmT}\loss_t(\ExtendedTrueParameterRowT)}
\newcommand{\directionalHessian}{\partial^2_{\uOmT^2} \loss_t(\ExtendedParameterRowT)}
\newcommand{\tdirectionalHessian}{\partial^2_{\uOmT^2} \loss_t(\tExtendedParameterRowT)}

\newcommand{\directionalGradientExternalField}{\partial_{\omI}(\loss^{(i)}(\ExternalFieldI))}
\newcommand{\tdirectionalGradientExternalField}{\partial_{\omI}(\loss^{(i)}(\tExternalFieldI))}
\newcommand{\directionalGradientExternalFieldTrue}{\partial_{\omI}(\loss^{(i)}(\TrueExternalFieldI))}
\newcommand{\directionalHessianExternalField}{\partial^2_{\normalbrackets{\omI}^2} \loss^{(i)}(\ExternalFieldI)}
\newcommand{\tdirectionalHessianExternalField}{\partial^2_{\normalbrackets{\omI}^2} \loss^{(i)}(\tExternalFieldI)}
\newcommand{\Probability}{\mathbb{P}}
\newcommand{\Expectation}{\mathbb{E}}
\newcommand{\Entropy}{h}
\newcommand{\Variance}{\mathbb{V}\text{ar}}
\newcommand{\Covariance}{\mathbb{C}\text{ov}}
\newcommand{\Reals}{\mathbb{R}} %
\newcommand{\real}{\Reals} %
\newcommand{\Indicator}{\mathds{1}}

\newcommand{\aGM}{\alpha}
\newcommand{\bGM}{\beta}
\newcommand{\dGM}{\tau}
\newcommand{\eGM}{\zeta}

\newcommand{\tSGM}[1][]{\dGM_{#1}\text{-}\mrm{\textsc{Sgm}}}

\newcommand{\sets}[1][\numindsets]{S_1, \cdots, S_{#1}} %
\newcommand{\barsets}[1][\wbar{\numindsets}]{\wbar{S}_1, \cdots, \wbar{S}_{#1}} %
\newcommand{\setU}[1][u]{S_{#1}} %
\newcommand{\barsetU}[1][u]{\wbar{S}_{#1}} %
\newcommand{\set}{S} %
\newcommand{\setC}{S^C} %

\newcommand{\cone}[1][]{C_{1,\tau}^{#1}}
\newcommand{\ctwo}[1][]{C_{2,\tau}^{#1}}
\newcommand{\cthree}[1][]{C_{3,\tau}^{#1}}
\newcommand{\cfour}[1][]{C_{4,\tau}^{#1}}
\newcommand{\cfive}[1][]{C_{5,\tau}^{#1}}

\newcommand{\numindsets}{L}

\newcommand{\KLD}[2]{\mathsf{KL}\left( #1\, \middle\Vert #2 \right)}
\newcommand{\Ent}[2]{\mathsf{Ent}_{#1}\left( #2 \right)}
\newcommand{\TV}[2]{\lVert #1 \!-\! #2 \rVert_{\mathsf{TV}}}
\newcommand{\tnabla}{\wtil{\nabla}}
\newcommand{\bom}{\wbar{\om}} %
\newcommand{\bpsi}{\wbar{\psi}} %

\newcommand{\LSI}[2]{\mathrm{LSI}_{#1}(#2)}

\newcommand{\EstimatedPhi}{\what{\Phi}}
\newcommand{\TruePhi}[1][i]{\Phi^{\star(#1)}} %
\newcommand{\Truephi}[1][i]{\gamma^{(#1)}} %

\newcommand{\ratio}{\gamma}
\newcommand{\radius}{\eta}

\newcommand{\diag}{\mathrm{diag}(\ParameterMatrix)}
\newcommand{\crx}{\wbar{\rvbx}} %
\newcommand{\mapsfrom}{\mathrel{\reflectbox{\ensuremath{\mapsto}}}}
\newcommand{\vsep}{\vspace{-2mm}}

\newcommand{\sless}[1]{\stackrel{#1}{\leq}}
\newcommand{\sgreat}[1]{\stackrel{#1}{\geq}}
\newcommand{\sequal}[1]{\stackrel{#1}{=}}

\newcommand{\normalbrackets}[1]{[ #1 ]}

\newcommand{\bigbrackets}[1]{\big[ #1 \big]}
\newcommand{\Bigbrackets}[1]{\Big[ #1 \Big]}
\newcommand{\biggbrackets}[1]{\bigg[ #1 \bigg]}
\newcommand{\normalparenth}[1]{( #1 )}
\newcommand{\parenth}[1]{\left( #1 \right)}
\newcommand{\bigparenth}[1]{\big( #1 \big)}
\newcommand{\Bigparenth}[1]{\Big( #1 \Big)}
\newcommand{\biggparenth}[1]{\bigg( #1 \bigg)}
\newcommand{\normalbraces}[1]{\{ #1  \}}
\newcommand{\braces}[1]{\left\{ #1 \right \}}
\newcommand{\sbraces}[1]{\{ #1 \}}
\newcommand{\bigbraces}[1]{\big\{ #1 \big \}}
\newcommand{\Bigbraces}[1]{\Big\{ #1 \Big \}}

\newcommand{\normalabs}[1]{| #1  |}
\newcommand{\bigabs}[1]{\big| #1 \big|}
\newcommand{\Bigabs}[1]{\Big| #1 \Big|}
\newcommand{\biggabs}[1]{\bigg| #1 \bigg|}

\newcommand{\ceils}[1]{\left\lceil #1 \right \rceil}

\newcommand{\tp}{^\top}
\newcommand{\inv}{^{-1}}

\newcommand{\qtext}[1]{\quad\text{#1}\quad} 
\newcommand{\stext}[1]{\ \text{#1}\ }

\def\defeq{\triangleq} %
\newcommand{\defn}{\defeq}

\def\norm#1{\left\|{#1}\right\|} %
 \def\snorm#1{\|{#1}\|} %
\newcommand{\zeronorm}[1]{\norm{#1}_0} %
\newcommand{\szeronorm}[1]{\snorm{#1}_0} %
\newcommand{\spnorm}[1]{\snorm{#1}_p} %
\newcommand{\sonenorm}[1]{\snorm{#1}_1} %
\newcommand{\twonorm}[1]{\norm{#1}_2} %
\newcommand{\stwonorm}[1]{\snorm{#1}_2} %
\newcommand{\infnorm}[1]{\norm{#1}_{\infty}} %
\newcommand{\sinfnorm}[1]{\snorm{#1}_{\infty}} %

\newcommand{\bmatnorm}[1]{\left|\!\left|\!\left| #1 \right|\!\right|\!\right|}
\newcommand{\matnorm}[1]{|\!|\!| #1 | \! | \!|}
\newcommand{\maxmatnorm}[1]{\matnorm{#1}_{\max}}
\newcommand{\onematnorm}[1]{\matnorm{#1}_{1}}
\newcommand{\infmatnorm}[1]{\matnorm{#1}_{\infty}}
\newcommand{\opnorm}[1]{\matnorm{#1}_{\mathrm{op}}} %
\newcommand{\bopnorm}[1]{\bmatnorm{#1}_{\mathrm{op}}} %
\newcommand{\fronorm}[1]{\matnorm{#1}_{\mrm{F}}} %

\newcommand{\moment}[2]{\norm{#1}_{L_{#2}}} %
\newcommand{\smoment}[2]{\snorm{#1}_{L_{#2}}}

\def\what#1{\widehat{#1}}

\def\mbf#1{\mathbf{#1}}
\def\mbb#1{\mathbb{#1}}

\def\mrm#1{\mathrm{#1}}

\def\tbf#1{\textbf{#1}}

\def\til#1{\widetilde{#1}}
\def\wtil#1{\til{#1}}
\def\wbar#1{\overline{#1}}

\def\balign#1\ealign{\begin{align}#1\end{align}}
\def\baligns#1\ealigns{\begin{align*}#1\end{align*}}
\def\balignat#1\ealign{\begin{alignat}#1\end{alignat}}
\def\balignats#1\ealigns{\begin{alignat*}#1\end{alignat*}}
\def\bitemize#1\eitemize{\begin{itemize}#1\end{itemize}}
\def\benumerate#1\eenumerate{\begin{enumerate}#1\end{enumerate}}

\newenvironment{talign*}
 {\csname align*\endcsname}
 {\endalign}
\newenvironment{talign}
 {\csname align\endcsname}
 {\endalign}

\def\balignst#1\ealignst{\begin{talign*}#1\end{talign*}}
\def\balignt#1\ealignt{\begin{talign}#1\end{talign}}

\let\svthefootnote\thefootnote
\newcommand\freefootnote[1]{%
	\let\thefootnote\relax%
	\footnotetext{#1}%
	\let\thefootnote\svthefootnote%
}

\title{On counterfactual inference with unobserved confounding\freefootnote{This work was supported, in part, by NSF under Grant No. DMS-2023528 as part of the Foundations of Data Science Institute (FODSI), the MIT-IBM Watson AI Lab under Agreement No. W1771646, MIT-IBM projects on Time Series and Causal Inference as well as project with DSO National Laboratory.}}

\author[1]{Abhin Shah}
\author[2]{Raaz Dwivedi}
\author[1]{Devavrat Shah}
\author[1]{Gregory W. Wornell}

\affil[1]{Massachusetts Institute of Technology}
\affil[2]{Cornell Tech}

\graphicspath{{../figures/},{figs/}}
\newtheorem{myproposition}{Proposition}

\newtheorem{mycorollary}{Corollary}

\newtheorem{myclaim}{Claim}

\newtheorem{mylemma}{Lemma}

\newtheorem{mydefinition}{Definition}

\newtheorem{myremark}{Remark}

\newtheorem{myassumption}{Assumption}

 \crefname{appendix}{Appendix.}{Appendices.}
\crefname{equation}{}{}
\crefname{lemma}{Lemma.}{Lemmas.}
\crefname{theorem}{Theorem.}{Theorems.}
\crefname{Corollary}{Corollary.}{Corollaries.}
\crefname{Claim}{Claim.}{Claims.}
\crefname{algorithm}{Algorithm.}{Algorithms.}
\crefname{example}{Example.}{Examples.}

\crefname{section}{Section.}{Sections.}
\crefname{table}{Table.}{Tables.}
\crefname{remark}{Remark.}{Remarks.}
\crefname{algorithm}{Algorithm}{Algorithms.}
\crefname{definition}{Definition.}{Definitions.}
\crefname{Proposition}{Proposition.}{Propositions.}
\crefname{myremark}{Remark.}{Remarks.}
\crefname{mylemma}{Lemma.}{Lemmas.}
\crefname{mydefinition}{Definition.}{Definitions.}
\crefname{myproposition}{Proposition.}{Propositions.}
\crefname{mycorollary}{Corollary.}{Corollaries.}
\crefname{myclaim}{Claim}{Claims.}
\crefname{myassumption}{Assumption.}{Assumptions.}
\crefname{figure}{Figure.}{Figures.}
\crefname{enumi}{}{}
\crefname{name}{}{} %

\date{}
\begin{document}

\begin{center}

  {\bf{\LARGE{On counterfactual inference with unobserved confounding}}}

\vspace*{.2in}

{\large{
\begin{tabular}{cccc}
Abhin Shah$^1$ & Raaz Dwivedi$^{2}$ & Devavrat Shah$^1$ & Gregory W. Wornell$^1$ 
\end{tabular}
\begin{tabular}{c}
  \tt{abhin@mit.edu}, \tt{dwivedi@cornell.edu}, \tt{\{devavrat, gww\}@mit.edu}
\end{tabular}
}}

\vspace*{.2in}

\begin{tabular}{c}
$^1$MIT and $^2$Cornell Tech
\end{tabular}

\today

\end{center}
\vspace*{.2in}

\begin{abstract}
Given an observational study with $n$ independent but heterogeneous units, our goal is to learn the counterfactual distribution for each unit using only one $p$-dimensional sample per unit containing covariates, interventions, and outcomes. Specifically, we allow for unobserved confounding that introduces statistical biases between interventions and outcomes as well as exacerbates the heterogeneity across units. Modeling the conditional distribution of the outcomes as an exponential family, we reduce learning the unit-level counterfactual distributions to learning $n$ exponential family distributions with heterogeneous parameters and only one sample per distribution. We introduce a convex objective that pools all $n$ samples to jointly learn all $n$ parameter vectors, and provide a unit-wise mean squared error bound that scales linearly with the metric entropy of the parameter space. For example, when the parameters are $s$-sparse linear combination of $k$ known vectors, the error is $O(s\log k/p)$.  En route, we derive sufficient conditions for compactly supported distributions to satisfy the logarithmic Sobolev inequality. As an application of the framework, our results enable consistent imputation of sparsely missing covariates. 
\end{abstract}

\addtocontents{toc}{\protect\setcounter{tocdepth}{0}}
\section{Introduction}
\label{section_introduction}
We are interested in the problem of unit-level counterfactual inference owing to the increasing importance of personalized decision-making in many domains.
As a motivating example, consider an observational dataset corresponding to an interaction between a recommender system and a user over time. At each time, the user was exposed to a product based on observed demographic factors as well as factors that are not observed in the dataset, e.g., user's energy level (i.e., whether they're feeling energetic or tired). Additionally, at each time, the user's engagement level, which could have sequentially depended on the prior interaction in addition to the ongoing interaction, was also recorded. Also, the system could have sequentially adapted its recommendation. Given such data of many heterogeneous users (e.g., a movie recommender system for a streaming media
platform), we want to infer each user's average engagement level if it were exposed to a different sequence of products while the observed and the unobserved factors remain unchanged. This task is challenging since: (a) the {unobserved} factors could give rise to spurious associations, (b) the users could be {heterogeneous} in that they may have different responses to same sequence of products, and (c) each user provides a {single} interaction trajectory.
More generally, to address problems of this kind, we consider an observational setting where a unit undergoes multiple interventions (or treatments) denoted by $\rvba$. We denote the outcomes of interest by $\rvby$, and allow the interventions $\rvba$ and the outcomes $\rvby$ to be confounded by observed covariates $\rvbv$ as well as unobserved covariates $\rvbz$. The graphical structure shown in \cref{fig_graphical_models}(a) captures these interactions and is at the heart of our problem. 
In the recommender system example above, {a unit corresponds to a user}, $\rvba$ corresponds to the products recommended, $\rvby$ corresponds to the engagement levels, $\rvbv$ corresponds to the observed demographic factors, and $\rvbz$ corresponds to the unobserved energy levels (see \cref{fig_graphical_models}(b)). We consider $n$ heterogeneous and independent units indexed by $i \in [n] \defn \{1,\cdots,n\}$, and assume access to one observation per unit with $(\svbv^{(i)}$, $\svba^{(i)}$, $\svby^{(i)})$ denoting the realizations of $(\rvbv$, $\rvba$, $\rvby)$ for unit $i$. 
\begin{figure}[t]
    \centering
    \begin{tabular}{cc}
    \includegraphics[width=0.25\linewidth,clip]{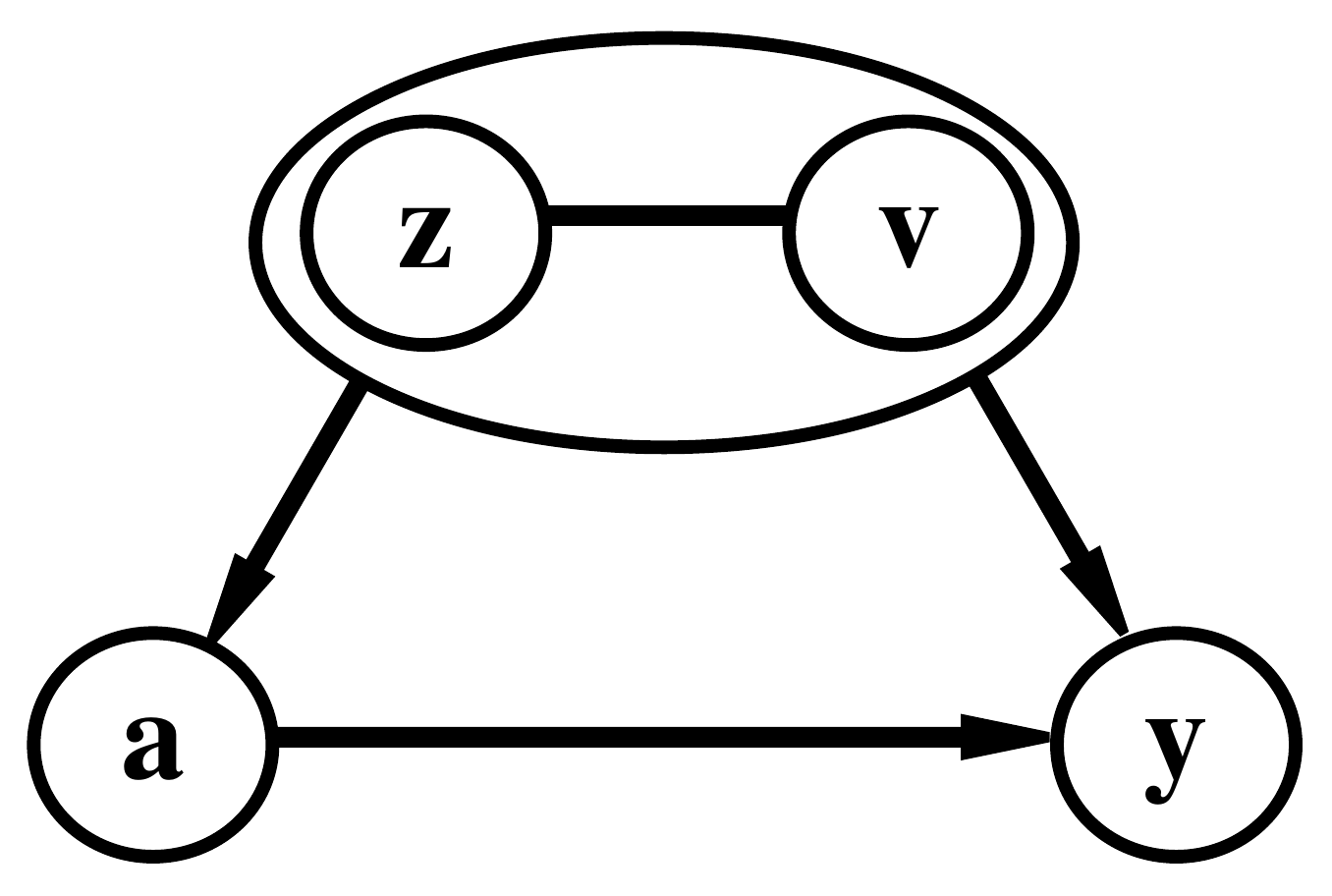}
    &
    \includegraphics[height=0.25\linewidth,clip]{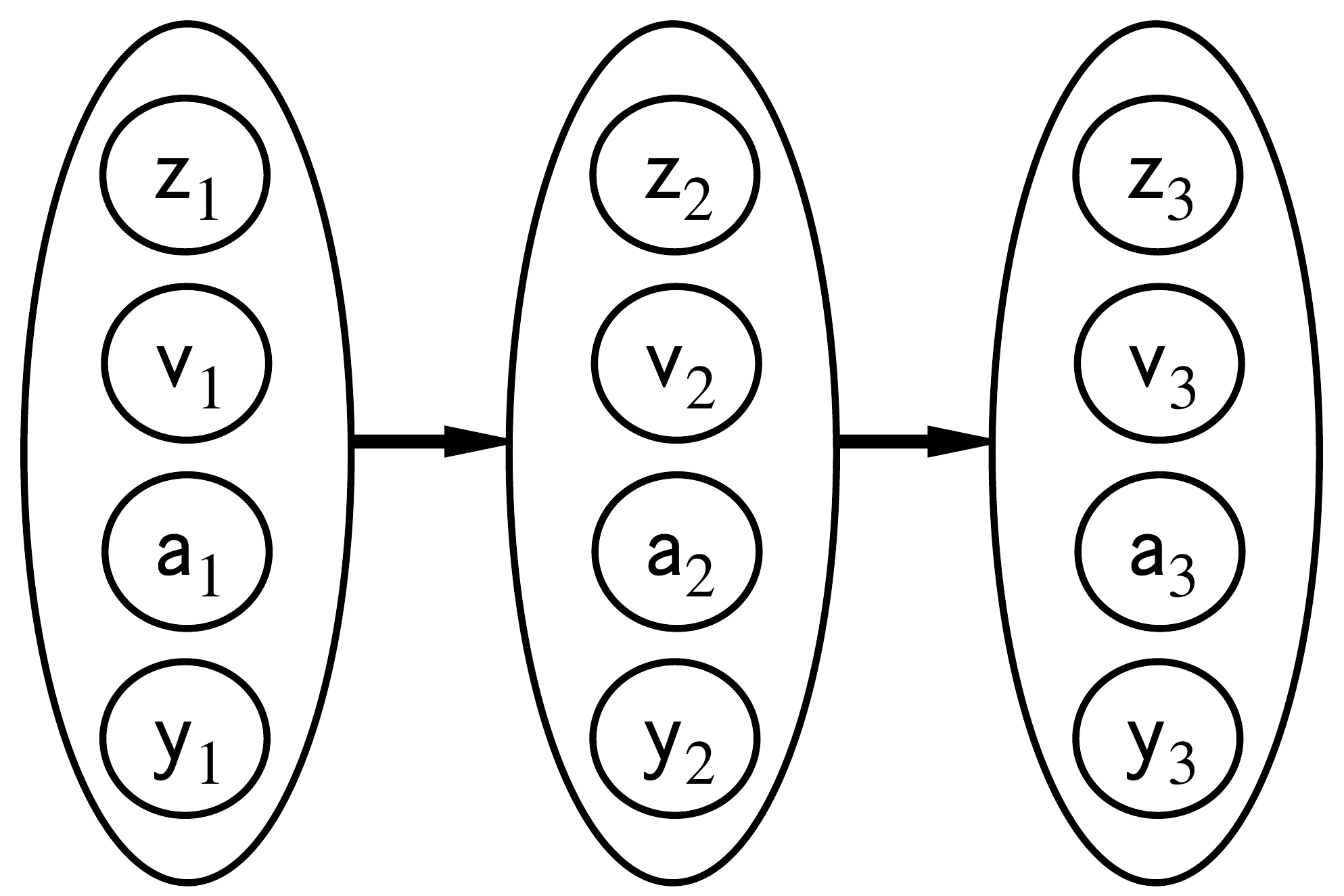} ~~
    \includegraphics[height=0.25\linewidth,clip]{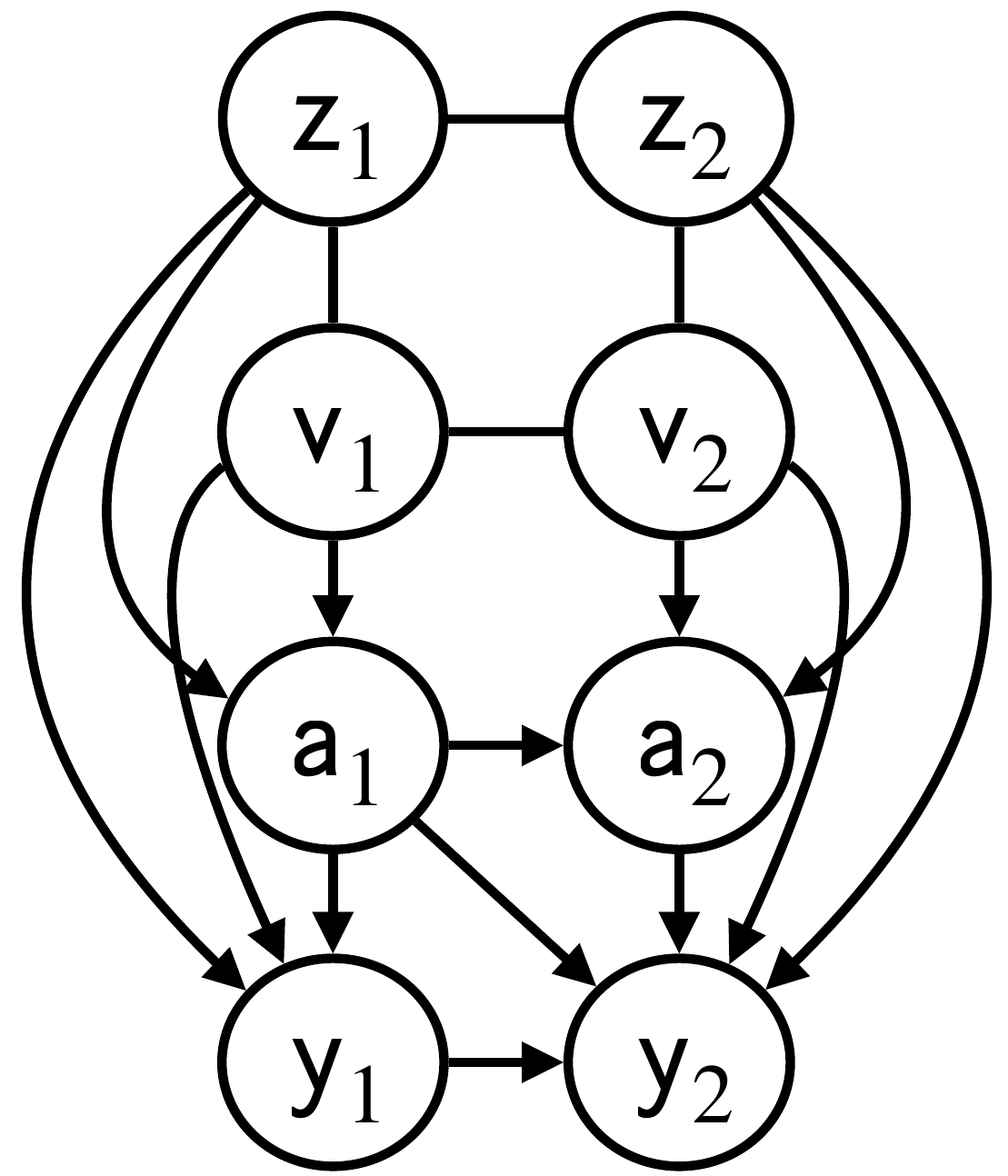}
    \\
    (a) A generic model for our setting \label{a}
    &
    (b) A graphical model for sequential recommender system 
    \end{tabular}
    \caption{Graphical models covered by our methodology. 
    Directed arrows denote causation and undirected arrows denote association. 
    Thin arrows denote low-level causal links and thick arrows denote high-level causal links, i.e., aggregated thin arrows. 
    Our methodology does not assume knowledge of  low-level causal links and is applicable to any graphical model with high-level causal links between variables as in panel \tbf{(a)}. Panel \tbf{(b)} presents an example of a sequential recommender system (consistent with the model in panel (a)) interacting with a user at 3 time points where {$\rvz_t$,} $\rvv_t$, $\rva_t$, and $\rvy_t$ denote the user's {unobserved energy levels,} observed demographic factors, the product exposed to the user, and the user's engagement level, respectively, at time $t$.
    The left subplot 
    illustrates the high-level dependency between the variables while the right subplot expands on it for time $1$ and $2$.
    {
    }}
    \label{fig_graphical_models}
\end{figure}
We operate within the Neyman-Rubin potential outcomes framework \citep{Neyman1923, Rubin1974} and denote the potential outcome of unit $i \in [n]$ under interventions $\svba$ by $\svby^{(i)}(\svba)$. Given the realizations $\braces{(\svbv^{(i)}, \svba^{(i)}, \svby^{(i)})}_{i=1}^{n}$, our goal is to answer counterfactual questions for these $n$ units. For example, what would the potential outcomes $\svby^{(i)}(\wtil{\svba}^{(i)})$ for interventions $\wtil{\svba}^{(i)} \neq \svba^{(i)}$ be, while the observed and unobserved covariates remain unchanged?  
Under the graphical model in \cref{fig_graphical_models}(a) and the stable unit treatment value assumption (SUTVA), i.e., the potential outcomes of unit $i$ are not affected by the interventions at other units, learning unit-level counterfactual distributions is equivalent to
learning unit-level conditional distributions
\begin{align}
    \braces{f_{\rvby | \rvba, \rvbz, \rvbv}(\rvby=\cdot | \rvba= \cdot, \svbz^{(i)}, \svbv^{(i)})}_{i=1}^{n}. \label{eq_set_conditional_distributions}
\end{align}
Here, the $i$-th distribution represents the conditional distribution for the outcomes $\rvby$ as a function of the interventions $\rvba$, while keeping the observed covariates $\rvbv$ and the unobserved covariates $\rvbz$ fixed at the corresponding realizations for unit $i$, i.e., $\svbv^{(i)}$ and $\svbz^{(i)}$, respectively.

\newcommand{\ranvarvec}{\rvbw}
\newcommand{\ranvarmat}{\newcommand{\varmat}{\begin{bmatrix} 
\rvbz, \rvbv, \rvba, \rvby^\tp
\end{bmatrix}}}
\newcommand{\varvec}{\svbw}
\newcommand{\varmat}{\begin{bmatrix} 
\svbz, \svbv, \svba, \svby
    \end{bmatrix}}

Such questions cannot be answered without structural assumptions 
due to two key challenges: (a) unobserved confounding and 
(b) single observation per unit.
First, the unobserved covariates $\rvbz$ introduce spurious statistical dependence between interventions and outcomes, termed unobserved confounding, which results in biased estimates. 
Second, we only observe one realization, namely the outcomes $\svby^{(i)}(\svba^{(i)})$ under the interventions $\svba^{(i)}$, that is consistent with the unit-level conditional distribution $f_{\rvby | \rvba, \rvbz, \rvbv}(\svby | \svba, \svbz^{(i)}, \svbv^{(i)})$. As a result, we need to learn $n$ heterogeneous conditional distributions while having access to only one sample from each of them. %

In this work, we model the conditional distribution of the outcomes of interest conditioned on the unobserved covariates, the observed covariates, the intervention as an exponential family distribution motivated by the principle of maximum entropy.\footnote{Exponential family distributions are the maximum entropy distributions given linear constraints on distributions such as specifying the moments (see \cite{jaynes1957information}).}
With this model structure, we show that both the aforementioned challenges can be tackled. 
In particular, we show that the $n$ unit-level conditional distributions in \eqref{eq_set_conditional_distributions} lead to $n$ distributions from the same exponential family, albeit with parameters that vary across units. The parameter corresponding to the $i^{th}$ unit, for brevity in terminology denoted by $\Truephi[i]$ (defined later), captures the effect of $\svbz^{(i)}$ and helps tackle the challenge of unobserved confounding. However, the challenge still remains to learn $n$ heterogeneous exponential family distributions with one sample per distribution. This challenge has been addressed in two specific scenarios in the literature: (a) if  the unobserved confounding is identical across units, i.e., the parameters $\normalbraces{\Truephi[i]}_{i=1}^n$ 
were all equal, then the challenge boils down to learning parameters of a single exponential family distribution from $n$ samples, which has been well-studied (cf. \cite{ShahSW2021B} for an overview); 
(b) if $\rvbv$, $\rvba$, and $\rvby$ take binary values and have pairwise interactions, then the challenge boils down to learning parameters of an Ising model (a special sub-class of exponential family defined later) with one sample. This specific challenge has been studied under restricted settings: (i) where the dependencies between the variables are known (e.g., \cite{KandirosDDGD2021, mukherjee2021high}) and (ii) where a specific
subset of the parameters are known \citep{DaganDDA2021}.
In this work, we consider a generalized setting where $\rvbv$, $\rvba$, and $\rvby$ can be either discrete, continuous, or both, and do not assume that the underlying dependencies or a specific subset of parameters are known.

\paragraph{Summary of contributions} 
This work introduces a method 
to learn unit-level counterfactual distributions from observational studies, in the presence of unobserved confounding, with one sample per unit, using exponential family modeling. For every unit $i \in [n]$, we reduce learning its counterfactual distribution to learning the unit-specific parameter  $\Truephi[i]$ with access to one sample $(\svbv^{(i)}, \svba^{(i)}, \svby^{(i)})$ from unit $i$. {Here, $\normalbraces{\Truephi[1], \cdots, \Truephi[n]}$ are parameters of $n$ different distributions from the same exponential family.}
The specific technical contributions are as follows:
\begin{enumerate}
\itemsep0em
    \item We introduce a convex (and strictly proper) loss function (\cref{def-loss-function}) that pools the data $\braces{(\svbv^{(i)}, \svba^{(i)}, \svby^{(i)})}_{i=1}^{n}$ across all $n$ samples to jointly learn all $n$ parameters $\normalbraces{\Truephi[i]}_{i=1}^n$.
    \item For every unit $i$, we prove that the mean squared errors of our estimates of (a) $\Truephi[i]$ (\cref{theorem_parameters}) and (b) the expected potential outcomes under alternate interventions (\cref{thm_causal_estimand}) scale linearly with the metric entropy of the underlying parameter space.
    For instance, when $\Truephi[i]$ is $s$-sparse linear combination of $k$ known vectors (\cref{cor_params}), 
    the error---just with one sample---decays as $O(s\log k/p)$, where $p$ is the dimension of the tuple $(\rvbv, \rvba, \rvby)$.
    \item We apply our method to impute missing covariates when they are sparse. Formally, we consider a setup (with no systematically unobserved covariates) where the observed covariates are entirely missing for some fixed fraction of the units. Specifically, for unit $i$ with missing covariates, only $(\svba^{(i)}, \svby^{(i)})$ is observed. For every such unit, we show that our method can recover the missing covariates with the mean squared error decaying as $O(p_v^2/p)$, where $p_v$ and $p$ are the dimensions of $\rvbv$ and $(\rvbv, \rvba, \rvby)$, respectively (\cref{prop_impute_missing_covariates}).
\item {Methodologically, our work advances three threads: (a) learning Ising models (and their extensions to discrete, continuous, or mixed variables) from a single sample, where we learn the dependencies between variables, generalizing prior work \citet{KandirosDDGD2021,DaganDDA2021}, (b) learning Markov random fields (a sub-class of exponential family) from multiple independent but non-identical samples, generalizing prior work \citet{vuffray2016interaction,VuffrayML2022,ShahSW2021A}, 
and (c) learning counterfactual outcomes with an exponential family model, allowing each unit to have different unobserved covariates and providing unit-level guarantees instead of average-level, generalizing \cite{arkhangelsky2018role}.}
 \item In our analysis, we (a) derive sufficient conditions for a continuous random vector supported on a compact set to satisfy the logarithmic Sobolev inequality (\cref{thm_LSI_main}) and (b) provide new concentration bounds for arbitrary functions of a continuous random vector that satisfies the logarithmic Sobolev inequality (\cref{thm_main_concentration}). These results may be of independent interest.
\end{enumerate}
\paragraph{Outline} \cref{sec_related_work} discusses background and related work. We discuss our formulation and algorithm in \cref{section_problem_formulation} and present their analysis in \cref{sec_main_results}. We develop an application of our methodology to impute missing covariates in \cref{sec_sparse_measurement_errors}. We sketch the proof of our main result in \cref{sec_proof_sketch} with detailed proofs deferred to the appendices. We conclude with a discussion in \cref{sec_discussion}.

\paragraph{Notation} 
For any positive integer $n$, let $[n] \coloneqq \{1,\cdots, n\}$.
For a deterministic sequence $u_1, \cdots , u_n$, we let $\svbu \coloneqq (u_1, \cdots, u_n)$. 
For a random sequence $\rvu_1, \cdots , \rvu_n$, we let $\rvbu \coloneqq (\rvu_1, \cdots, \rvu_n)$. 
For a vector $\svbu \in \Reals^p$, we use $u_t$ to denote its $t^{th}$ coordinate and $u_{-t} \in \Reals^{p-1}$ to denote the vector after deleting the $t^{th}$ coordinate. We denote the $\ell_0$, $\ell_p$ $(p \geq 1)$, and $\ell_{\infty}$ norms of a vector $\svbv$ by $\szeronorm{\svbv}$, $\spnorm{\svbv}$, and  $\sinfnorm{\svbv}$, respectively.  For a matrix $\tbf{M} \in \Reals^{p \times p}$, we denote the element in $t^{th}$ row and $u^{th}$ column by $\tbf{M}_{tu}$, the $t^{th}$ row by $\tbf{M}_t$, and the vector obtained after deleting $\tbf{M}_{tt}$ from $\tbf{M}_t$ by $\tbf{M}_{t,-t}$.  Further, we denote the matrix maximum norm by $\maxmatnorm{\tbf{M}}$, the Frobenius norm by $\fronorm{\tbf{M}}$, the spectral norm (operator $2$-norm) by $\opnorm{\tbf{M}}$, the induced $1-$norm (operator $1$-norm) by $\onematnorm{\tbf{M}}$, the induced $\infty$-norm (operator $\infty$-norm) by $\infmatnorm{\tbf{M}}$, and the $(2,\infty)$-norm by $\matnorm{\tbf{M}}_{2,\infty}$. Finally, for vectors $\what{\svbu} \in \Reals^p$ and $\wtil{\svbu} \in \Reals^p$, the mean squared error between $\what{\svbu}$ and $\wtil{\svbu}$ is defined as $\mathrm{MSE}(\what{\svbu}, \wtil{\svbu}) \defn p^{-1} \sum_{t \in [p]} (\what{u_t} - \wtil{u_t})^2$.
\section{Background and related work}
\label{sec_related_work}
This work builds on two vast bodies of literature: exponential family learning and unit-level counterfactual inference with unobserved confounding. For a detailed literature overview of the former, we refer the readers to \cite{bresler2015efficiently,klivans2017learning,VuffrayML2022,ShahSW2021A} (for a special sub-class, Markov random fields (MRFs)\footnote{MRFs can be naturally represented as exponential family distributions with certain sparsity constraints on the parameters via the principle of maximum entropy  \citep{wainwright2008graphical}.}) and \cite{ShahSW2021B} for general exponential families. For an introduction to counterfactual inference, see the books \cite{imbens2015causal, hernan2020causal} for settings with no unobserved confounding and \cite{Pearl2009,Pearl2016} for settings with known causal mechanism (in the form of a causal graph). 

\paragraph{Exponential family learning}
There is a series of works for learning Ising models, a special MRF with binary variables and an instance of a pair-wise exponential family, from a single sample. Such a model has two distinct sets of parameters capturing the contribution of nodes and edges in the underlying undirected graph, referred to as the external field and the interaction matrix.\footnote{E.g., in our model (defined later in \cref{eq_joint_distribution_zvay}), $\phi$ and $\Phi$ correspond to the external field  and the interaction matrix, respectively.} Many strategies exist  for learning such a model when the interaction matrix is known up to a constant and under varying assumptions on the external field; see, e.g., \cite{chatterjee2007estimation,bhattacharya2018inference,daskalakis2019regression,ghosal2020joint,KandirosDDGD2021,mukherjee2021high}. More recently, \cite{DaganDDA2021} provide guarantees for learning the interaction matrix from a single sample when the external field is known. \cite{KandirosDDGD2021} and \cite{mukherjee2021high} extend the tools in \cite{DaganDDA2021} to learn the external field for an Ising model with a known interaction matrix (up to a scalar multiple). Notably, all of these works are based on the pseudo-likelihood estimation \citep{besag1975statistical}. Our work extends the techniques and results from \cite{DaganDDA2021} to learn the external field from one sample of continuous variables with an estimated interaction matrix.
\cite{vuffray2016interaction} introduced a novel {M-estimation-based} loss function for learning Ising models from many independent and identically distributed samples. \cite{VuffrayML2022} and \cite{ShahSW2021A} generalize it to learn general MRFs with multi-ary discrete and continuous variables, respectively. {\cite{RenMVL2021} showed that this loss function has superior numerical performance compared to the ones based on pseudo-likelihood.}
We contribute to this line of work by generalizing that loss function further to learn MRFs with discrete, continuous, and mixed variables with independent but not identically distributed samples. 

 For settings closer to our work, namely, exponential families with unobserved variables, the two common modeling approaches include  restricted Boltzmann machines \citep{bresler2019learning, goel2020learning, bresler2020learning} and latent variable Gaussian graphical models; see, e.g., \cite{chandrasekaran2010latent,ma2013alternating,vinyes2018learning,wang2021learning}.
While the former assumes a bipartite structure with edges only across observed and unobserved variables, the latter imposes a Gaussian generative model. 
In this thread, most related to our set-up is the work by \cite{taeb2020learning} as they {model the conditional distribution of the observed variables conditioned on the unobserved variables as an exponential family similar to us.}
They provide empirically promising results for recovering the underlying graph and the number of unobserved variables (assumed to be small), albeit with limited theoretical guarantees. In contrast, here we provide parameter estimation error in the presence of unobserved variables (notably, we cover all the models they considered).

\paragraph{Unit-level counterfactual inference} Recent years have seen an active interest in developing different strategies for unit-level inference with unobserved confounding.

For the settings with univariate outcomes for each unit, a common approach to deal with unobserved confounding is the instrumental variable (IV) method~\citep{imbens1994identification} when one has access to a variable---the IV---that induces changes in intervention assignment but has no independent effect on outcomes allowing causal effect estimation. Recent works for IV methods with unit-level inference include
\cite{hartford2017deep,athey2019generalized,syrgkanis2019machine,singh2019kernel,xu2020learning,semenova2021debiased,wang2022instrumental}. 
Another approach for univariate outcomes, called causal sensitivity analysis \citep{rosenbaum1983assessing}, estimates the worst-case effect on the causal estimand as a function of the extent of unobserved confounding in a given dataset under varying assumptions on the generative model. For such analysis with unit-level guarantees, see, e.g., 
\cite{yadlowsky2018bounds,kallus2019interval,yin2021conformal,jin2021sensitivity,jesson2021quantifying}.

Closer to our work are those on panel or longitudinal data settings, where one observes multiple outcomes for each unit. For {linear panel data} settings, a common approach is factor modeling, where  potential outcomes and interventions (binary or multi-ary) are assumed to be independent conditional on some latent factors. See, e.g., difference-in-difference methods~\citep{bertrand2004much, angrist2009mostly}, synthetic control ~\citep{abadie1, abadie2}, its variants~\cite{arkhangelsky2021synthetic,dwivedi2022doubly}, and extensions to multi-ary interventions in synthetic interventions \citep{agarwal2020synthetic} and sequential experiments~\citep{dwivedi2022counterfactual}. 
{For non-linear panel data settings, the most commonly used models include probit, logit, Poisson, negative binomial, proportional hazard, and tobit models (see \cite{FV2018} for an overview) where some parametric model characterises the distribution of the outcomes conditional on the unobserved covariates, the observed covariates, and the interventions.} 
Notably, these works on linear and non-linear panel data directly estimate effects (averaged over all observed and unobserved covariates or unit-level for given observed and unobserved covariates) for finitely many interventions {when the intervention assignment has special structure}, while we focus on learning the counterfactual distributions while allowing for multi-ary discrete and continuous interventions {without any special structure}. 
In this thread, our work is most related to \cite{arkhangelsky2018role}, who also use an exponential family to model the unit-wise distribution of the observed covariates and interventions conditioned on the unobserved covariates.
They connect this model to the commonly used fixed effects model for the outcomes in latent factor modeling~\citep{angrist2009mostly}, and provide estimates for the average treatment effect given multiple units with the same set of unobserved covariates. Our work generalizes their set-up by allowing each unit to have a different set of unobserved covariates and provides the first unit-level counterfactual inference guarantee with an exponential family model.

\vsep

\section{Problem formulation and algorithm}
\label{section_problem_formulation}
This section formalizes the problem, specifies our model, and defines the inference tasks of interest. 

\subsection{Underlying causal mechanism and counterfactual distributions}
\label{subsec_causal_mech}
We consider a counterfactual inference task where units go through $p_a \geq 1$ interventions. For every unit, we observe $p_y \geq 1$ outcomes of interest. The interventions and the outcomes could be confounded by $p_v \geq 0$ observed covariates as well as $p_z \geq 0$ unobserved covariates. Additionally, the observed covariates and the unobserved covariates could be arbitrarily associated. We denote the random vector associated with the interventions, the outcomes,  the observed covariates, and the unobserved covariates by $\rvba \defn (\rva_{1}, \cdots, \rva_{p_a}) \in \cA^{p_a}$, $\rvby = (\rvy_{1}, \cdots, \rvy_{p_y}) \in \cY^{p_y}$, $\rvbv \defn (\rvv_{1}, \cdots, \rvv_{p_v}) \in \cV^{p_v}$, and $\rvbz \defn (\rvz_{1}, \cdots, \rvz_{p_z})  \in \cZ^{p_z}$, respectively, where $\cA, \cY, \cV$, and $\cZ$ denote the support of interventions, outcomes, observed covariates, and unobserved covariates, respectively. We allow these sets to contain discrete, continuous, or mixed values. 

\paragraph{Causal mechanism}
We summarize the causal relationship between the random vectors $\rvbz$, $\rvbv$, $\rvba$, and $\rvby$ in \cref{fig_graphical_models}(a) where we denote the arbitrary association between $\rvbz$ and $\rvbv$ by a undirected arrow, and the causal association between (i) $(\rvbz, \rvbv)$ and $\rvba$, (ii) $(\rvbz, \rvbv)$ and $\rvby$, and (iii) $\rvba$ and $\rvby$ by directed arrows. 
More generally, we are interested in any setup consistent with the  graphical model in \cref{fig_graphical_models}(a). We assume access to $n$ independent realizations indexed by $i \in [n]$: $\svbv^{(i)}$, $\svba^{(i)}$, and $\svby^{(i)}$ denote the realizations of $\rvbv$, $\rvba$, and $\rvby$ for unit $i$, respectively. For every realized tuple $(\svbv^{(i)}, \svba^{(i)}, \svby^{(i)})$, there is a corresponding realization $\svbz^{(i)}$ of the unobserved covariates $\rvbz$ that is unobserved. Next, we discuss some examples covered by our framework.

\paragraph{Examples: sequential and network settings}
\label{subsubsec_examples}
\begin{figure}[t]
    \centering
    \begin{tabular}{c}
    \includegraphics[height=0.28\linewidth,clip]{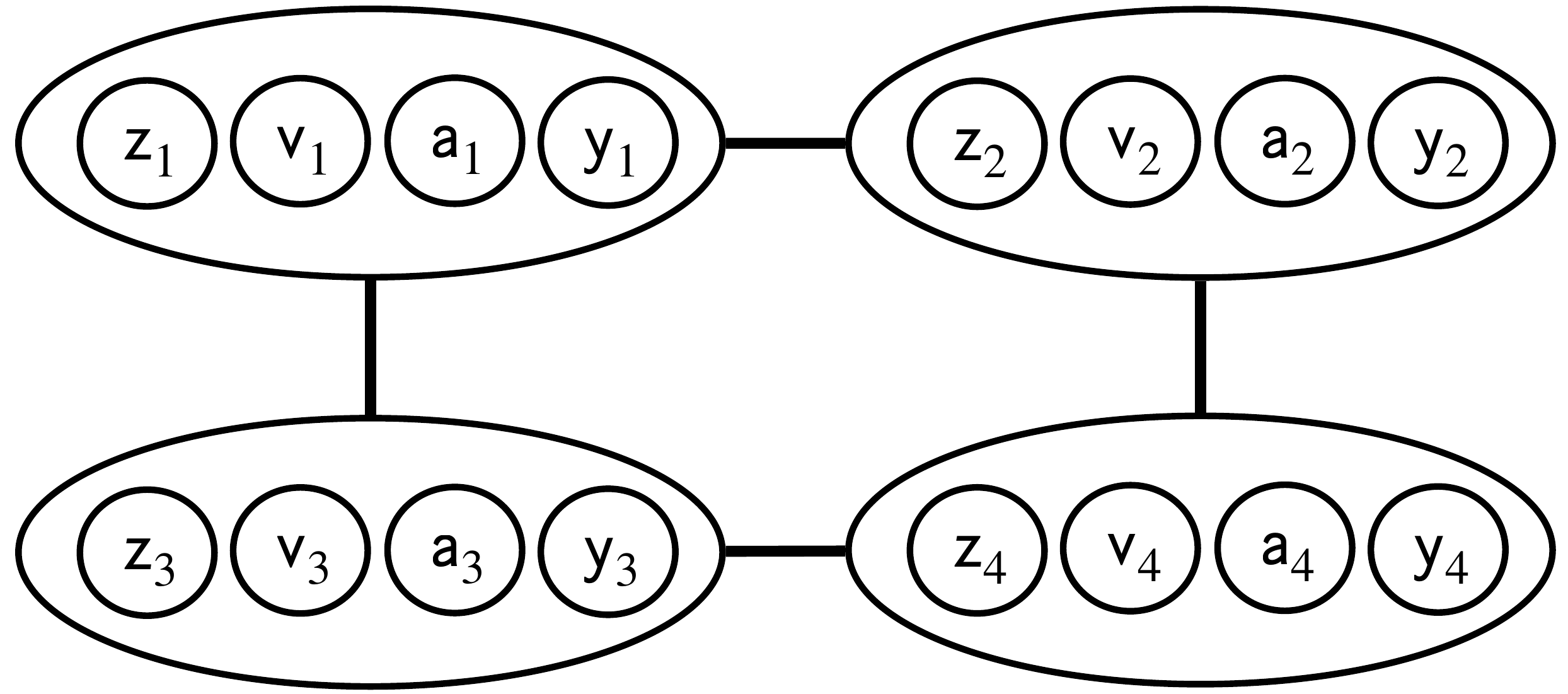} ~~
    \includegraphics[height=0.3\linewidth, clip]{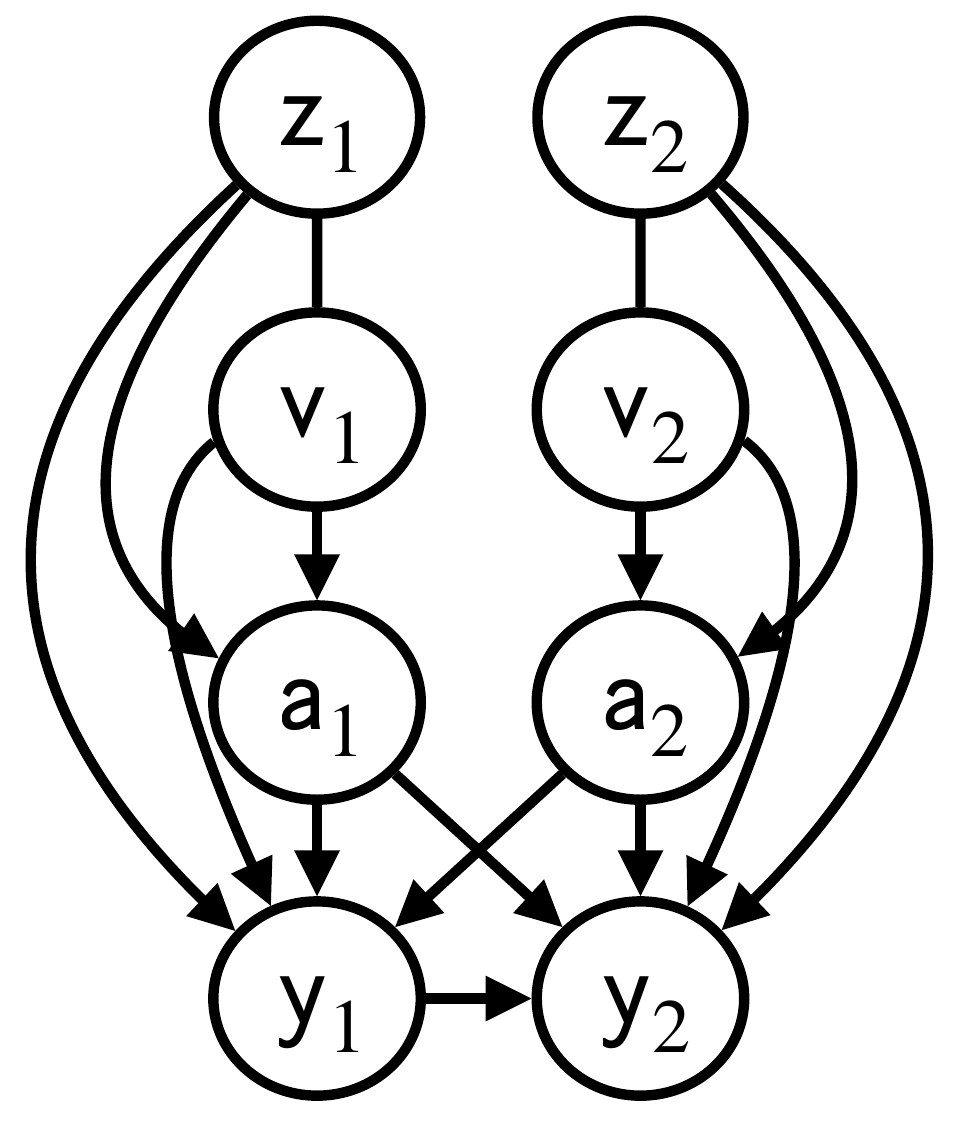}  \\
    \end{tabular}
    \caption{A graphical model for a single unit in the network setting with 4 users; arrows have same meaning as in \cref{fig_graphical_models}. Here $\rvv_t$, $\rvz_t$, $\rva_t$, and $\rvy_t$ denote user $t$'s observed factors, unobserved factors, exposed product, and engagement level, respectively. The left plot illustrates the high-level dependency between the variables of different users in the network, and the right plot expands on it for {(user $1$, user $2$) pair. Analogous dependencies exist for (user $1$, user $3$), (user $2$, user $4$), and (user $3$, user $4$) pairs.}}
    \label{fig_graphical_models_2}
\end{figure}
While \cref{fig_graphical_models}(a) exhibits the high-level causal links between $\rvbz$, $\rvbv$, $\rvba$, and $\rvby$, there could be complex low-level causal links between elements of these vectors. We do not assume any knowledge of such low-level causal links. In \cref{fig_graphical_models}(b), we provide an instance of a sequential setting covered by our work where every unit's (i) $\rva_{t+1}$ depends on $\rva_{t}$ in addition to $\rvv_{t+1}$ and $\rvbz$, and (ii) $\rvy_{t+1}$ depends on $\rva_{t}$ and $\rvy_{t}$ in addition to $\rva_{t+1}$, $\rvv_{t+1}$ and $\rvbz$. Another classical example covered by our framework includes the network setting where a unit represents a social network where users are linked to each other by interpersonal relationships as shown in \cref{fig_graphical_models_2}. Similar to the sequential recommender system, every user was exposed to a product based on observed demographic factors as well as certain unobserved factors, and the user's engagement level was recorded. The engagement level of user $t$, i.e., $\rvy_{t}$, depended its observed demographic factors $\rvv_{t}$, its unobserved factors $\rvz_{t}$, its exposed product $\rva_{t}$ as well as on the product exposed to its neighbor $u$, i.e., $\rva_{u}$. Further, $\rvy_{t}$ could have been associated with $\rvy_{u}$.

\paragraph{Unit-level counterfactual distributions}
We denote the Neyman-Rubin potential outcomes of unit $i \in [n]$ under interventions $\svba \in \cA^{p_a}$ by $\svby^{(i)}(\svba)$. We make the stable unit treatment value assumption (SUTVA) \citep{rubin1980randomization} for the observed outcome, 
i.e.,  
$\svby^{(i)} =  \svby^{(i)}(\svba^{(i)}) $ for all $i \in [n]$.
 For independent units with the causal mechanism and SUTVA assumed here, the unit-level counterfactual distributions are equivalent to certain unit-level conditional distributions as we now argue. Consider unit $i \in [n]$ and fix the observed covariates and the unobserved covariates at $\svbv^{(i)}$ and $\svbz^{(i)}$, respectively. Then, let $\wtil{\svby}^{(i)}$ be a realization of $\rvby$ when $\rvba = \wtil{\svba}^{(i)}$. We are interested in the distribution of the potential outcomes of unit $i$ for interventions $\wtil{\svba}^{(i)}$, i.e., the distribution of $\svby^{(i)}(\wtil{\svba}^{(i)})$ given $\rvbz = \svbz^{(i)}, \rvbv = \svbv^{(i)}$. Under the causal framework considered here (see \cref{fig_graphical_models}(a)), it is equivalent to the distribution of $\svby^{(i)}(\wtil{\svba}^{(i)})$ given $\rvba = \wtil{\svba}^{(i)}, \rvbz = \svbz^{(i)}, \rvbv = \svbv^{(i)}$ since $(\rvbz, \rvbv)$ satisfy ignorability \citep{Pearl2009, imbens2015causal}, i.e., the potential outcomes are independent of the interventions given $(\rvbz, \rvbv)$.  Further, under SUTVA, it is equivalent to the distribution of $\wtil{\svby}^{(i)}$ given $\rvba = \wtil{\svba}^{(i)}, \rvbz = \svbz^{(i)}, \rvbv = \svbv^{(i)}$, i.e., $f_{\rvby | \rvba, \rvbz, \rvbv}(\rvby=\cdot | \rvba= \wtil{\svba}^{(i)}, \svbz^{(i)}, \svbv^{(i)})$. Therefore, our goal is to learn the $n$ unit-level conditional distributions in \cref{eq_set_conditional_distributions}. 
 Now, we proceed to the modeling details.

\subsection{Exponential family modeling and its consequences}
\label{subsec_exp_fam}
Let $\ranvarvec \defeq \normalparenth{\rvbz, \rvbv, \rvba, \rvby}$ be the $\tilp$-dimensional random vector obtained by concatenating $\rvbz$, $\rvbv$, $\rvba$ and $\rvby$ where $\tilp \defn p_z + p_v + p_a + p_y$. {For notational convenience, we start by modeling the joint probability distribution $f_{\ranvarvec}$ as an exponential family and relax this model to the conditional distribution of the outcomes in \cref{subsec_joint_vs_conditional}.} In particular, we parameterize $f_{\ranvarvec}$ with natural parameters $\phi \in \Reals^{\tilp \times 1}$ and $\Phi \in \Reals^{\tilp \times \tilp}$, and natural statistics $\rvbw$ and $\rvbw \rvbw\tp$ so that
\begin{align}
    f_{\ranvarvec}(\varvec; \phi, \Phi) \propto \exp\Bigparenth{ \phi\tp \varvec
    +\varvec\tp \Phi \varvec},
    \qtext{where}
    \varvec \defeq (\svbz, \svbv, \svba, \svby),
    \label{eq_joint_distribution_zvay}
\end{align}
and $\svbz \defn (z_{1}, \cdots, z_{p_z})$, $\svbv \defn (v_{1}, \cdots, v_{p_v})$, $\svba \defn (a_{1}, \cdots, a_{p_a})$, and $\svby \defn (y_{1}, \cdots, y_{p_y})$ denote realizations of $\rvbz$, $\rvbv$, $\rvba$, and $\rvby$, respectively. 
Without loss of generality, we can assume $\Phi$ to be a symmetric matrix. 
Next, we show that with this modeling assumption, learning unit-level counterfactual distribution can be reduced to learning a suitable exponential family model.

Under the exponential family in \cref{eq_joint_distribution_zvay}, the unit-level conditional distribution of $\rvby$ conditioned on $\rvba = \svba$, $\rvbz = \svbz$, and $\rvbv = \svbv$ 
is an exponential family model with natural statistics $\rvby$ and $\rvby \rvby\tp$ and
\begin{align}
    f_{\rvby | \rvba, \rvbz, \rvbv}(\svby | \svba, \svbz, \svbv) \! \propto  \!\exp\Bigparenth{\!\bigbrackets{\phi^{(y)\tp} \!\!+\! 2\svbz\tp\Phi^{(z,y)} \!+\! 2\svbv\tp\Phi^{(v,y)} \!+\! 2\svba\tp\Phi^{(a,y)}} \svby \!+\! \svby\tp\Phi^{(y,y)} \svby}, \label{eq_conditional_distribution_y}
\end{align}
where $\phi^{(y)} \in \Reals^{p \times 1}$ is the component of $\phi$ corresponding to $\rvby$ and $\Phi^{(u, y)} \in \Reals^{p_u \times p_y}$ is the component of $\Phi$ corresponding to $\rvbu$ and $\rvby$ for all $\rvbu \in \{\rvbz, \rvbv, \rvba, \rvby\}$.\footnote{The exponential family in \cref{eq_conditional_distribution_y} is same as the one considered in \citet[Equation 1.3]{taeb2020learning}.} We make two key observations:
    (a) the term $\Phi^{(z,y)\top}\!\svbz$ captures the effect of unobserved covariates $\svbz$ on $f_{\rvby | \rvba, \rvbz, \rvbv}(\rvby\!=\!\cdot | \rvba\!=\! \cdot, \svbz, \svbv)$ and (b) the task of learning $f_{\rvby | \rvba, \rvbz, \rvbv}(\rvby=\cdot | \rvba= \cdot, \svbz, \svbv)$ in \cref{eq_conditional_distribution_y} as a function of $\rvba$ reduces to learning
\begin{align}
    \stext{(i)} \phi^{(y)} + 2\Phi^{(z,y)\top} \svbz + 2\Phi^{(v,y)\top} \svbv, \qtext{} \stext{(ii)} \Phi^{(a,y)}, \qtext{and} \stext{(iii)} \Phi^{(y,y)}. \label{eq_actual_parameters_of_interest}
\end{align}
That is, learning the unit-level conditional distribution for unit $i$ is equivalent to learning
\begin{align}
    \Truephi = \bigbraces{\phi^{(y)} + 2\Phi^{(z,y)\top} \svbz^{(i)} + 2\Phi^{(v,y)\top} \svbv^{(i)}, \Phi^{(a,y)}, \Phi^{(y,y)}},
    \label{eq:gammai}
\end{align}
where the notation $\Truephi$ is the same as in \cref{section_introduction}. {We note that, given $\rvba = \svba$, $\rvbz = \svbz$, and $\rvbv = \svbv$, $\rvby = \svba + \svbz + \svbv + \eta$ is one plausible data generating process (DGP) consistent with \cref{eq_conditional_distribution_y} when the noise variable $\boldsymbol{\eta}$ has an exponential family distribution. More specifically, this DGP, with $\boldsymbol{\eta}$ such that $f(\boldsymbol{\eta}) \propto  \exp\bigparenth{\phi^{(y)\top} \boldsymbol{\eta} + \boldsymbol{\eta}\tp \Phi^{(y,y)} \boldsymbol{\eta}}$, results in the conditional distribution in \cref{eq_conditional_distribution_y} with $\Phi^{(z,y)} = \Phi^{(v,y)} = \Phi^{(a,y)} = \Phi^{(y,y)}$.}

Next, we argue that learning the three quantities in \cref{eq_actual_parameters_of_interest} is subsumed in learning the parameters of the (unit-level) conditional distribution $f_{\rvbx |\rvbz}$ of the random vector $\rvbx \defeq \normalparenth{\rvbv, \rvba, \rvby}$ conditioned on $\rvbz = \svbz$. Note that $f_{\rvbx |\rvbz}$ belongs to an exponential family with natural statistics $\rvbx$ and $\rvbx \rvbx\tp$. For all $\rvbu \in \normalbraces{\rvbv, \rvba, \rvby}$, let $\phi^{(u)} \in \Reals^{p_u \times 1}$ be the component of $\phi$ corresponding to $\rvbu$, and $\Phi^{(z, u)} \in \Reals^{p_z \times p_u}$ be the component of $\Phi$ corresponding to $\rvbz$ and $\rvbu$. Then $f_{\rvbx |\rvbz}$ can be parameterized as follows:
\begin{align}
    \JointDist \!\propto\! \exp \Bigparenth{\! \normalbrackets{\ExternalField(\svbz)}\tp \! \svbx \!+\! \svbx\!\tp \! \ParameterMatrix \svbx\!} \label{eq_conditional_distribution_vay}, \!\! \stext{where}
    \ExternalField(\svbz)  \defn \!\begin{bmatrix}  \phi^{(v)} \!+\! 2 \Phi^{(z,v)\top} \svbz \\  \phi^{(a)} \!+\! 2 \Phi^{(z,a)\top} \svbz \\  \phi^{(y)} \!+\! 2 \Phi^{(z,y)\top} \svbz \end{bmatrix}  \! \! \in \! \Reals^{p\times 1},\, 
\end{align} 
$\svbx \defn (\svbv, \svba, \svby)$, $p \defn p_v + p_a + p_y$ and $\Theta \in \Reals^{p \times p}$ denotes the component of $\Phi$ corresponding to $\rvbx$. 
Given some estimates for $\ExternalField(\svbz)$ and $\ParameterMatrix$, using their appropriate components also yields an estimate of the three quantities in \cref{eq_actual_parameters_of_interest} for any $\rvbv = \svbv$. 
To summarize, the spurious associations or unobserved confounding between $\svba$ and $\svby$ introduced due to unobserved $\rvbz$ are fully captured by $\Phi^{(z,y)\top} \svbz$ or equivalently by $\ExternalField(\svbz)$; thereby, learning unit-level counterfactual distributions require us to learn these unit-level parameters. 

\subsubsection{Reduced inference task and modeling constraints}
\label{subsec_inference_tasks_of_interest}
Let $f_{\ranvarvec}(\cdot; \phi^*, \Phi^*)$ denote the true data generating distribution of $\rvbw$ in \cref{eq_joint_distribution_zvay}, and let $\TrueJointDistfun$ denote the true distribution of $\rvbx$ conditioned on $\rvbz = \svbz$ in \cref{eq_conditional_distribution_vay}.
Then, for all $i \in [n]$, we note that the realization $\svbx^{(i)} \defn (\svbv^{(i)}, \svba^{(i)}, \svby^{(i)})$ is consistent with the conditional distribution $\TrueJointDistfun[(i)]$ where we do not observe $\svbz^{(i)}$. Our primary goal is to learn the $n$ unit-level counterfactual distributions, which as noted above 
simplifies to estimating the following parameters:
\begin{align}
    \hspace{-1cm}\stext{(i) Unit-level} \TrueExternalFieldI[i] \defn \TrueExternalField(\svbz^{(i)}) \stext{for} i \in [n], \qtext{and (ii) Population-level} \TrueParameterMatrix. \label{eq_parameters_of_interest}
\end{align}
Our secondary goal is to estimate the expected potential outcomes for any given unit $i$ (with $\rvbz = \svbz^{(i)}, \rvbv = \svbv^{(i)}$) 
and an alternate intervention $\wtil{\svba}^{(i)}$:
\begin{align}
    \mu^{(i)}(\wtil{\svba}^{(i)}) \defn  \Expectation[\svby^{(i)}(\wtil{\svba}^{(i)}) | \rvbz = \svbz^{(i)}, \rvbv = \svbv^{(i)}],\label{eq_causal_estimand}
\end{align}
where $\svby^{(i)}(\wtil{\svba}^{(i)})$ denotes the potential outcomes for unit $i \in [n]$ under interventions $\wtil{\svba}^{(i)} \in \cA^{p_a}$.

For ease of exposition, we consider bounded continuous sets $\cV$, $\cA$, and $\cY$ 
with $\cV = \cA = \cY \defn \cX  = \normalbrackets{-\xmax, \xmax}$ for a given $\xmax$. In \cref{subsec_discrete}, we consider compact discrete and mixed sets. Throughout this paper, it is convenient to further constrain the model as follows:
\begin{assumption}[Bounded and sparse parameters]
\label{assumptions}
The true model parameters~\cref{eq_parameters_of_interest} satisfy
\begin{align}
    \TrueExternalFieldI[i] &\in \ParameterSet_{\ExternalField} \defn  
    \braces{ \ExternalField  \in \Reals^{p \times 1}: \infnorm{\ExternalField} \leq \aGM} 
    \stext{for all $i \in [n]$,} \label{eq_parameter_set_external_field}
    \intertext{and}
    \TrueParameterMatrix &\in \ParameterSet_{\ParameterMatrix} \defn  \braces{ \ParameterMatrix \in \Reals^{p \times p}: \ParameterMatrix = \ParameterMatrix\tp, ~ \maxmatnorm{\ParameterMatrix} \leq \aGM, ~ {\infmatnorm{\ParameterMatrix}} \leq \bGM}. \label{eq_parameter_set}
\end{align}
\end{assumption}
While \cref{eq_parameter_set_external_field} bounds the unit-level parameters (a necessary condition for model identifiability \citep{SanthanamW2012}), \cref{eq_parameter_set} {bounds the $\ell_1$ norm of the interaction of each $\rvx_t \in \rvbx$ with every $\rvx_u \in \rvbx$ in \cref{eq_conditional_distribution_vay}}. As a result, \cref{assumptions} implies that the exponential family in \cref{eq_conditional_distribution_vay} corresponds to MRFs (see \cref{sec_related_work}), also known as undirected graphical models (defined in \cref{sec_conditioning_trick}). We note that \cref{assumptions} is standard in the literature on learning MRFs \citep{bresler2015efficiently,vuffray2016interaction,klivans2017learning,VuffrayML2022,ShahSW2021A}.  
We are now ready to state our algorithm.

\subsection{An efficient algorithm via a convex objective}
\label{sec_algo}

We first describe our strategy to estimate the parameters in \cref{eq_parameters_of_interest}. Then, we use the estimated parameters to estimate the expected potential outcomes in \cref{eq_causal_estimand}.  We remark that for exponential families considered here, maximum likelihood for parameter estimation is not computationally tractable \citep{wainwright2008graphical,ShahSW2021B}. {As a result, we resort to an alternative objective function inspired by the convex loss functions used in \cite{vuffray2016interaction,VuffrayML2022,ShahSW2021A} as they do not depend on the partition function of the distribution. These loss functions are designed in a specific way (see below for details): (i) the sufficient statistics of the conditional distribution of a variable given all other variables are \textit{centered} by adding appropriate constants, (ii) the loss function is an empirical average of the sum of the inverses of all of these conditional distributions (without the partition function) with \textit{centered} sufficient statistics.}

\subsubsection{Parameter estimation}
\label{subsec_loss_function}

Our convex objective function jointly learns all the parameters of interest by pooling the observations across all $n$ units and exploiting the exponential family structure of $\rvbv$, $\rvba$, and $\rvby$ conditioned on $\rvbz = \svbz$ in \cref{eq_conditional_distribution_vay}, i.e., the objective explicitly utilizes the fact that the population-level parameter $\TrueParameterMatrix$ is shared across units.
In particular, we use the following two steps.

\paragraph{Centering sufficient statistics of the conditional distribution of a variable}
Consider the conditional distribution $f_{\rvx_t | \rvbx_{-t}, \rvbz}$ of the random variable $\rvx_t$ conditioned on $\rvbx_{-t} = \svbx_{-t}$ and $\rvbz = \svbz$ for any $t \in [p]$:
{\begin{align}
    f_{\rvx_t | \rvbx_{-t}, \rvbz}\bigparenth{x_t| \svbx_{-t}, \svbz; \ExternalFieldt(\svbz), \ParameterRowt} \propto \exp\biggparenth{ \bigbrackets{\ExternalFieldt(\svbz) + 2\ParameterRowttt\tp \svbx_{-t}} x_t + \ParameterTU[tt] x_t^2},
    \label{eq_conditional_dist_non_centered}
\end{align}}
where $\ExternalFieldt(\svbz)$ is the $t^{th}$ element of $\ExternalField(\svbz)$, $\ParameterRowt$ is the $t^{th}$ row of $\ParameterMatrix$, $\ParameterTU[tt]$ is the $t^{th}$ element of $\ParameterRowt$, and $\ParameterRowttt \defn \ParameterRowt \setminus \ParameterTU[tt] \in \Reals^{p-1}$ is the vector obtained after deleting $\ParameterTU[tt]$ from $\ParameterRowt$.
{Then, the sufficient statistics in \cref{eq_conditional_dist_non_centered}, namely $\rvx_t$ and $\rvx_t^2$, are centered by subtracting their expected value with respect to the uniform distribution on $\cX$ resulting in}
\begin{align}
    f_{\rvx_t | \rvbx_{-t}, \rvbz}\bigparenth{x_t| \svbx_{-t}, \svbz; \ExternalFieldt(\svbz), \ParameterRowt} \propto \exp\biggparenth{ \bigbrackets{\ExternalFieldt(\svbz) + 2\ParameterRowttt\tp \svbx_{-t}} x_t + \ParameterTU[tt] \Bigparenth{x_t^2 - \frac{\xmax^2}{3}}},
    \label{eq_conditional_dist}
\end{align}
{as the integral of $\rvx_t$ and $\rvx_t^2$ with respect to the uniform distribution on $\cX$ is $0$ and $\xmax^2/3$, respectively. As we see later (in \cref{prop_proper_loss_function}), this centering ensures that our loss function is a proper loss function as well as leads to connections with the surrogate likelihood~\citep[Proposition. 4.1]{ShahSW2021A}. We emphasize that the term $\xmax^2/3$ inside the exponent in \cref{eq_conditional_dist} is vacuous (as it is a constant) and the distribution in \cref{eq_conditional_dist} is equivalent to the one in \cref{eq_conditional_dist_non_centered}.} 

\paragraph{Constructing the loss function}
{Next, the loss function (defined below) is desgined to be an empirical average of the sum over $t \in [p]$ of the inverse of the term in the right hand side of \cref{eq_conditional_dist}.}
\begin{definition}[\tbf{Loss function}]\label{def-loss-function}
Given the samples $\sbraces{\svbx^{(i)}}_{i \in [n]}$, the loss $\loss:\Reals^{p \times (n+p)} \to \Reals$ is given by
\begin{align}
    \loss\bigparenth{\ExtendedParameterMatrix} \!=\! \frac{1}{n}\sump[t] \sumn[i] \! \exp\biggparenth{\!-\!\bigbrackets{\ExternalFieldtI \!+\! 2\ParameterRowttt\tp \svbx_{-t}^{(i)}} x_t^{(i)}\!-\!\ParameterTU[tt] \Bigparenth{[x_t^{(i)}]^2 - \frac{\xmax^2}{3}}}
    \!\!\qtext{where}
    \ExtendedParameterMatrix \!\defeq\!\! \begin{bmatrix} \ExtendedParameterRowT[1]\tp \\ \vdots \\ \ExtendedParameterRowT[p]\tp \end{bmatrix},
    \label{eq:loss_function}
    \end{align}
    and $\ExtendedParameterRowT[t] \!\defn\! \bigbraces{\ExternalFieldtI[1], \cdots, \ExternalFieldtI[n], \ParameterRowt}$ for $t\in[p]$.
\end{definition}
\noindent Our estimate of $\ExtendedTrueParameterMatrix$ (defined analogous to $\ExtendedParameterMatrix$) is given by
\begin{align}
    \ExtendedEstimatedParameterMatrix \in \argmin_{\ExtendedParameterMatrix \in \ParameterSet_{\ExternalField}^n \times \ParameterSet_{\ParameterMatrix}} \loss\bigparenth{\ExtendedParameterMatrix}.
    \label{eq_estimated_parameters}
\end{align}
We note \cref{eq_estimated_parameters} is a convex optimization problem, and a projected gradient descent algorithm (see \cref{subsec_alg}) returns an $\epsilon$-optimal estimate with $\tau = O(p/\epsilon)$ iterations\footnote{This follows from \cite[Theorem. 3.7]{bubeck2015convex} by noting that $\loss(\ExtendedParameterMatrix)$ is $O(p)$ smooth function of $\ExtendedParameterMatrix$.} where $\ExtendedEstimatedParameterMatrix_{\epsilon}$ is said to be an $\epsilon$-optimal estimate if $\loss\bigparenth{\ExtendedEstimatedParameterMatrix_{\epsilon}} \leq \loss\bigparenth{\ExtendedEstimatedParameterMatrix} + \epsilon$ for any $\epsilon > 0$. The loss function $\loss$ 
admits a notable property (see \cref{sec_proof_proper_loss_function} for the proof).
\newcommand{\properlossfunction}{Proper loss function}
\begin{proposition}[\tbf{\properlossfunction}]\label{prop_proper_loss_function} The loss function $\loss$ is strictly proper, i.e., $\ExtendedTrueParameterMatrix = \argmin_{\ExtendedParameterMatrix \in \ParameterSet_{\ExternalField}^n \times \ParameterSet_{\ParameterMatrix}} \Expectation_{\rvbx|\rvbz}\bigbrackets{\loss\bigparenth{\ExtendedParameterMatrix}}$.
\end{proposition}
\cref{prop_proper_loss_function} shows that the solution of the idealized convex program $\min_{\ExtendedParameterMatrix \in \ParameterSet_{\ExternalField}^n \times \ParameterSet_{\ParameterMatrix}} \Expectation_{\rvbx|\rvbz}\bigbrackets{\loss\bigparenth{\ExtendedParameterMatrix}}$ is unique and equal to $\ExtendedTrueParameterMatrix$. {In this idealized convex program, conditioned on the realized values of the unobserved covariates of the $n$ units $\svbz^{(1)}, \cdots, \svbz^{(n)}$, the loss function is averaged over all the randomness in the observed covariates, the interventions, and the outcomes. In other words, for every $i \in [n]$, the idealized convex program has infinite samples from $f_{\rvbx | \rvbz}$ with unobserved covariates $\rvbz$ conditioned to be $\svbz^{(i)}$.} Thus, the convex program in \cref{eq_estimated_parameters} can be seen as a {single} sample version of this idealized program, thereby providing an intuitive justification of our loss function (instead of a maximum likelihood objective, which is not  tractable here). As we show later in our proofs (see \cref{sec_proof_sketch} for an overview), different partial averages on the RHS of \cref{eq:loss_function} also admit useful properties and are critical to our analyses.

{We note that loss function  in \cref{eq:loss_function} is a generalization of the loss functions used in \cite{vuffray2016interaction,VuffrayML2022,ShahSW2021A}. In particular, if the unobserved confounding is identical across units, i.e., $\TrueExternalFieldI[1] = \cdots = \TrueExternalFieldI[n]$, then $\loss\bigparenth{\ExtendedParameterMatrix}$ in \cref{eq:loss_function} can be decomposed into $p$ independent loss functions, one for every $t \in [p]$. These decomposed loss functions are identical to the ones used in these prior works.}

\subsubsection{Causal estimate}
\label{subsec_causal_estimate}
Given the estimate $\ExtendedEstimatedParameterMatrix$, our estimate of the expected potential outcome $\mu^{(i)}(\wtil{\svba}^{(i)})$ under an alternate intervention $\wtil{\svba}^{(i)} \in \cA^{p_a}$~\cref{eq_causal_estimand} is derived as follows:
First, we identify $\EstimatedPhi^{(u, y)} \in \Reals^{p_u \times p_y}$ to be the component of $\EstimatedParameterMatrix$ corresponding to $\rvbu$ and $\rvby$ for all $\rvbu \in \{\rvbv, \rvba, \rvby\}$ and $\EstimatedExternalFieldI[i,y] \in \Reals^{p_y}$ to be the component of $\EstimatedExternalFieldI$ corresponding to $\rvby$. Next, we estimate the conditional distribution of $\rvby$ for unit $i$ as a function of the interventions $\rvba$, while keeping $\rvbv=\svbv^{(i)}$ and $\rvbz=\svbz^{(i)}$ fixed as 
\begin{align}
    \what{f}^{(i)}_{\rvby | \rvba}(\svby | \svba) \propto \exp\Bigparenth{\bigbrackets{\EstimatedExternalFieldI[i,y] + 2\svbv^{(i)\top}\EstimatedPhi^{(v,y)} + 2\svba\tp\EstimatedPhi^{(a,y)}} \svby + \svby\tp\EstimatedPhi^{(y,y)} \svby}. \label{eq_counterfactual_distribution_y}
\end{align}
Finally, we estimate $\mu^{(i)}(\wtil{\svba}^{(i)})$ as the mean under the above conditional distribution, given by
\begin{align}
    \what{\mu}^{(i)}(\wtil{\svba}^{(i)}) & \defn \Expectation_{\what{f}^{(i)}_{\rvby | \rvba}}[\rvby | \rvba = \wtil{\svba}^{(i)}],  
    \label{eq_causal_estimate}
\end{align}
which can be computed by standard algorithms for estimating marginals of graphical models, e.g., via the junction tree algorithm~\citep{wainwright2008graphical} or message-passing algorithms.\footnote{In general, estimating the marginals exactly is computationally hard for undirected graphical models. While the junction tree algorithm works well for graphical models with small treewidth~\citep[Section. 2.5]{wainwright2008graphical}, e.g., for trees or chains as in hidden Markov models or state-space models, message-passing algorithms are the default choice for computing approximate marginals for complex graphs, especially with cycles. However, message-passing algorithms may induce additional approximations, which we do not discuss here.} 

\section{Main results}
\label{sec_main_results}
In this section, we analyze our estimates. First, 
we provide our guarantee on estimating the unit-level and the population-level parameters in \cref{sub:parameter_result}. 
Next, we provide our guarantee on estimating the causal estimand of interest in \cref{subsec_guarantee_outcome_estimate}.  
Before stating our main results, we define a standard notion of complexity of the set $\ParameterSet_{\ExternalField}$, namely metric entropy (defined below) that our guarantees rely on.

\begin{definition}[$\varepsilon$-covering number and metric entropy]\label{def_covering_number_metric_entropy}
Given a set $\cV \subset \Reals^p$ and a scalar $\varepsilon > 0$, we use $\cC(\cV, \varepsilon)$ to denote the $\varepsilon$-covering number of $\cV$ with respect to $\sonenorm{\cdot}$, i.e., $\cC(\cV, \varepsilon)$ denotes the minimum cardinality over all possible subsets $\cU \subset \cV$ that satisfy $\cV \subset \cup_{u \in \cU} \ball(u; \varepsilon)$,
where $\ball(u; \varepsilon) \defn \braces{v \in \Reals^p: \sonenorm{u-v} \leq \varepsilon}$.
We let $\metric_{\ExternalField}(\varepsilon) \defn \log \cC(\ParameterSet_{\ExternalField},\varepsilon)$ denote the metric entropy of $\ParameterSet_{\ExternalField}$, and $\metric_{\ExternalField,n}(\varepsilon) \defn n \metric_{\ExternalField}(n \varepsilon)$ denote a scaled version of it.
\end{definition}

Next, we state two settings with upper bounds on the metric entropy, and we use them as running examples to unpack our general results throughout this paper.
\begin{example}[{Linear combination}]
    \label{exam:lc_dense}
    Consider a set $\ParameterSet_{\ExternalField}$ containing vectors with bounded entries that are also a linear combination of $k$ known vectors in $\real^p$ collected as $\mbf B \in \real^{p\times k}$, i.e., $\ParameterSet_{\ExternalField}=\sbraces{\mbf B \mbf a: \mbf a \in \real^k, \sinfnorm{\mbf B \mbf a} \leq \alpha }$. Then, \citet[Lemma. 11]{DaganDDA2021} implies that $\metric_{\ExternalField}(\radius) = O\bigparenth{k\log \bigparenth{1+\frac{\alpha}{\radius}}}$. Further, $\metric_{\ExternalField,n}(\radius) = O\bigparenth{\frac{\alpha k}{\radius}}$. 
\end{example}

\begin{example}[{Sparse linear combination}]
    \label{exam:sc}
    Consider a set $\ParameterSet_{\ExternalField}$ containing vectors with bounded entries that are also a $s$-\emph{sparse} linear combination of $k$ known vectors in $\real^p$ collected as $\mbf B \in \real^{p\times k}$, i.e., $\ParameterSet_{\ExternalField}=\sbraces{\mbf B \mbf a: \mbf a \in \real^k, \norm{a}_0 \leq s, \sinfnorm{\mbf B \mbf a} \leq \alpha }$. Then \citet[Corollary. 4]{DaganDDA2021} implies that $\metric_{\ExternalField}(\radius) = O\bigparenth{s\log k \log \bigparenth{1+\frac{\alpha}{\radius}}}$. Further, $\metric_{\ExternalField,n}(\radius) = O\bigparenth{\frac{\alpha s \log k}{\radius}}$. 
\end{example}

\newcommand{\edgeparammainresultname}{Recovering population-level parameter}
\newcommand{\parammainresultname}{Guarantee on quality of parameter estimate}
\newcommand{\nodeparammainresultname}{Recovering unit-level parameters}
\subsection{\parammainresultname}
\label{sub:parameter_result}
Our non-asymptotic guarantees use an assumption of a lower bound on the smallest eigenvalue of a suitable set of autocorrelation matrices.

\begin{assumption}\label{ass_pos_eigenvalue}
For any $\svbz \in \cZ^{p_z}$ and $t \in [p]$, let $\lambda_{\min}(\svbz, t)$ denote the smallest eigenvalue of the matrix
$\Expectation_{\rvbx | \rvbz} \bigbrackets{  \trvbx ~ \trvbx\tp | \rvbz = \svbz }$ where $\trvbx \defeq \bigparenth{ \rvx_t, 2\rvbx_{-t} \rvx_t, \rvx_t^2 -\xmax^2/3} \in \real^{p+1}$. We assume $\lambda_{\min} \defn \min_{\svbz \in \cZ^{p_z}, t\in [p]} \lambda_{\min}(\svbz, t)$ is strictly positive.
\end{assumption}
\noindent We note that all eigenvalues of any autocorrelation matrix are non-negative implying $\lambda_{\min}(\svbz, t) \geq 0$ for all $\svbz \in \cZ^{p_z}, t\in [p]$.  \cref{ass_pos_eigenvalue} requires $\lambda_{\min}(\svbz, t) > 0$ for all $\svbz \in \cZ^{p_z}, t\in [p]$ and serves as a sufficient condition to rule out certain singular distributions~\citep[Section. 5]{ShahSW2021B}.\footnote{Essentially, we use this assumption to lower bound the variance of a non-constant random variable (\cref{proof_of_lemma_parameter}).} 
In \cref{subsec_discussion_ass_pos_eigenvalue}, we show that $\lambda_{\min} = \Omega(e^{-c\bGM})$ when $\TrueParameterTU[tt] = 0$ for all $t \in [p]$ as in Ising model where $\rvx_t^2 = 1 $ for all $ t \in [p]$.

We are now ready to state our main result that characterizes a high probability bound on the estimation error for the estimate $\ExtendedEstimatedParameterMatrix$ computed via \cref{eq_estimated_parameters}. To simplify the presentation, we use $c$ and $c'$ to denote universal constants or constants that depend on the parameters $\aGM,\xmax,$ and $\lambda_{\min}$ and can take a different value in each appearance. 
\begin{theorem}[{\parammainresultname}]
\label{theorem_parameters}
Suppose \cref{assumptions,ass_pos_eigenvalue} hold. Fix an $\varepsilon > 0$ and $\delta \in (0,1)$, and define
\begin{align}
R(\varepsilon, \delta) & \!\defn\! \max \sbraces{ ce^{c'\bGM}\!\! \sqrt{\log(\log p/\delta) \!+\! \metric_{\ExternalField}( ce^{-c'\bGM}) }, \varepsilon \ratio} \stext{with} \ratio \!\defn\! \max_{\ExternalField, \bExternalField \in \ParameterSet_{\ExternalField}} \!\!\frac{\sonenorm{\ExternalField \!-\! \bExternalField}}{\stwonorm{\ExternalField \!-\! \bExternalField}} \label{eq_radius_node_thm}
\intertext{and}
\tmetric_{\ExternalField,n}(\varepsilon, \delta)  & \!\defn\! \metric_{\ExternalField,n}\Bigparenth{\frac{\varepsilon^2}{p}} \!+\! p \metric_{\ExternalField}\bigparenth{R^2\normalparenth{\varepsilon, \delta}}. \label{eq_tmetric_node_thm}
\end{align}
 Then, with probability at least $1-\delta$, the 
estimates $\EstimatedParameterMatrix,  \EstimatedExternalFieldI[1], \cdots, \EstimatedExternalFieldI[n]$ defined in \cref{eq_estimated_parameters}
satisfy
\begin{align}
    \matnorm{\EstimatedParameterMatrix \!-\! \TrueParameterMatrix}_{2,\infty} &\leq \varepsilon
    \qquad \quad \ \, \qtext{when}
    n \geq \frac{ce^{c'\bGM} p^2 \Bigparenth{p\log \frac{p}{\delta \varepsilon^2} + \metric_{\ExternalField,n}(\varepsilon^2)}}{\varepsilon^4} \label{eq:matrix_guarantee}
    \intertext{and}
    \max_{i\in[n]}\stwonorm{\EstimatedExternalFieldI - \TrueExternalFieldI} &  \leq  R\Bigparenth{ \varepsilon, \frac{\delta}{n}} 
    \qtext{when} n  \geq  \frac{ce^{c'\bGM} p^4 \Bigparenth{p \log \frac{np^2}{\delta \varepsilon^2} + \tmetric_{\ExternalField,n}\bigparenth{ \varepsilon, \frac{\delta}{n}}}}{\varepsilon^4}.   
    \label{eq:theta_guarantee}
\end{align}
\end{theorem}
\newcommand{\metricentropyone}{Linear combination of known vectors}

\noindent We split the proof into two parts: First, we establish the bound \cref{eq:matrix_guarantee} in \cref{sec:proof_of_theorem_parameters}, which we then use to establish the bound \cref{eq:theta_guarantee} in \cref{sec_proof_thm_node_parameters_recovery}. 
Our guarantee in \cref{eq:matrix_guarantee} provides a non-asymptotic error bound of order $\frac{p^2 (p\log p + \metric_{\ExternalField,n}(n^{-1/2}))}{n^{1/4}}$  (where we treat $\bGM$ as a constant) for estimating $\TrueParameterMatrix$ although the $n$ samples have different unit-level parameters $\sbraces{\TrueExternalFieldI}_{i=1}^{n}$. On the other hand, after squaring both sides and dividing by $p$, the guarantee~\cref{eq:theta_guarantee} for the unit-level parameters can be simplified as follows:\footnote{We replace $\delta/n$ in \cref{eq:theta_guarantee} by $\delta$ as we do not require a union bound over $i \in [n]$ for unit-wise guarantees.} whenever $n \geq c'\varepsilon^{-4} p^4\normalparenth{p\log \frac{p^2}{\delta \varepsilon^2}  + \metric_{\ExternalField,n}(\varepsilon^2/p) + p\metric_{\ExternalField}(c)}$, we have
\begin{align}
\label{eq:simplified_bound_theta}
    \mathrm{MSE}(\EstimatedExternalFieldI, \TrueExternalFieldI)
    \! \leq \! \max\Bigbraces{\varepsilon^2, \dfrac{\metric_{\ExternalField}(c) \!+\! \log (\log \frac{p}{\delta})}{p}}, 
\end{align}
where we use $\ratio \leq \sqrt p$ in \cref{eq_radius_node_thm} and treat $\bGM$ as a constant. For large $n$ so that $\varepsilon$ is small, this error scales linearly with the metric entropy $\metric_{\ExternalField}$---the error becomes worse as the unit-level parameter set $\ParameterSet_{\ExternalField}$ becomes more complex. 

 The next corollary (stated without proof)  provides a formal version of the population-level guarantee in \cref{eq:matrix_guarantee} and the unit-level guarantee in \cref{eq:simplified_bound_theta} for the two examples discussed earlier. We treat $\bGM$ as a constant and note that the dependence is exponential as in \cref{theorem_parameters}. 

 \begin{corollary}[{Consequences for examples}]\label{cor_params}
Suppose \cref{assumptions,ass_pos_eigenvalue} hold. Then, for any fixed $\varepsilon > 0$ and $\delta \in (0,1)$, the following results hold  with probability at least $1-\delta$.
\begin{enumerate}[leftmargin=*,label=(\alph*)]
\item\label{item:lc_dense} \emph{{Linear combination:}}\
If $\ParameterSet_{\ExternalField}$ is as in \cref{exam:lc_dense}, then for all $i \in [n]$,
\begin{align}
\matnorm{\EstimatedParameterMatrix \!-\! \TrueParameterMatrix}_{2,\infty} & \!\leq\! \varepsilon
    \qquad \quad \hspace{0.85cm} \qquad \quad \qquad \quad \qtext{for}
    n \!\geq \! \frac{c p^2 \bigparenth{p \log \frac{p}{\delta \varepsilon^2} \!+\! \frac{k}{\varepsilon^2}}}{\varepsilon^4}\\
    \hspace{-1cm} \mathrm{MSE}(\EstimatedExternalFieldI\!, \TrueExternalFieldI)
    & \!\leq\! \max\Bigbraces{\varepsilon^2, \dfrac{c \bigparenth{ k \!+\! \log (\log \frac{p}{\delta}) }}{p} \!}  \qtext{for}
    n \!\geq\! \frac{c p^5 \bigparenth{\log \frac{p^2}{\delta \varepsilon^2} \!+\! k \!+\! \frac{k}{\varepsilon^2}}}{\varepsilon^4}.
\end{align}
\item\label{item:sc}  \emph{{Sparse linear combination:}}\ 
If $\ParameterSet_{\ExternalField}$ is as in \cref{exam:sc}, then for all $i \in [n]$,
\begin{align}
\hspace{-0.5cm} \matnorm{\EstimatedParameterMatrix \!-\! \TrueParameterMatrix}_{2,\infty} & \!\leq\! \varepsilon
    \qquad \quad \hspace{1.0cm} \quad \qquad \quad \qquad \quad \qtext{for}
    n \!\geq\! \frac{c p^2 \bigparenth{p \log \frac{p}{\delta \varepsilon^2} \!+\! \frac{s \log k}{\varepsilon^2}}}{\varepsilon^4}\\
    \hspace{-0.5cm} \mathrm{MSE}(\EstimatedExternalFieldI \!\!, \TrueExternalFieldI)
    & \!\leq\! \max\!\Bigbraces{ \varepsilon^2\!, \!\dfrac{c \bigparenth{\!s \log k \!+\! \log (\log \frac{p}{\delta}) }}{p}\!}  \qtext{for}
    n \!\geq\! \! \frac{c p^5 \bigparenth{\log \frac{p^2}{\delta \varepsilon^2}  \!+\! s\log k \!+\! \frac{s \log k}{\varepsilon^2}\!}}{\varepsilon^4}.
\end{align}
\end{enumerate}
\end{corollary}
\noindent \cref{cor_params} states that, as long as $n$ is polynomially large in $p$, our strategy learns the unit-level parameters (on average in terms of mean square error across coordinates) for each user if $p$ is large compared to either the number of vectors $k$ (\cref{exam:lc_dense}) or the sparsity parameter $s$ (\cref{exam:sc}).

\paragraph{Sharpness of guarantees and generalization of prior results}
The exponential dependence on $\bGM$ in \cref{theorem_parameters} is unavoidable given the lower bounds for learning exponential families even with i.i.d. samples~\citep{SanthanamW2012}. {Regarding the dependence on error tolerance $\varepsilon$, prior works with suitable analogs of our loss function provide two different error scaling: (i) $1/\varepsilon^4$ in \cite{VuffrayML2022,ShahSW2021A,ShahSW2021B} and (ii) $1/\varepsilon^2$ in  \cite{vuffray2016interaction} and \cite{ShahSW2023}. The works in category (ii) use techniques from \cite{negahban2012unified}, and it remains an interesting future direction to see whether similar ideas could be used to sharpen the error scaling of $1/\varepsilon^4$ to the parametric rate of $1/\varepsilon^2$ in \cref{theorem_parameters}. We note that improving the dependence on $\varepsilon$ in \cref{eq:matrix_guarantee} improves the dependence on $\varepsilon$ as well as $p$ in \cref{eq:theta_guarantee}.} In the special case of equal unit-level parameters ($\TrueExternalFieldI[1] = \cdots =  \TrueExternalFieldI[n]$), the analysis in 
\cref{sec:proof_of_theorem_parameters} to establish the bound \cref{eq:matrix_guarantee} can be modified to recover (up to constants) prior guarantee \citep[Lemma.~9.1]{ShahSW2021A} on learning exponential family from $n$ i.i.d. samples. 
Further, the guarantee~\cref{eq:theta_guarantee} recovers the prior guarantee~\citep[Theorem.~6]{KandirosDDGD2021} as a special case where the authors consider learning an Ising model from one sample when the population-level parameter is known up to a scaling factor.
\newcommand{\outcomemainresultname}{Guarantee on quality of outcome estimate}
\subsection{\outcomemainresultname}
\label{subsec_guarantee_outcome_estimate}
Our non-asymptotic guarantee on outcome estimate assumes that the following matrices are suitably stable under small perturbation in the parameters: (i) the covariance matrix of $\rvby$ conditioned on $\rvba$, $\rvbz$, and $\rvbv$ and (ii) the cross-covariance matrix of $\rvby$ and $\rvy_t \rvby$ conditioned on $\rvba$, $\rvbz$, and $\rvbv$ for all $t \in [p_y]$.
\begin{assumption}\label{ass_bounded_op_norm_cov_matrices}
For any set $\mbb B$ containing $\ExternalField, \ParameterMatrix$, there exists a constant $C(\mbb B)$ such that
\begin{align}
  \sup\limits_{\ExternalField, \ParameterMatrix \in \mbb B} \max\Bigbraces{\opnorm{\Covariance_{\ExternalField,\ParameterMatrix}(\rvby, \rvby | {\svba}, \svbz, \svbv)}, \max\limits_{t \in [p_y]} \opnorm{\Covariance_{\ExternalField,\ParameterMatrix}(\rvby, \rvy_t \rvby | {\svba}, \svbz,\svbv)}} \leq C(\mbb B), \label{eq_cov_constraint}
\end{align}
almost surely. The expectation in \cref{eq_cov_constraint} is with respect to the distribution of $\rvby$ conditioned on $\rvba = \svba$, $\rvbz = \svbz$, and $\rvbv = \svbv$ which is fully parameterized by $\ExternalField$ and $\ParameterMatrix$, and can be obtained from \cref{eq_conditional_distribution_vay} after replacing $\ExternalField(\svbz)$ by $\ExternalField$.
\end{assumption}
\noindent In \cref{subsec_bounded_op_norms}, we show that $C(\mbb B)$ is a constant for a class of distributions. We note that this assumption is common in the literature on learning Gaussian graphical models to rule out singular distributions \citep{won2006maximum,zhou2011high,ma2016joint}.

We are now ready to state our guarantee for the estimate $\what{\mu}^{(i)}(\wtil{\svba}^{(i)})$ (see \cref{eq_causal_estimate}) of the expected potential outcomes for any unit $i \in [n]$ under an alternate intervention $\wtil{\svba}^{(i)} \in \cA^{p_a}$. We assume $p_v = p_a = p_y$ for brevity. See the proof in \cref{sec_proof_causal_estimand} where we also state a more general result. 

\begin{theorem}[{\outcomemainresultname}]
\label{thm_causal_estimand}
Suppose \cref{assumptions,ass_pos_eigenvalue,ass_bounded_op_norm_cov_matrices} hold. Then for any fixed $\varepsilon > 0$ and $\delta \in (0,1)$, 
 the estimates $\sbraces{\what{\mu}^{(i)}(\wtil{\svba}^{(i)})}_{i=1}^n$ defined in \cref{eq_causal_estimate} for any $\sbraces{\wtil{\svba}^{(i)} \in \cA^{p_a}}_{i=1}^n$ satisfy
\begin{align}
\label{eq:mu_error}
    \max_{i\in[n]} \! \frac{\stwonorm{\mu^{(i)}(\wtil{\svba}^{(i)}\!) \!-\! \what{\mu}^{(i)}(\wtil{\svba}^{(i)}\!)}}{C(\mbb B_i)} 
    \!\!\leq\!  R\Bigparenth{\varepsilon, \frac{\delta}{n}} \!\!+\! p \varepsilon ~\stext{for}~ n \!\geq\!\! \frac{ce^{c'\bGM} p^4 \!\bigparenth{p \log \frac{np^2}{\delta \varepsilon^2} \!+\! \tmetric_{\ExternalField,n}\normalparenth{ \varepsilon, \frac{\delta}{n}}\!}}{\varepsilon^4}\!,
\end{align}
 with probability at least $1-\delta$,
where $R(\varepsilon, \delta)$ was defined in \cref{eq_radius_node_thm}, $\tmetric_{\ExternalField,n}(\varepsilon, \delta) $ was defined in \cref{eq_tmetric_node_thm}, $C(\mbb B)$ was defined in \cref{eq_cov_constraint}, and
\begin{align}
\mbb B_i \defeq \bigbraces{\ExternalField \in \Lambda_{\theta}: \stwonorm{\ExternalField \!-\! \TrueExternalFieldI} \leq R\Bigparenth{\varepsilon, \frac{\delta}{n}}} \times \bigbraces{\ParameterMatrix \in \Lambda_{\Theta}: \max_{t\in[p]}\stwonorm{\ParameterRowt \!-\! \TrueParameterRowt} \leq \varepsilon}.
\end{align}
\end{theorem}
Repeating the algebra as in \cref{eq:simplified_bound_theta}
and treating $C(\mbb B_i)$ as a constant, the bound~\cref{eq:mu_error} yields the following simplified bound for the MSE of our mean outcome estimate $\mu^{(i)}(\wtil{\svba}^{(i)})$ for unit $i \in [n]$ under treatment $\wtil{\svba}^{(i)} \in \cA^{p_a}$: whenever $n \geq c'\varepsilon^{-4} p^4\normalparenth{p\log \frac{p^2}{\delta \varepsilon^2}  + \metric_{\ExternalField,n}(\varepsilon^2/p) + p\metric_{\ExternalField}(c)}$, we have
\begin{align}
    \mathrm{MSE}(\mu^{(i)}(\wtil{\svba}^{(i)}), \what{\mu}^{(i)}(\wtil{\svba}^{(i)}))
    \!\leq\! \varepsilon^2 \!+\! \dfrac{\metric_{\ExternalField}(c) \!+\! \log (\log \frac{p}{\delta})}{p}. 
\end{align}
This bound is of the same order as in \cref{eq:simplified_bound_theta} and can be formalized for the two examples (\cref{exam:lc_dense,exam:sc}) by deriving a suitable analog of \cref{cor_params}. In a nutshell, in both settings, the unit-level expected potential outcomes can be estimated well when the total number of units $n$ is large and the observations for each unit are high dimensional compared to the number of vectors $k$ in \cref{exam:lc_dense} or the sparsity parameter $s$ in \cref{exam:sc}.  We omit a formal statement for brevity. 

Finally, we also note that as in \cref{theorem_parameters}, the exponential dependence on $\bGM$ is expected to be unavoidable due to the principle of conjugate duality \citep{wainwright2008graphical}, i.e., the existence of a unique mapping from the parameters to the means and vice versa for the exponential family. 
Moreover, as in the discussion after \cref{cor_params}, the sharpness of the rate of $1/\varepsilon^4$ {is left} for future work. {Improving the dependency on $\varepsilon$ in \cref{eq:mu_error} would also improve the dependency on $p$.}
\newcommand{\gm}{\texttt{GM}}
\section{Possible extensions}
We now discuss how to extend our theoretical results with various relaxations of the exponential family modeling. 

\subsection{Modeling only the conditional distribution as exponential family}
\label{subsec_joint_vs_conditional}
Our framework and analysis can be extended to the setting where, instead of the joint distribution $f_{\ranvarvec}$ of $\ranvarvec = \normalparenth{\rvbz, \rvbv, \rvba, \rvby}$, we model only 
the conditional distribution $f_{\rvby|\rvba, \rvbz, \rvbv}$ of $\rvby$ conditioned on $\rvba$, $\rvbz$, and $\rvbv$ as an exponential family.
Note that when the joint distribution $f_{\ranvarvec}$ is an exponential family, the conditional distribution $f_{\rvby|\rvba, \rvbz, \rvbv}$ is also an exponential family, however a vice versa implication does not hold so that the setting considered here is a strict generalization of our previous setting. In fact, the conditional distribution $f_{\rvby|\rvba, \rvbz, \rvbv}$ being an exponential family puts no restrictions on the marginal distribution $f_{\rvbz, \rvbv, \rvba}$ of the unobserved covariates, the observed covariates, and the interventions as is the case with non-linear panel data models (\cref{sec_related_work}).

To estimate the expected potential outcomes $\mu^{(i)}(\wtil{\svba}^{(i)})$ in \cref{eq_causal_estimand} for any given unit $i$ 
and any alternate intervention $\wtil{\svba}^{(i)}$, 
it suffices to estimate the conditional distribution of $f_{\rvby|\rvba, \rvbz, \rvbv}(\cdot\vert \rvbv=\svba, \rvbv=\svbv^{(i)}\rvbz=\svbz^{(i)} )$ $\rvby$ for unit $i$ as a function of the intervention $\rvba$
(as in \cref{eq_counterfactual_distribution_y}). This task is equivalent to estimating $\Truephi$ in \cref{eq:gammai} under the exponential family models in \cref{eq_joint_distribution_zvay} or \cref{eq_conditional_distribution_y}. 

In \cref{subsec_exp_fam}, under the exponential family in \cref{eq_joint_distribution_zvay}, we argued (for analytical convenience) that learning $\Truephi$ is subsumed in learning the parameters corresponding to the conditional distribution $f_{\rvbx |\rvbz}$ of $\rvbx = \normalparenth{\rvbv, \rvba, \rvby}$ conditioned on $\rvbz$
(which also belongs to an exponential family with linear and quadratic interactions) as in \cref{eq_conditional_distribution_vay}. Then, we set the goal of estimating the parameters in \cref{eq_parameters_of_interest} and  designed a loss function to do so.  The loss function depended on the conditional distribution $f_{\rvx_t | \rvbx_{-t}, \rvbz}$ \cref{eq_conditional_dist_non_centered} of the random variable $\rvx_t$ conditioned on $\rvbx_{-t} = \svbx_{-t}$ and $\rvbz = \svbz$ for every $t \in [p]$.

Under the exponential family in \cref{eq_conditional_distribution_y}, we focus on directly learning the components of \cref{eq_parameters_of_interest} relevant to learning $\Truephi$, i.e., 
\begin{align}
    \TrueExternalFieldt(\svbz^{(i)}) & = \phi^{\star(y)} \!+\! 2 \Phi^{\star(z,y)\top} \svbz^{(i)} \in \Reals^{p_y \times 1}, \qtext{for all} t \in \normalbraces{p_v + p_a + 1, \cdots, p_v + p_a + p_y} \label{eq_new_parameters_1}\\
    \TrueParameterRowt & = \normalparenth{\Phi^{\star(v,y)}, \Phi^{\star(a,y)}, \Phi^{\star(y,y)}} \in \Reals^{p \times 1} \qtext{for all} t \in \normalbraces{p_v + p_a + 1, \cdots, p_v + p_a + p_y} \label{eq_new_parameters_2}.
\end{align}
We note that the conditional distribution $f_{\rvy_t | \rvby_{-t}, \rvbv, \rvba, \rvbz}$ of the random variable $\rvy_t$ conditioned on $\rvby_{-t} = \svby_{-t}$, $\rvbv = \svbv$, $\rvba = \svba$, and $\rvbz = \svbz$ for every $t \in [p_y]$ is consistent with the conditional distribution $f_{\rvx_{t'} | \rvbx_{-t'}, \rvbz}$ in \cref{eq_conditional_dist_non_centered} for every $t' \in \normalbraces{p_v + p_a + 1, \cdots, p_v + p_a + p_y}$. As a result, we can adapt the loss function in \cref{eq:loss_function} to learn the parameters in \cref{eq_new_parameters_1,eq_new_parameters_2} by summing over $t \in \normalbraces{p_v + p_a + 1, \cdots, p_v + p_a + p_y} $ instead of $t \in [p]$. Consequently, the guarantees in \cref{sec_main_results} continue to hold with $p$ replaced by $p_y$.

\subsection{Higher order terms in the conditional exponential family}
\label{subsec_high_terms}
In \cref{subsec_joint_vs_conditional}, we described how our framework and results apply when only the conditional distribution $f_{\rvby|\rvba, \rvbz, \rvbv}$ is modeled as the exponential family distribution in \cref{eq_conditional_distribution_y} where the term inside the exponent is linear in $(\rvbz, \rvbv, \rvba)$ and quadratic in $\rvby$.
We now describe how our framework and results are applicable when the conditional distribution $f_{\rvby|\rvba, \rvbz, \rvbv}$ is modeled as the following exponential family distribution
\begin{align}
    f_{\rvby | \rvba, \rvbz, \rvbv}(\svby | \svba, \svbz, \svbv) \! \propto  \exp\bigparenth{q_{\Phi}(\svbv, \svba, \svby)} \!\exp\bigparenth{2\svbz\tp\Phi^{(z,y)}\svby}, \label{eq_conditional_distribution_y_base_measure}
\end{align}
where $q_{\Phi}(\svbv, \svba, \svby)$ is some bounded degree polynomial
in $(\svbv, \svba, \svby)$ parameterized by $\Phi$, i.e., the term inside the exponent is linear in $\rvbz$ and arbitrary bounded degree polynomial in $(\rvbv, \rvba, \rvby)$. We note that every term in $q_{\Phi}(\svbv, \svba, \svby)$ needs to depend on $\rvby$ for it to contribute to $f_{\rvby | \rvba, \rvbz, \rvbv}$ in \cref{eq_conditional_distribution_y_base_measure}. For convenience, hereon, we ignore any dependence on $\rvbv$, and abuse notation to let $q_{\Phi}(\svba, \svby) = q_{\Phi}(\svbv, \svba, \svby)$. Then, in \cref{eq_conditional_distribution_y}, $q_{\Phi}(\svba, \svby)$ was a polynomial of degree 2 , i.e.,
\begin{align}
    q_{\Phi}(\svba, \svby) = q^{(2)}_{\Phi}(\svba, \svby) \defn \texttt{Sum}\Bigparenth{\phi^{(y)} \odot \svby + 2 \Phi^{(a,y)} \odot \bigparenth{\svba \otimes \svby} +  \Phi^{(y,y)} \odot \bigparenth{\svby \otimes \svby}},
\end{align}
where $\odot$ denotes the Hadamard product, $\otimes$ denotes the Kronecker product, $\Phi = (\phi^{(y)}, \Phi^{(a,y)}$, $\Phi^{(y,y)})$ with $\Phi^{(y,y)}$ being symmetric, and $\texttt{Sum}\normalparenth{s_1 + \cdots + s_h} \in \Reals$ sums, over all $i \in [h]$, all the entries of $s_i$ which could be a real number/vector/matrix/tensor. To explain how the loss function in \cref{eq:loss_function} needs to be modified for general $q_{\Phi}(\svba, \svby)$, we consider a polynomial of degree 3:
\begin{align}
    q_{\Phi}(\svba, \svby)  = q^{(2)}_{\Phi}(\svba, \svby) +   \texttt{Sum}\Bigparenth{ \hspace{-1.25cm}  \sum_{(u_1,u_2) \in \normalbraces{(a,a), (a,y), (y,y)}}  \hspace{-1.25cm} c_{u_1,u_2} \cdot \Phi^{(u_1,u_2,y)} \odot \bigparenth{\svbu_1 \otimes \svbu_2 \otimes \svby}}, 
\end{align}
where $c_{a,a} = c_{a,y} = 3$, $c_{y,y} = 1$ are constants chosen for consistency, and $\Phi^{(u_1,u_2,y)} \in \Reals^{p_{u_1} \times p_{u_2} \times p_y}$ is symmetric with respect to indices that are repeated for every $(u_1,u_2) \in \normalbraces{(a,a), (a,y), (y,y)}$. We illustrate the two steps from \cref{subsec_loss_function} below.
\paragraph{Centering sufficient statistics of the conditional distribution of a variable}
The conditional distribution $f_{\rvy_t | \rvby_{-t}, \rvba, \rvbz}$ of the random variable $\rvy_t$ conditioned on $\rvby_{-t} = \svby_{-t}$, $\rvba = \svba$, and $\rvbz = \svbz$ for every $t \in [p_y]$ is given by
\begin{align}
    f_{\rvy_t | \rvby_{-t}, \rvba, \rvbz}\bigparenth{y_t | \svby_{-t}, \svba, \svbz 
    } \!\propto  & \exp\!\biggparenth{\!\texttt{Sum}\Bigparenth{\Bigbrackets{\phi_t(\svbz) + \hspace{-0.322cm} \sum_{u \in \normalbraces{y_{-t}, a}} \hspace{-0.3cm}2\Phi^{(u,y_t)} \odot \svbu + \hspace{-1.79cm} \sum_{(u_1,u_2) \in \normalbraces{(a,a), (a,y_{-t}), (y_{-t},y_{-t})}} \hspace{-1.79cm} c_{u_1,u_2} \Phi^{(u_1,u_2,y_t)} \odot \bigparenth{\svbu_1 \otimes \svbu_2}} y_t \\
    &  \qquad + \Bigbrackets{\Phi^{(y_t,y_t)}  +  \hspace{-0.25cm} \sum_{u \in \normalbraces{y_{-t}, a}} \hspace{-0.25cm} 3\Phi^{(u,y_t,y_t)} \odot \svbu} \Bigparenth{y_t^2 - \frac{\xmax^2}{3}} + \Phi^{(y_t,y_t,y_t)} y_t^3}},
\end{align}
where $\phi_t(\svbz) \defn \phi^{(y_t)}  + 2 \Phi^{(z,y_t)} \odot \svbz$, $c_{y_{-t},y_{-t}} = 3$, and $c_{a,y_{-t}} = 6$. Let $\Phi_t$ denote the concatenation of all the remaining parameters.
As in \cref{eq_conditional_dist}, the term $\xmax^2/3$ inside the exponent is vacuous and centers the sufficient statistics $\rvy_t^2$. The other sufficient statistics, i.e., $\rvx_t$ and $\rvx_t^3$, are naturally centered as their integrals with respect to the uniform distribution on $\cX$ are both zeros. 
\paragraph{Constructing the loss function}
Now, it is easy to see that the corresponding loss $\loss$ is given by
\begin{align}
    \loss& 
    = \frac{1}{n} \sum_{t \in [p_y]} \sumn[i] \exp\!\biggparenth{-\texttt{Sum}\Bigparenth{\Bigbrackets{\phi_t^{(i)} + \hspace{-0.322cm} \sum_{u \in \normalbraces{y_{-t}, a}} \hspace{-0.3cm}2\Phi^{(u,y_t)} \odot \svbu^{(i)} + \hspace{-1.79cm} \sum_{(u_1,u_2) \in \normalbraces{(a,a), (a,y_{-t}), (y_{-t},y_{-t})}} \hspace{-1.79cm} c_{u_1,u_2} \Phi^{(u_1,u_2,y_t)}  \odot \bigparenth{\svbu_1^{(i)} \otimes \svbu_2^{(i)}}} y_t^{(i)} \\
    &  \qquad \qquad \qquad \qquad \qquad + \Bigbrackets{\Phi^{(y_t,y_t)}  +  \hspace{-0.25cm} \sum_{u \in \normalbraces{y_{-t}, a}} \hspace{-0.25cm} 3\Phi^{(u,y_t,y_t)} \odot \svbu^{(i)}} \Bigparenth{\bigbrackets{y_t^{(i)}}^2 - \frac{\xmax^2}{3}} + \Phi^{(y_t,y_t,y_t)} \bigbrackets{y_t^{(i)}}^3}},
\end{align}
and minimizing this convex loss results in the estimates of $\normalbraces{\phi_t^{(i)}}_{i \in [n]}$ and $\normalbraces{\Phi_{t}}_{t \in p_y}$. Consequently, the guarantees in \cref{sec_main_results} continue to hold with $p$ replaced by $p_y$ as long as \cref{assumptions,ass_pos_eigenvalue,ass_bounded_op_norm_cov_matrices} are appropriately generalized.

\paragraph{Tilting the base distribution} We note that the exponential family in \cref{eq_conditional_distribution_y} can be rewritten as
\begin{align}
    f_{\rvby | \rvba, \rvbz, \rvbv}(\svby | \svba, \svbz, \svbv) \! \propto  \exp\bigparenth{2\svbz\tp\Phi^{(z,y)}\svby}
    \exp\bigparenth{2\svbv\tp\Phi^{(v,y)}\svby}
    \exp\bigparenth{2\svba\tp\Phi^{(a,y)}\svby}
 \exp\bigparenth{\phi^{(y)\tp} \svby + \svby\tp\Phi^{(y,y)}\svby}, \label{eq_conditional_tilted}
\end{align}
where $\exp\bigparenth{\phi^{(y)\tp} \svby + \svby\tp\Phi^{(y,y)}\svby}$ stands for a base distribution on $\rvby$ which is exponentially tilted by $\rvbz$, $\rvbv$, and $\rvba$, i.e., by $ \exp\bigparenth{2\svbz\tp\Phi^{(z,y)}\svby}$, $\exp\bigparenth{2\svbv\tp\Phi^{(v,y)}\svby}$, and $\exp\bigparenth{2\svba\tp\Phi^{(a,y)}\svby}$, respectively. Then, generalizing the exponential family in \cref{eq_conditional_distribution_y} to the one in \cref{eq_conditional_distribution_y_base_measure} is equivalent to saying that our approach and results continue to apply when $(a)$ the base distribution on $\rvby$ is an exponential family distribution where the term inside the exponent is arbitrary bounded degree polynomial (instead of quadratic) and $(b)$ the exponent of the exponential tilting of this base distribution by $\normalparenth{\rvbv, \rvba}$ is arbitrary bounded degree polynomial (instead of linear).

\subsection{Discrete and mixed variables}
\label{subsec_discrete}
In \cref{subsec_exp_fam}, we described how our framework and results are applicable when the support of $\rvbv$, $\rvba$, and $\rvby$ are bounded continuous sets, i.e., $\cV = \cA = \cY = [-\xmax, \xmax]$. In \cref{subsec_joint_vs_conditional}, we showed that it suffices to only model the conditional distribution $f_{\rvby|\rvba, \rvbz, \rvbv}$ as an exponential family distribution implying that we do not need any restrictions on the support of $\rvbv$ and $\rvba$. Now, we describe how to adapt our loss function when $\rvby = (\rvy_{1}, \cdots, \rvy_{p_y}) \in \cY_1 \times \cdots \times \cY_{p_y}$ where $\cY_{t}$ is either a discrete compact set or a continuous compact set for $t \in [p_y]$.

We note that the conditional distribution $f_{\rvy_t | \rvby_{-t}, \rvbv, \rvba, \rvbz}$ of the random variable $\rvy_t$ conditioned on $\rvby_{-t} = \svby_{-t}$, $\rvbv = \svbv$, $\rvba = \svba$, and $\rvbz = \svbz$ for every $t \in [p_y]$ is still consistent with the conditional distribution $f_{\rvx_{t'} | \rvbx_{-t'}, \rvbz}$ in \cref{eq_conditional_dist_non_centered} for every $t' \in \normalbraces{p_v + p_a + 1, \cdots, p_v + p_a + p_y}$. However, the constants used to center the sufficient statistics in \cref{eq_conditional_dist} may change. More precisely, for any $t \in [p]$, the sufficient statistics $\rvx_t$ and $\rvx_t^2$ are centered by subtracting $\Expectation_{\cU_t}\bigbrackets{\rvx_t}$ and $\Expectation_{\cU_t}\bigbrackets{\rvx_t^2}$, respectively where $\cU_t$ denotes the uniform distribution supported over $\cY_t$. Consequently, the loss function in \cref{eq:loss_function} as well as \cref{ass_pos_eigenvalue} can be adapted, and the guarantees in \cref{sec_main_results} continue to hold.
\section{Application: Imputing missing covariates}
\label{sec_sparse_measurement_errors}
\newcommand{\impute}{Impute missing covariates}
Consider a setting with no systematically unobserved covariates $\rvbz$; instead, elements of $(\rvbv, \rvba, \rvby)$ are missing or have measurement error for some fraction of the units. Our goal is to impute these missing values or denoise the measurement error in the observed values.

\paragraph{Problem setup} For the ease of exposition, we assume the observed covariates $\rvbv$ can have measurement error\footnote{Our analysis remains the same when observed covariates $\rvbv$ are missing instead of having measurement error.} but the interventions and the outcomes do not have any measurement error. More concretely, for every unit $i \in [n]$, along with the interventions $\svba^{(i)}$ and the outcomes $\svby^{(i)}$, we observe $\csvbv^{(i)} = \svbv^{(i)} + \Delta \svbv^{(i)}$ instead of true covariates $\svbv^{(i)}$ where $\Delta \svbv^{(i)}$ denotes (unobserved) bounded measurement error. We assume that a certain number of units (known to us) have no measurement error: say, $\Delta \svbv^{(i)} = 0$ for all $i \in \normalbraces{n/2+1, \cdots, n}$. 

\paragraph{Questions of interest} Besides counterfactual estimates, our goal is to estimate $\Delta \svbv^{(i)}$ for units with measurement error.

\subsection{A theoretical guarantee}
\label{subsec_theo_gua} 
Our methodology can be applied to estimate these measurement errors when the joint distribution of the true covariates $\rvbv \in \cX^{p_v}$, the interventions $\rvba \in \cX^{p_a}$, and the observed outcomes $\rvby \in \cX^{p_y}$ can be modeled as an exponential family, parameterized by a vector $\phi \in \Reals^p$ and a symmetric matrix $\Phi \in \Reals^{p \times p}$ where $p \defn p_v + p_a + p_y$, i.e., with $\rvbw \defn (\rvbv, \rvba, \rvby)$
\begin{align}
    f_{{\ranvarvec}}({\varvec}; \phi, \Phi) \propto \exp\Bigparenth{ \phi\tp {\varvec}
    +{\varvec}\tp \Phi {\varvec}},
    \qtext{where}
    {\varvec} \defeq (\svbv, \svba, \svby),
    \label{eq_joint_distribution_vay}
\end{align}
and $\svbv \defn (v_{1}, \cdots, v_{p_v})$, $\svba \defn (a_{1}, \cdots, a_{p_a})$, and $\svby \defn (y_{1}, \cdots, y_{p_y})$ denote realizations of $\rvbv$, $\rvba$, and $\rvby$, respectively. To estimate the counterfactual distribution, we decompose $\rvbv$ into $\crvbv$ and $\Delta \rvbv$, and obtain the distribution of the observed quantities $\rvbx \defn (\crvbv, \rvba, \rvby)$ conditioned on $\Delta \rvbv = \Delta \svbv$ as follows (see \cref{proof_impute_missing_covariates} for details)
\begin{align}
    f_{\rvbx|\Delta \rvbv}\bigparenth{\svbx| \Delta \svbv; \ExternalField(\Delta \svbv), \!\ParameterMatrix} \!\propto\! \exp \Bigparenth{\! \normalbrackets{\ExternalField(\Delta \svbv)}\!\tp \svbx \!+\! \svbx\!\tp \ParameterMatrix \svbx\!} \stext{where}
    \ExternalField(\Delta \svbv) \! \defn \begin{bmatrix}  \phi^{(v)} \!-\! 2 \Phi^{(v,v)\!\top} \Delta \svbv  \\  \phi^{(a)} \!-\! 2 \Phi^{(v,a)\!\top} \Delta \svbv  \\ \phi^{(y)} \!-\! 2 \Phi^{(v,y)\!\top} \Delta \svbv \end{bmatrix}\!, \label{eq_x_given_v_application}
\end{align} 
$\svbx \defn (\csvbv, \svba, \svby)$, $\Theta \defn \Phi$, and $\csvbv$, $\svba$, and $\svby$ denote realizations of $\crvbv$, $\rvba$, and $\rvby$, respectively. As in \cref{subsec_exp_fam}, to estimate the counterfactual distribution, it suffices to learn $\ExternalField(\Delta \svbv)\in \Reals^{p\times 1}$ and $\ParameterMatrix \in\Reals^{p\times p}$.

Let $f_{{\ranvarvec}}(\cdot; \phi^{\star}, \Phi^{\star})$ denote the true data generating distribution of ${\rvbw}$ in \cref{eq_joint_distribution_vay} and let $f_{\rvbx|\Delta \rvbv}\bigparenth{\cdot| \Delta \svbv; \TrueExternalField(\Delta \svbv), \TrueParameterMatrix}$ denote the true distribution of $\rvbx$ conditioned on $\Delta \rvbv = \Delta \svbv$. We assume $(a)$ $\max\braces{\infnorm{\Delta \svbv}, \infnorm{\phi^{\star}}, \maxmatnorm{\Phi^{\star}}} \leq \aGM$ and {$(b)$ $\infmatnorm{\Phi^{\star}} \leq \bGM$ analogous to \cref{assumptions} where the row-wise $\ell_1$ sparsity in $(b)$ is assumed to be induced by row-wise $\ell_0$ sparsity, i.e., $\zeronorm{\Phi_t^{\star}} \leq \bGM/\aGM$ for all $t \in [p]$}. Then, given realizations $\normalbraces{ \svbx^{(i)}}_{i=1}^{n}$ consistent with $f_{\rvbx|\Delta \rvbv}\bigparenth{\cdot| \Delta \svbv^{(i)}; \TrueExternalField(\Delta \svbv^{(i)}), \TrueParameterMatrix}$, first, we estimate the parameters $\phi^{\star}$ and $\Phi^{\star} = \TrueParameterMatrix$ using the realizations for units $\normalbraces{n/2 + 1, \cdots, n}$. Next, we exploit the structure in the problem to show that $\TrueExternalFieldI \defn \TrueExternalField(\Delta \svbv^{(i)})$ can be written as a linear combination of known vectors 
with some error, for every unit $i \in \normalbraces{1, \cdots, n/2}$. Then, we use \cref{eq_estimated_parameters} to estimate $\normalbraces{\TrueExternalFieldI}_{i = 1}^{n}$ and obtain estimates of $\normalbraces{\Delta \svbv^{(i)}}_{i = 1}^{n}$ as by-products.
In particular, the estimate of the coefficients associated with the aforementioned linear combination for $\TrueExternalFieldI$ turn out to be our estimate of the measurement error $\Delta \svbv^{(i)}$ for every $i \in \normalbraces{1, \cdots, n/2}$. For $i \in \normalbraces{n/2 + 1, \cdots, n}$, estimating $\TrueExternalFieldI$ and $\Delta \svbv^{(i)}$ is straightforward since $\TrueExternalFieldI = \phi^{\star}$ and $\Delta \svbv^{(i)} = 0$. We provide our guarantee on estimating $\TrueParameterMatrix$, $\TrueExternalFieldI$ for $i \in [n]$, and $\Delta \svbv^{(i)}$ for $i \in [n]$ below with a proof in \cref{proof_impute_missing_covariates}.

\begin{proposition}[{\impute}]
\label{prop_impute_missing_covariates}
Suppose the eigenvalues of $\tbf{B}\tp\tbf{B}$ are lower bounded by $\kappa {p}$ for some $\kappa >0 $ where $\tbf{B}  \!\defn\! \begin{bmatrix} \phi^{\star}, -2 \Phi_1^{\star}, \cdots, -2 \Phi_{p_v}^{\star} \end{bmatrix}\!\in \! \Reals^{p\times (p_v+1)}$. Then, for any fixed $\varepsilon_1 > 0$ and $\delta \in (0,1)$, there exists estimates $\EstimatedParameterMatrix$ and  $\bigbraces{\EstimatedExternalFieldI}_{i=1}^{n}$ such that, 
with probability at least $1-\delta$,
\begin{align}
\matnorm{\EstimatedParameterMatrix - \TrueParameterMatrix}_{2,\infty}  \leq & \varepsilon_1  \qquad \qquad \qquad \qquad \qquad \hspace{0.2cm} \qquad  \qtext{when} n  \geq \frac{ce^{c'\bGM} \log\frac{p}{\sqrt{\delta}}}{\varepsilon_1^2},\label{mse_measurement_error_Theta}\intertext{and}
\max_{i\in[n]}
    \mathrm{MSE}(\EstimatedExternalFieldI, \! \TrueExternalFieldI)
    \leq & \max\Bigbraces{\varepsilon_1^2, \frac{ce^{c'\bGM} \bigparenth{p_v \!+\! \log \normalparenth{\log \frac{np}{\delta}}}}{p}} \stext{when}
    n \geq \frac{ce^{c'\bGM} \bigparenth{\log \frac{\sqrt{n}p}{\sqrt{\delta}} \!+\! p_v}}{\varepsilon_1^2}\label{mse_measurement_mse}.
\end{align}
Further, for any fixed $\varepsilon_2 > 0$, if $\varepsilon_2 \leq \frac{1}{8} \sqrt{\frac{p}{p_v+1}}$, there exist estimates $\bigbraces{\what{\Delta \svbv}^{(i)}}_{i=1}^{n}$ such that,
\begin{align}
    \max_{i\in[n]}
    \stwonorm{\what{\Delta \svbv}^{(i)}- \Delta \svbv^{(i)} }^2 \leq\max\Bigbraces{\frac{c_1 \varepsilon_2^2 \kappa}{p_v+1} , \dfrac{ce^{c'\bGM} \bigparenth{p_v \!+\! \log \normalparenth{\log \frac{np}{\delta}}}}{p \kappa}} \!+\! \varepsilon_2^2 \kappa,
    \label{mse_measurement_error}
\end{align}
with probability at least $1-\delta$, whenever $n \geq ce^{c'\bGM} \kappa^{-2} \varepsilon_2^{-2} (p_v\!+\!1) \bigparenth{\log \frac{\sqrt{n}p}{\sqrt{\delta}} + p_v}$.
\end{proposition}

The above guarantees can be simplified as follows by treating $\bGM$ and $\kappa$ as constants as well as ignoring the constants, and the logarithmic factors in $n$ and $\delta$ (denoted by $\precsim$ and $\succsim$): for any $\varepsilon_1 > 0$ and $\frac{1}{8} \sqrt{\frac{p}{p_v+1}} \geq \varepsilon_2 > 0$
\begin{align}
\matnorm{\EstimatedParameterMatrix \!-\! \TrueParameterMatrix}_{2,\infty} & \leq  \varepsilon_1   \qquad  \qquad  \qquad  \qquad \qtext{when} n  \succsim \frac{\log p}{\varepsilon_1^2}, \label{mse_measurement_error_Theta_simplified}\\
\max_{i\in[n]} \mathrm{MSE}(\EstimatedExternalFieldI, \TrueExternalFieldI) & \precsim  \max\Bigbraces{\varepsilon_1^2, \dfrac{p_v}{p}}  \qquad  \hspace{0.4cm} \qtext{when} n \succsim \frac{\log p + p_v}{\varepsilon_1^2}, \label{mse_measurement_mse_simplified}
\intertext{and}
\max_{i\in[n]} \stwonorm{\what{\Delta \svbv}^{(i)}- \Delta \svbv^{(i)} }^2 & \precsim \max\Bigbraces{\frac{\varepsilon_2^2}{p_v}, \dfrac{p_v}{p}} + \varepsilon_2^2 \hspace{3mm} \qtext{when}  n \succsim \frac{p_v (\log p + p_v)}{\varepsilon_2^2}. \label{mse_measurement_error_simplified}
\end{align}
For large $n$, whenever, $\max\bigbraces{\varepsilon_1^2, \frac{p_v}{p}}  = \frac{p_v}{p}$ and $\max\bigbraces{\frac{\varepsilon_2^2}{p_v}, \frac{p_v}{p}} = \frac{p_v}{p}$, the guarantees in \cref{mse_measurement_mse_simplified,mse_measurement_error_simplified} can be written as
\begin{align}
\max_{i\in[n]} \mathrm{MSE}(\EstimatedExternalFieldI, \TrueExternalFieldI) & \precsim  \dfrac{p_v}{p} \qquad  \qtext{when} n \succsim \frac{p \log p}{p_v},\label{measurement_mse_simple}
\intertext{and}
\max_{i\in[n]} \stwonorm{\what{\Delta \svbv}^{(i)}- \Delta \svbv^{(i)} }^2 & \precsim \dfrac{p_v^2}{p} \qquad \qtext{when} n \succsim \frac{p \log p}{p_v}. \label{measurement_error_simple}
\end{align}

\paragraph{Remark} The measurement errors can be recovered well as long as enough units with no measurement error are observed (i.e., $n/2$ is large) and the observation per unit is high dimensional (i.e., $p$ is large compared to $p_v^2$). We note that the quadratic dependence (on $p_v$) in \cref{measurement_error_simple} arises because of the error in expressing $\TrueExternalFieldI$ as a linear combination of known vectors. In contrast, we get a linear dependence (on $k$) in \cref{cor_params}\cref{item:lc_dense} where there is no error in expressing $\TrueExternalFieldI$ as a linear combination of known vectors (via \cref{exam:lc_dense}).

\subsection{Simulations} 
We now present some simulation results to empirically evaluate the error scaling of our parameter estimates with three key aspects of the application above: number of units $n$, dimension $p$, and dimension $p_v$ of covariates with measurement error.

\paragraph{Data generation}
We choose $\cX = [-1,1]$ and $p_a = p_y = (p-p_v)/2$. The true joint  distribution~\cref{eq_joint_distribution_vay} of $\rvbw \defn (\rvbv, \rvba, \rvby)$ is set as a truncated Gaussian distribution with the parameters $\phi^{\star} = \mathbf{1} \in \Reals^{p}$ and  
 a positive definite $\Phi^{\star} \in \Reals^{p \times p}$ generated using  \textit{sklearn} package \citep{pedregosa2011scikit} 
 such that $\aGM = 6$, $\bGM = 4$, and 
 $\kappa = 0.15$.
 We draw $n$ i.i.d. samples  $\normalbraces{\svbw^{(i)}}_{i = 1}^{n}$ from this true distrbution using \textit{tmvtnorm} package \citep{wilhelm2010tmvtnorm}.
Next, we generate $\Delta \svbv^{(i)}$ uniformly from $[0.9, 1]^{p_v}$ for units $i \in \normalbraces{1, \cdots, n/2}$ while setting $\Delta \svbv^{(i)} =\mathbf{0}$ for other units.
Combining $\normalbraces{\svbw^{(i)}}_{i = 1}^{n}$ and $\normalbraces{\Delta \svbv^{(i)}}_{i = 1}^{n}$ yields $\normalbraces{\svbx^{(i)}}_{i = 1}^{n}$ (see \cref{eq_x_given_v_application}).

 \begin{figure}[!t]
    \centering
    \begin{tabular}{c}
        \includegraphics[width=1\textwidth]{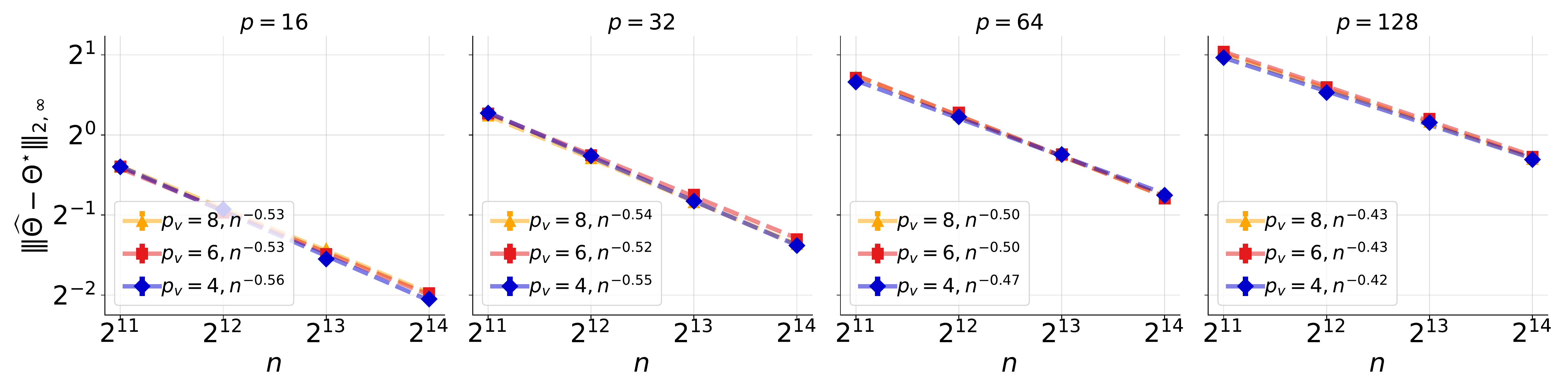}\\
        (a)\\ 
        \includegraphics[width=1\textwidth]{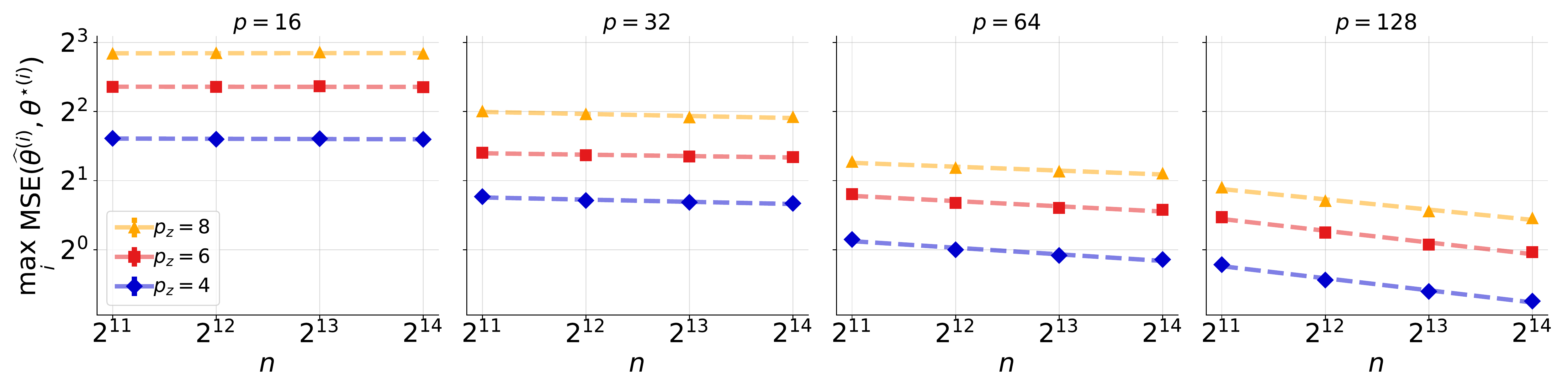}\\
        (b)\\
        \includegraphics[width=1\textwidth]{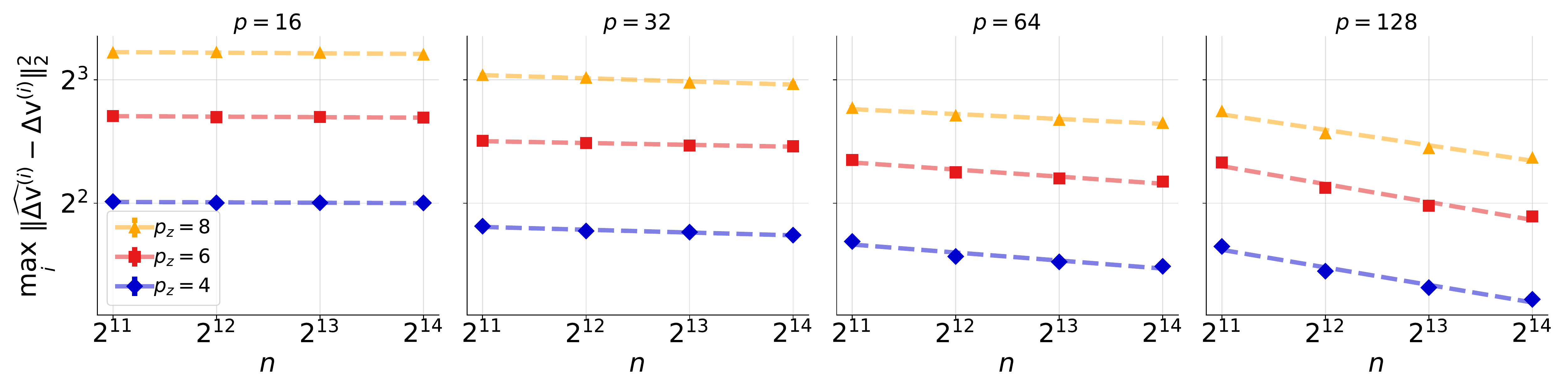}\\
        (c)
    \end{tabular}
    \caption{
    Error scaling with number of units $n$, for various $p$ and $p_v$, for our estimates of $\TrueParameterMatrix$ (top row), $\sbraces{\TrueExternalFieldI}_{i=1}^{n}$ (middle row), and $\sbraces{\Delta \svbv^{(i)}}_{i=1}^{n}$ (bottom row).}
    \label{fig:vs_p_n}
\end{figure}

\paragraph{Plot details} 
In \cref{fig:vs_p_n}, we plot the scaling of errors in our estimates for $\TrueParameterMatrix$ in the top row, $\sbraces{\TrueExternalFieldI}_{i=1}^{n}$ in the middle row, and $\sbraces{\Delta \svbv^{(i)}}_{i=1}^{n}$ in the bottom row. In particular, we present how the error scales as the dimension $n$ grows for various $p$ and $p_v$. We plot the averaged error across 50 independent trials along with $\pm 1$ standard error (the standard error is too small to be visible in our results). 
To help see the error scaling, we provide the least squares fit on the log-log scale (log error vs log x-axis). We display the best linear fit and mention an empirical decay rate in the legend based on the slope of that fit, e.g., for a slope of $-0.56$ for estimating $\TrueParameterMatrix$ when $p = 16$ and $p_v = 4$, we report an empirical rate of $n^{-0.56}$ for the averaged error. In the middle row and the bottom row of \cref{fig:vs_p_n}, the rates vary from $n^{0.00}$ to $n^{-0.17}$, and we omit these weak dependencies in the legend to reduce clutter.

\paragraph{Error scaling for $\EstimatedParameterMatrix$}
From the first row of \cref{fig:vs_p_n}, we observe that the error $\matnorm{\EstimatedParameterMatrix \!-\! \TrueParameterMatrix}_{2,\infty}$ admits a scaling of  between $n^{-0.56}$ and $n^{-0.42}$ for various $p$ and $p_v$. These empirical rates indicate a parametric error rate of $n^{-0.5}$ for $\matnorm{\EstimatedParameterMatrix \!-\! \TrueParameterMatrix}_{2,\infty}$, consistent with the scaling of $\varepsilon^{-2}$ in \cref{mse_measurement_error_Theta_simplified}. Further, as expected, the error $\matnorm{\EstimatedParameterMatrix \!-\! \TrueParameterMatrix}_{2,\infty}$ does not depend on $p_v$ but increases with an increase in $p$.

\paragraph{Error scaling for $\EstimatedExternalFieldI$}
In the middle row of \cref{fig:vs_p_n}, we see the error $\max_{i\in[n]}\mathrm{MSE}(\EstimatedExternalFieldI\!, \TrueExternalFieldI)$ has a weak dependence on $n$ for a fixed $p$ and $p_v$, decreases with an increase in $p$ for any fixed $n$ and $p_v$, and increases with an increase in $p_v$ for any fixed $n$ and $p$.  
This is consistent with \cref{mse_measurement_mse_simplified} when $\max\bigbraces{\varepsilon_1^2, \frac{p_v}{p}} \!=\! \frac{p_v}{p}$ (see \cref{measurement_mse_simple}). Further, we note that the decay of the error with $p$ is slower for smaller $n$ (cf. $n = 2^{11}$ vs $n = 2^{14}$). This is expected from \cref{mse_measurement_mse_simplified} where the $n$  required to ensure $\max\bigbraces{\varepsilon_1^2, \frac{p_v}{p}} \!=\! \frac{p_v}{p}$ increases with an increase in $p$. As a result, for larger $p$, $\varepsilon_1^2$ comes into the picture explaining the increased dependence of the error on $n$ (cf. $p = 16$ vs $p = 128$).

\paragraph{Error scaling for $\what{\Delta \svbv}^{(i)}$ \!}
The trends in the error $\max_{i\in[n]}\!\stwonorm{\what{\Delta \svbv}^{\!(i)}\!-\! \Delta \svbv^{(i)}}^2$ are similar to the error $\max_{i\in[n]}\mathrm{MSE}(\EstimatedExternalFieldI, \TrueExternalFieldI)$. In the bottom row of \cref{fig:vs_p_n}, we see $\max_{i\in[n]}\!\stwonorm{\what{\Delta \svbv}^{\!(i)}\!-\! \Delta \svbv^{(i)} }^2$ has a weak dependence on $n$ for a fixed $p$ and $p_v$, decreases with an increase in $p$ for any fixed $n$ and $p_v$, and increases with an increase in $p_v$ for any fixed $n$ and $p$. This is consistent with \cref{mse_measurement_error_simplified} when $\max\bigbraces{\frac{\varepsilon_2^2}{p_v}, \frac{p_v}{p}} \!=\! \frac{p_v}{p}$ (see \cref{measurement_error_simple}). For the same reason mentioned in the previous paragraph, we see a slower decay in the error with $p$  for smaller $n$ (cf. $n = 2^{11}$ vs $n = 2^{14}$), and a higher dependence of the error on $n$ for larger $p$ (cf. $p = 16$ vs $p = 128$).

\section{Proof Sketch for \cref{theorem_parameters}: \parammainresultname}
\label{sec_proof_sketch}
Our proof of \cref{theorem_parameters} proceeds in two stages (see \cref{fig:proof_sketch} for an overview). First, we establish \cref{eq:matrix_guarantee} for estimating $\TrueParameterMatrix$. Next, we use this guarantee to establish the unit-level guarantee~\cref{eq:theta_guarantee} for each of $\bigbraces{\TrueExternalFieldI[1], \cdots, \TrueExternalFieldI[n]}$ by substituting $\ParameterMatrix = \EstimatedParameterMatrix$ in \cref{eq_estimated_parameters}, i.e.,  analyzing the following convex optimization problem: 
\begin{align}
    \normalbraces{\EstimatedExternalFieldI[1], \cdots, \EstimatedExternalFieldI[n]} \in \argmin_{\normalbraces{\ExternalFieldI[1], \cdots, \ExternalFieldI[n]} \in \ParameterSet_{\ExternalField}^n} \loss\bigparenth{\EstimatedParameterMatrix, \ExternalFieldI[1], \cdots, \ExternalFieldI[n]}. \label{eq_decomposed_second_minimum}
\end{align}

\subsection{Estimating the population-level parameter}
\label{sec_proof_sketch_pop}
In the first part, we show that all points $\ExtendedParameterMatrix \in \ParameterSet_{\ParameterMatrix} \times \ParameterSet_{\ExternalField}^n$, such that $\stwonorm{\ParameterRowt - \TrueParameterRowt} \geq \varepsilon$ for at least one $t \in [p]$, uniformly satisfy 
\begin{align}
    \loss(\ExtendedParameterMatrix) \geq \loss(\ExtendedTrueParameterMatrix) + \Omega(\varepsilon^2) \stext{for} n \geq \frac{ce^{c'\bGM} p^2}{\varepsilon^4} \cdot \Bigparenth{p \log \frac{p}{\delta \varepsilon^2} + \metric_{\ExternalField, n}\big(\varepsilon^2 \big)},
    \label{eq_proof_sketch_thm_edge}
\end{align}
with probability at least $1-\delta$. Then, we conclude the proof using contraposition. 

To prove \cref{eq_proof_sketch_thm_edge}, we first decompose the convex (and positive) objective $\loss(\ExtendedParameterMatrix)$ in \cref{eq:loss_function} as a sum of $p$ convex (and positive) auxiliary objectives $\loss_t$, namely, $\loss(\ExtendedParameterMatrix) = \sump \loss_t\bigparenth{\ExtendedParameterRowT}$ where
\begin{align}
    \loss_t\bigparenth{\ExtendedParameterRowT} \defn \frac{1}{n} \sumn[i] \exp\Bigparenth{-\normalbrackets{\ExternalFieldtI + 2\EstimatedParameterRowttt\tp \svbx_{-t}^{(i)}} x_t^{(i)} - \EstimatedParameterTU[tt] \Bigbrackets{[x_t^{(i)}]^2 - \frac{\xmax^2}{3}}}.
    \label{eq:loss_function_p_original}
\end{align}
Next, for any fixed $t \in [p]$, $\varepsilon > 0$, and $\ExtendedParameterMatrix \in \ParameterSet_{\ExternalField}^n \times \ParameterSet_{\ParameterMatrix}$ with $\stwonorm{\ParameterRowt - \TrueParameterRowt} \geq \varepsilon$, we show (see \cref{lemma_parameter})
\begin{align}
    \loss_t(\ExtendedParameterRowT) \geq \loss_t(\ExtendedTrueParameterRowT) + \Omega(\varepsilon^2) - \varepsilon_1 \qtext{whenever} n \geq \frac{c e^{c'\bGM}\log \frac{p}{\delta}}{\varepsilon_1^2} , \label{eq_firststage_key_step}
\end{align}
and then establish the same bound uniformly for all $t \in [p]$ with probability $1-\delta$. Taking a sum over $t$ on both sides of \cref{eq_firststage_key_step}, we conclude that for any fixed  
$\ExtendedParameterMatrix$ with $\stwonorm{\ParameterRowt - \TrueParameterRowt} \geq \varepsilon$ for some $t \in [p]$, 
\begin{align}
    \loss(\ExtendedParameterMatrix) \geq \loss(\ExtendedTrueParameterMatrix) + \Omega(\varepsilon^2) \qtext{whenever} n \geq \frac{c e^{c'\bGM} p^2 \log \frac{p}{\delta}}{\varepsilon^4}, \label{eq_proof_sketch_first_part_summed}
\end{align}
with probability at least $1\!-\!\delta$ where we substituted $\varepsilon_1 \!=\! c\varepsilon^2/p$. Finally, we conclude \cref{eq_proof_sketch_thm_edge} by using \cref{eq_proof_sketch_first_part_summed}, the Lipschitzness of  $\loss$ (see \cref{lemma_lipschitzness_first_stage}), and a covering number argument (see \cref{sec:proof_of_theorem_parameters}). 

We establish \cref{eq_firststage_key_step} (\cref{lemma_parameter}) via \cref{lemma_conc_first_sec_der}, which provides suitable concentration and anti-concentration results for the first-order and second-order derivatives, respectively, for the auxiliary objective $\loss_t$ in \cref{eq:loss_function_p_original}. We prove \cref{lemma_conc_first_sec_der} by extending the results from \cite{ShahSW2021A} to the setting with non-identical but independent samples $\normalbraces{\svbx^{(i)}\sim \TrueJointDistfunT}_{i = 1}^{n}$.

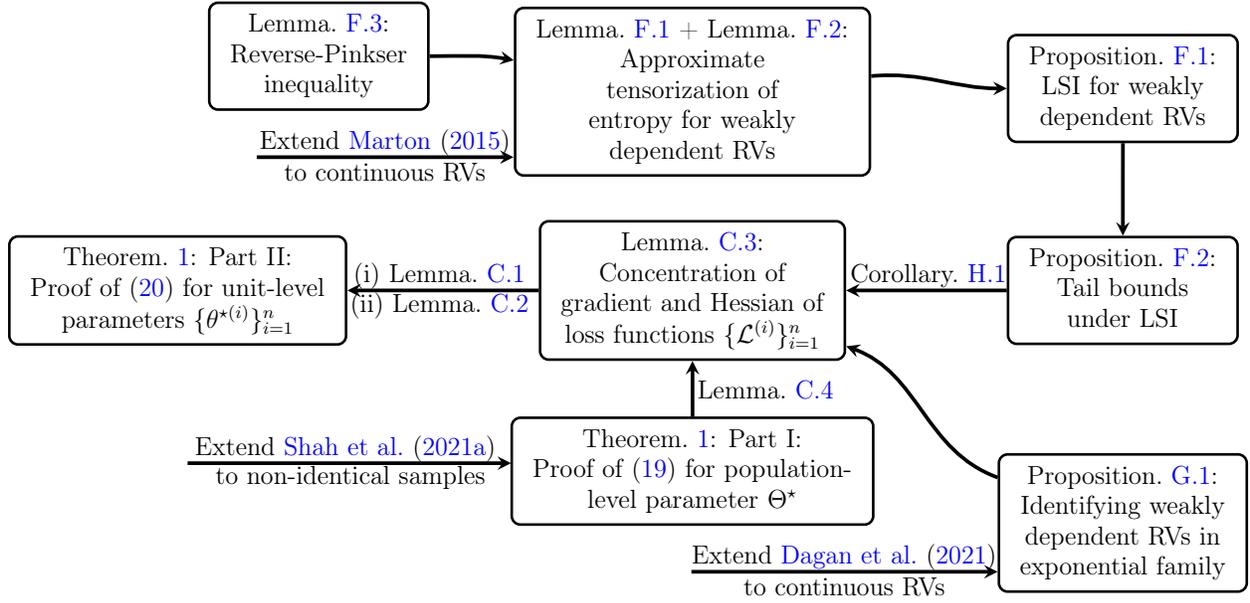
\begin{figure}[t!]
\resizebox{\textwidth}{!}
{%
\begin{tikzpicture}[scale=0.8, >=stealth,main node/.style={shape=rectangle,draw,rounded corners,}]
    \node[main node][very thick, fill=red!0, font = {\large}] (t2) {\begin{tabular}{c} \cref{theorem_parameters}: Part II: \\ Proof of \cref{eq:theta_guarantee} for unit-level \\parameters $\normalbraces{\TrueExternalFieldI}_{i = 1}^{n}$ \end{tabular}};
    \node[main node][very thick, font = {\large}] (l1) [right=3.1 cm of t2]{\begin{tabular}{c} \cref{lemma_conc_first_sec_der_stage_two}: \\ Concentration of \\gradient and Hessian of\\ loss functions $\normalbraces{\loss^{(i)}}_{i = 1}^{n}$ \end{tabular}};
    \node[main node][very thick, fill=red!0, font = {\large}] (t1) [below=0.9 cm of l1]{\begin{tabular}{c} \cref{theorem_parameters}: Part I: \\ Proof of \cref{eq:matrix_guarantee} for population-\\level parameter $\TrueParameterMatrix$ \end{tabular}};
    \node[main node][very thick, fill=orange!0, font = {\large}] (p3) [right=2.6 cm of l1]{\begin{tabular}{c} \cref{thm_main_concentration}:\\ Tail bounds\\ under LSI \end{tabular}};
    \node[main node][very thick, fill=orange!0, font = {\large}] (p2) [above=1.5cm of p3]{\begin{tabular}{c} \cref{thm_LSI_main}:\\ LSI for weakly\\ dependent RVs \end{tabular}};
    \node[main node][very thick, fill=orange!0, font = {\large}] (p4) [below=1.75 cm of p3]{\begin{tabular}{c} \cref{lemma_conditioning_trick}:\\ Identifying weakly\\dependent RVs in \\exponential family \end{tabular}};
    \node[main node][very thick] (l2) [above=0.7 cm of l1, font = {\large}]{\begin{tabular}{c} \cref{lemma_tenorization_kld} + \cref{lemma_dobrushin_implies_tensorization}:\\ Approximate \\
    tensorization of \\ entropy for weakly \\ dependent RVs \end{tabular}};
    \node at (5.3,0.3) (lipschitzness) {\large (i) \cref{lemma_parameter_single_external_field}};
    \node at (5.3,-0.3) {\large (ii) \cref{lemma_lipschitzness}};
    \node[main node][very thick, font = {\large}] (l3) [above left=2.5cm of l1]{\begin{tabular}{c} \cref{lemma_reverse_pinsker}:\\ Reverse-Pinkser \\inequality \end{tabular}};
    \draw[->,  line width=.6mm] (t1) to[out=90,in=270] (l1);
    \node at (11.9,-2.0) {\large \cref{lemma_expected_psi_upper_bound}};
    \draw[->,  line width=.6mm] (p2) to[out=270,in=90] (p3);
    \draw[->,  line width=.6mm] (p3) to[out=180,in=0] (l1);
    \draw[->,  line width=.6mm] (p4) to[out=160,in=340] (l1);
    \draw[->,  line width=.6mm] (l2) to[out=5,in=180] (p2);
    \draw[->,  line width=.6mm] (l3) to[out=0,in=170] (l2);
    \draw[->,  line width=.6mm] (l1) to[out=180,in=0] (t2);
    \node at (15.2,0.3) {\large \cref{coro}};
    \draw[->,  line width=.6mm] (0.2, -3.5) to (6.75, -3.5);
    \node at (3.4,-3.2) {\large Extend \cite{ShahSW2021A}};
    \node at (3.5,-3.8) {\large to non-identical samples};
    \draw[->,  line width=.6mm] (10.4, -5.7) to (16.6, -5.7);
    \node at (13.5,-5.4) {\large Extend \cite{DaganDDA2021}};
    \node at (13.5,-6.0) {\large to continuous RVs};
    \draw[->,  line width=.6mm] (1.6, 2.7) to (6.8, 2.7);
    \node at (4.2,3.0) {\large Extend \cite{Marton2015}};
    \node at (4.2,2.4) {\large to continuous RVs};
\end{tikzpicture}
}
  \caption{\tbf{Sketch diagram of the results and the proof techniques for \cref{theorem_parameters}.}
  First, we establish~\cref{eq:matrix_guarantee} for estimating $\TrueParameterMatrix$ by extending \citet[Proposition I.1, Proposition I.2]{ShahSW2021A} for i.i.d. data to non-identical samples. Next, we use \cref{eq:matrix_guarantee} to establish \cref{eq:theta_guarantee} for the unit-level parameters $\normalbraces{\TrueExternalFieldI}_{i = 1}^{n}$ via suitable concentration results for derivatives of the auxiliary loss functions in k\cref{eq:loss_function_n_original}. 
  En route, we establish three results of independent interest: (i) \cref{thm_LSI_main} that shows that weakly dependent and bounded random variables satisfy logarithmic Sobolev inequality (LSI) by both extending \citet[Theorem. 1, Theorem. 2]{Marton2015} and establishing a reverse-Pinkser  inequality to continuous random vectors; (ii)  \cref{thm_main_concentration} that extends the tail bounds \citet[Theorem. 6]{DaganDDA2021} to continuous distributions satisfying LSI; and (iii)
\cref{lemma_conditioning_trick} that extends the conditioning trick \citet[Lemma. 2]{DaganDDA2021} for identifying a weakly dependent subset to continuous random vectors.}
    \label{fig:proof_sketch}
\end{figure}

\subsection{Estimating the unit-level parameters}
\label{sec_proof_sketch_unit}
In the second part, we decompose the convex optimization problem in \cref{eq_decomposed_second_minimum} into $n$ convex optimization problems:
\begin{align}
    \loss^{(i)}\bigparenth{\ExternalFieldI} \defn \sump[t] \exp\Bigparenth{-\bigbrackets{\ExternalFieldtI + 2\EstimatedParameterRowttt\tp \svbx_{-t}^{(i)}} x_t^{(i)} - \EstimatedParameterTU[tt] \bigparenth{[x_t^{(i)}]^2 - \frac{\xmax^2}{3}}}
    \stext{for $i \in [n]$.}
    \label{eq:loss_function_n_original}
\end{align}
Noting that the set $\ParameterSet_{\ExternalField}^n$ places independent constraints on the $n$ unit-level parameters, namely $\ExternalFieldI[i] \in \ParameterSet_{\ExternalField}$, independently for all $ i \in [n]$ and combining \cref{eq:loss_function,eq_decomposed_second_minimum}, we find that
\begin{align}
    \min_{\normalbraces{\ExternalFieldI[1], \cdots, \ExternalFieldI[n]} \in \ParameterSet_{\ExternalField}^n} \!\!\!\! \loss\bigparenth{\EstimatedParameterMatrix, \ExternalFieldI[1], \cdots, \ExternalFieldI[n]} \sequal{\cref{eq:loss_function_n_original}} \frac{1}{n}\!\sumn \min_{\ExternalFieldI[i] \in \ParameterSet_{\ExternalField}} \loss^{(i)}\bigparenth{\ExternalFieldI[i]}
    \!\implies
    \EstimatedExternalFieldI \! \in \argmin_{\ExternalFieldI[i] \in \ParameterSet_{\ExternalField}} \loss^{(i)}\bigparenth{\ExternalFieldI[i]}, \label{eq_opt_problem_unit}
\end{align}
for each $i \in [n]$.
Next, we establish that with probability at least $1-\delta$,
\begin{align}
    \loss^{(i)}\bigparenth{\ExternalFieldI} & \geq  \loss^{(i)}\bigparenth{\TrueExternalFieldI} + R^2(\varepsilon, \delta) \stext{when} n \geq \frac{ce^{c'\bGM} p^4}{\varepsilon^4}  \Bigparenth{p \log \frac{p^2}{\delta \varepsilon^2}  + \tmetric_{\ExternalField,n}\normalparenth{ \varepsilon, \delta}}, \label{eq_secondstage_key_step}
\end{align}
uniformly for all points $\ExternalFieldI \in \ParameterSet_{\ExternalField}$ with $\stwonorm{\ExternalFieldI - \TrueExternalFieldI} \geq R(\varepsilon, \delta)$ (see \cref{eq_radius_node_thm}). We conclude the proof by contraposition with the basic inequality $\loss^{(i)}(\EstimatedExternalFieldI) \leq \loss^{(i)}(\TrueExternalFieldI)$ and a standard union bound over all $i \in [n]$.

The proof of \cref{eq_secondstage_key_step} mimics the same road map as that for \cref{eq_proof_sketch_thm_edge}. \cref{lemma_parameter_single_external_field} shows that for any fixed $\ExternalFieldI \in \ParameterSet_{\ExternalField}$, if $\ExternalFieldI$ is far from $\TrueExternalFieldI$, then with high probability $\loss^{(i)}\bigparenth{\ExternalFieldI}$ is significantly larger than $\loss^{(i)}\bigparenth{\TrueExternalFieldI}$. We prove \cref{lemma_parameter_single_external_field} via  concentration of derivatives of $\loss^{(i)}$~\cref{eq:loss_function_n_original} in \cref{lemma_conc_first_sec_der_stage_two}, this objective's Lipschitznes 
 in \cref{lemma_lipschitzness}, and a covering number argument (see \cref{sec_proof_thm_node_parameters_recovery}).

The proof of \cref{lemma_conc_first_sec_der_stage_two} involves several novel arguments: First, for a $\dGM$-Sparse Graphical Model (\cref{def:tau_sgm}), i.e., a generalization of the random vector $\rvbw$ in \cref{eq_joint_distribution_zvay}, \cref{lemma_conditioning_trick} identifies a   subset that satisfies Dobrushin's uniqueness condition (\cref{def_dobrushin_condition}) after conditioning on the complementary subset. Second, \cref{thm_LSI_main} shows that a bounded and weakly dependent continuous random vector (defined using Dobrushin's uniqueness condition) satisfies the logarithmic Sobolev inequality (LSI). Third, \cref{thm_main_concentration} establishes tail bounds for arbitrary functions of a continuous random vector that satisfies LSI. Putting together these results and a robustness result (\cref{lemma_expected_psi_upper_bound}) while invoking concentration results to account for the estimation error for $\TrueParameterMatrix$, yields \cref{lemma_conc_first_sec_der_stage_two}.
\section{Discussion}
\label{sec_discussion}
We introduce an exponential family approach to learn unit-level counterfactual distributions from a single sample per unit even when there is unobserved confounding. By conditioning on the latent confounders and using a novel convex loss function, we estimate the parameters of unit-level counterfactual distributions given the information about what actually happened. {The resulting estimates of unit-level counterfactual distributions enable us to estimate any functional of each unit's potential outcomes under alternate interventions. We analyze each unit's expected potential outcomes under alternate interventions, thereby providing a guarantee on unit-level counterfactual effects, i.e., individual treatment effects.}
We note that our approach makes only macro-level assumptions about the underlying causal graph and does not assume the knowledge of the micro-level causal graph.

A side product of our results is a strategy for answering interventional questions, e.g., to estimate average treatment effects. These questions are equivalent to estimating distributions of the form $f_{\rvby | \mathrm{do}(\rvba)}(\svby | \mathrm{do}(\rvba = \svba))$ where the do-operator \citep{Pearl2009} forces $\rvba$ to be $\svba$. Under the causal framework considered (\cref{fig_graphical_models}(b)), we have $f_{\rvby | \mathrm{do}(\rvba)}(\svby | \mathrm{do}(\rvba = \svba))=\Expectation_{\rvbv, \rvbz}\normalbrackets{f_{\rvby | \rvba, \rvbz, \rvbv}(\svby | \svba, \svbz, \svbv)}$. Consequently, the mixture distribution $n^{-1} \sum_{i \in [n]} \what{f}^{(i)}_{\rvby | \rvba}(\svby | \svba)$ with $\what{f}^{(i)}_{\rvby | \rvba}(\svby | \svba)$ defined in \cref{eq_counterfactual_distribution_y}, serves as a natural estimate via our strategy. 
Investigating the efficacy of this estimator is an interesting future direction. 

{In this work, the conditional exponential family distribution of $\rvby$ in \cref{subsec_exp_fam} or in \cref{subsec_high_terms} was such that the effect of unobserved covariates $\rvbz$---after conditioning on them---was captured by a first-order interaction term varying with the realized value of $\rvbz$ for each unit, e.g., $\sbraces{\ExternalField(\svbz^{(i)})}_{i=1}^n$for the conditional distribution in \cref{subsec_exp_fam}. Focusing on \cref{subsec_exp_fam}, when one considers higher-order interaction terms in the joint distribution, the conditional distributions would also have higher-order interaction terms (the highest order in the conditional distribution is one less than the highest order in the joint distribution) that vary with $\rvbz$. Focusing on \cref{subsec_high_terms}, the exponent of the exponential tilting of the base distribution of the outcomes by the unobserved covariates could have higher-order terms.} For such cases, while our analysis for population-level parameters (\cref{theorem_parameters} Part I's proof in \cref{sec:proof_of_theorem_parameters}) is likely to extend easily, new arguments for analyzing quadratic (or higher-order) interaction terms that vary for each unit seem necessary. Developing these results, e.g., suitable analogs of Dobrushin's condition for higher-order exponential family, present an exciting future venue for research.

Our methodology can be useful for a class of multi-task learning problems \citep{caruana1997multitask}, e.g., when we have multiple logistic regression tasks with some commonalities. For a logistic regression task, the exponential family model~\cref{eq_conditional_distribution_vay} has been used by \cite{DaganDDA2021} to allow dependencies between the labels via the parameter $\ParameterMatrix$ (instead of assuming independence between the labels), e.g., for spatio-temporal data. They consider a single regression task and assume that the dependency matrix $\ParameterMatrix$ is known up to a constant and learn a task-specific parameter $\ExternalField(\svbz)$ (where $\svbz$ denotes a task). Our model and methodology apply to the case of fully unknown $\ParameterMatrix$ given multiple datasets that share the same dependency parameter $\ParameterMatrix$ but have varying task-specific parameters $\ExternalField(\svbz)$; and provide a tractable way to estimate all these parameters together. 
{In fact, our framework and results also apply beyond the quadratic dependencies captured by $\ParameterMatrix$ as described in \cref{subsec_high_terms}.} Analyzing whether our methodology can be extended beyond logistic regression models for multi-task learning is a question worthy of further investigation.
\subsubsection*{Acknowledgments}
The authors thank Thomas Courtade, Yuzhou Gu, Anuran Makur, Wenlong Mou, Felix Otto, and Yury Polyanskiy for helpful pointers regarding Logarithmic Sobolev inequalities.  The authors thank Yuval Dagan and Anthimos Vardis Kandiros for helpful discussion about \cite{DaganDDA2021} and \cite{KandirosDDGD2021}. The authors also thank Alberto Abadie, Avi Feller, and Martin Wainwright for helpful comments. Lastly, the authors thank the anonymous reviewers of Workshop on Causality for Real-world Impact (NeurIPS 2022) for their comments and suggestions.

\subsubsection*{Funding}
This work was supported, in part, by NSF under Grant No. DMS-2023528 as part of the Foundations of Data Science Institute (FODSI), the MIT-IBM Watson AI Lab under Agreement No. W1771646, MIT-IBM projects on Time Series and Causal Inference as well as project with DSO National Laboratory.
\appendix 
\renewcommand\contentsname{
    \begin{center}
        Appendix
    \end{center}
}
\addtocontents{toc}{\protect\setcounter{tocdepth}{2}}
\tableofcontents

\counterwithin{mylemma}{section}
\counterwithin{mydefinition}{section}
\counterwithin{myproposition}{section}
\counterwithin{mycorollary}{section}

\section{Proper loss function and projected gradient descent}
\label{sec_proper_convex}
{In this section, we prove  \cref{prop_proper_loss_function} showing that the loss function in \cref{eq:loss_function} is a proper loss function. We also 
provide an algorithm to obtain an $\epsilon$-optimal estimate of $\ExtendedEstimatedParameterMatrix$.}

\subsection{Proof of {Proposition \ref{prop_proper_loss_function}}}
\label{sec_proof_proper_loss_function}
Fix any $\svbz \in \cZ^{p_z}$. For every $t \in [p]$, define the following parametric distribution 
\begin{align}
    u_{\rvbx|\rvbz}\bigparenth{\svbx| \svbz; \ExternalFieldt(\svbz), \ParameterRowt} \propto \dfrac{\TrueJointDist}{\ConditionalDistT}, \label{eq_u_distribution}
\end{align}
where $\TrueJointDist$ is as defined in \cref{eq_conditional_distribution_vay} and $\ConditionalDistT$ is as defined in \cref{eq_conditional_dist}. Letting $\cx_t \defn x_t^2 - \xmax^2/3$ and using \cref{eq_conditional_dist}, we can write $u_{\rvbx|\rvbz}\bigparenth{\svbx| \svbz; \ExternalFieldt(\svbz), \ParameterRowt}$ in \cref{eq_u_distribution} as
\begin{align}
    u_{\rvbx|\rvbz}\bigparenth{\svbx| \svbz; \ExternalFieldt(\svbz), \ParameterRowt} \propto \TrueJointDist \exp\bigparenth{ -\normalbrackets{\ExternalFieldt(\svbz) + 2\ParameterRowttt\tp \svbx_{-t}}x_t - \ParameterTU[tt]\cx_t }. \label{eq_u_distribution_2}
\end{align}
Then, we have
\begin{align}
    u_{\rvbx|\rvbz}\bigparenth{\svbx| \svbz; \ExternalFieldt(\svbz), \ParameterRowt} & = \dfrac{\TrueJointDist \exp\bigparenth{\!-\!\normalbrackets{\ExternalFieldt(\svbz) \!+\! 2\ParameterRowttt\tp \svbx_{-t}}x_t \!-\! \ParameterTU[tt]\cx_t }}{\!\!\!\!\int_{\svbx \in \cX^p} \! \TrueJointDist \exp\bigparenth{ \!-\!\normalbrackets{\ExternalFieldt(\svbz) \!+\! 2\ParameterRowttt\tp \svbx_{-t}}x_t \!-\! \ParameterTU[tt]\cx_t } d\svbx}\\
    & = \dfrac{\TrueJointDist \exp\bigparenth{ -\normalbrackets{\ExternalFieldt(\svbz) + 2\ParameterRowttt\tp \svbx_{-t}}x_t - \ParameterTU[tt]\cx_t }}{\Expectation_{\rvbx|\rvbz}\Bigbrackets{\exp\bigparenth{ -\normalbrackets{\ExternalFieldt(\svbz) + 2\ParameterRowttt\tp \svbx_{-t}}x_t - \ParameterTU[tt]\cx_t }}}. \label{eq_u_alternate}
\end{align}
Further, for $\ExternalFieldt(\svbz) = \TrueExternalFieldt(\svbz),$ and $\ParameterRowt = \TrueParameterRowt$, we can write an expression for  $u_{\rvbx|\rvbz}\bigparenth{\svbx| \svbz; \TrueExternalFieldt(\svbz), \TrueParameterRowt}$ which does not depend on $\rvx_t$ functionally. From \cref{eq_conditional_dist}, we have
\begin{align}
    u_{\rvbx|\rvbz}\bigparenth{\svbx| \svbz; \TrueExternalFieldt(\svbz), \TrueParameterRowt} \propto f_{\rvbx_{-t}|\rvbz}\bigparenth{\svbx_{-t}| \svbz; \TrueExternalField(\svbz), \TrueParameterMatrix}. \label{eq_u_star_no_dependence}
\end{align}
Now, consider the difference between $\KLD{u_{\rvbx|\rvbz}\bigparenth{\svbx| \svbz; \TrueExternalFieldt(\svbz), \TrueParameterRowt}}{u_{\rvbx|\rvbz}\bigparenth{\svbx| \svbz; \ExternalFieldt(\svbz), \ParameterRowt}}$ and \\
$\KLD{u_{\rvbx|\rvbz}\bigparenth{\svbx| \svbz; \TrueExternalFieldt(\svbz), \TrueParameterRowt}}{\TrueJointDist}$. We have
\begin{align}
    & \KLD{u_{\rvbx|\rvbz}\bigparenth{\cdot| \svbz; \TrueExternalFieldt(\svbz), \TrueParameterRowt}}{u_{\rvbx|\rvbz}\bigparenth{\cdot| \svbz; \ExternalFieldt(\svbz), \ParameterRowt}} \\
    & \qquad \qquad - \KLD{u_{\rvbx|\rvbz}\bigparenth{\cdot| \svbz; \TrueExternalFieldt(\svbz), \TrueParameterRowt}}{\TrueJointDistfun}\\
    & \sequal{(a)} \int_{\svbx \in \cX^p} \!\!\!\! u_{\rvbx|\rvbz}\bigparenth{\svbx| \svbz; \TrueExternalFieldt(\svbz), \TrueParameterRowt} \log \dfrac{\TrueJointDist}{u_{\rvbx|\rvbz}\bigparenth{\svbx| \svbz; \ExternalFieldt(\svbz), \ParameterRowt}} d\svbx \\
    & \sequal{\cref{eq_u_alternate}} \int_{\svbx \in \cX^p} \!\!\!\! u_{\rvbx|\rvbz}\bigparenth{\svbx| \svbz; \TrueExternalFieldt(\svbz), \TrueParameterRowt} \log \dfrac{\Expectation_{\rvbx|\rvbz}\Bigbrackets{\exp\bigparenth{ -\normalbrackets{\ExternalFieldt(\svbz) + 2\ParameterRowttt\tp \svbx_{-t}}x_t - \ParameterTU[tt]\cx_t }}}{\exp\bigparenth{ -\normalbrackets{\ExternalFieldt(\svbz) + 2\ParameterRowttt\tp \svbx_{-t}}x_t - \ParameterTU[tt]\cx_t }} d\svbx \\
    & = \log \Expectation_{\rvbx|\rvbz}\Bigbrackets{\exp\bigparenth{ -\normalbrackets{\ExternalFieldt(\svbz)+2\ParameterRowttt\tp \svbx_{-t}}x_t -\ParameterTU[tt]\cx_t }} \\
    & \qquad \qquad - \int_{\svbx \in \cX^p}  \!\!\!\! u_{\rvbx|\rvbz}\bigparenth{\svbx| \svbz; \TrueExternalFieldt(\svbz), \TrueParameterRowt} \bigparenth{ \normalbrackets{\ExternalFieldt(\svbz) + 2\ParameterRowttt\tp \svbx_{-t}}x_t + \ParameterTU[tt]\cx_t }   d\svbx \\
    & \sequal{(b)} \log \Expectation_{\rvbx|\rvbz}\Bigbrackets{\exp\bigparenth{ -\normalbrackets{\ExternalFieldt(\svbz) + 2\ParameterRowttt\tp \svbx_{-t}}x_t - \ParameterTU[tt]\cx_t }}, \label{eq_kl_difference}
\end{align}
where $(a)$ follows from the definition of KL-divergence and $(b)$ follows because integral is zero since $u_{\rvbx|\rvbz}\bigparenth{\svbx| \svbz; \TrueExternalFieldt(\svbz), \TrueParameterRowt}$ does not functionally depend on $x_t$ as in \cref{eq_u_star_no_dependence}, and $\int_{x_t \in \cX} x_t dx_t= 0$ and $\int_{x_t \in \cX} \cx_t dx_t= 0$. Now, we can write
\begin{align}
    \Expectation_{\rvbx|\rvbz}\bigbrackets{\loss\bigparenth{\ExtendedParameterMatrix}} & = \frac{1}{n} \sump[t] \sumn[i]  \Expectation_{\rvbx|\rvbz}\Bigbrackets{\exp\bigparenth{ -\normalbrackets{\ExternalFieldt(\svbz^{(i)}) + \ParameterRowttt\tp \svbx_{-t}^{(i)}}x_t^{(i)} - \ParameterTU[tt]\cx_t^{(i)} }} \\
    & \sequal{\cref{eq_kl_difference}} \frac{1}{n} \sump[t] \sumn[i] \exp\Bigparenth{\KLD{u_{\rvbx|\rvbz}\bigparenth{\cdot| \svbz^{(i)}; \TrueExternalFieldt(\svbz^{(i)}), \TrueParameterRowt}}{u_{\rvbx|\rvbz}\bigparenth{\cdot| \svbz^{(i)}; \ExternalFieldt(\svbz^{(i)}), \ParameterRowt}} \\
    &  \qquad \qquad\qquad  - \KLD{u_{\rvbx|\rvbz}\bigparenth{\cdot| \svbz^{(i)}; \TrueExternalFieldt(\svbz^{(i)}), \TrueParameterRowt}}{\TrueJointDistfunT}}.\label{eq_population_loss_kl_relationship}
\end{align}
We note that the parameters 
only show up in the first KL-divergence term in the right-hand-side of \cref{eq_population_loss_kl_relationship}. Therefore, it is easy to see that $\Expectation_{\rvbx|\rvbz}\bigbrackets{\loss\bigparenth{\ExtendedParameterMatrix}}$ is minimized uniquely when $\ExternalFieldt(\svbz^{(i)}) = \TrueExternalFieldt(\svbz^{(i)})$ and $\ParameterRowt = \TrueParameterRowt$ for all $t \in [p]$ and all $i \in [n]$, i.e., when $\ExtendedParameterMatrix = \ExtendedTrueParameterMatrix$.

\subsection{Algorithm}
\label{subsec_alg}
{In this section, we provide a projected gradient descent algorithm to return an $\epsilon$-optimal estimate of the convex optimization in \cref{eq_estimated_parameters}. We note that alternative algorithms (including Frank-Wolfe) can also be used.}

\begin{algorithm}[h]
    \SetCustomAlgoRuledWidth{0.4\textwidth} 
    \KwInput{number of iterations $\tau$, step size $\eta$, $\epsilon$, parameter sets $\ParameterSet_{\ExternalField}$ and $ \ParameterSet_{\ParameterMatrix}$}
    \KwOutput{$\epsilon$-optimal estimate $\ExtendedEstimatedParameterMatrix_{\epsilon}$}
    \KwInitialization{$\ExtendedParameterMatrix^{(0)} = \boldsymbol{0}$} 
    {
    \For{$j = 0,\cdots,\tau$}
    {
        $\ExtendedParameterMatrix^{(j+1)} \leftarrow \argmin_{\ExtendedParameterMatrix \in \ParameterSet_{\ExternalField}^n \times \ParameterSet_{\ParameterMatrix}} \stwonorm{\ExtendedParameterMatrix^{(j)} - \eta  \nabla \loss\bigparenth{\ExtendedParameterMatrix^{(j)}} - \ExtendedParameterMatrix}$
    }
    $\ExtendedEstimatedParameterMatrix_{\epsilon} \leftarrow \ExtendedParameterMatrix^{(\tau+1)}$
    }
    \caption{Projected Gradient Descent}
    \label{alg:GradientDescent}
\end{algorithm}
{We note that, in general, projecting onto the space $\ParameterSet_{\ExternalField}^n \times \ParameterSet_{\ParameterMatrix}$ may not be easy depending on the specific form of $\ParameterSet_{\ExternalField}$. For \cref{exam:lc_dense,exam:sc}, projecting on $\ParameterSet_{\ExternalField}$ is equivalent to projecting onto the $k$-dimensional vector $\mbf a$. For \cref{exam:sc}, the $\ell_0$-sparsity is relaxed to $\ell_1$ sparsity. We also do not focus on any issues that may arise due to the choice of the step size $\eta$.}
\section{Proof of {Theorem \ref{theorem_parameters}} Part I: \edgeparammainresultname}
\label{sec:proof_of_theorem_parameters}

To prove this part, it is sufficient to show that all points $\ExtendedParameterMatrix \in \ParameterSet_{\ParameterMatrix} \times \ParameterSet_{\ExternalField}^n$, such that $\stwonorm{\ParameterRowt - \TrueParameterRowt} \geq \varepsilon$ for at least one $t \in [p]$, uniformly satisfy 
\begin{align}
\loss(\ExtendedParameterMatrix) \geq \loss(\ExtendedTrueParameterMatrix) + \Omega(\varepsilon^2) \stext{for} n \geq \frac{ce^{c'\bGM} p^2}{\varepsilon^4} \cdot \biggparenth{p \log \frac{p}{\delta} + \metric_{\ExternalField,n}\big(\varepsilon^2 \big)},
\label{eq_sufficienct_condition_thm_edge}
\end{align}
with probability at least $1-\delta$. Then, the guarantee in \cref{theorem_parameters} follows from \cref{eq_estimated_parameters} by contraposition.

To that end, we decompose $\loss(\ExtendedParameterMatrix)$ in \cref{eq:loss_function} as a sum of $p$ convex (and positive) auxiliary objectives $\loss_t\bigparenth{\ExtendedParameterRowT}$, i.e., $\loss(\ExtendedParameterMatrix) = \sump \loss_t\bigparenth{\ExtendedParameterRowT}$ where
\begin{align}\loss_t\bigparenth{\ExtendedParameterRowT} \defn \frac{1}{n} \sumn[i] \exp\Bigparenth{\!-\!\normalbrackets{\ExternalFieldtI \!+\! 2\ParameterRowttt\tp \svbx_{-t}^{(i)}} x_t^{(i)}\!-\!\ParameterTU[tt]\cx_t^{(i)}},
\label{eq:loss_function_p}
\end{align}
with $\cx_t^{(i)} = \bigbrackets{x_t^{(i)}}^2 - \xmax^2/3$ and $\ExtendedParameterRowT = \bigbraces{\ExternalFieldtI[1], \cdots, \ExternalFieldtI[n], \ParameterRowt}$ as defined in \cref{eq:loss_function}. The lemma below, proven in \cref{proof_of_lemma_parameter}, shows that for any fixed and feasible $\ExtendedParameterRowT$, if $\ParameterRowt$ is far from $\TrueParameterRowt$, then with high probability $\loss_t\bigparenth{\ExtendedParameterRowT}$ is significantly larger than $\loss_t\bigparenth{\ExtendedTrueParameterRowT}$. The lemma uses 
the following constants that depend on model parameters $\tau \defeq (\aGM, \bGM, \xmax, \ParameterMatrix)$:
\begin{align}
\label{eq:constants}
\cone \!\defeq\! \aGM \!+\! 4 \bGM \xmax
\qtext{and}
\ctwo \defeq \exp{(\xmax(\aGM+ 2\bGM \xmax))}.
\end{align}
\newcommand{\parameterseparation}{Gap between the loss function for a fixed parameter}
\begin{lemma}[{\parameterseparation}]\label{lemma_parameter}
	Consider any $\ExtendedParameterMatrix \in \ParameterSet_{\ExternalField}^n \times \ParameterSet_{\ParameterMatrix}$. Fix any $\delta \in (0,1)$. Then, we have uniformly for all $t \in [p]$
	\begin{align}
	\loss_t\bigparenth{\ExtendedParameterRowT} \geq \loss_t\bigparenth{\ExtendedTrueParameterRowT} + \frac{\lambda_{\min} \twonorm{\ParameterRowt - \TrueParameterRowt}^2}{2\ctwo} - \varepsilon
	\qtext{for} n \geq
	\dfrac{ce^{c'\bGM} \log(p/\delta)}{\varepsilon^2},
	\end{align}
	with probability at least  $1-\delta$, where $\ctwo$ was defined in \cref{eq:constants}. 
\end{lemma}
\newcommand{\lipschitznesslossfunction}{Lipschitzness of the loss function}
{
	\noindent Next, we show that the loss function $\loss$ is Lipschitz (see \cref{sub:proof_lemma_lipschitzness_first_stage} for the proof).
	\begin{lemma}[{\lipschitznesslossfunction}]\label{lemma_lipschitzness_first_stage}
		Consider any $\ExtendedParameterMatrix, \tExtendedParameterMatrix \in \ParameterSet_{\ExtendedParameterMatrix}$. Then, the loss function $\loss$ is $2\xmax^2 \ctwo$-Lipschitz in a suitably-adjusted $\ell_1$ norm:
		\begin{align}  \bigabs{\loss\bigparenth{\tExtendedParameterMatrix} - \loss\bigparenth{\ExtendedParameterMatrix}} \leq 2\xmax^2 \ctwo \Bigparenth{\sump \sonenorm{\tParameterRowt - \ParameterRowt}  + \frac{1}{n} \sumn \sonenorm{\tExternalFieldI -\ExternalFieldI}}, \label{eq_lipschitz_property_first_stage}
		\end{align}
		where the constant $\ctwo$ was defined in \cref{eq:constants}.
\end{lemma}}
\noindent Given these lemmas, we now proceed with the proof.

\paragraph{Proof strategy} We want to show that all points $\ExtendedParameterMatrix \in \ParameterSet_{\ParameterMatrix} \times \ParameterSet_{\ExternalField}^n$, such that $\stwonorm{\ParameterRowt - \TrueParameterRowt} \geq \varepsilon$ for at least one $t \in [p]$, uniformly satisfy \cref{eq_sufficienct_condition_thm_edge} with probability at least $1-\delta$. To do so, 
we consider the set of feasible $\ExtendedParameterMatrix$ such that the distance of $\ParameterRowt$ 
from $\TrueParameterRowt$ is at least $\varepsilon > 0$ in $\ell_2$ norm for some $t \in [p]$, and denote the set by $\ParameterSet_{\ParameterMatrix}^{\varepsilon} \times \ParameterSet_{\ExternalField}^n$ (see \cref{eq_set_cover_1} and \cref{eq_parameter_set_external_field}). Then, using an appropriate covering set of $\ParameterSet_{\ParameterMatrix}^{\varepsilon} \times \ParameterSet_{\ExternalField}^n$ and the Lipschitzness of $\loss$, we show that the value of $\loss$ at all points in $\ParameterSet_{\ParameterMatrix}^{\varepsilon} \times \ParameterSet_{\ExternalField}^n$ is uniformly $\Omega(\varepsilon^2)$ larger than the value of $\loss$ at $\ExtendedTrueParameterMatrix$ with high probability. 

\paragraph{Arguments for points in the covering set} 
Define the set
\begin{align}
\ParameterSet_{\ParameterMatrix}^{\varepsilon} \defn  \braces{ \ParameterMatrix \in \Reals^{p \times p}: \ParameterMatrix = \ParameterMatrix\tp, \maxmatnorm{\ParameterMatrix} \leq \aGM, \infmatnorm{\ParameterMatrix} \leq \bGM, \max_{t\in[p]}\stwonorm{\TrueParameterRowt - \ParameterRowt} \geq \varepsilon}.\label{eq_set_cover_1}
\end{align}
Let $\cU(\ParameterSet_{\ParameterMatrix}^{\varepsilon}, \varepsilon')$ be the $\varepsilon'$-cover of smallest size for the set $\ParameterSet_{\ParameterMatrix}^{\varepsilon}$ with respect to $\sonenorm{\cdot}$ (see \cref{def_covering_number_metric_entropy}) and let $\cC(\ParameterSet_{\ParameterMatrix}^{\varepsilon}, \varepsilon') = \normalabs{\cU(\ParameterSet_{\ParameterMatrix}^{\varepsilon}, \varepsilon')}$ be the $\varepsilon'$-covering number. Similarly, let $\cU(\ParameterSet_{\ParameterMatrix}^{\varepsilon}, \varepsilon'')$ be the $\varepsilon''$-cover of the smallest size for the set $\ParameterSet_{\ExternalField}^n$ with respect to $\sonenorm{\cdot}$ and let $\cC(\ParameterSet_{\ExternalField}^n, \varepsilon'') = \normalabs{\cU(\ParameterSet_{\ParameterMatrix}^{\varepsilon}, \varepsilon'')}$ be the $\varepsilon''$-covering number.
We choose
\begin{align}
\varepsilon' \defn  \dfrac{\lambda_{\min} \varepsilon^2}{32 \xmax^2 \ctwo[2]} \qtext{and} \varepsilon'' \defn  \dfrac{\lambda_{\min} \varepsilon^2 n}{32 \xmax^2 \ctwo[2]}. \label{eq_varepsilon_stage_1}
\end{align}
Now, we argue by a union bound that the value of $\loss$ at all points in $\cU(\ParameterSet_{\ParameterMatrix}^{\varepsilon}, \varepsilon') \times \cU(\ParameterSet_{\ExternalField}^n, \varepsilon'')$ is uniformly $\Omega(\varepsilon^2)$ larger than $\loss(\ExtendedTrueParameterMatrix)$ with high probability. 
For any $\ExtendedParameterMatrix \in \cU(\ParameterSet_{\ParameterMatrix}^{\varepsilon}, \varepsilon') \times \cU(\ParameterSet_{\ExternalField}^n, \varepsilon'')$, we have
\begin{align}
\sump \stwonorm{\TrueParameterRowt - \ParameterRowt}^2 \sgreat{(a)} \varepsilon^2,\label{eq_lower_bound_two_norm_theta_diff_stage_1}
\end{align}
where $(a)$ follows because $\cU(\ParameterSet_{\ParameterMatrix}^{\varepsilon}, \varepsilon') \subseteq \ParameterSet_{\ParameterMatrix}^{\varepsilon}$. Now, applying \cref{lemma_parameter} with $\varepsilon \mapsfrom \lambda_{\min} \varepsilon^2/4\ctwo p$ and $\delta \mapsfrom \delta/(\cC(\ParameterSet_{\ParameterMatrix}^{\varepsilon}, \varepsilon') + \cC(\ParameterSet_{\ExternalField}^n, \varepsilon''))$ and summing over $t\in[p]$, we find that
\begin{align}
\sum_{t\in[p]}\loss_t\bigparenth{\ExtendedParameterRowT} &\geq \sum_{t\in[p] }\parenth{\loss_t\bigparenth{\ExtendedTrueParameterRowT} + \frac{\lambda_{\min} \twonorm{\ParameterRowt - \TrueParameterRowt}^2}{2\ctwo} - \frac{\lambda_{\min} \varepsilon^2}{4\ctwo p}} \\
\implies \qquad\qquad \loss\bigparenth{\ExtendedParameterMatrix} &\geq \loss\bigparenth{\ExtendedTrueParameterMatrix} + \frac{\lambda_{\min}}{2\ctwo} \sump \stwonorm{\TrueParameterRowt - \ParameterRowt}^2 - \frac{\lambda_{\min} \varepsilon^2}{4\ctwo} 
\\
&\sgreat{\cref{eq_lower_bound_two_norm_theta_diff_stage_1}} \loss\bigparenth{\ExtendedTrueParameterMatrix} + \frac{\lambda_{\min} \varepsilon^2}{4\ctwo},
\end{align}
with probability at least $1-\delta/(\cC(\ParameterSet_{\ParameterMatrix}^{\varepsilon}, \varepsilon') + \cC(\ParameterSet_{\ExternalField}^n, \varepsilon''))$ whenever
\begin{align}
n \geq \frac{ce^{c'\bGM} p^2 \log\bigparenth{(\cC(\ParameterSet_{\ParameterMatrix}^{\varepsilon}, \varepsilon') \times \cC(\ParameterSet_{\ExternalField}^n, \varepsilon'')) \cdot p/\delta}}{\lambda_{\min}^2 \varepsilon^4}. \label{eq_n_condition_edge_recovery_2}
\end{align}
By applying the union bound over $\cU(\ParameterSet_{\ParameterMatrix}^{\varepsilon}, \varepsilon') \times \cU(\ParameterSet_{\ExternalField}^n, \varepsilon'')$, as long as $n$ satisfies \cref{eq_n_condition_edge_recovery_2}, we have
\begin{align}
\loss\bigparenth{\ExtendedParameterMatrix} \geq \loss\bigparenth{\ExtendedTrueParameterMatrix}  + \frac{\lambda_{\min} \varepsilon^2}{4\ctwo} \stext{uniformly for every} \ExtendedParameterMatrix \in \cU(\ParameterSet_{\ParameterMatrix}^{\varepsilon}, \varepsilon') \times \cU(\ParameterSet_{\ExternalField}^n, \varepsilon''), \label{eq_union_bound_covering_set_stage_1} 
\end{align}
with probability at least $1-\delta$.

\paragraph{Arguments for points outside the covering set} Next, we establish the claim~\cref{eq_sufficienct_condition_thm_edge} for an arbitrary $\tExtendedParameterMatrix \in \ParameterSet_{\ParameterMatrix}^{\varepsilon} \times \ParameterSet_{\ExternalField}^n$ conditional on the event that \cref{eq_union_bound_covering_set_stage_1} holds.
Given a fixed $\tExtendedParameterMatrix \in \ParameterSet_{\ParameterMatrix}^{\varepsilon} \times \ParameterSet_{\ExternalField}^n$, let $\ExtendedParameterMatrix$ be (one of) the point(s) in the cover $\cU(\ParameterSet_{\ParameterMatrix}^{\varepsilon}, \varepsilon') \times \cU(\ParameterSet_{\ExternalField}^n, \varepsilon'')$ that satisfies $\sump \sonenorm{\tParameterRowt - \ParameterRowt} \leq \varepsilon'$ and $\sumn \sonenorm{\tExternalFieldI -\ExternalFieldI} \leq \varepsilon''$ (there exists such a point by~\cref{def_covering_number_metric_entropy}). Then, the choices~\cref{eq_varepsilon_stage_1} and \cref{lemma_lipschitzness_first_stage} put together imply that
\begin{align}
\loss\bigparenth{\tExtendedParameterMatrix} & \geq \loss\bigparenth{\ExtendedParameterMatrix} \!-\! 2\xmax^2 \ctwo \Bigparenth{\sump \sonenorm{\tParameterRowt - \ParameterRowt}  + \frac{1}{n} \sumn \sonenorm{\tExternalFieldI -\ExternalFieldI}} \label{eq_effect_average_1}\\
& \geq \loss\bigparenth{\ExtendedParameterMatrix} - 2\xmax^2 \ctwo \Bigparenth{ \varepsilon' + \frac{\varepsilon''}{n}} \sgreat{\cref{eq_varepsilon_stage_1}} \loss\bigparenth{\ExtendedParameterMatrix} \!-\!  \frac{\lambda_{\min} \varepsilon^2}{8\ctwo} \sgreat{\cref{eq_union_bound_covering_set_stage_1}} \loss\bigparenth{\ExtendedTrueParameterMatrix} \!+\! \frac{\lambda_{\min} \varepsilon^2}{8\ctwo}.
\end{align}

\paragraph{Bounding $n$} 
Using $\ParameterSet_{\ParameterMatrix}^{\varepsilon} \subseteq \ParameterSet_{\ParameterMatrix}$ and the outer product definition of $\ExternalField^n$, we find that
\begin{align}
\cC(\ParameterSet_{\ParameterMatrix}^{\varepsilon}, \varepsilon') \leq \cC(\ParameterSet_{\ParameterMatrix}, \varepsilon')
\qtext{and}
\cC(\ParameterSet_{\ExternalField}^n, \varepsilon'')
= (\cC(\ParameterSet_{\ExternalField}, \varepsilon''))^n.
\label{eq_bound_covering_number_stage_1}
\end{align}
Putting together \cref{eq_varepsilon_stage_1,eq_bound_covering_number_stage_1}, the lower bound \cref{eq_n_condition_edge_recovery_2} can be replaced by
\begin{align}
n \geq \frac{ce^{c'\bGM} p^2}{\lambda_{\min}^2 \varepsilon^4} \cdot \biggparenth{\log \frac{p}{\delta} + \log \cC\Big(\ParameterSet_{\ParameterMatrix}, \frac{\lambda_{\min} \varepsilon^2}{ce^{c'\bGM}}\Big) + n \log \cC\Big(\ParameterSet_{\ExternalField}, \frac{\lambda_{\min} n \varepsilon^2}{ce^{c'\bGM}}\Big)},
\end{align}
which yields the claim immediately after noting that
\begin{align}
\log \cC\Big(\ParameterSet_{\ParameterMatrix}, \frac{\lambda_{\min} \varepsilon^2}{ce^{c'\bGM}}\Big) \!=\! O\Big(\bGM^2 p \log\Big(\frac{1}{\lambda_{\min} \varepsilon^2}\Big)\Big) \stext{and} \log \cC\Big(\ParameterSet_{\ExternalField}, \frac{\lambda_{\min} n \varepsilon^2}{ce^{c'\bGM}}\Big) \!=\! \metric_{\ExternalField}\Big( \frac{\lambda_{\min} n \varepsilon^2}{ce^{c'\bGM}} \Big).
\end{align}

\subsection{Proof of \cref{lemma_parameter}: \parameterseparation}
\label{proof_of_lemma_parameter}
Fix any $\varepsilon >0$, any $\delta \in (0,1)$, and $t \in [p]$. Consider any direction $\uOmT \defn \bigbraces{\omtI[1], \cdots, \omtI[n], \Omt} \in \Reals^{n+p}$ along the parameter $\ExtendedParameterRowT$, i.e.,
\begin{align}
\label{eq:omega_defn}
\uOmT = \ExtendedParameterRowT - \ExtendedTrueParameterRowT,
\qtext{and}
\Omt = \ParameterRowt-\TrueParameterRowt.
\end{align}
We denote the first-order and the second-order directional derivatives of the loss function $\loss_t$ in \cref{eq:loss_function_p} along the direction $\uOmT$ evaluated at $\ExtendedParameterRowT$ by $\directionalGradient$ and $\directionalHessian$, respectively. 
Below, we state a lemma (with proof divided across \cref{sub:proof_of_lemma_conc_first_der} and \cref{sub:proof_of_lemma_conc_sec_der}) that provides us a control on $\directionalGradient$ and $\directionalHessian$. 
The assumptions of \cref{lemma_parameter} remain in force.

\newcommand{\concpopresultname}{Control on first and second directional derivatives}
\newcommand{\concgradresultname}{Concentration of first directional derivative}
\newcommand{\conchessresultname}{Anti-concentration of second directional derivative}

\begin{lemma}[{\concpopresultname}]\label{lemma_conc_first_sec_der}
	For any fixed $\varepsilon_1, \varepsilon_2 > 0$, $\delta_1, \delta_2 \in (0,1)$, $t \in [p]$, $\ExtendedParameterMatrix \in \ParameterSet_{\ExternalField}^n \times \ParameterSet_{\ParameterMatrix}$ defined in \cref{eq:loss_function} and $\Omt$ defined in \cref{eq:omega_defn}, we have the following:
	\begin{enumerate}[label=(\alph*)]
		\item\label{item_conc_first_der} \textnormal{{\concgradresultname}}: with probability at least $1-\delta_1$,
		\begin{align}
		\bigabs{\directionalGradientTrue} \leq \varepsilon_1
		\stext{for}
		n \geq \dfrac{8\cone[2] \ctwo[2] \xmax^2  \log \frac{2p}{\delta_1} }{\varepsilon_1^2}
		\stext{and uniformly for all $t \in [p]$.}
		\end{align}
		\item\label{item_conc_sec_der} \textnormal{{\conchessresultname}}: with probability at least $1-\delta_2$,
		\begin{align}
		\directionalHessian \!\geq\! \frac{\lambda_{\min} \twonorm{\Omt}^2}{\ctwo}- \varepsilon_2 \stext{for}
		n \!\geq\! \dfrac{32\cone[4] \xmax^4  \log\frac{2p}{\delta_2}}{\varepsilon_2^2 \ctwo[2]}
		\stext{and uniformly for all $t \in [p]$.}
		\end{align}
	\end{enumerate}
\end{lemma}

\noindent Given this lemma, we now proceed with the proof. 
Define a function $g : [0,1] \to \Reals^{n+p}$ 
\begin{align}
g(a) \defn \ExtendedTrueParameterRowT + a(\ExtendedParameterRowT - \ExtendedTrueParameterRowT).
\end{align}
Notice that $g(0) = \ExtendedTrueParameterRowT$ and $g(1) = \ExtendedParameterRowT$ 
as well as
\begin{align}
\dfrac{d\loss_t(g(a))}{da} = \tdirectionalGradient\bigr|_{\tExtendedParameterRowT = g(a)} \qtext{and} \dfrac{d^2\loss_t(g(a))}{da^2} = \tdirectionalHessian\bigr|_{\tExtendedParameterRowT = g(a)}. \label{eq_der_mapping}
\end{align}
By the fundamental theorem of calculus, we have
\begin{align}
\dfrac{d\loss_t(g(a))}{da} \geq \dfrac{d\loss_t(g(a))}{da}\bigr|_{a = 0} + a \min_{a \in (0,1)}\dfrac{d^2\loss_t(g(a))}{da^2}. \label{eq_fundamental}
\end{align}
Integrating both sides of \cref{eq_fundamental} with respect to $a$, we obtain
\begin{align}
\loss_t(g(a)) - \loss_t(g(0)) & \geq  a \dfrac{d\loss_t(g(a))}{da}\bigr|_{a = 0} +  \dfrac{a^2}{2} \min_{a \in (0,1)}\dfrac{d^2\loss_t(g(a))}{da^2}\\
& \sequal{\cref{eq_der_mapping}} a \tdirectionalGradient\bigr|_{\tExtendedParameterRowT = g(0)} +  \dfrac{a^2}{2} \min_{a \in (0,1)}\tdirectionalHessian\bigr|_{\tExtendedParameterRowT = g(a)}\\
& \sequal{(a)} a\directionalGradientTrue +  \dfrac{a^2}{2} \min_{a \in (0,1)}\tdirectionalHessian\bigr|_{\tExtendedParameterRowT = g(a)}\\
& \sgreat{(b)} - a \bigabs{\directionalGradientTrue} +  \dfrac{a^2}{2} \min_{a \in (0,1)}\tdirectionalHessian\bigr|_{\tExtendedParameterRowT = g(a)}, \label{eq_taylor_expansion}
\end{align}
where $(a)$ follows because $g(0) = \ExtendedTrueParameterRowT$ and $(b)$ follows by the triangle inequality. Plugging in $a = 1$ in \cref{eq_taylor_expansion} as well as using $g(0) = \ExtendedTrueParameterRowT$ and $g(1) = \ExtendedParameterRowT$, we find that
\begin{align}
\loss_t(\ExtendedParameterRowT) - \loss_t(\ExtendedTrueParameterRowT) \geq - \bigabs{\directionalGradientTrue} +  \dfrac{1}{2} \min_{a \in (0,1)}\tdirectionalHessian\bigr|_{\tExtendedParameterRowT = g(a)}.
\end{align}
Now, we use \cref{lemma_conc_first_sec_der} with 
\begin{align}
\varepsilon_1 \mapsfrom 
\frac{\varepsilon}{2},
\quad \delta_1 \mapsfrom \frac{\delta}{2}, 
\quad 
\varepsilon_2 \mapsfrom 
\varepsilon,
\qtext{and} \delta_2 \mapsfrom \frac{\delta}{2}.
\end{align}
Thus for $n \geq  \dfrac{ce^{c'\bGM} \log(p/\delta)}{\varepsilon^2}$, we have
\begin{align}
\loss_t(\ExtendedParameterRowT) 
\!-\! \loss_t(\ExtendedTrueParameterRowT) & \geq - \frac{\varepsilon}{2} \!+\!  \dfrac{1}{2} \biggparenth{\frac{\lambda_{\min} \twonorm{\Omt}^2}{\ctwo}-\varepsilon} \!=\! \frac{\lambda_{\min} \twonorm{\Omt}^2}{2\ctwo} - \varepsilon, \label{eq:taylor_expansion_with_grad_hess_conc}
\end{align}
uniformly for all $t \in [p]$, with probability at least $1-\delta$.

\subsubsection{Proof of \cref{lemma_conc_first_sec_der}\cref{item_conc_first_der}: \concgradresultname}
\label{sub:proof_of_lemma_conc_first_der}
\newcommand{\expressionfirstder}{Expression for first directional derivative}
For every $t \in [p]$ with $\uOmT$ defined in \cref{eq:omega_defn}, we claim that the first-order directional derivative of the loss function defined in \cref{eq:loss_function_p} is given by
\begin{align}
\directionalGradient  = -\frac{1}{n}\sumn  \Bigparenth{\DeltatIp  \tsvbx^{(i)}}  \exp\Bigparenth{-\normalbrackets{\ExternalFieldtI + 2\ParameterRowttt\tp \svbx^{(i)}_{-t}} x_t^{(i)} - \ParameterTU[tt] \cx^{(i)}_t}, \label{eq:first_dir_derivative}
\end{align}
where $\DeltatI \defeq \begin{bmatrix} \omtI \\ \Omttt\tp \\ \Omtu[t] \end{bmatrix} \in \real^{p+1}$ and $\tsvbx^{(i)} \defeq \begin{bmatrix} x_t^{(i)} \\ 2\svbx_{-t}^{(i)}x_t^{(i)}  \\ \cx_t^{(i)} \end{bmatrix} \in \real^{p+1} $ for all $ i \in [n]$ with $\cx^{(i)}_t = \bigbrackets{x^{(i)}_t}^2 - \xmax^2/3$.
We provide a proof at the end.\\

\noindent Next, we claim that the mean of the first-order directional derivative evaluated at the true parameter is zero. We provide a proof at the end.

\newcommand{\zeromeangradient}{Zero-meanness of first directional derivative}
\begin{lemma}[{\zeromeangradient}]
	\label{prop_zero_mean_first_der}
	For every $t \in [p]$ with $\uOmT$ defined in \cref{eq:omega_defn}, we have
	$\Expectation\bigbrackets{\directionalGradientTrue} = 0$.
\end{lemma}
\noindent Given these, we proceed to show the concentration of the first-order directional derivative evaluated at the true parameter. Fix any $t \in [p]$. From \cref{eq:first_dir_derivative}, we have
\begin{align}
\directionalGradientTrue & \sequal{\cref{eq:first_dir_derivative}}  -\frac{1}{n}\sumn  \Bigparenth{\DeltatIp  \tsvbx^{(i)}} \exp\Bigparenth{-\normalbrackets{\TrueExternalFieldtI + 2\TrueParameterRowtttTop \svbx^{(i)}_{-t}} x_t^{(i)} - \TrueParameterTU[tt] \cx^{(i)}_t }. 
\end{align}
Each term in the above summation is an independent random variable and is bounded as follows
\begin{align}
& \Bigabs{\Bigparenth{\DeltatIp  \tsvbx^{(i)}} \times \exp\Bigparenth{-\normalbrackets{\TrueExternalFieldtI + 2\TrueParameterRowtttTop \svbx^{(i)}_{-t}} x_t^{(i)} - \TrueParameterTU[tt] \cx^{(i)}_t }} \\
& \sequal{(a)} \Bigabs{\Bigparenth{\omtI x_t^{(i)} + 2\Omttt\tp  \svbx_{-t}^{(i)} x_t^{(i)} + \Omtu[t] \cx_t^{(i)}} \times \exp\Bigparenth{-\normalbrackets{\TrueExternalFieldtI + 2\TrueParameterRowtttTop \svbx^{(i)}_{-t}} x_t^{(i)} - \TrueParameterTU[tt] \cx^{(i)}_t }} \\
& \sless{(b)} \bigabs{\normalabs{\omtI} + 2\sonenorm{\Omt} \sinfnorm{\svbx^{(i)}}} \times \xmax \times \exp\Bigparenth{\bigparenth{\normalabs{\TrueExternalFieldtI} + 2\sonenorm{\TrueParameterRowt} \sinfnorm{\svbx^{(i)}}} \xmax} \\
& \sless{(c)} \bigparenth{2\aGM + 8 \bGM \xmax} \times \xmax \times \exp\Bigparenth{\normalparenth{\aGM + 2 \bGM \xmax} \xmax} \sequal{\cref{eq:constants}} 2 \cone \ctwo \xmax, 
\end{align}
where $(a)$ follows by plugging in $\DeltatI$ and $\tsvbx^{(i)}$, 
$(b)$ follows from triangle inequality, Cauchy–Schwarz inequality, and because $\sinfnorm{\svbx^{(i)}} \leq \xmax $ for all $ i \in [n]$, and $(c)$ follows because $\TrueExternalFieldI \in \ParameterSet_{\ExternalField} $ for all $ i \in [n]$, $\TrueParameterMatrix \in \ParameterSet_{\ParameterMatrix}$, $\omI \in 2\ParameterSet_{\ExternalField} $ for all $ i \in [n]$, $\Om \in 2\ParameterSet_{\ParameterMatrix}$, and $\sinfnorm{\svbx^{(i)}} \leq \xmax $ for all $ i \in [n]$.\\

\noindent Further, from \cref{prop_zero_mean_first_der}, we have $\Expectation\bigbrackets{\directionalGradientTrue} = 0$. Therefore, using the Hoeffding's inequality results in
\begin{align}
\Probability \Bigparenth{\bigabs{\directionalGradientTrue} > \varepsilon_1} < 2 \exp\biggparenth{-\dfrac{n \varepsilon_1^2}{8\cone[2] \ctwo[2] \xmax^2}}. \label{eq_hoeffding_first_dir}
\end{align}
The proof follows by using the union bound over all $t \in [p]$.

\paragraph{Proof of \cref{eq:first_dir_derivative}: {\expressionfirstder}}
\label{sub_sub_sec_proof_of_expression_first_der}
Fix any $t \in [p]$. The first-order partial derivatives of $\loss_t$ with respect to entries of $\ExtendedParameterRowT$ defined in \cref{eq:loss_function_p} are given by
\begin{align}
\frac{\partial \loss_t(\ExtendedParameterRowT)}{\partial \ExternalFieldtI} & \!=\! \frac{-1}{n}x_t^{(i)}\exp\Bigparenth{\!-\!\normalbrackets{\ExternalFieldtI \!+\! 2\ParameterRowttt\tp \svbx_{-t}^{(i)}} x_t^{(i)}\!-\!\ParameterTU[tt]\cx_t^{(i)}} \stext{for all} i \in [n],
\qtext{and}\\
\frac{\partial \loss_t(\ExtendedParameterRowT)}{\partial \ParameterTU[tu]} & \!=\!
\begin{cases} \frac{-2}{n}\sumn x_t^{(i)} x_u^{(i)} \exp\Bigparenth{\!-\!\normalbrackets{\ExternalFieldtI \!+\! 2\ParameterRowttt\tp \svbx_{-t}^{(i)}} x_t^{(i)}\!-\!\ParameterTU[tt]\cx_t^{(i)}} \stext{for all} u \in [p]\!\setminus\!\braces{t}\!.\\
\frac{-1}{n}\sumn \cx_t^{(i)} \exp\Bigparenth{\!-\!\normalbrackets{\ExternalFieldtI \!+\! 2\ParameterRowttt\tp \svbx_{-t}^{(i)}} x_t^{(i)}\!-\!\ParameterTU[tt]\cx_t^{(i)}} \stext{for} u = t.
\end{cases}
\label{eq:theta_first_derivatives} 
\end{align}
Now, we can write the first-order directional derivative of $\loss_t$ as
\begin{align}
\directionalGradient &\!\defeq\!\lim_{h\to 0}\frac{\loss_t(\ExtendedParameterRowT + h \uOmT)-\loss_t(\ExtendedParameterRowT)}{h} =  \sumn  \omtI \frac{\partial \loss_t(\ExtendedParameterRowT)}{\partial \ExternalFieldtI} + \sumu \Omtu \frac{\partial \loss_t(\ExtendedParameterRowT)}{\partial \ParameterTU[tu]}\\ 
& \!\!\!\!=\! \frac{\!-1}{n}\!\!\sumn \!\! \Bigparenth{\!\omtI x_t^{(i)} \!\!+\! 2\sum_{u \in [p] \setminus \{t\}} \Omtu x_t^{(i)} \!x_u^{(i)} \!\!+\! \Omtu[t] \cx_t^{(i)}\!} \!\exp\Bigparenth{\!\!\!-\!\!\normalbrackets{\ExternalFieldtI \!\!+\! 2\ParameterRowttt\tp \svbx_{\!-t}^{(i)}} x_t^{(i)}\!\!\!-\!\ParameterTU[tt]\cx_t^{(i)}\!} \\
& \!\!\!\!=\! \frac{\!-1}{n}\!\!\sumn  \Bigparenth{\!\omtI x_t^{(i)} \!\!+\! 2\Omttt\tp  \svbx_{-t}^{(i)} x_t^{(i)} \!\!+\! \Omtu[t] \cx_t^{(i)}}  \!\exp\Bigparenth{\!\!\!-\!\!\normalbrackets{\ExternalFieldtI \!\!+\! 2\ParameterRowttt\tp \svbx_{\!-t}^{(i)}} x_t^{(i)}\!\!\!-\!\ParameterTU[tt]\cx_t^{(i)}\!} \\
& \!\!\!\!\sequal{(a)}\! \frac{-1}{n}\sumn  \Bigparenth{\DeltatIp  \tsvbx^{(i)}}  \exp\Bigparenth{\!-\!\normalbrackets{\ExternalFieldtI \!+\! 2\ParameterRowttt\tp \svbx_{-t}^{(i)}} x_t^{(i)}\!-\!\ParameterTU[tt]\cx_t^{(i)}},
\end{align}
where $(a)$ follows from the definitions of $\DeltatI$ and $\tsvbx^{(i)}$.

\paragraph{Proof of \cref{prop_zero_mean_first_der}: {\zeromeangradient}}
\label{sub:proof_of_prop_zero_mean_first_der}
Fix any $t \in [p]$. From \cref{eq:first_dir_derivative}, we have
\begin{align}
& \Expectation\bigbrackets{\directionalGradientTrue}\\
& \sequal{\cref{eq:first_dir_derivative}}  -\frac{1}{n} \sumn \Expectation_{\svbx^{(i)}, \svbz^{(i)}} \biggbrackets{  \Bigparenth{\DeltatIp  \tsvbx^{(i)}} \exp\Bigparenth{-\normalbrackets{\TrueExternalFieldtI + 2\TrueParameterRowtttTop \svbx^{(i)}_{-t}} x_t^{(i)} - \TrueParameterTU[tt] \cx^{(i)}_t }}\\ 
&  \sequal{(a)} -\frac{1}{n}\sumn \!  \sum_{u \in [p+1]} \!\! \Expectation_{\svbz^{(i)}} \biggbrackets{ \DeltaItu \Expectation_{\svbx^{(i)} | \svbz^{(i)}} \Bigbrackets{\tsvbx^{(i)}_u \exp\Bigparenth{\!-\!\normalbrackets{\TrueExternalFieldt(\svbz^{(i)}) \!+\! 2\TrueParameterRowtttTop \svbx^{(i)}_{-t}} x_t^{(i)} \!-\! \TrueParameterTU[tt] \cx^{(i)}_t }}},
\end{align}
where $(a)$ follows by linearity of expectation and by plugging in $\TrueExternalFieldtI = \TrueExternalFieldt(\svbz^{(i)})$.
Now to complete the proof, we show that for any $i \in [n], u \in [p+1]$ and $\svbz^{(i)} \in \cZ^{p_z}$, we have
\begin{align}
\Expectation_{\svbx^{(i)} | \svbz^{(i)}} \Bigbrackets{\tsvbx^{(i)}_u \exp\Bigparenth{-\normalbrackets{\TrueExternalFieldt(\svbz^{(i)}) + 2\TrueParameterRowtttTop \svbx^{(i)}_{-t}} x_t^{(i)} - \TrueParameterTU[tt] \cx^{(i)}_t }} = 0.
\end{align}
Fix any $i \in [n], u \in [p+1]$ and $\svbz^{(i)} \in \cZ^{p_z}$. We have
\begin{align}
& \Expectation_{\svbx^{(i)} | \svbz^{(i)}} \Bigbrackets{\tsvbx^{(i)}_u \exp\Bigparenth{-\normalbrackets{\TrueExternalFieldt(\svbz^{(i)}) + 2\TrueParameterRowtttTop \svbx^{(i)}_{-t}} x_t^{(i)} - \TrueParameterTU[tt] \cx^{(i)}_t }}\\
& =   \int\limits_{\cX^p}  \tx^{(i)}_u  \exp\Bigparenth{-\normalbrackets{\TrueExternalFieldt(\svbz^{(i)}) + 2\TrueParameterRowtttTop \svbx_{-t}^{(i)}} x_t^{(i)} - \TrueParameterTU[tt] \cx_t^{(i)}} f_{\rvbx| \rvbz}\bigparenth{\svbx^{(i)}| \svbz^{(i)}} d\svbx^{(i)}\\
& =  \int\limits_{\cX^p}  \tx^{(i)}_u  \exp\Bigparenth{\!\!-\!\normalbrackets{\TrueExternalFieldt(\svbz^{(i)}) + 2\TrueParameterRowtttTop \svbx_{-t}^{(i)}} x_t^{(i)} - \TrueParameterTU[tt] \cx_t^{(i)}} f_{\rvbx_{-t}| \rvbz}\bigparenth{\svbx_{-t}^{(i)}| \svbz^{(i)}}  \times \\
& \qquad\qquad\qquad\qquad\qquad\qquad \TrueConditionalDistIt   d\svbx^{(i)} \\
& \sequal{(a)}  \int\limits_{\cX^p } \dfrac{\tx^{(i)}_u  f_{\rvbx_{-t}|\rvbz}\bigparenth{\svbx_{-t}^{(i)}| \svbz^{(i)}} d\svbx^{(i)}} {\int_{\cX} \exp\Bigparenth{\normalbrackets{\TrueExternalFieldt(\svbz^{(i)}) + 2\TrueParameterRowtttTop \svbx_{-t}^{(i)}} x_t^{(i)} + \TrueParameterTU[tt] \cx_t^{(i)}}d x_t^{(i)}} \\
& = \int\limits_{\cX^{p-1}} \biggbrackets{\int_{\cX} \tx_u^{(i)} dx_t^{(i)}} \dfrac{ f_{\rvbx_{-t}| \rvbz}\bigparenth{\svbx_{-t}^{(i)}|\svbz^{(i)}} d\svbx_{-t}^{(i)}} {\int_{\cX} \exp\Bigparenth{\normalbrackets{\TrueExternalFieldt(\svbz^{(i)}) + 2\TrueParameterRowtttTop \svbx_{-t}^{(i)}} x_t^{(i)} + \TrueParameterTU[tt] \cx_t^{(i)}}d x_t^{(i)}} \\
\sequal{(b)} & ~~~ 0,
\end{align}
where $(a)$ follows by plugging in $ \TrueConditionalDistIt $ from \cref{eq_conditional_dist} and $(b)$ follows because $\int_{\cX} x_t^{(i)} dx_t^{(i)} = 0$ and $\int_{\cX} \cx_t^{(i)} dx_t^{(i)} = 0$.

\subsubsection{Proof of \cref{lemma_conc_first_sec_der}\cref{item_conc_sec_der}: \conchessresultname}
\label{sub:proof_of_lemma_conc_sec_der}
\newcommand{\lowerboundhessian}{Lower bound on the second directional derivative}
\noindent We start by claiming that the second-order directional derivative can be lower bounded by a quadratic form. We provide a proof in \cref{sub:proof_of_lemma_lower_bound_sec_der}. 

\begin{lemma}[{\lowerboundhessian}]\label{lemma_lower_bound_sec_der}
	For every $t \in [p]$ with $\uOmT$ defined in \cref{eq:omega_defn}, we have
	\begin{align}
	\directionalHessian \geq \frac{1}{n \ctwo } \sumn  \Bigparenth{\DeltatIp  \tsvbx^{(i)}}^2,
	\end{align}
	where $\DeltatI \defeq \begin{bmatrix} \omtI \\ \Omttt\tp \\ \Omtu[t] \end{bmatrix} \in \real^{p+1}$ and $\tsvbx^{(i)} \defeq \begin{bmatrix} x_t^{(i)} \\ 2\svbx_{-t}^{(i)}x_t^{(i)}  \\ \cx_t^{(i)} \end{bmatrix} \in \real^{p+1} $ for all $ i \in [n]$ with $\cx^{(i)}_t = \bigbrackets{x^{(i)}_t}^2 - \xmax^2/3$ and the constant $\ctwo$ was defined in \cref{eq:constants}.
\end{lemma}
\noindent Given this, we proceed to show the anti-concentration of the second-order directional derivative. Fix any $t \in [p]$ and any $\ExtendedParameterMatrix \in \ParameterSet_{\ExternalField}^n \times \ParameterSet_{\ParameterMatrix}$. From \cref{lemma_lower_bound_sec_der}, we have
\begin{align}
\directionalHessian & \geq  \frac{1}{n \ctwo}\sumn  \Bigparenth{\DeltatIp  \tsvbx^{(i)}}^2. \label{eq_sec_der_form_lower_bound}
\end{align}
First, using the Hoeffding’s inequality, let us show concentration of $\frac{1}{n}\sumn  \Bigparenth{\DeltatIp  \tsvbx^{(i)}}^2$ around its mean. We observe that each term in the summation is an independent random variable and is bounded
as follows
\begin{align}
\Bigparenth{\DeltatIp  \tsvbx^{(i)}}\!^2 & \sequal{(a)} \bigparenth{\omtI x_t^{(i)} + 2\Omttt\tp  \svbx_{-t}^{(i)} x_t^{(i)} + \Omtu[t] \cx_t^{(i)}}^2 \\
& 
\sless{(b)} \bigparenth{\normalabs{\omtI} \!+\! 2\sonenorm{\Omt} \sinfnorm{\svbx^{(i)}}}^2  \xmax^2
\sless{(c)} \bigparenth{2\aGM \!+\! 8 \bGM \xmax}^2 \xmax^2 \sequal{\cref{eq:constants}} 4 \cone[2] \xmax^2,
\end{align}
where $(a)$ follows by plugging in $\DeltatI$ and $\tsvbx^{(i)}$, 
$(b)$ follows from triangle inequality, Cauchy–Schwarz inequality and because $\sinfnorm{\svbx^{(i)}} \leq \xmax $ for all $ i \in [n]$, and $(c)$ follows because $\Om \in 2\ParameterSet_{\ParameterMatrix}$, $\omI \in 2\ParameterSet_{\ExternalField}$, and $\sinfnorm{\svbx^{(i)}} \leq \xmax $ for all $ i \in [n]$. Then, from the Hoeffding's inequality, for any $\varepsilon > 0$ we have
\begin{align}
\Probability \Bigparenth{\biggabs{\frac{1}{n}\sumn  \Bigparenth{\DeltatIp  \tsvbx^{(i)}}^2 - \frac{1}{n}\sumn  \Expectation\biggbrackets{\Bigparenth{\DeltatIp  \tsvbx^{(i)}}^2} } > \varepsilon} < 2 \exp\biggparenth{-\dfrac{n \varepsilon^2}{32\cone[4] \xmax^4}}. \label{eq_hoeffding_first_dir}
\end{align}
Applying the union bound over all $t \in [p]$, for any $\delta \in (0,1)$ and uniformly for all $t \in [p]$, we have
\begin{align}
\frac{1}{n}\sumn  \Bigparenth{\DeltatIp  \tsvbx^{(i)}}^2 \geq \frac{1}{n}\sumn  \Expectation\biggbrackets{\Bigparenth{\DeltatIp  \tsvbx^{(i)}}^2} - \varepsilon, \label{eq_empirical_quad_form_lower_bound}
\end{align}
with probability at least $1-\delta$ as long as
\begin{align}
n \geq \dfrac{32\cone[4] \xmax^4}{\varepsilon^2}\log\biggparenth{\dfrac{2p}{\delta}}.
\end{align}
Now, we lower bound $\Expectation\Bigbrackets{ \Bigparenth{\DeltatIp  \tsvbx^{(i)}}^2 }$ for every $t \in [p]$ and every $i \in [n]$. Fix any $t \in [p]$ and $i \in [n]$. We have
\begin{align}
\Expectation_{\svbx^{(i)}, \svbz^{(i)}}\biggbrackets{ \Bigparenth{\DeltatIp  \tsvbx^{(i)}}^2 } & = \Expectation_{\svbz^{(i)}}\biggbrackets{ \DeltatIp \Expectation_{\svbx^{(i)} | \svbz^{(i)}} \Bigbrackets{  \tsvbx^{(i)} \tsvbx^{(i)\tp} | \svbz^{(i)} }\DeltatI}  \\
& \sgreat{(a)} \lambda_{\min} \Expectation_{\svbz^{(i)}}\Bigbrackets{ \stwonorm{\DeltatI}^2} \sgreat{(b)} \lambda_{\min}  \stwonorm{\Omt}^2, \label{eq_lower_bound_squared_expectation_lambda_min}
\end{align}
where $(a)$ follows from \cref{ass_pos_eigenvalue} and $(b)$ follows from the definition of $\DeltatI$. Combining \cref{eq_sec_der_form_lower_bound,eq_empirical_quad_form_lower_bound,eq_lower_bound_squared_expectation_lambda_min}, for any $\delta \in (0,1)$ and uniformly for all $t \in [p]$, we have
\begin{align}
\directionalHessian \geq \frac{1}{\ctwo} \biggparenth{\lambda_{\min}  \stwonorm{\Omt}^2 - \varepsilon},
\end{align}
with probability at least $1-\delta$ as long as
\begin{align}
n \geq \dfrac{32\cone[4] \xmax^4}{\varepsilon^2}\log\biggparenth{\dfrac{2p}{\delta}}.
\end{align}
Choosing $\varepsilon = \varepsilon_2 \ctwo$ and $\delta = \delta_2$ yields the claim. 

\paragraph{Proof of \cref{lemma_lower_bound_sec_der}: \lowerboundhessian}
\label{sub:proof_of_lemma_lower_bound_sec_der}
For every $t \in [p]$ with $\uOmT$ defined in \cref{eq:omega_defn}, we claim that the second-order directional derivative of the loss function defined in \cref{eq:loss_function_p} is given by
\begin{align}
\directionalHessian = \frac{1}{n}\sumn  \Bigparenth{\DeltatIp  \tsvbx^{(i)}}^2  \exp\Bigparenth{-\normalbrackets{\ExternalFieldtI + 2\ParameterRowttt\tp \svbx_{-t}^{(i)}} x_t^{(i)} - \ParameterTU[tt]  \cx_t^{(i)}}, 
\label{eq:second_dir_derivative}
\end{align}
where $\DeltatI \defeq \begin{bmatrix} \omtI \\ \Omttt\tp \\ \Omtu[t] \end{bmatrix} \in \real^{p+1}$ and $\tsvbx^{(i)} \defeq \begin{bmatrix} x_t^{(i)} \\ 2\svbx_{-t}^{(i)}x_t^{(i)}  \\ \cx_t^{(i)} \end{bmatrix} \in \real^{p+1} $ for all $ i \in [n]$ with $\cx^{(i)}_t = \bigbrackets{x^{(i)}_t}^2 - \xmax^2/3$. We provide a proof at the end.\\

\newcommand{\expressionsecder}{Expression for second directional derivative}

\noindent Given this claim, we proceed to prove the lower bound on the second directional derivative. Fix any 
$t \in [p]$. From \cref{eq:second_dir_derivative}, we have
\begin{align}
\directionalHessian &  = \frac{1}{n}\sumn   \Bigparenth{\DeltatIp  \tsvbx^{(i)}}^2 \times \exp\Bigparenth{-\normalbrackets{\ExternalFieldtI + 2\ParameterRowttt\tp \svbx_{-t}^{(i)}} x_t^{(i)} - \ParameterTU[tt]  \cx_t^{(i)}}  \\
& \sgreat{(a)} \frac{1}{n}\sumn   \Bigparenth{\DeltatIp  \tsvbx^{(i)}}^2 \times \exp\Bigparenth{-\bigparenth{\normalabs{\ExternalFieldtI} + 2\sonenorm{\ParameterRowt} \sinfnorm{\svbx^{(i)}}} \xmax}\\
& \sgreat{(b)} \frac{1}{n}\sumn   \Bigparenth{\DeltatIp  \tsvbx^{(i)}}^2 \times  \exp\Bigparenth{-\normalparenth{\aGM + 2 \bGM \xmax} \xmax}\\
&  \sequal{\cref{eq:constants}} \frac{1}{\ctwo n}\sumn  \Bigparenth{\DeltatIp  \tsvbx^{(i)}}^2, \label{eq_lower_bound_sec_der}
\end{align}
where $(a)$ follows from triangle inequality, Cauchy–Schwarz inequality and because $\sinfnorm{\svbx^{(i)}} \leq \xmax $ for all $ i \in [n]$, and $(b)$ follows because $\ExternalFieldI \in \ParameterSet_{\ExternalField} $ for all $ i \in [n]$, $\ParameterMatrix \in \ParameterSet_{\ParameterMatrix}$, and $\sinfnorm{\svbx^{(i)}} \leq \xmax $ for all $ i \in [n]$.

\paragraph{Proof of \cref{eq:second_dir_derivative}: {\expressionsecder}}
Fix any $t \in [p]$. The second-order partial derivatives of $\loss_t$ with respect to entries of $\ExtendedParameterRowT$ defined in \cref{eq:loss_function} are given by
\begin{align}
\frac{\partial^2 \loss_t(\ExtendedParameterRowT)}{\partial \bigbrackets{\ExternalFieldtI}^2} & = \frac{1}{n}\bigbrackets{x_t^{(i)}}^2\exp\Bigparenth{-\normalbrackets{\ExternalFieldtI + 2\ParameterRowttt\tp \svbx_{-t}^{(i)}} x_t^{(i)} - \ParameterTU[tt]  \cx_t^{(i)}} \qtext{for all} i \in [n],\\
\frac{\partial^2 \loss_t(\ExtendedParameterRowT)}{\partial \ParameterTU[tu] \ParameterTU[tv]} & = 
\begin{cases} \frac{4}{n}\sumn\bigbrackets{x_t^{(i)}}^2 x_u^{(i)} x_v^{(i)} \exp\Bigparenth{\!\!-\normalbrackets{\ExternalFieldtI \!+\! 2\ParameterRowttt\tp \svbx_{-t}^{(i)}} x_t^{(i)} \!-\! \ParameterTU[tt]  \cx_t^{(i)}} \\
\qquad \qquad \qquad\qquad ~~~~~ \qquad \qquad \qquad \qquad \qquad \stext{for all} u,v \in [p]\!\setminus\!\braces{t}. \\
\frac{2}{n}\sumn \cx_t^{(i)} x_t^{(i)} x_u^{(i)} \exp\Bigparenth{\!\!-\normalbrackets{\ExternalFieldtI \!+\! 2\ParameterRowttt\tp \svbx_{-t}^{(i)}} x_t^{(i)} \!-\! \ParameterTU[tt]  \cx_t^{(i)}} \\
\qquad \qquad \qquad\qquad ~~~~~ \qquad \qquad \qquad \qquad \qquad \stext{for all} u \in [p]\!\setminus\!\braces{t} \stext{and} v \!=\! t.\\
\frac{2}{n}\sumn \cx_t^{(i)} x_t^{(i)} x_v^{(i)} \exp\Bigparenth{\!\!-\normalbrackets{\ExternalFieldtI \!+\! 2\ParameterRowttt\tp \svbx_{-t}^{(i)}} x_t^{(i)} \!-\! \ParameterTU[tt]  \cx_t^{(i)}} \\
\qquad \qquad \qquad\qquad ~~~~~ \qquad \qquad \qquad\qquad \qquad \stext{for all} v \in [p]\!\setminus\!\braces{t} \stext{and} u \!=\! t.\\
\frac{1}{n}\sumn \bigbrackets{\cx_t^{(i)}}^2 \exp\Bigparenth{\!\!-\normalbrackets{\ExternalFieldtI \!+\! 2\ParameterRowttt\tp \svbx_{-t}^{(i)}} x_t^{(i)} \!-\! \ParameterTU[tt]  \cx_t^{(i)}} \\
\qquad \qquad \qquad\qquad ~~~~~ \qquad \qquad \qquad \qquad \qquad \stext{for} v \!=\! t \stext{and} u \!=\! t.
\end{cases}\\
\frac{\partial^2 \loss_t(\ExtendedParameterRowT)}{\partial \ParameterTU[tu] \ExternalFieldtI} & \!=\! \frac{\partial^2 \loss_t(\ExtendedParameterRowT)}{\partial \ExternalFieldtI \ParameterTU[tu]} \!=\! \begin{cases} 
\frac{2}{n}\bigbrackets{x_t^{(i)}}^2 x_u^{(i)} \exp\!\Bigparenth{\!\!\!-\!\normalbrackets{\ExternalFieldtI \!\!+\! 2\ParameterRowttt\tp \svbx_{-t}^{(i)}} x_t^{(i)} \!\!-\! \ParameterTU[tt]  \cx_t^{(i)}} \\
\qquad \qquad \qquad\qquad \qquad \qquad ~~ \qquad \stext{for all} i \in [n], u \in [p]\! \setminus\!\braces{t}. \\
\frac{1}{n} x_t^{(i)} \cx_t^{(i)} \exp\!\Bigparenth{\!\!\!-\!\normalbrackets{\ExternalFieldtI \!\!+\! 2\ParameterRowttt\tp \svbx_{-t}^{(i)}} x_t^{(i)} \!\!-\! \ParameterTU[tt]  \cx_t^{(i)}} \\
\qquad \qquad \qquad\qquad \qquad \qquad ~~ \qquad \stext{for all} i \in [n], u  = t.
\end{cases}
\label{eq:theta_second_derivatives} 
\end{align}
Now, we can write the second-order directional derivative of $\loss_t$ as
\begin{align}
& \directionalHessian  \defeq \lim_{h\to 0}\frac{\partial_{\uOmT}\loss_t(\ExtendedParameterRowT+h \uOmT)\!-\!\partial_{\uOmT}\loss_t(\ExtendedParameterRowT)}{h} \\
& =  \sumn  \bigbrackets{\omtI}^2 \frac{\partial^2 \loss_t(\ExtendedParameterRowT)}{\partial \bigbrackets{\ExternalFieldtI}^2} + \sumu \sumu[v] \Omtu \Omtu[v] \frac{\partial^2 \loss_t(\ExtendedParameterRowT)}{\partial \ParameterTU[tu] \ParameterTU[tv]} + 2\sumn \sumu \omtI \Omtu \frac{\partial^2 \loss_t(\ExtendedParameterRowT)}{\partial \ParameterTU[tu] \ExternalFieldtI} \\ 
& = \frac{1}{n}\sumn \!\! \Bigparenth{ \bigbrackets{\omtI x_t^{(i)}}^2 \!\!+\! 4 \!\!\sumu \!\! \Omtu x_t^{(i)} \!x_u^{(i)} \!\!\sumu[v] \! \Omtu[v]  x_t^{(i)} \!x_v^{(i)} \!+\! 4 \Omtu[t]  \cx_t^{(i)} \!\!\sumu \!\!\Omtu x_t^{(i)} \!x_u^{(i)} \!+\! \bigbrackets{\Omtu[t]  \cx_t^{(i)}}^2 \\
	& \qquad \!\!\!+\! 4\omtI x_t^{(i)} \!\! \sumu \!\! \Omtu x_t^{(i)} \! x_u^{(i)} \!+\! 2\omtI x_t^{(i)}\bigbrackets{\Omtu[t]  \cx_t^{(i)}}\!} \!\times\! \exp\!\Bigparenth{\!\!\!-\!\normalbrackets{\ExternalFieldtI \!\!+\! 2\ParameterRowttt\tp \svbx_{-t}^{(i)}} x_t^{(i)} \!\!\!-\! \ParameterTU[tt]  \cx_t^{(i)}\!} \\
& = \frac{1}{n}\sumn \bigparenth{\omtI x_t^{(i)} + 2\Omttt\tp  \svbx_{-t}^{(i)} x_t^{(i)} + \Omtu[t] \cx_t^{(i)}}^2  \exp\Bigparenth{-\normalbrackets{\ExternalFieldtI + 2\ParameterRowttt\tp \svbx_{-t}^{(i)}} x_t^{(i)} - \ParameterTU[tt]  \cx_t^{(i)}} \\
& \sequal{(a)} \frac{1}{n}\sumn  \Bigparenth{\DeltatIp  \tsvbx^{(i)}}^2  \exp\Bigparenth{-\normalbrackets{\ExternalFieldtI + 2\ParameterRowttt\tp \svbx_{-t}^{(i)}} x_t^{(i)} - \ParameterTU[tt]  \cx_t^{(i)}},
\end{align}
where $(a)$ follows from the definitions of $\DeltatI$ and $\tsvbx^{(i)}$.

\subsection{Example for \cref{ass_pos_eigenvalue}}
\label{subsec_discussion_ass_pos_eigenvalue}
As seen in \cref{eq_lower_bound_squared_expectation_lambda_min}, \cref{ass_pos_eigenvalue} is used to lower bound $\Expectation_{\svbx^{(i)}, \svbz^{(i)}}\Bigbrackets{ \bigparenth{\DeltatIp  \tsvbx^{(i)}}^2}$ by $\stwonorm{\Omt}^2$. In this section, we show that $\Expectation_{\svbx^{(i)}, \svbz^{(i)}}\Bigbrackets{ \bigparenth{\DeltatIp  \tsvbx^{(i)}}^2}$ can be lower bounded by $\stwonorm{\Omt}^2$ without requiring \cref{ass_pos_eigenvalue} if $\TrueParameterTU[tt] = 0$ for all $t \in [p]$ and {the row-wise $\ell_1$ sparsity of $\ParameterMatrix$ in \cref{assumptions} is assumed to be induced by row-wise $\ell_0$ sparsity, i.e., $\zeronorm{\ParameterRowt} \leq \bGM/\aGM$ for all $t \in [p]$}. To that end, first we claim that the conditional variance of $x_t^{(i)}$ conditioned on $\rvbx_{-t} = \svbx_{-t}^{(i)}$ and $\rvbz = \svbz^{(i)}$ is lower bounded by a constant for every $t \in [p]$ and $i \in [n]$. We provide a proof in \cref{sub:proof_of_prop_lower_bound_variance}.  
\newcommand{\lowerboundcondvar}{Lower bound on the conditional variance}
\begin{lemma}[{\lowerboundcondvar}]\label{prop_lower_bound_variance}
	We have
	\begin{align}
	\Variance\bigparenth{x_t^{(i)} \big| \svbx_{-t}^{(i)}, \svbz^{(i)}}  \geq \frac{2\xmax^2}{\pi e \ctwo[4]}
	\quad \stext{for all} t\in[p] \stext{and} i\in[n],
	\end{align}
	where the constant $\ctwo$ was defined in \cref{eq:constants}. 
\end{lemma}
Given this lemma, we proceed. We have     
\begin{align}
\Expectation\biggbrackets{ \Bigparenth{\DeltatIp  \tsvbx^{(i)}}^2 } \sgreat{(a)} \Variance \biggbrackets{ \DeltatIp  \tsvbx^{(i)}} \sequal{(b)} \Variance \biggbrackets{\omtI x_t^{(i)} + 2\Omt\tp  \svbx^{(i)} x_t^{(i)}}, \label{eq_expected_quad_form_lower_bound}
\end{align}
where $(a)$ follows from the fact that for any random variable a, $\Expectation\normalbrackets{a^2} \geq \Variance\normalbrackets{a}$ and $(b)$ follows because we let $\Omtu[t] = 0$ since $\TrueParameterTU[tt] = 0$. We define the following set to lower bound $\Variance \bigbrackets{ \omtI x_t^{(i)} + 2\Omt\tp  \svbx^{(i)} x_t^{(i)}}$:
\begin{align}
\cE(\TrueParameterMatrix) \defn \braces{(t,u) \in [p]^2: t < u, \TrueParameterTU \neq 0},\label{eq:edge_set}
\end{align}
and consider the graph $\cG(\TrueParameterMatrix) = ([p], \cE(\TrueParameterMatrix))$ with $[p]$ as nodes and $\cE(\TrueParameterMatrix)$ as edges such that $\TrueJointDist$ is Markov with respect to $\cG(\TrueParameterMatrix)$. We claim that there exists a non-empty set $\cR_t \subset [p]\setminus\braces{t}$ such that 
\begin{enumerate}[label=(\roman*)]
	\item\label{item_independence_set} $\cR_t$ is an independent set of $\cG(\TrueParameterMatrix)$, i.e., there are no edges between any pair of nodes in $\cR_t$, and
	\item\label{eq_independentSetProperty} the row vector $\Omt$ satisfies $\sum_{u \in \cR_t} \normalabs{\Omtu}^2 
	\geq 
	\frac{1}{\bGM/\aGM+1}\twonorm{\Omt}^2$.
\end{enumerate}
Taking this claim as given at the moment, we continue our proof. Denoting $\cR_t^c \defeq [p] \setminus \cR_t$, and using the law of total variance, the variance term in \cref{eq_expected_quad_form_lower_bound} can be lower bounded as
\begin{align}
\Variance \biggbrackets{\omtI x_t^{(i)} + 2\Omt\tp  \svbx^{(i)} x_t^{(i)}}  & \geq \Expectation\biggbrackets{\Variance \Bigbrackets{\omtI x_t^{(i)} + 2\Omt\tp  \svbx^{(i)} x_t^{(i)} \Big| \svbx_{\cR_t^c}^{(i)},  \svbz^{(i)}}}\\
& \sequal{(a)} 4\Expectation\Bigbrackets{\bigparenth{x_t^{(i)}}^2 \Variance\Bigparenth{\sum_{u \in \cR_t} \Omtu x_u^{(i)} \Big| \svbx_{\cR_t^c}^{(i)}, \svbz^{(i)}}}\\
& \sequal{(b)} 4\Expectation\Bigbrackets{\bigparenth{x_t^{(i)}}^2 \sum_{u \in \cR_t} \Omtu^2 \Variance\Bigparenth{ x_u^{(i)} \Big| \svbx_{\cR_t^c}^{(i)}, \svbz^{(i)}}}\\
& \sequal{(c)} 4\Expectation\Bigbrackets{\bigparenth{x_t^{(i)}}^2 \sum_{u \in \cR_t} \Omtu^2 \Variance\Bigparenth{ x_u^{(i)} \Big| \svbx_{-u}^{(i)}, \svbz^{(i)}}}\\
& \sgreat{(d)} \frac{8\xmax^2}{\pi e \ctwo[4]} \sum_{u \in \cR_t} \Omtu^2 \Expectation\Bigbrackets{\bigparenth{x_t^{(i)}}^2  }\\
& \sgreat{(e)} \frac{8\xmax^2}{\pi e \ctwo[4]} \sum_{u \in \cR_t} \Omtu^2 
\Variance\Bigparenth{ x_t^{(i)} \Big| \svbx_{-t}^{(i)}, \svbz^{(i)}}\\
& \sgreat{(f)} \frac{16\xmax^4}{\pi^2 e^2 \ctwo[8]} \sum_{u \in \cR_t} \Omtu^2 \sgreat{\cref{eq_independentSetProperty}} \frac{16\xmax^4 \twonorm{\Omt}^2}{\pi^2 e^2  \normalparenth{\bGM/\aGM + 1} \ctwo[8]}, \label{eq_var_intermediate_lower_bound}
\end{align}
where $(a)$ follows because $(x_u^{(i)})_{u \in \cR_t^c}$ are deterministic when conditioned on themselves, and $t \in \cR_t^c$, $(b)$ follows because $(x_u^{(i)})_{u \in \cR_t}$ are conditionally independent given $\svbx_{\cR_t^c}^{(i)}$ and $\svbz^{(i)}$ which is a direct consequence of \cref{item_independence_set}, $(c)$ follows because of the local Markov property (as the conditioning set includes all the neighbors in $\cG(\TrueParameterMatrix)$ of each node in  $\cR_t$),
$(d)$ and $(f)$ follow from \cref{prop_lower_bound_variance}, and $(e)$ follows because $\Expectation\Bigbrackets{\bigparenth{x_t^{(i)}}^2} = \Expectation\Bigbrackets{\Expectation\Bigbrackets{\bigparenth{x_t^{(i)}}^2 \Big| \svbx_{-t}^{(i)}, \svbz^{(i)}}} \geq \Variance\bigparenth{ x_t^{(i)} \Big| \svbx_{-t}^{(i)}, \svbz^{(i)}}$.\\

\noindent Combining \cref{eq_expected_quad_form_lower_bound} and \cref{eq_var_intermediate_lower_bound}, we have
\begin{align}
\Expectation_{\svbx^{(i)}, \svbz^{(i)}}\Bigbrackets{ \bigparenth{\DeltatIp  \tsvbx^{(i)}}^2} \geq \frac{16\xmax^4}{\pi^2 e^2  \normalparenth{\bGM/\aGM + 1} \ctwo[8]} \cdot \stwonorm{\Omt}^2.
\end{align}
\noindent It remains to construct the set $\cR_t$ that is an independent set of $\cG(\TrueParameterMatrix)$ and satisifies~\cref{eq_independentSetProperty}.

\paragraph{Construction of the set $\cR_t$}
For every $u \in [p]$, let $\cN(u)$ denote the set of neighbors of $u$ in $\cG(\TrueParameterMatrix)$, i.e., $\cN(u) \defeq \braces{v \in [p]: (u, v) \in \cE(\TrueParameterMatrix)} \bigcup \braces{v \in [p]: (v, u) \in \cE(\TrueParameterMatrix)}$. 
We start by selecting $r_1 \in [p] \setminus \braces{t}$ such that
\begin{align}
\normalabs{\Omtu[r_1]} \geq \normalabs{\Omtu} \qtext{for all} u \in [p] \setminus \braces{t , r_1}.
\end{align}
Next, we identify $r_2 \in [p] \setminus \braces{t, r_1, \cN(r_1)}$ such that
\begin{align}
\normalabs{\Omtu[r_2]} \geq \normalabs{\Omtu} \qtext{for all} u \in [p] \setminus \braces{t , r_1, \cN(r_1), r_2}.
\end{align}
We continue identifying $r_3, \ldots, r_s$ in such a manner till no more nodes are left, where $s$ denotes the total number of nodes selected. Now we define $\cR_t \defeq \{ r_1, \cdots , r_s\}$. For any $u \in [p]$, we have $\normalabs{\cN(u)} \leq \szeronorm{\TrueParameterRowt[u]} \leq \bGM/\aGM$ from \cref{eq:edge_set} and \cref{assumptions}. Using this, we see that $\cR_t$ is an independent set of $\cG(\TrueParameterMatrix)$ as claimed in \cref{item_independence_set} such
that it satisfies \cref{eq_independentSetProperty} by construction.

\subsubsection{Proof of \cref{prop_lower_bound_variance}: \lowerboundcondvar}
\label{sub:proof_of_prop_lower_bound_variance}
For any random variable $\rvx$, let $\Entropy(\rvx)$ denote the differential entropy of $\rvx$.  Fix any $t \in [p]$ and $i \in [n]$. Then, from Shannon's entropy inequality $(2\Entropy(\cdot) \leq \log \sqrt{2\pi e  \Variance(\cdot)})$, we have
\begin{align}
    2\pi e \Variance\bigparenth{x_t^{(i)} \big| \svbx_{-t}^{(i)}, \svbz^{(i)}} \sgreat{(a)} {\exp\Bigparenth{2\Entropy\bigparenth{x_t^{(i)} \big| \svbx_{-t}^{(i)}, \svbz^{(i)}}}}. \label{eq_shannon}
\end{align}
Therefore, to bound the variance, it suffices to bound the differential entropy. We have
\begin{align}
& - \Entropy\bigparenth{x_t^{(i)} \big| \svbx_{-t}^{(i)}, \svbz^{(i)}}\\
& =    \int\limits_{\cX^p \times \cZ^{p_z}}   f_{\rvbx, \rvbz}(\svbx^{(i)}\!\!, \svbz^{(i)}) \log \Bigparenth{ \TrueConditionalDistIt } d\svbx^{(i)} d\svbz^{(i)} \\
& =   \int\limits_{\cX^p \times \cZ^{p_z}}  f_{\rvbx, \rvbz}(\svbx^{(i)}\!\!, \svbz^{(i)}) \log \!\biggparenth{\!\!\frac{\exp\bigparenth{\normalbrackets{\TrueExternalFieldt(\svbz^{(i)}) \!+\! 2\TrueParameterRowtttTop \svbx_{-t}^{(i)}} x_t^{(i)} \!+\! \TrueParameterTU[tt] \cx_t^{(i)}}}{\int_{\cX} \!\exp\bigparenth{\normalbrackets{\TrueExternalFieldt(\svbz^{(i)}) \!+\! 2\TrueParameterRowtttTop \svbx_{-t}^{(i)}} x_t^{(i)} \!+\! \TrueParameterTU[tt] \cx_t^{(i)}}d x_t^{(i)}}\!} d\svbx^{(i)} d\svbz^{(i)} \!\\
& \sgreat{(a)}  \int\limits_{\cX^p \times \cZ^{p_z}}   f_{\rvbx, \rvbz}(\svbx^{(i)}\!\!, \svbz^{(i)}) \log \!\biggparenth{\!\!\frac{\exp\bigparenth{\bigparenth{\normalabs{\TrueExternalFieldt(\svbz^{(i)})} \!+\! 2\sonenorm{\TrueParameterRowt} \sinfnorm{\svbx^{(i)}}} \xmax\!}}{ \int_{\cX} \!\exp\!\bigparenth{\!\!-\!\!\bigparenth{\normalabs{\TrueExternalFieldt(\svbz^{(i)})} \!+\! 2\sonenorm{\TrueParameterRowt} \sinfnorm{\svbx^{(i)}}} \xmax\!} d x_t^{(i)}}\!\!} d\svbx^{(i)} d\svbz^{(i)} \!\\
& \sgreat{(b)}   \int\limits_{\cX^p \times \cZ^{p_z}}   f_{\rvbx, \rvbz}(\svbx^{(i)}\!\!, \svbz^{(i)}) \log \!\biggparenth{\!\!\frac{\exp\bigparenth{\normalparenth{\aGM + 2 \bGM \xmax} \xmax}}{ \int_{\cX} \exp\bigparenth{-\normalparenth{\aGM + 2 \bGM \xmax} \xmax} d x_t^{(i)}}} d\svbx^{(i)} d\svbz^{(i)} \!\\
& \sequal{(c)}   \int\limits_{\cX^p \times \cZ^{p_z}}   f_{\rvbx, \rvbz}(\svbx^{(i)}\!\!, \svbz^{(i)}) \log \biggparenth{\frac{\cthree[2]}{2\xmax}} d\svbx^{(i)} d\svbz^{(i)}\! = \log \biggparenth{\frac{\cthree[2]}{2\xmax}}, \label{eq_conditional_variance_lower_bound}
\end{align}
where $(a)$ follows from triangle inequality and Cauchy–Schwarz inequality and because $\sinfnorm{\svbx^{(i)}} \leq \xmax $ for all $ i \in [n]$, $(b)$ follows because $\TrueExternalField(\svbz^{(i)}) \in \ParameterSet_{\ExternalField} $ for all $ i \in [n]$, $\TrueParameterMatrix \in \ParameterSet_{\ParameterMatrix}$, $\sinfnorm{\svbx^{(i)}} \leq \xmax $ for all $ i \in [n]$, and $(c)$ follows because $\int_{\cX} dx_t^{(i)} = 2\xmax$. Combining \cref{eq_shannon,eq_conditional_variance_lower_bound} completes the proof.

\newcommand{\singleparameterseparation}{Gap between the loss function for a fixed parameter}

{\subsection{Proof of \cref{lemma_lipschitzness_first_stage}: \lipschitznesslossfunction}
	\label{sub:proof_lemma_lipschitzness_first_stage}
	Consider any direction $\uOm$ $= \tExtendedParameterMatrix - \ExtendedParameterMatrix$. Now, define the function $q : [0,1] \to \Reals$ as follows
	\begin{align}
	q(a) = \loss\bigparenth{\ExtendedParameterMatrix + a(\tExtendedParameterMatrix - \ExtendedParameterMatrix)}. \label{eq_func_f_lipschitz_first_stage}
	\end{align}
	Then, the desired inequality in \cref{eq_lipschitz_property_first_stage} is equivalent to $$\normalabs{q(1) - q(0)} \leq 2\xmax^2 \ctwo \Bigparenth{\sump \sonenorm{\Omt}  + \frac{1}{n} \sumn \sonenorm{\omI}}.$$
	From the mean value theorem, there exists $a' \in (0,1)$ such that
	\begin{align}
	\normalabs{q(1) - q(0)} 
	= \biggabs{\dfrac{dq(a')}{da}} \sequal{\cref{eq_func_f_lipschitz_first_stage}} \Bigabs{\dfrac{d\loss\bigparenth{\ExtendedParameterMatrix + a(\tExtendedParameterMatrix - \ExtendedParameterMatrix)}}{da}} 
	\sequal{\cref{eq_der_mapping}} \Bigabs{\directionalGradientFull\bigr|_{\ExtendedParameterMatrix = \ExtendedParameterMatrix + a(\tExtendedParameterMatrix - \ExtendedParameterMatrix)}}. \label{eq_mvt_lipschitz_first_stage}
	\end{align}
	Using \cref{eq:first_dir_derivative} in \cref{eq_mvt_lipschitz_first_stage}, we can write
	\begin{align}
	& \bigabs{q(1) \!-\! q(0)} \\
	& =\frac{1}{n} \biggabs{\sump \sumn \! \Bigparenth{\DeltatIp  \tsvbx^{(i)}} \times \exp\Bigparenth{\!-\!\Bigbrackets{\bigparenth{\ExternalFieldtI + a'(\tExternalFieldtI \!-\! \ExternalFieldtI)} + \\
		& \qquad \qquad \qquad 2 \bigparenth{\ParameterRowttt + a'(\tParameterRowttt \!-\! \ParameterRowttt)}\tp \svbx^{(i)}_{-t}} x_t^{(i)} - \bigparenth{\ParameterTU[tt] + a'(\tParameterTU[tt] - \ParameterTU[tt] )} \cx_t^{(i)} }}\\
	& \sless{(a)}  \exp\Bigparenth{\bigparenth{\normalbrackets{(1\!-\!a') \aGM \!+\! a' \aGM} + 2\normalbrackets{(1\!-\!a') \bGM \!+\! a'  \bGM}\xmax} \xmax} \frac{1}{n} \biggabs{\!\sump \!\sumn  \! \Bigparenth{\DeltatIp  \tsvbx^{(i)}}}\\
	& \sless{(b)}  \frac{2\xmax^2 \ctwo}{n} \sump \sumn \sonenorm{\DeltatI} 
	\sequal{(c)} 2\xmax^2 \ctwo \Bigparenth{\sump \sonenorm{\Omt}  + \frac{1}{n} \sumn \sonenorm{\omI}} ,
	\end{align}
	where $(a)$ follows from triangle inequality, Cauchy–Schwarz inequality, $\ExternalFieldI, \tExternalFieldI \in \ParameterSet_{\ExternalField}$, $\ParameterMatrix, \tParameterMatrix \in \ParameterSet_{\ParameterMatrix}$, and $\sinfnorm{\svbx^{(i)}} \leq \xmax $ for all $ i \in [n]$, 
	$(b)$ follows from \cref{eq:constants}, the triangle inequality, and because $\sinfnorm{\svbx^{(i)}} \leq \xmax $ for all $ i \in [n]$, and $(c)$ follows from the definition of $\DeltatI$.

}
\section{Proof of {Theorem \ref{theorem_parameters}} Part II: \nodeparammainresultname}
\label{sec_proof_thm_node_parameters_recovery}
To analyze our estimate of the unit-level parameters, we use the estimate $\EstimatedParameterMatrix$ of the population-level parameter $\TrueParameterMatrix$ along with the associated guarantee provided in \cref{theorem_parameters} Part I. We note that the constraints on the unit-level parameters in \cref{eq_estimated_parameters} are independent across units, i.e., $\ExternalFieldI[i] \in \ParameterSet_{\ExternalField}$ independently for all $ i \in [n]$. Therefore, we look at $n$ independent convex optimization problems by decomposing the loss function $\loss$ in \cref{eq:loss_function} and the estimate $\ExtendedEstimatedParameterMatrix$ in \cref{eq_estimated_parameters} as follows: For $i \in [n]$, we define
\begin{align}
\loss^{(i)}\bigparenth{\ExternalFieldI} &\defn \sump[t] \exp\Bigparenth{-\normalbrackets{\ExternalFieldtI + 2\EstimatedParameterRowttt\tp \svbx_{-t}^{(i)}} x_t^{(i)} - \EstimatedParameterTU[tt] \cx_t^{(i)}} \\
\qtext{and} \EstimatedExternalFieldI &\defn \argmin_{\ExternalFieldI \in \ParameterSet_{\ExternalField}} \loss^{(i)}\bigparenth{\ExternalFieldI}.
\label{eq:loss_function_n}
\end{align}
Now, fix any $i \in [n]$. From \cref{eq:loss_function_n}, we have $\loss^{(i)}\bigparenth{\EstimatedExternalFieldI} \leq  \loss^{(i)}\bigparenth{\TrueExternalFieldI}$. Using contraposition, to prove this part, it is sufficient to show that all points $\ExternalFieldI \in \ParameterSet_{\ExternalField}$ that satisfy $\stwonorm{\ExternalFieldI - \TrueExternalFieldI} \geq R(\varepsilon, \delta)$ also uniformly satisfy
\begin{align}
\loss^{(i)}\bigparenth{\ExternalFieldI} & \geq  \loss^{(i)}\bigparenth{\TrueExternalFieldI} + R^2(\varepsilon, \delta) \stext{when} n \geq \frac{ce^{c'\bGM}p^4}{\varepsilon^4}  \Bigparenth{p \log \frac{p^2}{\delta \varepsilon^2} + \tmetric_{\ExternalField,n}(\varepsilon, \delta)}, \label{eq_sufficienct_condition_thm_node}
\end{align}
with probability at least $1-\delta$ where $R(\varepsilon, \delta)$ 
was defined in \cref{eq_radius_node_thm} and $\tmetric_{\ExternalField,n}(\varepsilon, \delta)$ was defined in \cref{eq_tmetric_node_thm}.
Then, the guarantee in \cref{theorem_parameters} follows by applying a union bound over all $i \in [n]$.\\

\noindent To that end, the lemma below, proven in \cref{sub_proof_lemma_parameter_single_external_field}, shows that for any fixed $\ExternalFieldI \in \ParameterSet_{\ExternalField}$, if $\ExternalFieldI$ is far from $\TrueExternalFieldI$, then with high probability $\loss^{(i)}\bigparenth{\ExternalFieldI}$ is significantly larger than $\loss^{(i)}\bigparenth{\TrueExternalFieldI}$.
\begin{lemma}[{\singleparameterseparation}]\label{lemma_parameter_single_external_field}
	Fix any $\varepsilon > 0$, $\delta \in (0,1)$, and $i \in [n]$. Then, for any $\ExternalFieldI \in \ParameterSet_{\ExternalField}$ such that $\stwonorm{\ExternalFieldI - \TrueExternalFieldI} \geq  \varepsilon \ratio$ (see \cref{eq_radius_node_thm}), we have
	\begin{align}
	\loss^{(i)}\bigparenth{\ExternalFieldI} \!\geq\!  \loss^{(i)}\bigparenth{\TrueExternalFieldI} \!+\! \frac{2^{2.5}\bGM \xmax^4}{\pi e \ctwo[5]} \stwonorm{\ExternalFieldI \!-\! \TrueExternalFieldI}^2 \stext{for}
	n \!\geq \! \frac{ce^{c'\bGM}p^4}{\varepsilon^4} \!\Bigparenth{\!p \log \frac{p^2}{\delta \varepsilon^2} \!+\!  \metric_{\ExternalField,n}\Bigparenth{ \frac{\varepsilon^2}{p}}\!},
	\end{align}
	with probability at least 
	$1-\delta - c\bGM^2  \log p \cdot \exp(-e^{-c'\bGM}\stwonorm{\ExternalFieldI - \TrueExternalFieldI}^2 ) 
	$
	where $\ctwo$ was defined in \cref{eq:constants}.
\end{lemma}
\noindent \textbf{Note.} When we invoke \cref{lemma_parameter_single_external_field}, we ensure that $c\bGM^2  \log p \cdot \exp(-e^{-c'\bGM}\stwonorm{\ExternalFieldI - \TrueExternalFieldI}^2 )$ is of the same order as $\delta$.\\

\noindent Next, we show that the loss function $\loss^{(i)}$ is Lipschitz (see \cref{sub:proof_lemma_lipschitzness} for the proof).
\begin{lemma}[{\lipschitznesslossfunction}]\label{lemma_lipschitzness}
	Consider any $i \in [n]$. Then, the loss function $\loss^{(i)}$ is Lipschitz with respect to the $\ell_1$ norm $\sonenorm{\cdot}$ and with Lipschitz constant $\xmax \ctwo$, i.e.,
	\begin{align}
	\bigabs{\loss^{(i)}\bigparenth{\tExternalFieldI} - \loss^{(i)}\bigparenth{\ExternalFieldI}} \leq \xmax \ctwo \sonenorm{\tExternalFieldI - \ExternalFieldI} \qtext{for all} \ExternalFieldI, \tExternalFieldI \in \ParameterSet_{\ExternalField}, \label{eq_lipschitz_property}
	\end{align}
	where the constant $\ctwo$ was defined in \cref{eq:constants}.
\end{lemma}

\noindent Given these lemmas, we now proceed with the proof. 

\paragraph{Proof strategy} We want to show that all points $\ExternalFieldI \in \ParameterSet_{\ExternalField}$, that satisfy $\stwonorm{\ExternalFieldI - \TrueExternalFieldI} \geq R(\varepsilon, \delta)$, uniformly satisfy \cref{eq_sufficienct_condition_thm_node} with probability at least $1-\delta$. To do so, we consider the set of points $\ParameterSet_{\ExternalField}^{r} \subset \ParameterSet_{\ExternalField}$ whose distance from $\TrueExternalFieldI$ is at least $r > 0$ in $\ell_2$ norm. Then, using an appropriate covering set of $\ParameterSet_{\ExternalField}^{r}$ and the Lipschitzness of $\loss^{(i)}$, we show that the value of $\loss^{(i)}$ at all points in $\ParameterSet_{\ExternalField}^{r}$ is uniformly $\Omega(r^2)$ larger than the value of $\loss^{(i)}$ at $\TrueExternalFieldI$ with high probability. 
Finally, we choose $r$ small enough to make the failure probability smaller than $\delta$.

\paragraph{Arguments for points in the covering set} Consider any $r \geq \varepsilon \ratio$ (where $\ratio$ is defined in \cref{eq_radius_node_thm}) and the set of elements $\ParameterSet_{\ExternalField}^{r} \defn \braces{ \ExternalFieldI \in  \ParameterSet_{\ExternalField}: \stwonorm{\TrueExternalFieldI - \ExternalFieldI} \geq r}$. 
Let $\cU(\ParameterSet_{\ExternalField}^{r}, \varepsilon')$ be the $\varepsilon'$-cover of the  smallest size for the set $\ParameterSet_{\ExternalField}^{r}$ with respect to $\sonenorm{\cdot}$ (see \cref{def_covering_number_metric_entropy}) and let
$\cC(\ParameterSet_{\ExternalField}^{r}, \varepsilon')$ be the $\varepsilon'$-covering number
where 
\begin{align}
\varepsilon' \defn \frac{2\sqrt{2}  \bGM \xmax^3 r^2}{ \pi e \ctwo[6]}. \label{eq_eps'}
\end{align}
Now, we argue by a union bound that the value of $\loss^{(i)}$ at all points in $\cU(\ParameterSet_{\ExternalField}^{r}, \varepsilon')$ is uniformly $\Omega(r^2)$ larger than $\loss^{(i)}(\TrueExternalFieldI)$ with high probability. 
For any $\ExternalFieldI \in \cU(\ParameterSet_{\ExternalField}^{r}, \varepsilon')$, we have
\begin{align}
\stwonorm{\TrueExternalFieldI - \ExternalFieldI} \sgreat{(a)} r,
\label{eq_lower_bound_two_norm_theta_diff}
\end{align}
where $(a)$ follows because $\cU(\ParameterSet_{\ExternalField}^{r}, \varepsilon') \subseteq \ParameterSet_{\ExternalField}^{r}$. 
Now, applying \cref{lemma_parameter_single_external_field} with $\varepsilon \mapsfrom \varepsilon$ and $\delta \mapsfrom \delta/2\cC(\ParameterSet_{\ExternalField}^{r}, \varepsilon')$, we have
\begin{align}
\loss^{(i)}\bigparenth{\ExternalFieldI} \geq \loss^{(i)}\bigparenth{\TrueExternalFieldI} + \frac{4\sqrt{2}\bGM \xmax^4}{\pi e \ctwo[5]}\stwonorm{\TrueExternalFieldI - \ExternalFieldI}^2 \sgreat{\cref{eq_lower_bound_two_norm_theta_diff}} \loss^{(i)}\bigparenth{\TrueExternalFieldI} + \frac{4\sqrt{2}\bGM \xmax^4 r^2}{\pi e \ctwo[5]},
\end{align}
with probability at least $1-\delta/2\cC(\ParameterSet_{\ExternalField}^{r}, \varepsilon') - c\bGM^2  \log p \cdot \exp(-e^{-c'\bGM}\stwonorm{\ExternalFieldI - \TrueExternalFieldI}^2 )$ 
whenever
\begin{align}
n \geq \frac{ce^{c'\bGM}p^4}{\varepsilon^4} \Bigparenth{p \log \frac{\cC(\ParameterSet_{\ExternalField}^{r}, \varepsilon') \cdot p^2}{\delta \varepsilon^2} + \metric_{\ExternalField,n}\Bigparenth{ \frac{\varepsilon^2}{p}}}.
\label{eq_n_condition_node_recovery}
\end{align}
By applying the union bound over $\cU(\ParameterSet_{\ExternalField}^{r}, \varepsilon')$, as long as $n$ satisfies \cref{eq_n_condition_node_recovery},
we have
\begin{align}
\loss^{(i)}\bigparenth{\ExternalFieldI} \geq  \loss^{(i)}\bigparenth{\TrueExternalFieldI} + \frac{4\sqrt{2}\bGM \xmax^4 r^2}{\pi e \ctwo[5]} \stext{uniformly for every} \ExternalFieldI \in \cU(\ParameterSet_{\ExternalField}^{r}, \varepsilon'), \label{eq_union_bound_covering_set} 
\end{align}
with probability at least $1-\delta/2 - c\bGM^2 \cC(\ParameterSet_{\ExternalField}^{r}, \varepsilon') \log p \cdot \exp(-e^{-c'\bGM}\stwonorm{\ExternalFieldI - \TrueExternalFieldI}^2)$ which can lower bounded by $1-\delta/2 - c\bGM^2 \cC(\ParameterSet_{\ExternalField}^{r}, \varepsilon') \log p \cdot \exp(-e^{-c'\bGM}r^2)$ using \cref{eq_lower_bound_two_norm_theta_diff}.

\paragraph{Arguments for points outside the covering set} Next, we establish the claim \cref{eq_sufficienct_condition_thm_node} for an arbitrary $\tExternalFieldI \in \ParameterSet_{\ExternalField}^{r}$ conditional on the event that \cref{eq_union_bound_covering_set} holds.
Given a fixed $\tExternalFieldI \in \ParameterSet_{\ExternalField}^{r}$, let $\ExternalFieldI$ be (one of) the point(s) in the $\cU(\ParameterSet_{\ExternalField}^{r}, \varepsilon')$ that satisfies $\sonenorm{\ExternalFieldI - \tExternalFieldI} \leq \varepsilon'$ (there exists such a point by \cref{def_covering_number_metric_entropy})
Then, the choices \cref{eq_eps'} and \cref{lemma_lipschitzness} put together imply that
\begin{align}
\loss^{(i)}\bigparenth{\tExternalFieldI} \!\geq\! \loss^{(i)}\bigparenth{\ExternalFieldI} \!-\! \xmax \ctwo \sonenorm{\ExternalFieldI \!-\! \tExternalFieldI} \!\! 
& \geq \loss^{(i)}\bigparenth{\ExternalFieldI} - \xmax \ctwo  \varepsilon' \\
& \sgreat{\cref{eq_eps'}} \loss^{(i)}\bigparenth{\ExternalFieldI} \!-\!  \frac{2\sqrt{2} \bGM \xmax^4 r^2}{ \pi e \ctwo[5]}\\ 
& \sgreat{\cref{eq_union_bound_covering_set}} \loss^{(i)}\bigparenth{\TrueExternalFieldI} \!+\! \frac{2\sqrt{2} \bGM \xmax^4 r^2}{\pi e \ctwo[5]}, 
\end{align}
It remains to bound sample size $n$ and the failure probability $\delta$. 

\paragraph{Bounding $n$} 
Using $\ParameterSet_{\ExternalField}^{r} \subseteq \ParameterSet_{\ExternalField}$, we find that
\begin{align}
\cC(\ParameterSet_{\ExternalField}^{r}, \varepsilon') \sless{(a)} \cC(\ParameterSet_{\ExternalField}, \varepsilon').
\label{eq_bound_covering_number}
\end{align}
Putting together \cref{eq_eps'} and \cref{eq_bound_covering_number}, the lower bound  \cref{eq_n_condition_node_recovery} can be replaced by
\begin{align}
n \geq \frac{ce^{c'\bGM}p^4}{\varepsilon^4}  \Bigparenth{p \log \frac{p^2}{\delta \varepsilon^2}+ p\metric_{\ExternalField}\bigparenth{r^2}  + \metric_{\ExternalField,n}\Bigparenth{ \frac{\varepsilon^2}{p}} }.
\end{align}

\paragraph{Bounding $\delta$} To bound the failure probability by $\delta$, it is sufficient to chose $r$ such that
\begin{align}
\delta & \geq \delta/2 + c\bGM^2 \cC(\ParameterSet_{\ExternalField}^{r}, \varepsilon') \log p \cdot \exp(-e^{-c'\bGM}r^2).
\label{eq_failure_prob}
\end{align}
From \cref{eq_bound_covering_number} and \cref{eq_failure_prob}, it is sufficient to chose $r$ such that 
\begin{align}
\delta & \geq \delta/2 + c\bGM^2 \cC(\ParameterSet_{\ExternalField}, \varepsilon') \log p \cdot \exp(-e^{-c'\bGM}r^2).\label{eq_failure_prob_2}
\end{align}
Re-arranging and taking logarithm on both sides of \cref{eq_failure_prob_2} and using \cref{eq_eps'}, we have
\begin{align}
\log \delta \geq c\biggbrackets{\log \bigparenth{\bGM^2 \log p} +  \metric_{\ExternalField}\Bigparenth{\frac{r^2}{ce^{c'\bGM}}} - e^{-c'\bGM} r^2}.
\label{eq_failure_prob_log}
\end{align}
Finally, \cref{eq_failure_prob_log} holds whenever
\begin{align}
r \geq ce^{c'\bGM} \sqrt{\log \dfrac{\bGM^2 \log p}{\delta} + \metric_{\ExternalField}(c e^{-c'\bGM})}.
\end{align}
Recalling that the choice of $r$ was such that $r \geq \varepsilon \ratio$ completes the proof.

\newcommand{\expressionfirstderstagetwo}{Expression for first directional derivative}

\subsection{Proof of \cref{lemma_parameter_single_external_field}: \singleparameterseparation}
\label{sub_proof_lemma_parameter_single_external_field}
Fix any $\varepsilon > 0$, any $\delta \in (0,1)$, and any $i \in [n]$. Consider any direction $\omI \in \Reals^{p}$ along the parameter $\ExternalFieldI$, i.e.,
\begin{align}
\omI = \ExternalFieldI - \TrueExternalFieldI. \label{eq:omega_defn_stage_2}
\end{align}
We denote the first-order and the second-order directional derivatives of the loss function $\loss^{(i)}$ in \cref{eq:loss_function_n} along the direction $\omI$ evaluated at $\ExternalFieldI$ by $\directionalGradientExternalField$ and $\directionalHessianExternalField$, respectively. 
Below, we state a lemma (with proof divided across \cref{sub:proof_of_lemma_conc_grad_stage_2} and \cref{sub:proof_of_lemma_lower_bound_sec_der_stage_2}) that provides us a control on $\directionalGradientExternalFieldTrue$ and $\directionalHessianExternalField$. The assumptions of \cref{lemma_parameter_single_external_field} remain in force.

\newcommand{\conclocalresultname}{Control on first and second directional derivatives}
\newcommand{\concgradstagetwo}{Concentration of first directional derivative}
\newcommand{\anticoncgradstagetwo}{Anti-concentration of second directional derivative}

\begin{lemma}[{\conclocalresultname}]\label{lemma_conc_first_sec_der_stage_two}
	For any fixed $\varepsilon_1, \varepsilon_2 > 0$, $\delta_1 \in (0,1)$, $i \in [n]$, $\ExternalFieldI \in \ParameterSet_{\ExternalField}$ with $\omI$ defined in \cref{eq:omega_defn_stage_2}, we have the following:
	\begin{enumerate}[label=(\alph*)]
		\item\label{item_conc_first_der_stage_two} \textnormal{{\concgradstagetwo}}: We have
		\begin{align}
		\bigabs{\directionalGradientExternalFieldTrue} \leq\varepsilon_1 \sonenorm{\omI} \!+\! \varepsilon_2 \stwonorm{\omI}^2 \qtext{for} \!\! n \geq \dfrac{ce^{c'\bGM}  p^4 \bigparenth{p \log \frac{p^2}{\delta_1 \varepsilon_1^2} \!+\!  \metric_{\ExternalField,n}\bigparenth{ \frac{\varepsilon_1^2}{p}}}}{\varepsilon_1^4},
		\end{align}
		with probability at least $1-\delta_1 - O\biggparenth{\bGM^2 \log p \exp\biggparenth{\dfrac{-\varepsilon_2^2\stwonorm{\omI}^2}{e^{c'\bGM}}}}$.
		\item\label{item_conc_sec_der_stage_two} \textnormal{{\anticoncgradstagetwo}}: We have \begin{align}
		\directionalHessianExternalField \geq \frac{32\sqrt{2}\bGM \xmax^4}{\pi e \ctwo[5]} \stwonorm{\omI}^2,
		\end{align}
		with probability at least $1- O\biggparenth{\bGM^2 \log p \exp\biggparenth{\dfrac{-\stwonorm{\omI}^2}{e^{c'\bGM}}}}$ where $\ctwo$ was defined in \cref{eq:constants}.
	\end{enumerate}
\end{lemma}

\noindent Given this lemma, we now proceed with the proof.
Define a function $g : [0,1] \to \Reals^{p}$ as follows
\begin{align}
g(a) = \TrueExternalFieldI + a(\ExternalFieldI - \TrueExternalFieldI).
\end{align}
Notice that $g(0) = \TrueExternalFieldI$ and $g(1) = \ExternalFieldI$ as well as
\begin{align}
\dfrac{d\loss^{(i)}(g(a))}{da} = \tdirectionalGradientExternalField\bigr|_{\tExternalFieldI = g(a)} \qtext{and} \dfrac{d^2\loss^{(i)}(g(a))}{da^2} = \tdirectionalHessianExternalField\bigr|_{\tExternalFieldI = g(a)}. \label{eq_der_mapping_external_field}
\end{align}
By the fundamental theorem of calculus, we have
\begin{align}
\dfrac{d\loss^{(i)}(g(a))}{da} \geq \dfrac{d\loss^{(i)}(g(a))}{da}\bigr|_{a = 0} + a \min_{a \in (0,1)}\dfrac{d^2\loss^{(i)}(g(a))}{da^2}. \label{eq_fundamental_external_field}
\end{align}
Integrating both sides of \cref{eq_fundamental_external_field} with respect to $a$, we obtain
\begin{align}
\loss^{(i)}(g(a)) \!-\! \loss^{(i)}(g(0)) & \geq  a \dfrac{d\loss^{(i)}(g(a))}{da}\bigr|_{a = 0} +  \dfrac{a^2}{2} \min_{a \in (0,1)}\dfrac{d^2\loss^{(i)}(g(a))}{da^2}\\
& \sequal{\cref{eq_der_mapping_external_field}} a\tdirectionalGradientExternalField\bigr|_{\tExternalFieldI = g(0)} \!+\!  \dfrac{a^2}{2} \min_{a \in (0,1)}\tdirectionalHessianExternalField\bigr|_{\tExternalFieldI = g(a)}\\
& \sequal{(a)} a\directionalGradientExternalFieldTrue \!+\!  \dfrac{a^2}{2} \min_{a \in (0,1)}\tdirectionalHessianExternalField\bigr|_{\tExternalFieldI = g(a)}\\
& \sgreat{(b)} - a \bigabs{\directionalGradientExternalFieldTrue} \!+\!  \dfrac{a^2}{2} \min_{a \in (0,1)}\tdirectionalHessianExternalField\bigr|_{\tExternalFieldI = g(a)},
\label{eq_taylor_expansion_external_field}
\end{align}
where $(a)$ follows because $g(0) = \TrueExternalFieldI$, and $(b)$ follows by the triangle inequality.
Plugging in $a = 1$ in \cref{eq_taylor_expansion_external_field} as well as using $g(0) = \TrueExternalFieldI$ and $g(1) = \ExternalFieldI$, we find that
\begin{align}
\loss^{(i)}(\ExternalFieldI) - \loss^{(i)}(\TrueExternalFieldI) & \geq -  \bigabs{\directionalGradientExternalFieldTrue} +  \dfrac{1}{2} \min_{a \in (0,1)}\tdirectionalHessianExternalField\bigr|_{\tExternalFieldI = g(a)}. \label{eq_loss_separation_single_par_inter}
\end{align}
Now, we use \cref{lemma_conc_first_sec_der_stage_two} with $\varepsilon_1 \mapsfrom 4\sqrt{2}\bGM \xmax^4 \varepsilon/\pi e \ctwo[5]$, $\varepsilon_2 \mapsfrom 8\sqrt{2}\bGM \xmax^4/\pi e \ctwo[5]$, and $\delta_1 \mapsfrom \delta$. Therefore, with probability at least $1-\delta - O\biggparenth{\bGM^2 \log p \exp\biggparenth{\dfrac{-\stwonorm{\omI}^2}{e^{c'\bGM}}}}$ and as long as $n \geq O\biggparenth{\dfrac{e^{c'\bGM} p^4 \bigparenth{p \log \frac{p^2}{\delta} +  \metric_{\ExternalField,n}\bigparenth{ \frac{\varepsilon^2}{p}}}}{\varepsilon^4}}$, we have
\begin{align}
\loss^{(i)}(\ExternalFieldI) \!-\! \loss^{(i)}(\TrueExternalFieldI) & \! \geq \!- \frac{2^{2.5}\bGM \xmax^4\varepsilon}{\pi e \ctwo[5]} \sonenorm{\omI} \!\!-\! \frac{2^{3.5}\bGM \xmax^4}{\pi e \ctwo[5]} \stwonorm{\omI}^2 \!+\!  \frac{2^{4.5}\bGM \xmax^4}{\pi e \ctwo[5]} \stwonorm{\omI}^2 \\
& \!\! = - \frac{2^{2.5}\bGM \xmax^4\varepsilon}{\pi e \ctwo[5]} \sonenorm{\omI} + \frac{2^{3.5}\bGM \xmax^4}{\pi e \ctwo[5]} \stwonorm{\omI}^2 \\
& \!\! \sgreat{\cref{eq_radius_node_thm}} -  \frac{2^{2.5}\bGM \xmax^4\varepsilon \ratio}{\pi e \ctwo[5]} \stwonorm{\omI} + \frac{2^{3.5}\bGM \xmax^4}{\pi e \ctwo[5]} \stwonorm{\omI}^2 \\
& \!\! \sgreat{(a)} \!-\!  \frac{2^{2.5}\bGM \xmax^4}{\pi e \ctwo[5]} \stwonorm{\omI}^2 \!+\! \frac{2^{3.5}\bGM \xmax^4}{\pi e \ctwo[5]} \stwonorm{\omI}^2 \!=\! \frac{2^{2.5}\bGM \xmax^4}{\pi e \ctwo[5]} \stwonorm{\omI}^2,
\end{align}
where $(a)$ follows because $\stwonorm{\omI} = \stwonorm{\ExternalFieldI - \TrueExternalFieldI} \geq  \varepsilon \ratio$ according to the lemma statement.

\subsubsection{Proof of \cref{lemma_conc_first_sec_der_stage_two}\cref{item_conc_first_der_stage_two}: \concgradstagetwo}\label{sub:proof_of_lemma_conc_grad_stage_2}
Fix some $i \in [n]$ and some $\ExternalFieldI \in \ParameterSet_{\ExternalField}$. Let $\omI$ be as defined in \cref{eq:omega_defn_stage_2}. We claim that the first-order directional derivative of $\loss^{(i)}$ defined in \cref{eq:loss_function_n} is given by 
\begin{align}
\directionalGradientExternalField = - \sump \omIt x_t^{(i)} \exp\Bigparenth{-\normalbrackets{\ExternalFieldtI + 2\EstimatedParameterRowttt\tp \svbx_{-t}^{(i)}} x_t^{(i)} - \EstimatedParameterTU[tt] \cx_t^{(i)}}. \label{eq:first_dir_derivative_stage_2}
\end{align}
We provide a proof at the end. For now, we assume the claim and proceed.

We note that the pair $\braces{\rvbx, \rvbz}$ corresponds to a $\tSGM$ (see \cref{def:tau_sgm}) with $\dGM \defn (\aGM,  \bGM, \xmax, \ParameterMatrix)$. To show the concentration, we use \cref{lemma_conditioning_trick} (see \cref{sec_conditioning_trick}) with $\lambda = \frac{1}{4\sqrt{2}\xmax^2}$, decompose $\directionalGradientExternalFieldTrue$ as a sum of  $\numindsets = 1024 \bGM^2 \xmax^4 \log 4p$, and focus on these  $\numindsets$ terms. 
Consider the $\numindsets$ subsets $\sets \in [p]$ obtained from \cref{lemma_conditioning_trick} with $\lambda = \frac{1}{4\sqrt{2}\xmax^2}$ and define
\begin{align}
\psi_u(\svbx^{(i)}; \omI) \defeq \sum_{t \in \setU} \omIt x_t^{(i)} \exp\Bigparenth{\!-\!\normalbrackets{\TrueExternalFieldtI \!+\! 2\EstimatedParameterRowttt\tp \svbx_{-t}^{(i)}} x_t^{(i)} \!-\! \EstimatedParameterTU[tt] \cx_t^{(i)}} \stext{for every} u \in \numindsets. \label{eq_def_psi}
\end{align}
Now, we decompose $\directionalGradientExternalFieldTrue$ as a sum of the $\numindsets$ terms defined above. More precisely, we have
\begin{align}
\directionalGradientExternalFieldTrue & \sequal{\cref{eq:first_dir_derivative_stage_2}} -\sump \omIt x_t^{(i)} \exp\Bigparenth{-\normalbrackets{\TrueExternalFieldtI + 2\EstimatedParameterRowttt\tp \svbx_{-t}^{(i)}} x_t^{(i)} - \EstimatedParameterTU[tt] \cx_t^{(i)}} \\
& \sequal{(a)} - \frac{1}{\numindsets'} \sum_{u \in [\numindsets]} \sum_{t \in \setU} \omIt x_t^{(i)} \exp\Bigparenth{-\normalbrackets{\TrueExternalFieldtI + 2\EstimatedParameterRowttt\tp \svbx_{-t}^{(i)}} x_t^{(i)} - \EstimatedParameterTU[tt] \cx_t^{(i)}} \\
& \sequal{\cref{eq_def_psi}} - \frac{1}{\numindsets'} \sum_{u \in [\numindsets]} \psi_u(\svbx^{(i)}; \omI), \label{eq_first_order_derivative_expressed_via_psi}
\end{align}
where $(a)$ follows because each $t \in [p]$ appears in exactly $\numindsets' = \lceil \numindsets/32\sqrt{2}\bGM \xmax^2 \rceil$ of the sets $\sets$ according to \cref{lemma_conditioning_trick}\cref{item:cardinality_independence_set} (with $\lambda = \frac{1}{4\sqrt{2}\xmax^2}$). Now, we focus on the $\numindsets$ terms in \cref{eq_first_order_derivative_expressed_via_psi}.\\

\noindent Consider any $u \in [\numindsets]$. We claim that conditioned on $\svbx_{-\setU}^{(i)}$ and $\svbz^{(i)}$, the expected value of $\psi_u(\svbx^{(i)}; \omI)$ can be upper bounded uniformly across all $u \in [\numindsets]$. We provide a proof at the end.

\newcommand{\upperboundpsi}{Upper bound on expected $\psi_u$}
\begin{lemma}[{\upperboundpsi}]\label{lemma_expected_psi_upper_bound}
	Fix $\varepsilon > 0$, $\delta \in (0,1)$, $i \in [n]$ and $\ExternalFieldI \in \ParameterSet_{\ExternalField}$. Then, with $\omI$ defined in \cref{eq:omega_defn_stage_2}
	and given $\svbz^{(i)}$ and $\svbx_{-\setU}^{(i)}$ for all $u \in [\numindsets]$, 
	we have
	\begin{align}
	\max\limits_{u \in [\numindsets]} \Expectation\Bigbrackets{\psi_u(\svbx^{(i)}; \omI) \bigm\vert \svbx_{-\setU}^{(i)}, \svbz^{(i)}} \leq \varepsilon \sonenorm{\omI} \!\!\qtext{for}\!\! n \geq \dfrac{c e^{c'\bGM} p^4 \bigparenth{p \log \frac{p^2}{\delta \varepsilon^2} \!+\!  \metric_{\ExternalField,n}\bigparenth{ \frac{\varepsilon^2}{p}}}}{\varepsilon^4},
	\end{align}
	with probability at least $1-\delta$.
\end{lemma}

\noindent Consider again any $u \in [\numindsets]$. Now, we claim that conditioned on $\svbx_{-\setU}^{(i)}$ and $\svbz^{(i)}$, $\psi_u(\svbx^{(i)}; \omI)$ concentrates around its conditional expected value. We provide a proof at the end.

\newcommand{\concpsi}{Concentration of $\psi_u$}

\begin{lemma}[{\concpsi}]\label{lemma_concentration_psi}
	Fix $\varepsilon > 0$, $i \in [n]$, $u \in [\numindsets]$, and $\ExternalFieldI \in \ParameterSet_{\ExternalField}$. Then, with $\omI$ defined in \cref{eq:omega_defn_stage_2} and given $\svbz^{(i)}$ and $\svbx_{-\setU}^{(i)}$, we have
	\begin{align}
	\Bigabs{\psi_u(\svbx^{(i)}; \omI) - \Expectation\bigbrackets{\psi_u(\svbx^{(i)}; \omI) \bigm\vert \svbx_{-\setU}^{(i)}, \svbz^{(i)}}} \leq \varepsilon,
	\end{align}
	with probability at least $1-\exp\biggparenth{ \dfrac{- \varepsilon^2}{e^{c'\bGM}\stwonorm{\omI}^2}}$. 
\end{lemma}

\noindent Given these lemmas, we proceed to show the concentration of $\directionalGradientExternalFieldTrue$.
To that end, for any $u \in [\numindsets]$, given $\svbx_{-\setU}^{(i)}$ and $\svbz^{(i)}$, let $E_u$ denote the event that
\begin{align}
\psi_u(\svbx^{(i)}; \omI) \leq   \Expectation\bigbrackets{\psi_u(\svbx^{(i)}; \omI) \vert \svbx_{-\setU}^{(i)}, \svbz^{(i)}} + \frac{1}{32\sqrt{2}\bGM \xmax^2} \varepsilon_2 \stwonorm{\omI}^2. \label{eq_event_Ej}
\end{align}
Since $E_u$ in an indicator event, using the law of total expectation results in
\begin{align}
\Probability(E_u) = \Expectation\Bigbrackets{\Probability(E_u | \svbx_{-\setU}^{(i)}, \svbz^{(i)})} \sgreat{(a)} 1 - \exp\biggparenth{ \dfrac{- \varepsilon_2^2\stwonorm{\omI}^2}{e^{c'\bGM}}}.
\end{align}
where $(a)$ follows from \cref{lemma_concentration_psi} with $\varepsilon \mapsfrom \dfrac{\varepsilon_2\stwonorm{\omI}^2}{32\sqrt{2}\bGM \xmax^2}$. Now, by applying the union bound over all $u \in [\numindsets]$ where $\numindsets = 1024 \bGM^2 \xmax^4 \log 4p$, we have
\begin{align}
\Probability\Bigparenth{\bigcap_{u \in \numindsets} E_u} \geq 1 - O\biggparenth{\bGM^2 \log p \exp\biggparenth{\dfrac{-\varepsilon_2^2\stwonorm{\omI}^2}{e^{c'\bGM}}}}.
\end{align}
Now, assume the event $\cap_{u \in \numindsets} E_u$ holds. Whenever this holds, we also have
\begin{align}
\bigabs{\directionalGradientExternalFieldTrue} & \sless{\cref{eq_first_order_derivative_expressed_via_psi}} \frac{1}{\numindsets'} \sum_{u \in [\numindsets]} \bigabs{\psi_u(\svbx^{(i)}; \omI)}\\ & \sless{\cref{eq_event_Ej}} \frac{1}{\numindsets'} \sum_{u \in [\numindsets]} \Bigabs{\Expectation\bigbrackets{\psi_u(\svbx^{(i)}; \omI) \vert \svbx_{-\setU}^{(i)}, \svbz^{(i)}} + \frac{1}{32\sqrt{2}\bGM \xmax^2} \varepsilon_2 \stwonorm{\omI}^2}, \label{eq_grad_intermediate_bound}
\end{align}
where $\numindsets' = \lceil \numindsets/32\sqrt{2}\bGM \xmax^2 \rceil$. Further, using \cref{lemma_expected_psi_upper_bound} in \cref{eq_grad_intermediate_bound} with $\varepsilon \mapsfrom \dfrac{\varepsilon_1}{32\sqrt{2}\bGM \xmax^2}$ and $\delta \mapsfrom \delta_1$, whenever
\begin{align}
n \geq \dfrac{ce^{c'\bGM} \cdot p^4 \bigparenth{p \log \frac{p^2}{\delta_1 \varepsilon_1^2} +  \metric_{\ExternalField,n}\bigparenth{ \frac{\varepsilon_1^2}{p}}}}{\varepsilon_1^4},
\end{align}
with probability at least $1-\delta_1$, we have,
\begin{align}
\bigabs{\directionalGradientExternalFieldTrue} & \leq \frac{1}{\numindsets'} \sum_{u \in [\numindsets]} \Bigparenth{\frac{1}{32\sqrt{2}\bGM \xmax^2} \varepsilon_1 \sonenorm{\omI} + \frac{1}{32\sqrt{2}\bGM \xmax^2} \varepsilon_2 \stwonorm{\omI}^2} \\
& = \frac{\numindsets}{32 \sqrt{2}\bGM \xmax^2 \numindsets'} \Bigparenth{\varepsilon_1 \sonenorm{\omI} \!+\! \varepsilon_2 \stwonorm{\omI}^2} \sless{(a)} \varepsilon_1 \sonenorm{\omI} \!+\! \varepsilon_2 \stwonorm{\omI}^2,
\end{align}
where $(a)$ follows because $\numindsets' = \lceil \numindsets/32 \sqrt{2}\bGM \xmax^2 \rceil$.

\paragraph{Proof of \cref{eq:first_dir_derivative_stage_2}: \expressionfirstderstagetwo}
Fix any $i \in [n]$. The first-order partial derivatives of $\loss^{(i)}$ (defined in \cref{eq:loss_function_n}) with respect to the entries of the parameter vector $\ExternalFieldI$ are given by
\begin{align}
\frac{\partial \loss^{(i)}(\ExternalFieldI)}{\partial \ExternalFieldIt} & = -x_t^{(i)} \exp\Bigparenth{-\normalbrackets{\ExternalFieldtI + 2\EstimatedParameterRowttt\tp \svbx_{-t}^{(i)}} x_t^{(i)} - \EstimatedParameterTU[tt] \cx_t^{(i)}} \qtext{for all} t \in [p].
\end{align}
Now, we can write the first-order directional derivative of $\loss^{(i)}$ as
\begin{align}
\directionalGradientExternalField &\defeq\lim_{h\to 0}\frac{\loss^{(i)}(\ExternalFieldI + h \omI)-\loss^{(i)}(\ExternalFieldI)}{h} = \sump \omIt \frac{\partial \loss^{(i)}(\ExternalFieldI)}{\partial \ExternalFieldIt} \\
& = - \sump \omIt x_t^{(i)} \exp\Bigparenth{-\normalbrackets{\ExternalFieldtI + 2\EstimatedParameterRowttt\tp \svbx_{-t}^{(i)}} x_t^{(i)} - \EstimatedParameterTU[tt] \cx_t^{(i)}}.
\end{align}

\paragraph{Proof of \cref{lemma_expected_psi_upper_bound}: \upperboundpsi}
\label{sub:proof_of_lemma_expected_psi_upper_bound}
Fix any $i \in [n]$, $u \in [\numindsets]$, and $\ExternalFieldI \in \ParameterSet_{\ExternalField}$. Then, given $\svbx_{-\setU}^{(i)}$ and $\svbz^{(i)}$, we have
\begin{align}
& \Expectation\biggbrackets{\psi_u(\svbx^{(i)}; \omI) \bigm\vert \svbx_{-\setU}^{(i)}, \svbz^{(i)}} \\
&  \sequal{(a)} \Expectation\Bigbrackets{\sum_{t \in \setU} \omIt x_t^{(i)} \exp\Bigparenth{-\normalbrackets{\TrueExternalFieldtI + 2\EstimatedParameterRowttt\tp \svbx_{-t}^{(i)}} x_t^{(i)} - \EstimatedParameterTU[tt] \cx_t^{(i)}} \Bigm\vert \svbx_{-\setU}^{(i)}, \svbz^{(i)}} \\
& \sequal{(b)} \sum_{t \in \setU} \omIt \Expectation\Bigbrackets{x_t^{(i)} \exp\Bigparenth{-\normalbrackets{\TrueExternalFieldtI + 2\EstimatedParameterRowttt\tp \svbx_{-t}^{(i)}} x_t^{(i)} - \EstimatedParameterTU[tt] \cx_t^{(i)}} \Bigm\vert \svbx_{-\setU}^{(i)}, \svbz^{(i)}} \\
& \sequal{(c)} \sum_{t \in \setU} \omIt \Expectation\biggbrackets{\Expectation\Bigbrackets{x_t^{(i)} \!\exp\!\bigparenth{\!-\!\normalbrackets{\TrueExternalFieldtI \!\!+\! 2\EstimatedParameterRowttt\tp \svbx_{-t}^{(i)}} x_t^{(i)} \!\!-\! \EstimatedParameterTU[tt] \cx_t^{(i)}} \!\Bigm\vert \! \svbx_{-t}^{(i)}, \svbz^{(i)}\!} \! \biggm\vert \svbx_{-\setU}^{(i)}, \svbz^{(i)}},
\label{eq_conditional_expectation_psi_intermediate}
\end{align}
where $(a)$ follows from the definition of $\psi_u(\svbx^{(i)}; \omI)$ in \cref{eq_def_psi}, $(b)$ follows from linearity of expectation, and $(c)$ follows from the law of total expectation, i.e., $\Expectation[\Expectation[Y|X,Z]|Z] = \Expectation[Y|Z]$ since $\svbx_{-\setU}^{(i)} \subseteq \svbx_{-t}^{(i)}$. Now, we bound $\Expectation\Bigbrackets{x_t^{(i)} \!\exp\!\bigparenth{\!\!-\!\normalbrackets{\TrueExternalFieldtI \!\!\!+\! 2\EstimatedParameterRowttt\tp \svbx_{-t}^{(i)}} x_t^{(i)} \!\!\!-\! \EstimatedParameterTU[tt] \cx_t^{(i)}} \!\bigm\vert \! \svbx_{-t}^{(i)}, \svbz^{(i)}\!} \!$ for every $t \in \setU$. We have
\begin{align}
&  \Expectation\Bigbrackets{x_t^{(i)} \!\exp\!\bigparenth{\!-\!\normalbrackets{\TrueExternalFieldtI \!\!+\! 2\EstimatedParameterRowttt\tp \svbx_{-t}^{(i)}} x_t^{(i)} \!\!-\! \EstimatedParameterTU[tt] \cx_t^{(i)}} \!\Bigm\vert \! \svbx_{-t}^{(i)}, \svbz^{(i)}\!} \\
& \!=\!\! \int\limits_{\cX} \!\! x_t^{(i)} \!\!\exp\!\bigparenth{\!\!-\!\normalbrackets{\TrueExternalFieldtI \!\!\!+\! 2\EstimatedParameterRowttt\tp \svbx_{\!-t}^{(i)}} x_t^{(i)} \!\!\!-\! \EstimatedParameterTU[tt] \cx_t^{(i)}}  f_{\rvx_t|\rvbx_{\!-t}, \rvbz}\bigparenth{x_t^{(i)} | \svbx_{\!-t}^{(i)}, \svbz^{(i)}\!;  \TrueExternalFieldt(\!\svbz^{(i)}\!), \TrueParameterRowt} d x_t^{(i)}\\
& \!\sequal{(a)}\! \frac{\int_{\cX} x_t^{(i)}  \exp\bigparenth{2\normalbrackets{\TrueParameterRowttt - \EstimatedParameterRowttt}\tp \svbx_{-t}^{(i)} x_t^{(i)} + \normalbrackets{\TrueParameterTU[tt] - \EstimatedParameterTU[tt]} \cx_t^{(i)}} dx_t^{(i)}} {\int_{\cX} \exp\Bigparenth{\normalbrackets{\TrueExternalFieldt(\svbz^{(i)}) + 2\TrueParameterRowtttTop \svbx_{-t}^{(i)}} x_t^{(i)} + \TrueParameterTU[tt] \cx_t^{(i)}}d x_t^{(i)}}\\
& \!\sequal{(b)}\! \frac{\int_{\cX} x_t^{(i)} \Bigbrackets{ 1 + 2\normalbrackets{\TrueParameterRowttt \!-\! \EstimatedParameterRowttt}\tp \svbx_{-t}^{(i)} x_t^{(i)} + \normalbrackets{\TrueParameterTU[tt] \!-\! \EstimatedParameterTU[tt]} \cx_t^{(i)}} dx_t^{(i)}} {\int_{\cX} \exp\Bigparenth{\normalbrackets{\TrueExternalFieldt(\svbz^{(i)}) + 2\TrueParameterRowtttTop \svbx_{-t}^{(i)}} x_t^{(i)} + \TrueParameterTU[tt] \cx_t^{(i)}}d x_t^{(i)}}\\
& \qquad \qquad \qquad \qquad  + \frac{\int_{\cX} x_t^{(i)} \Bigbrackets{o\Bigparenth{\normalbrackets{\TrueParameterRowttt \!-\! \EstimatedParameterRowttt}\tp \svbx_{-t}^{(i)} x_t^{(i)} + \normalbrackets{\TrueParameterTU[tt] \!-\! \EstimatedParameterTU[tt]} \cx_t^{(i)}}^2} dx_t^{(i)}} {\int_{\cX} \exp\Bigparenth{\normalbrackets{\TrueExternalFieldt(\svbz^{(i)}) + 2\TrueParameterRowtttTop \svbx_{-t}^{(i)}} x_t^{(i)} + \TrueParameterTU[tt] \cx_t^{(i)}}d x_t^{(i)}}\\
& \!\sequal{(c)}\! \frac{ 4\xmax^3 \normalbrackets{\TrueParameterRowttt - \EstimatedParameterRowttt}\tp \svbx_{-t}^{(i)} }{3\int_{\cX} \exp\Bigparenth{\normalbrackets{\TrueExternalFieldt(\svbz^{(i)}) + 2\TrueParameterRowtttTop \svbx_{-t}^{(i)}} x_t^{(i)} + \TrueParameterTU[tt] \cx_t^{(i)}}d x_t^{(i)}} \\
& \qquad \qquad \qquad  \qquad \!+\! \frac{ \xmax^5 \bigparenth{\normalbrackets{\TrueParameterRowttt - \EstimatedParameterRowttt}\tp \svbx_{-t}^{(i)}} \bigparenth{\TrueParameterTU[tt] \!-\! \EstimatedParameterTU[tt]} o(1) }{\int_{\cX} \exp\Bigparenth{\normalbrackets{\TrueExternalFieldt(\svbz^{(i)}) + 2\TrueParameterRowtttTop \svbx_{-t}^{(i)}} x_t^{(i)} + \TrueParameterTU[tt] \cx_t^{(i)}}d x_t^{(i)}}, \label{eq_taylor_series}
\end{align}
where $(a)$ follows from \cref{eq_conditional_dist} and $\TrueExternalFieldI = \TrueExternalField(\svbz^{(i)}) ~\forall i \in [n]$, $(b)$ follows by using the Taylor series expansion $\exp(y) = 1 + y + o(y^2)$ around zero, $(c)$ follows because $\int_{\cX} x_t^{(i)} dx_t^{(i)} = \int_{\cX} x_t^{(i)} \cx_t^{(i)} d x_t^{(i)} = \int_{\cX} \bigparenth{ x_t^{(i)}}^3 d x_t^{(i)} = \int_{\cX} x_t^{(i)} \bigparenth{\cx_t^{(i)}}^2 d x_t^{(i)} = 0$, $\int_{\cX} \bigparenth{ x_t^{(i)}}^2 d x_t^{(i)} = 2\xmax^3/3$, and $\int_{\cX} \bigparenth{ x_t^{(i)}}^2 \cx^{(i)}_t d x_t^{(i)} = 8\xmax^5/45$.

Now, we bound the numerators in \cref{eq_taylor_series} by using $\sonenorm{\TrueParameterRowt - \EstimatedParameterRowt} \leq \sqrt{p} \stwonorm{\TrueParameterRowt - \EstimatedParameterRowt}$. Then, we invoke \cref{theorem_parameters} to bound $\stwonorm{\TrueParameterRowt - \EstimatedParameterRowt}$ by $\varepsilon \mapsfrom \frac{3\varepsilon}{2 \ctwo \xmax^3 \sqrt{p}}$. Therefore, we subsume the second term by the first term resulting in the following bound:
\begin{align}
\Expectation\Bigbrackets{x_t^{(i)} \!\exp\!\bigparenth{\!\!-\!\normalbrackets{\TrueExternalFieldtI \!\!\!+\! 2\EstimatedParameterRowttt\tp \svbx_{-t}^{(i)}} x_t^{(i)} \!\!\!-\! \EstimatedParameterTU[tt] \cx_t^{(i)}} \!\bigm\vert \! \svbx_{-t}^{(i)}, \svbz^{(i)}\!} \! \leq \! \frac{2\ctwo  \xmax^3 \sqrt{p} \stwonorm{\TrueParameterRowt \!-\! \EstimatedParameterRowt}}{3}\!, \label{eq_expected_psi_bound_two_norm}
\end{align}
where we have used the triangle inequality, $\sinfnorm{\svbx^{(i)}} \leq \xmax $ for all $ i \in [n]$ as well as $\sonenorm{\TrueParameterRowt - \EstimatedParameterRowt} \leq \sqrt{p} \stwonorm{\TrueParameterRowt - \EstimatedParameterRowt}$ to upper bound the numerator, and the arguments used in the proof of \cref{prop_lower_bound_variance} as well as $\int_{\cX} dx_t^{(i)} = 2\xmax$ to lower bound the denominator.

\noindent Using \cref{theorem_parameters} in \cref{eq_expected_psi_bound_two_norm} with $\varepsilon \mapsfrom \dfrac{3\varepsilon}{2\ctwo \xmax^3 \sqrt{p} }$ and $\delta \mapsfrom \delta$, we have
\begin{align}
\Expectation\Bigbrackets{x_t^{(i)} \exp\Bigparenth{-\normalbrackets{\TrueExternalFieldtI + 2\EstimatedParameterRowttt\tp \svbx_{-t}^{(i)}} x_t^{(i)} - \EstimatedParameterTU[tt] \cx_t^{(i)}} \Bigm\vert \svbx_{-t}^{(i)}, \svbz^{(i)}} \leq \varepsilon, \label{eq_upper_bound_intermediate}
\end{align}
with probability at least $1-\delta$ as long as
\begin{align}
n \geq \dfrac{ce^{c'\bGM} \cdot p^4 \bigparenth{p \log \frac{p^2}{\delta \varepsilon^2} +  \metric_{\ExternalField,n}\bigparenth{ \frac{\varepsilon^2}{p}}}}{\varepsilon^4}. \label{eq_n_bound_stability}
\end{align}
Using \cref{eq_upper_bound_intermediate} and triangle inequality in \cref{eq_conditional_expectation_psi_intermediate}, we have
\begin{align}
\Expectation\Bigbrackets{\psi_u(\svbx^{(i)}; \omI)\bigm\vert \svbx_{-\setU}^{(i)}, \svbz^{(i)}} & \leq \varepsilon \sum_{t \in \setU} \bigabs{\omIt} \leq \varepsilon \sonenorm{\omI},
\end{align}
with probability at least $1-\delta$ as long as $n$ satisfies \cref{eq_n_bound_stability}.

\paragraph{Proof of \cref{lemma_concentration_psi}: \concpsi}
\label{sub:proof_of_lemma_concentration_psi}
To show this concentration result, we use \cref{coro}~\cref{eq_coro_combined} for the function $q_2$. To that end, we note that the pair $\braces{\rvbx, \rvbz}$ corresponds to a $\tSGM$ (\cref{def:tau_sgm}) with $\dGM \defn (\aGM, \bGM, \xmax, \ParameterMatrix)$. However, the random vector $\rvbx$ conditioned on $\rvbz$ need not satisfy the Dobrushin's uniqueness condition (\cref{def_dobrushin_condition}). Therefore, we cannot apply \cref{coro}~\cref{eq_coro_combined} as is. To resolve this, we resort to \cref{lemma_conditioning_trick} with $\lambda = \frac{1}{4\sqrt{2}\xmax^2}$ to reduce the random vector $\rvbx$ conditioned on $\rvbz$ to Dobrushin's regime.

Fix any $u \in [\numindsets]$. Then, from \cref{lemma_conditioning_trick}\cref{item:conditional_sgm_independence_set}, (i) the pair of random vectors $\braces{\rvbx_{\setU}, (\rvbx_{-\setU}, \rvbz)}$ corresponds to a $\tSGM[1]$ with $\dGM_1 \defn (\aGM+2\bGM\xmax, \frac{1}{4\sqrt{2}\xmax^2}, \xmax, \ParameterMatrix_{\setU})$, and (ii) the random vector $\rvbx_{\setU}$ conditioned on $(\rvbx_{-\setU}, \rvbz)$ satisfies the Dobrushin's uniqueness condition (\cref{def_dobrushin_condition}) with coupling matrix $2\sqrt{2} \xmax^2 \normalabs{\ParameterMatrix_{\setU}}$ with $2\sqrt{2} \xmax^2 \opnorm{\normalabs{\ParameterMatrix_{\setU}}} \leq 2\sqrt{2} \xmax^2 \lambda \leq 1/2$. Now, for any fixed $i \in [n]$, 
we apply \cref{coro}~\cref{eq_coro_combined} for the function $q_2$ with $\varepsilon \mapsfrom \varepsilon$ for a given $\svbx_{-\setU}^{(i)}$ and $\svbz^{(i)}$, to obtain
\begin{align}
\Probability\biggparenth{\Bigabs{\psi_u(\svbx^{(i)}; \omI) - \Expectation\Bigbrackets{\psi_u(\svbx^{(i)}; \omI) \Bigm\vert \svbx_{-\setU}^{(i)}, \rvbz}} \geq \varepsilon \Bigm\vert \svbx_{-\setU}^{(i)}, \rvbz} \leq \exp\biggparenth{ \dfrac{- \varepsilon^2}{e^{c'\bGM}\stwonorm{\omI}^2}}.
\end{align}

\subsubsection{Proof of \cref{lemma_conc_first_sec_der_stage_two}\cref{item_conc_sec_der_stage_two}: \anticoncgradstagetwo}
\label{sub:proof_of_lemma_lower_bound_sec_der_stage_2}

\newcommand{\expressionsecderstagetwo}{Expression for second directional derivative}
Fix some $i \in [n]$ and some $\ExternalFieldI \in \ParameterSet_{\ExternalField}$. Let $\omI$ be as defined in \cref{eq:omega_defn_stage_2}. We claim that the second-order directional derivative of $\loss^{(i)}$ defined in \cref{eq:loss_function_n} is given by 
\begin{align}
\directionalHessianExternalField = \sump \bigparenth{\omIt x_t^{(i)}}^2  \exp\Bigparenth{-\normalbrackets{\ExternalFieldtI + 2\EstimatedParameterRowttt\tp \svbx_{-t}^{(i)}} x_t^{(i)} - \EstimatedParameterTU[tt] \cx_t^{(i)}}. \label{eq:second_dir_derivative_stage_2}
\end{align}
We provide a proof at the end. For now, we assume the claim and proceed.
Now, we lower bound $\directionalHessianExternalField$ by a quadratic form as follows
\begin{align}
\directionalHessianExternalField 
& \! \sgreat{(a)} \sump   \bigparenth{\omIt x_t^{(i)}}^2 \times \exp\Bigparenth{\!-\!\bigparenth{\normalabs{\ExternalFieldtI} \!+\! 2\sonenorm{\EstimatedParameterRowt} \sinfnorm{\svbx^{(i)}}} \xmax}\\
&  \! \sgreat{(b)} \sump  \bigparenth{\omIt x_t^{(i)}}^2 \times \exp\Bigparenth{\!-\!\normalparenth{\aGM \!+\! 2 \bGM \xmax} \xmax} \sequal{\cref{eq:constants}} \frac{1}{\ctwo}\!\! \sump  \! \bigparenth{\omIt x_t^{(i)}}^2 \label{eq_lower_bound_sec_der_stage_2_cont}\!,
\end{align}
where $(a)$ follows from \cref{eq:second_dir_derivative_stage_2} by triangle inequality, Cauchy–Schwarz inequality, and because $\sinfnorm{\svbx^{(i)}} \leq \xmax $ for all $ i \in [n]$, and $(b)$ follows because $\EstimatedParameterMatrix \in \ParameterSet_{\ParameterMatrix}$, $\ExternalFieldI \in \ParameterSet_{\ExternalField}$, and $\sinfnorm{\svbx^{(i)}} \leq \xmax $ for all $ i \in [n]$.\\

\noindent Now, to show the anti-concentration of $\directionalHessianExternalField$, we show the anti-concentration of the quadratic form in \cref{eq_lower_bound_sec_der_stage_2_cont}. To that end, we note that the pair $\braces{\rvbx, \rvbz}$ corresponds to a $\tSGM$ (\cref{def:tau_sgm}) with $\dGM \defn (\aGM, \bGM, \xmax, \ParameterMatrix)$. Then, we decompose the quadratic form in \cref{eq_lower_bound_sec_der_stage_2_cont} as a sum of  $\numindsets = 1024  \bGM^2 \xmax^4 \log 4p$ terms using \cref{lemma_conditioning_trick} (see \cref{sec_conditioning_trick}) with $\lambda = \frac{1}{4\sqrt{2}\xmax^2}$ and focus on these  $\numindsets$ terms. Consider the $\numindsets$ subsets $\sets \in [p]$ obtained from \cref{lemma_conditioning_trick} and define
\begin{align}
\bpsi_u(\svbx^{(i)}; \omI) \defeq \sum_{t \in \setU} \bigparenth{\omIt x_t^{(i)}}^2 \qtext{for every} u \in \numindsets. \label{eq_def_bar_psi}
\end{align}
Then, we have
\begin{align}
\sump \bigparenth{\omIt x_t^{(i)}}^2 & \sequal{(a)} \frac{1}{\numindsets'} \sum_{u \in [\numindsets]} \sum_{t \in \setU} \bigparenth{\omIt x_t^{(i)}}^2 \sequal{\cref{eq_def_bar_psi}}  \frac{1}{\numindsets'} \sum_{u \in [\numindsets]} \bpsi_u(\svbx^{(i)}; \omI), \label{eq_second_order_derivative_expressed_via_psi}
\end{align}
where $(a)$ follows because each $t \in [p]$ appears in exactly $\numindsets' = \lceil \numindsets/32\sqrt{2}\bGM \xmax^2 \rceil$ of the sets $\sets$ according to \cref{lemma_conditioning_trick}\cref{item:cardinality_independence_set} (with $\lambda = \frac{1}{4\sqrt{2}\xmax^2}$). Now, we focus on the $\numindsets$ terms in \cref{eq_second_order_derivative_expressed_via_psi}.\\

\noindent Consider any $u \in [\numindsets]$. We claim that conditioned on $\svbx_{-\setU}^{(i)}$ and $\svbz^{(i)}$, the expected value of $\bpsi_u(\svbx^{(i)}; \omI)$ can be upper bounded uniformly across all $u \in [\numindsets]$. We provide a proof at the end.

\newcommand{\lowerboundbarpsi}{Lower bound on expected $\bpsi_u$}

\begin{lemma}[{\lowerboundbarpsi}]\label{lemma_expected_psi_lower_bound}
	Fix $i \in [n]$ and $\ExternalFieldI \in \ParameterSet_{\ExternalField}$. Then, with $\omI$ defined in \cref{eq:omega_defn_stage_2} and given $\svbz^{(i)}$ and $\svbx_{-\setU}^{(i)}$, we have
	\begin{align}
	\min\limits_{u \in [\numindsets]} \Expectation\Bigbrackets{\bpsi_u(\svbx^{(i)}; \omI) \bigm\vert \svbx_{-\setU}^{(i)}, \svbz^{(i)}} \geq \frac{2\xmax^2}{\pi e \ctwo[4]} \stwonorm{\omI}^2,
	\end{align}
	where the constant $\ctwo$ was defined in \cref{eq:constants}.
\end{lemma}

\noindent Consider again any $u \in [\numindsets]$. Now, we claim that conditioned on $\svbx_{-\setU}^{(i)}$ and $\svbz^{(i)}$, $\bpsi_u(\svbx^{(i)}; \omI)$ concentrates around its conditional expected value. We provide a proof at the end.

\newcommand{\concbarpsi}{Concentration of $\bpsi_u$}

\begin{lemma}[{\concbarpsi}]\label{lemma_concentration_bar_psi}
	Fix $\varepsilon > 0$, $i \in [n]$, $u \in [\numindsets]$, and $\ExternalFieldI \in \ParameterSet_{\ExternalField}$.
	Then, with $\omI$ defined in \cref{eq:omega_defn_stage_2} and given $\svbz^{(i)}$ and $\svbx_{-\setU}^{(i)}$, we have
	\begin{align}
	\Bigabs{\bpsi_u(\svbx^{(i)}; \omI) - \Expectation\bigbrackets{\bpsi_u(\svbx^{(i)}; \omI) \bigm\vert \svbx_{-\setU}^{(i)}, \svbz^{(i)}}} \leq \varepsilon,
	\end{align}
	with probability at least $1-\exp\biggparenth{ \dfrac{- \varepsilon^2}{e^{c'\bGM}\stwonorm{\omI}^2}}$. 
\end{lemma}

\noindent Given these lemmas, we proceed to show the anti-concentration of the quadratic form in  \cref{eq_lower_bound_sec_der_stage_2_cont} implying the anti-concentration of $\directionalHessianExternalField$. To that end, for any $u \in [\numindsets]$, given $\svbx_{-\setU}^{(i)}$ and $\svbz^{(i)}$, let $E_u$ denote the event that
\begin{align}
\bpsi_u(\svbx^{(i)}; \omI) \geq   \Expectation\bigbrackets{\bpsi_u(\svbx^{(i)}; \omI) \vert \svbx_{-\setU}^{(i)}, \svbz^{(i)}} - \frac{\xmax^2}{\pi e \ctwo[4]} \stwonorm{\omI}^2. \label{eq_event_Ej_Hess}
\end{align}
Since $E_u$ in an indicator event, using the law of total expectation results in
\begin{align}
\Probability(E_u) = \Expectation\Bigbrackets{\Probability(E_u | \svbx_{-\setU}^{(i)}, \svbz^{(i)})} \sgreat{(a)} 1 - \exp\biggparenth{ \dfrac{ \stwonorm{\omI}^2}{e^{c'\bGM}}},
\end{align}
where $(a)$ follows from \cref{lemma_concentration_bar_psi} with $\varepsilon \mapsfrom \dfrac{\xmax^2}{\pi e \ctwo[4]} \stwonorm{\omI}^2$. Now, by applying the union bound over all $u \in [\numindsets]$ where $\numindsets = 1024  \bGM^2 \xmax^4 \log 4p$, we have
\begin{align}
\Probability\Bigparenth{\bigcap_{u \in \numindsets} E_u} \geq 1 - O\biggparenth{\bGM^2 \log p \exp\biggparenth{\dfrac{\stwonorm{\omI}^2}{e^{c'\bGM}}}}.
\end{align}
Now, assume the event $\cap_{u \in \numindsets} E_u$ holds. Whenever this holds, we also have
\begin{align}
\sump \bigparenth{\omIt x_t^{(i)}}^2  & \sequal{\cref{eq_second_order_derivative_expressed_via_psi}}  \frac{1}{\numindsets'} \sum_{u \in [\numindsets]} \bpsi_u(\svbx^{(i)}; \omI) \\
& \sgreat{\cref{eq_event_Ej_Hess}} \frac{1}{\numindsets'} \sum_{u \in [\numindsets]}  \biggparenth{\Expectation\bigbrackets{\bpsi_u(\svbx^{(i)}; \omI) \vert \svbx_{-\setU}^{(i)}, \svbz^{(i)}} - \frac{\xmax^2}{\pi e \ctwo[4]} \stwonorm{\omI}^2} \\
& \sgreat{(a)} \frac{1}{\numindsets'} \sum_{u \in [\numindsets]} \frac{\xmax^2}{\pi e \ctwo[4]} \stwonorm{\omI}^2 = \frac{\xmax^2\numindsets}{\pi e \numindsets'\ctwo[4]} \stwonorm{\omI}^2, \label{eq_hess_intermediate_bound}
\end{align}
where $\numindsets' = \lceil \numindsets/32\sqrt{2}\bGM \xmax^2 \rceil$ and $(a)$ follows from \cref{lemma_expected_psi_lower_bound}. Finally, approximating $\numindsets' = \numindsets/32\sqrt{2}\bGM \xmax^2$ and using \cref{eq_lower_bound_sec_der_stage_2_cont}, we have
\begin{align}
\directionalHessianExternalField \geq \frac{1}{\ctwo} \sump  \bigparenth{\omIt x_t^{(i)}}^2  \sgreat{\cref{eq_hess_intermediate_bound}} \frac{32\sqrt{2}\bGM \xmax^4}{\pi e \ctwo[5]} \stwonorm{\omI}^2,
\end{align}
which completes the proof.

\paragraph{Proof of \cref{eq:second_dir_derivative_stage_2}: \expressionsecderstagetwo}
Fix any $i \in [n]$. The second-order partial derivatives of $\loss^{(i)}$ (defined in \cref{eq:loss_function_n}) with respect to the entries of the parameter vector $\ExternalFieldI$ are given by
\begin{align}
\frac{\partial^2 \loss^{(i)}(\ExternalFieldI)}{\partial \bigbrackets{\ExternalFieldIt}^2}
& = \bigbrackets{x_t^{(i)}}^2 \exp\Bigparenth{-\normalbrackets{\ExternalFieldtI + 2\EstimatedParameterRowttt\tp \svbx_{-t}^{(i)}} x_t^{(i)} - \EstimatedParameterTU[tt] \cx_t^{(i)}}
\qtext{for all} t \in [p].
\end{align}
Now, we can write the second-order directional derivative of $\loss^{(i)}$ as
\begin{align}
\directionalHessianExternalField &\defeq 
\lim_{h\to 0}\frac{\partial_{\omI}\loss^{(i)}(\ExternalFieldI+h \omI)\!-\!\partial_{\omI}\loss^{(i)}(\ExternalFieldI)}{h}
= \sump \bigbrackets{\omIt}^2 \frac{\partial^2 \loss^{(i)}(\ExternalFieldI)}{\partial \bigbrackets{\ExternalFieldIt}^2} \\
& = \sump \bigparenth{\omIt x_t^{(i)}}^2  \exp\Bigparenth{-\normalbrackets{\ExternalFieldtI + 2\EstimatedParameterRowttt\tp \svbx_{-t}^{(i)}} x_t^{(i)} - \EstimatedParameterTU[tt] \cx_t^{(i)}}.
\end{align}

\paragraph{Proof of \cref{lemma_expected_psi_lower_bound}: \lowerboundbarpsi}
Fix any $i \in [n]$, $u \in [\numindsets]$, and $\ExternalFieldI \in \ParameterSet_{\ExternalField}$. Then, given $\svbx_{-\setU}^{(i)}$ and $\svbz^{(i)}$, we have
\begin{align}
\Expectation\Bigbrackets{\bpsi_u(\svbx^{(i)}; \omI) \bigm\vert \svbx_{-\setU}^{(i)}, \svbz^{(i)}} & \sequal{\cref{eq_def_bar_psi}} \Expectation\Bigbrackets{\sum_{t \in \setU}  \bigparenth{\omIt x_t^{(i)}}^2 \bigm\vert \svbx_{-\setU}^{(i)}, \svbz^{(i)}} \\
& \sequal{(a)} \sum_{t \in \setU}  \Expectation \Bigbrackets{\bigparenth{\omIt x_t^{(i)}}^2 \bigm\vert \svbx_{-\setU}^{(i)}, \svbz^{(i)}}\\
& \sequal{(b)} \sum_{t \in \setU}  \Expectation \Bigbrackets{\Expectation \Bigbrackets{\bigparenth{\omIt x_t^{(i)}}^2 \Big \vert \svbx_{-t}^{(i)}, \svbz^{(i)}} \Bigm\vert \svbx_{-\setU}^{(i)}, \svbz^{(i)}} \\
& \sgreat{(c)} \sum_{t \in \setU}  \Expectation \Bigbrackets{\Variance\Bigparenth{\omIt x_t^{(i)} \Big \vert \svbx_{-t}^{(i)}, \svbz^{(i)}} \Bigm\vert \svbx_{-\setU}^{(i)}, \svbz^{(i)}} \\
& \sgreat{(d)} \frac{2\xmax^2}{\pi e \ctwo[4]} \stwonorm{\omI}^2,
\end{align}
where $(a)$ follows from linearity of expectation, $(b)$ follows from the law of total expectation i.e., $\Expectation[\Expectation[Y|X,Z]|Z] = \Expectation[Y|Z]$ since $\svbx_{-\setU}^{(i)} \subseteq \svbx_{-t}^{(i)}$, $(c)$ follows follows from the fact that for any random variable a, $\Expectation\normalbrackets{a^2} \geq \Variance\normalbrackets{a}$, and $(d)$ follows from \cref{prop_lower_bound_variance}.

\paragraph{Proof of \cref{lemma_concentration_bar_psi}: \concbarpsi}
\label{sub:proof_of_lemma_concentration_bar_psi}
To show this concentration result, we use \cref{coro}~\cref{eq_coro_combined} for the function $q_1$. To that end, we note that the pair $\braces{\rvbx, \rvbz}$ corresponds to a $\tSGM$ (\cref{def:tau_sgm}) with $\dGM \defn (\aGM,  \bGM, \xmax, \ParameterMatrix)$. However, the random vector $\rvbx$ conditioned on $\rvbz$ need not satisfy the Dobrushin's uniqueness condition (\cref{def_dobrushin_condition}). Therefore, we cannot apply \cref{coro}~\cref{eq_coro_combined} as is. To resolve this, we resort to \cref{lemma_conditioning_trick} with $\lambda = \frac{1}{4\sqrt{2}\xmax^2}$ to reduce the random vector $\rvbx$ conditioned on $\rvbz$ to Dobrushin's regime.

Fix any $u \in [\numindsets]$. Then, from \cref{lemma_conditioning_trick}\cref{item:conditional_sgm_independence_set}, (i) the pair of random vectors $\braces{\rvbx_{\setU}, (\rvbx_{-\setU}, \rvbz)}$ corresponds to a $\tSGM[1]$ with $\dGM_1 \defn (\aGM+2\bGM\xmax, \frac{1}{4\sqrt{2}\xmax^2}, \xmax, \ParameterMatrix_{\setU})$, and (ii) the random vector $\rvbx_{\setU}$ conditioned on $(\rvbx_{-\setU}, \rvbz)$ satisfies the Dobrushin's uniqueness condition (\cref{def_dobrushin_condition}) with coupling matrix $2\sqrt{2} \xmax^2 \normalabs{\ParameterMatrix_{\setU}}$ with $2\sqrt{2} \xmax^2 \opnorm{\normalabs{\ParameterMatrix_{\setU}}} \leq 2\sqrt{2} \xmax^2 \lambda \leq 1/2$. Now, for any fixed $i \in [n]$, 
we apply \cref{coro}~\cref{eq_coro_combined} for the function $q_1$ with $\varepsilon = \varepsilon$ for a given $\svbx_{-\setU}^{(i)}$ and $\svbz^{(i)}$, to obtain
\begin{align}
\Probability\biggparenth{\Bigabs{\bpsi_u(\svbx^{(i)}; \omI) \!-\! \Expectation\Bigbrackets{\bpsi_u(\svbx^{(i)}; \omI) \Bigm\vert \svbx_{-\setU}^{(i)}, \svbz^{(i)}}} \geq \varepsilon \Bigm\vert \svbx_{-\setU}^{(i)}, \svbz^{(i)}} \leq \exp\!\biggparenth{\! \dfrac{- \varepsilon^2}{e^{c'\bGM}\stwonorm{\omI}^2}\!}\!.
\end{align}

\subsection{Proof of \cref{lemma_lipschitzness}: \lipschitznesslossfunction}
\label{sub:proof_lemma_lipschitzness}
Fix any $i \in [n]$, any $\ExternalFieldI, \tExternalFieldI \in \ParameterSet_{\ExternalField}$. Consider the direction $\omI = \tExternalFieldI - \ExternalFieldI$, and define the function $q : [0,1] \to \Reals$ as follows
\begin{align}
q(a) = \loss^{(i)}\bigparenth{\ExternalFieldI + a(\tExternalFieldI - \ExternalFieldI)}. \label{eq_func_f_lipschitz}
\end{align}
Then, the desired inequality in \cref{eq_lipschitz_property} is equivalent to $$\normalabs{q(1) - q(0)} \leq \xmax \ctwo \sonenorm{\omI}.$$
From the mean value theorem, there exists $a' \in (0,1)$ such that
\begin{align}
\normalabs{q(1) - q(0)} 
= \biggabs{\dfrac{dq(a')}{da}}.  \label{eq_mvt_lipschitz}
\end{align}
Therefore, we have
\begin{align}
\bigabs{q(1) - q(0)} & \sequal{\cref{eq_mvt_lipschitz}} \biggabs{\dfrac{dq(a')}{da}} \sequal{\cref{eq_func_f_lipschitz}} \Bigabs{\dfrac{d\loss^{(i)}\bigparenth{\ExternalFieldI + a'(\tExternalFieldI - \ExternalFieldI)}}{da}} \\
& \sequal{\cref{eq_der_mapping_external_field}} \Bigabs{\directionalGradientExternalField\bigr|_{\ExternalFieldI = \ExternalFieldI + a'(\tExternalFieldI - \ExternalFieldI)}}. \label{eq_mvt_lipschitz_second_stage_stage}
\end{align}
Using \cref{eq:first_dir_derivative_stage_2} in \cref{eq_mvt_lipschitz_second_stage_stage}, we have
\begin{align}
\bigabs{q(1) - q(0)} & = \Bigabs{\sump \omtI x_t^{(i)} \exp\Bigparenth{-\normalbrackets{\ExternalFieldtI + a'(\tExternalFieldtI - \ExternalFieldtI) + 2\EstimatedParameterRowttt\tp \svbx_t^{(i)}} x_t^{(i)} - \EstimatedParameterTU[tt] \cx_t^{(i)}}}\\
& \sless{(a)} \xmax \sump \bigabs{\omIt}  \exp\Bigparenth{\Bigbrackets{\bigabs{(1-a')\ExternalFieldtI} + \bigabs{a' \tExternalFieldtI} + 2\sonenorm{\EstimatedParameterRowt} \sinfnorm{\svbx^{(i)}}} \xmax}\\
& \sless{(b)} \xmax \exp\Bigparenth{\bigparenth{(1-a') \aGM + a' \aGM + 2 \bGM\xmax} \xmax} \sump \bigabs{\omIt} \\
& \sequal{\cref{eq:constants}} \xmax \ctwo \sonenorm{\omI},
\end{align}
where $(a)$ follows from triangle inequality, Cauchy–Schwarz inequality, and because $\sinfnorm{\svbx^{(i)}} \leq \xmax $ for all $ i \in [n]$
and $(b)$ follows because $\ExternalFieldI, \tExternalFieldI \in \ParameterSet_{\ExternalField}$, $\EstimatedParameterMatrix \in \ParameterSet_{\ParameterMatrix}$, and $\sinfnorm{\svbx^{(i)}} \leq \xmax $ for all $ i \in [n]$.

\section{Proof of {Theorem \ref{thm_causal_estimand}}: \outcomemainresultname}
\label{sec_proof_causal_estimand}
\newcommand{\hpsi}{\what{\psi}}
\newcommand{\tpsi}{\wtil{\psi}}
\newcommand{\spsi}{\psi^{\star}}
\newcommand{\hPsi}{\what{\Psi}}
\newcommand{\tPsi}{\wtil{\Psi}}
\newcommand{\sPsi}{\Psi^{\star}}
Fix any unit $i \in [n]$ and an alternate intervention $\wtil{\svba}^{(i)} \in \cA^{p_a}$. Then, we have
\begin{align}
\mu^{(i)}(\wtil{\svba}^{(i)}) & \sequal{\cref{eq_causal_estimand}} \Expectation[\svby^{(i)}(\wtil{\svba}^{(i)}) | \rvbz = \svbz^{(i)}, \rvbv = \svbv^{(i)}] \sequal{(a)} \Expectation[\rvby | \rvba = \wtil{\svba}^{(i)}, \rvbz = \svbz^{(i)}, \rvbv = \svbv^{(i)}],
\end{align}
where $(a)$ follows because the unit-level counterfactual distribution is equivalent to unit-level conditional distribution under the causal framework considered as described in \cref{subsec_causal_mech}. To obtain a convenient expression for $\Expectation[\rvby | \rvba = \wtil{\svba}^{(i)}, \rvbz = \svbz^{(i)}, \rvbv = \svbv^{(i)}]$, we identify $\TruePhi[u,y] \in \Reals^{p_u \times p_y}$ to be the component of $\TrueParameterMatrix$ corresponding to $\rvbu$ and $\rvby$ for all $\rvbu \in \{\rvbv, \rvba, \rvby\}$ and $\TrueExternalFieldI[i,y] \in \Reals^{p_y}$ to be the component of $\TrueExternalFieldI$ corresponding to $\rvby$. Then, the conditional distribution of $\rvby$ as a function of the interventions $\rvba$, while keeping $\rvbv$ and $\rvbz$ fixed at the corresponding realizations for unit $i$, i.e., $\svbv^{(i)}$ and $\svbz^{(i)}$, respectively, can be written as
\begin{align}
f^{(i)}_{\rvby | \rvba}(\svby | \svba) \propto \exp\Bigparenth{\bigbrackets{\TrueExternalFieldI[i,y] + 2\svbv^{(i)\top}\TruePhi[v,y] + 2\svba\tp\TruePhi[a,y]} \svby + \svby\tp\TruePhi[y,y] \svby}. \label{eq_conditional_distribution_y_alternate}
\end{align}
Therefore, we have
\begin{align}
\Expectation[\rvby | \rvba = \wtil{\svba}^{(i)}, \rvbz = \svbz^{(i)}, \rvbv = \svbv^{(i)}] = \Expectation_{f^{(i)}_{\rvby | \rvba}}[\rvby | \rvba = \wtil{\svba}^{(i)}].
\end{align}
\noindent Now, consider the $p_u$ dimensional random vector $\rvbu$ supported on $\cX^{p_u}$ with distribution $f_{\rvbu}$ parameterized by   $\psi \in \Reals^{p_y}$ and $\Psi \in \Reals^{p_y \times p_y}$ as follows
\begin{align}
f_{\rvbu}(\svbu | \psi, \Psi) \propto \exp(\psi\tp \svbu+ \svbu\tp\Psi\svbu). \label{eq_dist_u} 
\end{align}
Then, note that $\what{f}^{(i)}_{\rvby | \rvba}(\svby | \svba)$ in \cref{eq_counterfactual_distribution_y} and $f^{(i)}_{\rvby | \rvba}(\svby | \svba)$ in \cref{eq_conditional_distribution_y_alternate} belong to the set $\normalbraces{f_{\rvbu}(\cdot | \psi, \Psi): \psi \in \Reals^{p_y}, \Psi \in \Reals^{p_y \times p_y}}$ for some $\psi$ and $\Psi$. Now, we consider any two distributions in this set, namely $f_{\rvbu}(\svbu | \hpsi, \hPsi)$ and $f_{\rvbu}(\svbu | \spsi, \sPsi)$. Then, we claim that the two norm of the difference of the mean vectors of these distributions is bounded as below. We provide a proof at the end.
\newcommand{\expfamperturbationresultname}{Perturbation in the mean vector}
\begin{lemma}[{\expfamperturbationresultname}]
	\label{lemma_exp_fam_parameter_perturbation}
	For any $\psi \in \Reals^{p_y}$ and $\Psi \in \Reals^{p_y \times p_y}$, let $\mu_{\psi, \Psi}(\rvbu) \in \Reals^{p_u}$ and $\Covariance_{\psi, \Psi}(\rvbu, \rvbu) \in \Reals^{p_u \times p_u}$ denote the mean vector and the covariance matrix of $\rvbu$, respectively, with respect to $f_{\rvbu}$ in \cref{eq_dist_u}. Then, for any $\hpsi, \spsi \in \Reals^{p_y}$ and $\hPsi, \sPsi \in \Reals^{p_y \times p_y}$, there exists some $t \in (0,1)$, $\tpsi \defn t \hpsi + (1-t) \spsi$ and $\tPsi \defn t \tpsi + (1-t) \tpsi$ such that
	\begin{align}
	\stwonorm{ \mu_{\hpsi, \hPsi}(\rvbu) - \mu_{\spsi, \sPsi}(\rvbu)} & \leq    \opnorm{\Covariance_{\tpsi, \tPsi}(\rvbu, \rvbu)}\stwonorm{(\hpsi - \spsi)} \\
 & \qquad + \sum_{t_3 \in [p]}  \opnorm{\Covariance_{\tpsi, \tPsi}(\rvbu, \rvu_{t_3} \rvbu)}\stwonorm{(\hPsi_{t_3} \!-\! \sPsi_{t_3})}.
	\end{align}
\end{lemma}
Given this lemma, we proceed with the proof. By applying this lemma to $\what{f}^{(i)}_{\rvby | \rvba}(\svby | \svba)$ in \cref{eq_counterfactual_distribution_y} and $f^{(i)}_{\rvby | \rvba}(\svby | \svba)$ in \cref{eq_conditional_distribution_y_alternate}, we see that it is sufficient to show the following bound
\begin{align}
& \stwonorm{\normalparenth{\TrueExternalFieldI[i,y] - \EstimatedExternalFieldI[i,y]} + 2\svbv^{(i)\top}\!\normalparenth{\TruePhi[v,y] - \EstimatedPhi^{(v,y)} } + 2\wtil{\svba}^{(i)\top}\!\normalparenth{\TruePhi[a,y] - \EstimatedPhi^{(a,y)}}}  \\
& \qquad \qquad \qquad \qquad \qquad \qquad \qquad \qquad \qquad + \sum_{t \in [p_y]} \stwonorm{\TruePhi[y,y]_{t} - \EstimatedPhi^{(y,y)}_t} \leq R(\varepsilon, \delta/n) + p \varepsilon.
\end{align}
To that end, we have
\begin{align}
\sum_{t \in [p_y]} \stwonorm{\TruePhi[y,y]_{t} - \EstimatedPhi^{(y,y)}_t} \sless{(a)} \sum_{t \in [p_y]} \stwonorm{\TrueParameterRowt -\EstimatedParameterRowt},
\label{eq_mean_outcome_part_1}
\end{align}
where $(a)$ follows because $\ell_2$ norm of any sub-vector is no more than $\ell_2$ norm of the vector.
Similarly, we have
\begin{align}
& \stwonorm{\normalparenth{\TrueExternalFieldI[i,y] \!-\! \EstimatedExternalFieldI[i,y]} \!+ \!2\svbv^{(i)\top}\normalparenth{\TruePhi[v,y] \!-\!\EstimatedPhi^{(v,y)} } \!+\! 2\wtil{\svba}^{(i)\top}\normalparenth{\TruePhi[a,y] \!-\! \EstimatedPhi^{(a,y)}}}\\
& \sless{(a)}  \stwonorm{\TrueExternalFieldI[i,y] \!-\! \EstimatedExternalFieldI[i,y]} \!+\! 2\stwonorm{\svbv^{(i)\top}\normalparenth{\TruePhi[v,y] \!-\!\EstimatedPhi^{(v,y)}}} \!+\! 2\stwonorm{\wtil{\svba}^{(i)\top}\normalparenth{\TruePhi[a,y] \!-\! \EstimatedPhi^{(a,y)}}}\\
& \sless{(b)}  \stwonorm{\TrueExternalFieldI[i,y] \!-\! \EstimatedExternalFieldI[i,y]} \!+\! 2\stwonorm{\svbv^{(i)}}\opnorm{\TruePhi[v,y] \!-\!\EstimatedPhi^{(v,y)}} \!+\! 2\stwonorm{\wtil{\svba}^{(i)}}\opnorm{\normalparenth{\TruePhi[a,y] \!-\! \EstimatedPhi^{(a,y)}}}\\
& \sless{(c)}  \stwonorm{\TrueExternalFieldI[i] \!-\! \EstimatedExternalFieldI[i]} + 2\Bigparenth{\stwonorm{\svbv^{(i)}} + \stwonorm{\wtil{\svba}^{(i)}}} \opnorm{\TrueParameterMatrix -\EstimatedParameterMatrix}\\
& \sless{(d)}  \stwonorm{\TrueExternalFieldI[i] \!-\! \EstimatedExternalFieldI[i]} + 2\Bigparenth{\stwonorm{\svbv^{(i)}} + \stwonorm{\wtil{\svba}^{(i)}}} \onematnorm{\TrueParameterMatrix -\EstimatedParameterMatrix} \\
& \sless{(e)}  \stwonorm{\TrueExternalFieldI[i] \!-\! \EstimatedExternalFieldI[i]} + 2\xmax \bigparenth{\sqrt{p_v} + \sqrt{p_a}} \onematnorm{\TrueParameterMatrix -\EstimatedParameterMatrix}, \label{eq_mean_outcome_part_2}
\end{align}
where $(a)$ follows from triangle inequality, $(b)$ follows because induced matrix norms are submultiplicative, $(c)$ follows because operator norm of any sub-matrix is no more than operator norm of the matrix and $\ell_2$ norm of any sub-vector is no more than $\ell_2$ norm of the vector, $(d)$ follows because $\TrueParameterMatrix -\EstimatedParameterMatrix$ is symmetric and because matrix operator norm is bounded by square root of the product of matrix one norm and matrix infinity norm, and $(e)$ follows because $\max\{\sinfnorm{\svbv^{(i)}}, \sinfnorm{\svba^{(i)}}\} \leq \xmax$ for all $i \in [n]$.\\

\noindent Now, combining \cref{eq_mean_outcome_part_1,eq_mean_outcome_part_2}, we have
\begin{align}
& \stwonorm{\normalparenth{\TrueExternalFieldI[i,y] \!\!-\! \EstimatedExternalFieldI[i,y]} \!+\! 2\svbv^{(i)\top}\!\normalparenth{\TruePhi[v,y] \!\!-\! \EstimatedPhi^{(v,y)} } \!+\! 2\wtil{\svba}^{(i)\top}\!\normalparenth{\TruePhi[a,y] \!\!-\! \EstimatedPhi^{(a,y)}}} \!+\!\! \!\!\sum_{t \in [p_y]} \!\!\!\stwonorm{\TruePhi[y,y]_{t} \!\!-\! \EstimatedPhi^{(y,y)}_t} \\
& \leq \stwonorm{\TrueExternalFieldI[i] \!\!-\! \EstimatedExternalFieldI[i]} \!+\! 2\xmax \bigparenth{\sqrt{p_v} \!+\! \sqrt{p_a}} \onematnorm{\TrueParameterMatrix \!\!-\!\EstimatedParameterMatrix} \!+\! \sum_{t \in [p_y]} \stwonorm{\TrueParameterRowt \!\!-\!\EstimatedParameterRowt}\\
& \sless{(a)}  R(\varepsilon, \delta/n) + 2\xmax \bigparenth{\sqrt{p_v} + \sqrt{p_a}} \sqrt{p} \varepsilon + p_y \varepsilon,
\end{align}
and $(a)$ follows from \cref{theorem_parameters} by using the relationship between vector norms. The proof is complete by rescaling $\varepsilon$ and absorbing the constants in $c$.

\paragraph{Proof of \cref{lemma_exp_fam_parameter_perturbation}: \expfamperturbationresultname}
Let $Z(\psi, \Psi) \in \Reals_{+}$ denote the log-partition function of $f_{\rvbu}(\cdot | \psi, \Psi) $ in \cref{eq_dist_u}. Then, from \cite[Theorem 1]{BusaFSZ2019}, we have
\begin{align}
\stwonorm{ \mu_{\hpsi, \hPsi}(\rvbu) - \mu_{\spsi, \sPsi}(\rvbu)} = \stwonorm{\nabla_{\hpsi} Z(\hpsi, \hPsi) - \nabla_{\spsi} Z(\spsi, \sPsi)}. \label{eq_mean_vector_difference}
\end{align}
For $t_1, t_2, t_3 \in [p]$, consider $\frac{\partial^2 Z(\psi, \Psi)}{\partial \psi_{t_1}\partial \psi_{t_2}}$ and $\frac{\partial^2 Z(\psi, \Psi)}{\partial \psi_{t_1}\partial \Psi_{t_2, t_3}}$. Using the fact that the Hessian of the log partition function of any regular exponential family is the covariance matrix of the associated sufficient statistic, we have
\begin{align}
\frac{\partial^2 Z(\psi, \Psi)}{\partial \psi_{t_1}\partial \psi_{t_2}} = \Covariance_{\psi, \Psi}(\rvu_{t_1}, \rvu_{t_2}) \qtext{and} \frac{\partial^2 Z(\psi, \Psi)}{\partial \psi_{t_1}\partial \Psi_{t_2, t_3}} = \Covariance_{\psi, \Psi}(\rvu_{t_1}, \rvu_{t_2} \rvu_{t_3}). \label{eq_exp_fam_hessian_cov}
\end{align}
Now, for some $c \in (0,1)$, $\tpsi \defn c \hpsi + (1-c) \spsi$ and $\tPsi \defn c \tpsi + (1-c) \tpsi$, we have the following from the mean value theorem
\begin{align}
& \frac{\partial Z(\hpsi, \hPsi)}{\partial \hpsi_{t_1}} \!-\!  \frac{\partial Z(\spsi, \sPsi)}{\partial \spsi_{t_1}} \\
& \!=\! \sum_{t_2 \in [p]}\frac{\partial^2 Z(\tpsi, \tPsi)}{\partial \tpsi_{t_2}\partial \tpsi_{t_1}} \cdot (\hpsi_{t_2} - \spsi_{t_2}) + \sum_{t_2 \in [p]} \sum_{t_3 \in [p]} \frac{\partial^2 Z(\tpsi, \tPsi)}{\partial \tPsi_{t_2, t_3} \partial \tpsi_{t_1}} \cdot (\hPsi_{t_2, t_3} - \sPsi_{t_2, t_3})\\
& \!\!\!\sequal{\cref{eq_exp_fam_hessian_cov}} \!\!\sum_{t_2 \in [p]}\!\! \Covariance_{\tpsi, \tPsi}(\rvu_{t_1}, \rvu_{t_2}) 
\!\cdot\! (\hpsi_{t_2} \!-\! \spsi_{t_2}) \!+\!\!\! \sum_{t_3 \in [p]} \!\sum_{t_2 \in [p]} \!\!\Covariance_{\tpsi, \tPsi}(\rvu_{t_1}, \rvu_{t_3} \rvu_{t_2}) \!\cdot\!(\hPsi_{t_3, t_2} \!-\! \sPsi_{t_3, t_2}).
\end{align}
Now, using the triangle inequality and sub-multiplicativity of induced matrix norms, we have
\begin{align}
\stwonorm{\nabla_{\hpsi} Z(\hpsi, \hPsi) \!-\! \nabla_{\spsi} Z(\spsi, \sPsi)} & \leq \opnorm{\Covariance_{\tpsi, \tPsi}(\rvbu, \rvbu)}\stwonorm{(\hpsi \!-\! \spsi)} \\
& \qquad \qquad \qquad + \sum_{t_3 \in [p]} \! \opnorm{\Covariance_{\tpsi, \tPsi}(\rvbu, \rvu_{t_3} \rvbu)}\stwonorm{(\hPsi_{t_3} \!-\! \sPsi_{t_3})}.  \label{eq_mean_value_theorem}
\end{align}
Combining \cref{eq_mean_vector_difference,eq_mean_value_theorem} completes the proof.

\subsection{Bounded operator norms for perturbations in the parameters}
\label{subsec_bounded_op_norms}
In \cref{subsec_guarantee_outcome_estimate}, we assumed the operator norms of (i) the covariance matrix of $\rvby$ conditioned on $\rvba$, $\rvbz$, and $\rvbv$ and (ii) the cross-covariance matrix of $\rvby$ and $\rvy_t \rvby$ conditioned on $\rvba$, $\rvbz$, and $\rvbv$ for all $t \in [p_y]$ to remain bounded for small perturbation in the parameters. In this section, we provide examples where these hold. 

Suppose the distribution of $\rvby$ conditioned on $\rvba$, $\rvbz$, and $\rvbv$ is a Gaussian distribution. For simplicity, let the mean of this distribution be zero. Then, for any $t,u,v \in [p_y]$,
\begin{align}
\Covariance_{\ExternalField,\ParameterMatrix}(\rvy_u, \rvy_t \rvy_v | {\svba}, \svbz, \svbv) = \Expectation_{\ExternalField,\ParameterMatrix}(\rvy_u \rvy_t \rvy_v | {\svba}, \svbz, \svbv) \sequal{(a)}  0.
\end{align}
where $(a)$ follows because $\Expectation_{\ExternalField,\ParameterMatrix}(\rvy_u \rvy_t \rvy_v | {\svba}, \svbz, \svbv)$ is the third cumulant of $\rvy_u \rvy_t \rvy_v | \rvba, \rvbz, \rvbv$ and the third cumulant for any Gaussian distribution is zero \citep{holmquist1988moments}. Then,
\begin{align}
\max\limits_{t \in [p_y]} \opnorm{\Covariance_{\ExternalField,\ParameterMatrix}(\rvby, \rvy_t \rvby | {\svba}, \svbz, \svbv)} = 0. \label{eq_op_norm_cross_zero}
\end{align}
Further, \cref{eq_op_norm_cross_zero} also holds for small perturbations in $\ExternalField$ and $\ParameterMatrix$ as the distribution of $\rvby$ conditioned on $\rvba$, $\rvbz$, and $\rvbv$ would still be a Gaussian distribution. 

Now, we bound $\opnorm{\Covariance_{\ExternalField,\ParameterMatrix}(\rvby, \rvby | {\svba}, \svbz, \svbv)}$ under additional conditions. For simplicity, suppose $\Variance_{\ExternalField,\ParameterMatrix}(\rvy_t| {\svba}, \svbz, \svbv) = 1$ for all $t \in [p_y]$. Further, suppose the (undirected) graphical structure associated with elements of $\rvby$, i.e., $\rvy_{1}, \cdots, \rvy_{p_y}$, is a chain (This would be true for the motivating example in \cref{fig_graphical_models}(a)). If the correlation between any two elements of $\rvby$ connected by an edge in the tree is equal to $\rho \in [0,1]$ (This is equivalent to all the off-diagonal non-zero entries of $\ParameterMatrix$ being the same), then for any $u,v \in [p_y]$,
\begin{align}
\Covariance_{\ExternalField,\ParameterMatrix}(\rvy_u, \rvy_v | {\svba}, \svbz, \svbv)  \sequal{(a)}  \rho^{\normalabs{u-v}},
\end{align}
where $(a)$ follows by the correlation decay property for Gaussian tree models \citep[Equation. 18]{tan2010learning}. Then, for any $0 \leq \rho < 1$
\begin{align}
\opnorm{\Covariance_{\ExternalField,\ParameterMatrix}(\rvby, \rvby | {\svba}, \svbz, \svbv)} \sless{(a)} \frac{1+\rho}{1-\rho}, \label{eq_op_norm_cov_zero}
\end{align}
where $(a)$ follows from \cite{trench1999asymptotic}. Further, \cref{eq_op_norm_cov_zero} holds for small perturbations in $\ExternalField$ and $\ParameterMatrix$ as long as $\rho < 1$. Therefore, $C(\mbb B)$ in \cref{eq_cov_constraint} is a constant (with respect to $p$) for small perturbations in $\ExternalField$ and $\ParameterMatrix$.

While we showed that $C(\mbb B)$ is a constant for a class of Gaussian distributions, we except similar results for truncated Gaussian distributions and exponential family distributions in \cref{eq_conditional_distribution_y}.

\section{Proof of {Proposition \ref{prop_impute_missing_covariates}}: \impute}
\label{proof_impute_missing_covariates}
We start by decomposing the true covariates $\rvbv$ into two variables: one to capture the randomness in the noisy observations $\crvbv$ and the other to capture the randomness in the measurement error $\Delta \rvbv$, i.e., $\rvbv = \crvbv - \Delta \rvbv$.  
Then, by letting $\barp \defn 2 p_v + p_a + p_y$ and using \cref{eq_joint_distribution_vay}, the joint probability distribution $f_{\wbar{\ranvarvec}}$ of the $\barp$-dimensional random vector $\wbar{\ranvarvec} \defeq \normalparenth{\Delta \rvbv, \crvbv, \rvba, \rvby}$ can be parameterized by a vector $\wbar{\phi} \in \Reals^{\barp \times 1}$ and a symmetric matrix $\wbar{\Phi} \in \Reals^{\barp \times \barp}$ as follows 
\begin{align}
f_{\wbar{\ranvarvec}}(\wbar{\varvec}; \wbar{\phi}, \wbar{\Phi}) \propto \exp\Bigparenth{ \wbar{\phi}\tp \wbar{\varvec}
	+\wbar{\varvec}\tp \wbar{\Phi} \wbar{\varvec}},
\qtext{where}
\wbar{\varvec} \defeq (\Delta \svbv, \csvbv, \svba, \svby),
\label{eq_joint_distribution_vvay}
\end{align}
and $\Delta \svbv$, $\csvbv$, $\svba$, and $\svby$ denote realizations of $\Delta \rvbv$, $\crvbv$, $\rvba$, and $\rvby$, respectively. More importantly, $\wbar{\phi}$ and $\wbar{\Phi}$ are derived completely from $\phi$ and $\Phi$, respectively, and have special structure: 
\begin{align}
    \wbar{\phi}^{(\cv)} = - \wbar{\phi}^{(\Delta v)} & =  \phi^{(v)},\\
    \wbar{\phi}^{(u)} & = \phi^{(u)} \stext{for all} \rvbu \in \normalbraces{\rvba, \rvby},\\
    \wbar{\Phi}^{(\Delta v, \Delta v)} = \wbar{\Phi}^{(\cv, \cv)} = - \wbar{\Phi}^{(\cv, \Delta v)} & = \Phi^{(v, v)}\\
    \wbar{\Phi}^{(u, \cv)} = - \wbar{\Phi}^{(u, \Delta v)} & = \Phi^{(u, v)} \stext{for all} \rvbu \in \normalbraces{\rvba, \rvby}, \stext{and}\\
    \wbar{\Phi}^{(u_1, u_2)} & = \Phi^{(u_1, u_2)} \stext{for all} \rvbu_1, \rvbu_2 \in \normalbraces{\rvba, \rvby}.
\end{align}
Now, to learn counterfactuals and measurement errors for units $i \in \normalbraces{1, \cdots, n/2}$, we use the methodology developed in \cref{section_problem_formulation} by replacing the role of unobserved covariates $\rvbz$ by $\Delta \svbv$. In particular, we consider learning $f_{\rvby | \rvba, \Delta \rvbv, \crvbv}(\rvby=\cdot | \rvba= \cdot, \Delta \svbv, \csvbv)$ as a function of $\rvba$. From \cref{eq_actual_parameters_of_interest} and the structure on $\wbar{\phi}$ and $\wbar{\Phi}$ described above, this reduces to learning
\begin{align}
& \stext{(i)} \wbar{\phi}^{(y)} + 2\wbar{\Phi}^{(\Delta v,y)\top} \Delta \svbv + 2\wbar{\Phi}^{(\cv,y)\top} \csvbv = \phi^{(y)} - 2\Phi^{(v,y)\top} \Delta \svbv + 2\Phi^{(v,y)\top} \csvbv, \stext{}\\ & \stext{(ii)} \wbar{\Phi}^{(a,y)} = \Phi^{(a, y)}, \stext{and}\\ &\stext{(iii)} \wbar{\Phi}^{(y,y)} = \Phi^{(y, y)}.
\end{align}
To learn these, we consider the distribution of $\rvbx \defn (\crvbv, \rvba, \rvby)$ conditioned on $\Delta \rvbv = \Delta \svbv$. From \cref{eq_conditional_distribution_vay}, we have
\begin{align}
f_{\rvbx|\Delta \rvbv}\bigparenth{\! \svbx| \Delta \svbv; \ExternalField(\!\Delta \svbv\!), \! \ParameterMatrix} \!\propto\! \exp\!\Bigparenth{ \!\normalbrackets{\ExternalField(\!\Delta \svbv\!)}\!\tp \! \! \svbx \!+\! \svbx\!\tp \!\ParameterMatrix \svbx \!} \stext{with}
\ExternalField(\!\Delta \svbv\!)  \! \defeq \!\! \begin{bmatrix} \! \phi^{(v)} \!-\! 2 \Phi^{(v,v)\top}\!  \Delta \svbv \! \\ \! \phi^{(a)} \!-\! 2 \Phi^{(v,a)\top}\!  \Delta \svbv \! \\ \! \phi^{(y)} \!-\! 2 \Phi^{(v,y)\top} \! \Delta \svbv \!\end{bmatrix}\!\!\!, \label{eq_coditional_distribution_vay|delta_v}
\end{align} 
$\svbx \defn (\csvbv, \svba, \svby)$, $\Theta \defn \Phi$, and $\csvbv$, $\svba$, and $\svby$ denoting realizations of $\crvbv$, $\rvba$, and $\rvby$, respectively. The special structure on $\wbar{\Phi}$ discussed above implies that $\Phi^{(v,v)}, \Phi^{(v,a)}$, and $\Phi^{(v,y)}$ affect both $\ExternalField(\Delta \svbv)$ and $\ParameterMatrix$ which we exploit. As mentioned in \cref{subsec_theo_gua}, we denote the true distribution of $\rvbx$ conditioned on $\Delta \rvbv = \Delta \svbv$ by $f_{\rvbx|\Delta \rvbv}\bigparenth{\cdot| \Delta \svbv; \TrueExternalField(\Delta \svbv), \TrueParameterMatrix}$. 

\paragraph{Proof idea} First, we use units $i \in \normalbraces{n/2+1, \cdots, n}$ without any measurement error to estimate $\phi^{\star}$ and $\Phi^{\star} = \TrueParameterMatrix$, i.e., the parameters corresponding to the distribution of $(\rvbv, \rvba, \rvby)$ (see \cref{subsec_theo_gua}). Next, for units $i \in \normalbraces{1, \cdots, n/2}$ with measurement error, we estimate $\TrueExternalField(\Delta \svbv^{(i)})$ by expressing it as a linear combination of the estimates of $\phi^{\star}$ and $\Phi^{\star}$ (enabling the use of \cref{exam:lc_dense}). The coefficients of this linear combination turn out to be our estimates of the measurement error $\Delta \svbv^{(i)}$.

\paragraph{Estimate $\phi^{\star}$ and $\Phi^{\star}$}  For units $i \in \normalbraces{n/2+1, \cdots, n}$, under our assumption $\Delta \svbv^{(i)} = 0$ implying $\TrueExternalField(\Delta \svbv^{(i)}) = \phi^{\star}$. Therefore, in addition to the population-level parameter $\TrueParameterMatrix = \Phi^{\star}$, the unit-level parameter $\TrueExternalField(\Delta \svbv) = \phi^{\star}$ is also shared for these units. As a result, the set of distributions $\bigbraces{f_{\rvbx|\Delta \rvbv}\bigparenth{\cdot| \Delta \svbv; \TrueExternalField(\Delta \svbv), \TrueParameterMatrix}}_{i=1}^n$ all coincide. Thus, 
 learning $\phi^{\star}$ and $\Phi^{\star}$ boils down to learning parameters of a sparse graphical model (because of the assumptions in \cref{subsec_theo_gua}) from $n/2$ samples. We use the methodology and analysis from \cite{ShahSW2023} (which is closely related to the one in this work) to obtain estimates 
 $\what{\phi}$ and $\what{\Phi}$
 such that with probability at least $1-\delta$, we have
\begin{align}
\max\bigbraces{\stwonorm{\phi^{\star} - \what{\phi}}, \matnorm{\Phi^{\star} \!-\! \what{\Phi}}_{2,\infty}} & \leq \varepsilon_1 \qtext{whenever} n  \geq \frac{ce^{c'\bGM} \log\frac{p}{\sqrt{\delta}}}{\varepsilon_1^2}. \label{eq_theta_estimate_guarantee}
\end{align}

\paragraph{Recover the unit-level parameters} Now, for units $i \in \normalbraces{1, \cdots, n/2}$, we express the true unit-level parameters $\TrueExternalField(\Delta \svbv^{(i)})$ as a linear combination of known vectors. To that end, fix any $i \in [n/2]$. Then, using \cref{eq_coditional_distribution_vay|delta_v},
we can write $\TrueExternalFieldI \defn \TrueExternalField(\Delta \svbv^{(i)})$ as a linear combination of $p_v + 1$ vectors, i.e.,
\begin{align}
\TrueExternalFieldI = \tbf{B} \tbf{a}^{(i)},\label{eq_sparse_lasso_first}
\end{align}
where
\begin{align}
    \tbf{B}  \defn \begin{bmatrix} \phi^{\star}, -2 \Phi^{\star}_1, \cdots, -2 \Phi^{\star}_{p_v} \end{bmatrix}\in  \Reals^{p\times (p_v+1)} \qtext{and} \tbf{a}^{(i)} \defn \begin{bmatrix} 1  \\ \Delta \svbv^{(i)} \end{bmatrix} \in \Reals^{(p_v+1) \times 1}.  \label{eq_appl_b_defn}
\end{align}
While we do not know the matrix $\tbf{B}$, we can produce an estimate $\what{\tbf{B}}$ using $\what{\phi}$ and $\what{\Phi}$ such that, with probability at least $1-\delta$,
\begin{align}
\matnorm{\what{\tbf{B}} \!-\! \tbf{B}}_{2,\infty} & \leq \varepsilon_1 \qtext{whenever} n  \geq \frac{ce^{c'\bGM} \log\frac{p}{\sqrt{\delta}}}{\varepsilon_1^2}. \label{eq_b_estimate_guarantee}
\end{align}
This guarantee follows directly from \cref{eq_theta_estimate_guarantee} and the definition of $\tbf{B}$ in \cref{eq_appl_b_defn}. Then, we can write
\begin{align}
\TrueExternalFieldI = \what{\tbf{B}} \wtil{\tbf{a}}^{(i)} \qtext{where} \wtil{\tbf{a}}^{(i)} \defn  \tbf{a}^{(i)} + \zeta, \label{eq_sparse_lasso_third}
\end{align}
for some error term $\zeta$. 
Conditioned on the event
\cref{eq_b_estimate_guarantee}, $\zeta$ can be controlled in following manner
\begin{align}
\stwonorm{\what{\tbf{B}} \zeta} \sequal{\cref{eq_sparse_lasso_third}}  \stwonorm{\TrueExternalFieldI - \what{\tbf{B}} \tbf{a}^{(i)}} & \sequal{\cref{eq_sparse_lasso_first}} \stwonorm{\tbf{B} \tbf{a}^{(i)} - \what{\tbf{B}} \tbf{a}^{(i)}} \\
& \sless{(a)} \opnorm{\tbf{B} - \what{\tbf{B}}} \stwonorm{\tbf{a}^{(i)}} \\
& \sless{(b)} \bigparenth{\sqrt{p} \matnorm{\tbf{B} - \what{\tbf{B}}}_{2,\infty}} \cdot  \bigparenth{\sqrt{p_v+1} \sinfnorm{\tbf{a}^{(i)}}} \sless{(c)} \aGM \varepsilon_1\sqrt{(p_v+1)p},
\label{eq_bounded_zeta} 
\end{align}
where $(a)$ follows from sub-multiplicativity of induced matrix norms, $(b)$ follows from standard matrix norm inequalities, and $(c)$ follows from \cref{eq_b_estimate_guarantee} and because the measurement errors are bounded by $\aGM$.

Then, performing an analysis similar to one in \cref{sec_proof_thm_node_parameters_recovery} while using the bound on $n$ in \cref{eq_theta_estimate_guarantee} instead of the one in \cref{eq:matrix_guarantee}, and using \cref{exam:lc_dense}, we obtain estimates $\EstimatedExternalFieldI[1], \cdots, \EstimatedExternalFieldI[n/2]$ such that (see \cref{cor_params}\cref{item:lc_dense} for reference), with probability at least $1 - \delta$, we have
\begin{align}
\max_{i\in[n/2]}
\mathrm{MSE}(\EstimatedExternalFieldI, \TrueExternalFieldI)
&\leq \max\Bigbraces{\varepsilon_1^2, \dfrac{ce^{c'\bGM} \bigparenth{p_v + \log \normalparenth{\log \frac{np}{\delta}}}}{p}}, 
\label{eq_mse_application}
\end{align}
whenever $n \geq ce^{c'\bGM} \varepsilon_1^{-2} \bigparenth{\log \frac{\sqrt{n} p}{\sqrt{\delta}} + p_v}$.

\paragraph{Recover the measurement error} We condition on the event \cref{eq_mse_application} happening and note that the above estimate $\EstimatedExternalFieldI$ of the unit-level parameter $\TrueExternalFieldI$ is of the form $\EstimatedExternalFieldI  = \what{\tbf{B}} \what{\tbf{a}}^{(i)}$ for $i \in [n/2]$. We declare $\what{\tbf{a}}^{(i)}$ as our estimate of the measurement error for unit $i \in [n/2]$ and prove the corresponding guarantee below.\\

\noindent  Fix any $i \in [n/2]$. From \cref{eq_sparse_lasso_third} and triangle inequality, we find that
\begin{align}
\stwonorm{\TrueExternalFieldI - \EstimatedExternalFieldI} & = \stwonorm{\what{\tbf{B}} \tbf{a}^{(i)} + \what{\tbf{B}} \zeta \!-\! \what{\tbf{B}} \what{\tbf{a}}^{(i)} } \geq \stwonorm{\what{\tbf{B}} \tbf{a}^{(i)} - \what{\tbf{B}} \what{\tbf{a}}^{(i)}} \!-\! \stwonorm{\what{\tbf{B}} \zeta}.\label{eq_lower_bound_theta_error}
\end{align}
Then, doing standard algebra with \cref{eq_lower_bound_theta_error} yields that
\begin{align}
\mathrm{MSE}(\EstimatedExternalFieldI, \TrueExternalFieldI) + \frac{\stwonorm{\what{\tbf{B}} \zeta}^2}{p} & \geq \frac{\stwonorm{\what{\tbf{B}} \tbf{a}^{(i)} - \what{\tbf{B}} \what{\tbf{a}}^{(i)}}^2}{2p} = \frac{(\tbf{a}^{(i)} - \what{\tbf{a}}^{(i)})\tp \what{\tbf{B}}\tp \what{\tbf{B}} (\tbf{a}^{(i)} - \what{\tbf{a}}^{(i)})}{2p}.\label{eq_a_bound_0}
\end{align}
Combining \cref{eq_bounded_zeta,eq_mse_application,eq_a_bound_0} with the choice $\varepsilon_1 = \kappa \varepsilon_2/\aGM \sqrt{p_v + 1}$, we have
\begin{align}
\frac{(\tbf{a}^{(i)} \!-\! \what{\tbf{a}}^{(i)})\tp \what{\tbf{B}}\tp \!\what{\tbf{B}} (\tbf{a}^{(i)} \!-\! \what{\tbf{a}}^{(i)})}{2p} \!\leq\! \max\Bigbraces{\frac{\varepsilon_2^2\kappa^2}{\aGM^2(p_v+1)}\!, \dfrac{ce^{c'\bGM} \bigparenth{p_v \!+\! \log \normalparenth{\log \frac{np}{\delta}}}}{p}} \!+\! \varepsilon_2^2 \kappa^2,  
\label{eq_a_bound_2}
\end{align}
uniformly for all $i \in [n/2]$, with probability at least $1 - \delta$, whenever $n \geq ce^{c'\bGM} \kappa^{-2} \varepsilon_2^{-2} (p_v\!+\!1) \bigparenth{\log \frac{\sqrt{n} p}{\sqrt{\delta}} + p_v}$. Next, we claim that the eigenvalues of $\what{\tbf{B}}\tp \what{\tbf{B}}$ can be lower bounded by $\kappa p / 2$ whenever $\varepsilon_2 \leq \sqrt{p/(p_v+1)}/8$. Taking this claim as given at the moment, we continue our proof. We have
\begin{align}
\frac{\kappa} {4}\stwonorm{\tbf{a}^{(i)} - \what{\tbf{a}}^{(i)}}^2 \leq  \frac{(\tbf{a}^{(i)} - \what{\tbf{a}}^{(i)})\tp \what{\tbf{B}}\tp \what{\tbf{B}} (\tbf{a}^{(i)} - \what{\tbf{a}}^{(i)})}{2p} \qtext{whenever} \varepsilon_2 \leq  \frac{1}{8} \sqrt{\frac{p}{p_v+1}}, \label{eq_a_bound_1}
\end{align}
uniformly for all $i \in [n/2]$. Combining \cref{eq_a_bound_2,eq_a_bound_1} completes the proof.\\

\noindent It remains to show that the eigenvalues of $\what{\tbf{B}}\tp \what{\tbf{B}}$ can be lower bounded by $\kappa p / 2$ conditioned on \cref{eq_theta_estimate_guarantee}. For any matrix $\tbf{M}$, let $\lambda_{\max}(\tbf{M})$ and $\lambda_{\min}(\tbf{M})$ denote the largest and the smallest eigenvalues of $\tbf{M}$, respectively. Then from Weyl's inequality \citep[Theorem. 8.2]{bhatia2007perturbation}, we have
\begin{align}
\lambda_{\min}(\what{\tbf{B}}\tp \what{\tbf{B}}) \geq \lambda_{\min}({\tbf{B}}\tp {\tbf{B}}) - \lambda_{\max}( {\tbf{B}}\tp {\tbf{B}} - \what{\tbf{B}}\tp \what{\tbf{B}} ) \sgreat{(a)} \kappa p - \lambda_{\max}( {\tbf{B}}\tp {\tbf{B}} - \what{\tbf{B}}\tp \what{\tbf{B}} ),
\end{align}
where $(a)$ follows from the assumption on the eigenvalues of ${\tbf{B}}\tp {\tbf{B}}$. 
Now, it suffices to upper bound $\lambda_{\max}( {\tbf{B}}\tp {\tbf{B}} - \what{\tbf{B}}\tp \what{\tbf{B}} )$ by $\kappa p/ 2$. We have 
\begin{align}
\bigabs{ \lambda_{\max}({\tbf{B}}\tp {\tbf{B}} - \what{\tbf{B}}\tp \what{\tbf{B}}) } & \sequal{(a)} \opnorm{ {\tbf{B}}\tp {\tbf{B}} - \what{\tbf{B}}\tp \what{\tbf{B}} } \\
&  \sless{(b)} (p_v+1) \maxmatnorm{ {\tbf{B}}\tp {\tbf{B}} - \what{\tbf{B}}\tp \what{\tbf{B}} } \\
& \sless{(c)} (p_v+1) \Bigparenth{\maxmatnorm{ {\tbf{B}}\tp \bigparenth{\tbf{B} - \what{\tbf{B}}}} + \maxmatnorm{ {\bigparenth{\tbf{B} - \what{\tbf{B}}}\tp\what{\tbf{B}} } }} \\& \sless{(d)} (p_v+1) \bigparenth{\matnorm{ \tbf{B}\tp}_{2,\infty} + \matnorm{ \what{\tbf{B}}\tp}_{2,\infty}}\matnorm{ \bigparenth{\tbf{B} - \what{\tbf{B}}}\tp}_{2,\infty}  \\
& \sless{(e)} (p_v+1) (2\aGM\sqrt{p} + 2\aGM\sqrt{p}) \cdot \varepsilon_1 \sless{(f)} 4 \kappa \varepsilon_2 \sqrt{p_v+1}\sqrt{p} \sless{(g)} \frac{\kappa p}{2}, 
\end{align}
where $(a)$ follows because ${\tbf{B}}\tp {\tbf{B}} - \what{\tbf{B}}\tp \what{\tbf{B}} $ is symmetric, $(b)$ follows from because $\opnorm{\tbf{M}} \leq \fronorm{\tbf{M}} \leq d \maxmatnorm{\tbf{M}}$ for any square matrix $\tbf{M} \in \Reals^{d \times d}$, $(c)$ follows from the triangle inequality, $(d)$ follows by Cauchy–Schwarz inequality, $(e)$ follows because $\maxmatnorm{\what{\tbf{B}}} \leq 2\aGM$, $\maxmatnorm{{\tbf{B}}} \leq 2\aGM$ (because of the assumptions in \cref{subsec_theo_gua}), and from \cref{eq_theta_estimate_guarantee,eq_appl_b_defn}, $(f)$ follows from the choice of $\varepsilon_1$, and $(g)$ follows whenever $\varepsilon_2 \leq \frac{1}{8} \sqrt{\frac{p}{p_v+1}}$.

\section{Logarithmic Sobolev inequality and tail bounds} 
\label{section_lsi_tail_bounds}
In this section, we present two results which may be of independent interest. First, we show that a random vector supported on a compact set satisfies the logarithmic Sobolev inequality (to be defined) if it satisfies the Dobrushin's uniqueness condition (to be defined). This result is a generalization of the result in \cite{Marton2015} for discrete random vectors to continuous random vectors supported on a compact set. Next, we show that if a random vector satisfies the logarithmic Sobolev inequality, then any arbitrary function of the random vector concentrates around its mean. This result is a generalization of the result in \cite{DaganDDA2021}  for discrete random vectors to continuous random vectors.\\

\noindent Throughout this section, we consider a $p$-dimensional random vector $\rvbx$ supported on $\cX^p$ with distribution $f_{\rvbx}$ where $p \geq 1$. We start by defining the logarithmic Sobolev inequality (LSI). We use the convention $0\log 0=0$.

\begin{definition}[{Logarithmic Sobolev inequality}]\label{def_lsi}
	A random vector $\rvbx$ satisfies the logarithmic Sobolev inequality with constant $\sigma^2 > 0$ (abbreviated as $\LSI{\rvbx}{\sigma^2}$) if
	\begin{align}
	\Ent{\rvbx}{q^2} \leq \sigma^2 \Expectation_{\rvbx}\Bigbrackets{\twonorm{\nabla_{\rvbx} q(\rvbx)}^2} \qtext{for all} q : \cX^p \to \Reals, \label{eq_LSI_definition}
	\end{align}
	where $\Ent{\rvbx}{g}\!\defn\! \Expectation_{\rvbx}[g(\rvbx) \log g(\rvbx)] \!-\!\Expectation_{\rvbx}[g(\rvbx)] \log \Expectation_{\rvbx}[g(\rvbx)]$ denotes the entropy of the function $g\! :\! \cX^p \!\to\! \real_{+}$.
\end{definition}
\vspace{2mm}
\noindent  
Next, we state the Dobrushin's uniqueness condition. For any distributions $f$ and $g$, let $\TV{f}{g}$ denote the total variation distance between $f$ and $g$.
\begin{definition}[{Dobrushin's uniqueness condition}]\label{def_dobrushin_condition}
	A random vector $\rvbx$ satisfies the Dobrushin's uniqueness condition with coupling matrix $\ParameterMatrix \in \Reals_+^{p \times p}$ if
	$\opnorm{\ParameterMatrix} < 1$, and
	for every $t \in [p], u \in [p] \!\setminus\! \{t\}$, and $\svbx_{-t}, \tsvbx_{-t} \in \cX^{p-1}$ differing only in the $u^{th}$ coordinate,
	\begin{align}
	\TV{f_{\rvx_t | \rvbx_{-t} = \svbx_{-t}}}{f_{\rvx_t | \rvbx_{-t} = \tsvbx_{-t}}} \leq \ParameterTU[tu]. \label{eq_dob_tv_bound}
	\end{align}
\end{definition}
\noindent We note that the Dobrushin's uniqueness condition, as originally stated (see \cite{Marton2015}) for Ising model, also requires $\ParameterTU[tt] = 0$ for all $t \in [p]$. This condition makes sense for Ising model where $\rvx_t^2 = 1$ for all $t \in [p]$. However, this is not true for continuous random vectors necessitating a need for modification in the condition.\\

\vspace{2mm}
\noindent From hereon, we let $\cX^p$ be compact unless otherwise specified. Moreover, we define
\begin{align}
\label{eq:smin}   
f_{\min} \defeq \min_{t \in [p], \svbx \in \cX^p} f_{\rvx_t | \rvbx_{-t}}(x_t | \svbx_{-t}).
\end{align}
\noindent Now, we provide the first main result of this section with a proof in \cref{subsec_proof_lsi}.

\newcommand{\lsiresultname}{Logarithmic Sobolev inequality}
\begin{proposition}[{\lsiresultname}]\label{thm_LSI_main}
	If a random vector $\rvbx$ with $f_{\min}>0$ (see \cref{eq:smin}) satisfies (a) the Dobrushin's uniqueness condition (\cref{def_dobrushin_condition}) with coupling matrix $\ParameterMatrix \in \Reals_+^{p \times p}$, and (b) $\rvx_t | \rvbx_{-t}$ satisfies $\LSI{\rvx_t | \rvbx_{-t} = \svbx_{-t}}{\sigma^2}$ for all $t \in [p]$ and $\svbx_{-t} \in \cX^{p-1}$ (see \cref{def_lsi}), then it satisfies 
	$\mathrm{LSI}_{\rvbx}(2\sigma^2/(f_{\min}(1-\opnorm{\ParameterMatrix})^2))$.
\end{proposition}
\noindent Next, we define the notion of pseudo derivative and pseudo Hessian that come in handy in our proofs for providing upper bounds on the norm of the derivative and the Hessian.

\begin{definition}[{Pseudo derivative and Hessian}]\label{def_pseudo_der_hes}
	For a function $q : \cX^p \to \Reals$, the functions $\tnabla q : \cX^p \to \Reals^{p_1}$ and $\tnabla^2 q : \cX^p \to \Reals^{p_1 \times p_2}$  ($p_1,p_2 \geq 1$) are, respectively, called a pseudo derivative and a pseudo Hessian for $q$ if for all $\svby \in \cX^p$ and $\rho \in \Reals^{p_1 \times 1}$, we have
	\begin{align}
	\stwonorm{\tnabla q(\svby)} \geq  \stwonorm{\nabla q(\svby)}
	\qtext{and}
	\stwonorm{\rho\tp \tnabla^2 q(\svby)} \geq  \stwonorm{\nabla \bigbrackets{\rho\tp \tnabla q(\svby)}}. 
	\label{eq:pseudo_Hessian}
	\end{align}
\end{definition}

\noindent Finally, we provide the second main result of this section with a proof in \cref{subsec_proof_main_concentration}.

\newcommand{\mainconcresultname}{Tail bounds for arbitrary functions under LSI}
\begin{proposition}[{\mainconcresultname}]\label{thm_main_concentration}
	Given a random vector $\rvbx$ satisfying $\LSI{\rvbx}{\sigma^2}$, any function $q :\cX^p \to \Reals$ with a pseudo derivative $\tnabla q$ and pseudo Hessian $\tnabla^2 q$ (see \cref{def_pseudo_der_hes}), $\rvbx$ satisfies a tail bound, namely for any fixed $\varepsilon > 0$, we have
	\begin{align}
	\Probability\Bigbrackets{\bigabs{q_c(\rvbx)} \!\geq\! \varepsilon} \!\leq\! \exp\biggparenth{\! \frac{-c}{\sigma^4} \min \Bigparenth{\frac{\varepsilon^2}{ \Expectation\bigbrackets{\stwonorm{\tnabla q(\rvbx)}}^2 \!+\! \max\limits_{\svbx \in \cX^p} \fronorm{\tnabla^2 q(\svbx)}^2}, \frac{\varepsilon}{\max\limits_{\svbx \in \cX^p} \opnorm{\tnabla^2 q(\svbx)}}}},
	\end{align}
	where $q_c(\rvbx) = q(\rvbx) - \Expectation\bigbrackets{q(\rvbx)}$ and $c$ is a universal constant.
\end{proposition}

\subsection{Proof of \cref{thm_LSI_main}: \lsiresultname}
\label{subsec_proof_lsi}
We start by defining the notion of $W_2$ distance \citep{Marton2015} which is useful in the proof. We note that $W_2$ distance is a metric on the space of probability measures and satisfies triangle inequality.

\begin{definition}\cite[{$W_2$ distance}]{Marton2015}\label{def_w2_distance}
	For random vectors $\rvbx$ and $\rvby$ supported on $\cX^p$ with distributions $f$ and $g$, respectively, the $W_2$ distance is given by
	$ W_2^2(g_{\rvby}, f_{\rvbx}) \defeq \inf_{\pi} \sump \Bigbrackets{\Probability_{\pi}(\rvx_t \neq \rvy_t)}^2$,
	where the infimum is taken over all couplings $\pi(\rvbx, \rvby)$ such that $\pi(\rvbx) = f(\rvbx)$ and $\pi(\rvby) = g(\rvby)$.
\end{definition}
\noindent Given \cref{def_w2_distance}, our next lemma states that if appropriate $W_2$ distances are bounded, then the KL divergence (denoted by $\KLD{\cdot}{\cdot}$) and the entropy approximately tensorize. We provide a proof in \cref{subsec_proof_lemma_tenorization_kld}.

\newcommand{\approxtensorofKLresultname}{Approximate tensorization of KL divergence and entropy}
\begin{lemma}[{\approxtensorofKLresultname}]\label{lemma_tenorization_kld}
	Given random vectors $\rvbx$ and $\rvby$ supported on $\cX^p$ with distributions $f$ and $g$, respectively, such that $f_{\min} > 0$ (see \cref{eq:smin}), if for all subsets $\set \subseteq [p]$ (with $\setC \defn [p] \setminus \set$) and all $\svby_{\setC} \in \cX^{p - |\set|}$,
	\begin{align}
	W_2^2\bigparenth{g_{\rvby_{\set} | \rvby_{\setC} = \svby_{\setC}}, f_{\rvbx_{\set} | \rvbx_{\setC} = \svby_{\setC}}} \!\leq\! C \! \sumset \Expectation\Bigbrackets{\!
 \lVert g_{\rvy_t | \rvby_{-t} = \svby_{-t}} \!\!-\!\! f_{\rvx_t | \rvbx_{-t} = \svby_{-t}} \rVert_{\mathsf{TV}}^2 \Big| \rvby_{\setC} \!=\! \svby_{\setC}\!},
	\label{eq_w2_distance_bounded_assumption}
	\end{align}
	almost surely for some constant $C \geq 1$, then 
	\begin{align}
	\KLD{g_{\rvby}}{f_{\rvbx}} &\leq \frac{2C}{f_{\min}} \sump \Expectation\bigbrackets{\KLD{g_{\rvy_t | \rvby_{-t} = \svby_{-t}}}{f_{\rvx_t | \rvbx_{-t} = \svby_{-t}}}},
	\qtext{and} 
	\label{eq_tensorization_kld}\\
	\Ent{\rvbx}{q} &\leq \frac{2C}{f_{\min}} \sump \Expectation_{\rvbx_{-t}}\bigbrackets{\Ent{\rvx_t | \rvbx_{-t}}{q}}
	\qtext{for any function $q : \cX^p \to \real_{+}$.} \label{eq_tensorization_entropy}
	\end{align}
\end{lemma}
\noindent Next, we claim that if the random vector $\rvbx$ satisfies Dobrushin's uniqueness condition, then the condition \cref{eq_w2_distance_bounded_assumption} of \cref{lemma_tenorization_kld} is naturally satisfied. We provide a proof in \cref{subsec_proof_lemma_dobrushin_implies_tensorization}.

\newcommand{\dobimpliesapproxtensorresultname}{Dobrushin's uniqueness implies approximate tensorization}

\begin{lemma}[{\dobimpliesapproxtensorresultname}]\label{lemma_dobrushin_implies_tensorization}
	Given random vectors $\rvbx$ and $\rvby$ supported on $\cX^p$ with distributions $f$ and $g$, respectively, if $\rvbx$ satisfies Dobrushin's uniqueness condition (see \cref{def_dobrushin_condition}) with coupling matrix $\ParameterMatrix \in \Reals^{p \times p}$, then for all subsets $\set \subseteq [p]$ (with $\setC \defn [p] \setminus \set$) and all $\svby_{\setC} \in \cX^{p - |\set|}$,
\begin{align}
 W_2^2\bigparenth{g_{\rvby_{\set} | \rvby_{\setC} = \svby_{\setC}}, f_{\rvbx_{\set} | \rvbx_{\setC} = \svby_{\setC}}} \!\leq\! C \! \sumset \Expectation\Bigbrackets{\!
 \lVert g_{\rvy_t | \rvby_{-t} = \svby_{-t}} \!\!-\!\! f_{\rvx_t | \rvbx_{-t} = \svby_{-t}} \rVert_{\mathsf{TV}}^2 \Big| \rvby_{\setC} \!=\! \svby_{\setC}\!}, 
 \label{eq_w2_distance_bounded}
\end{align}
almost surely where $C = {\bigparenth{1\!-\!\opnorm{\ParameterMatrix}}^2}$.
\end{lemma}
\noindent 
Now to prove \cref{thm_LSI_main}, applying \cref{lemma_tenorization_kld,lemma_dobrushin_implies_tensorization} for an arbitrary function $f : \cX^p \to \Reals$, we find that
\begin{align}
\Ent{\rvbx}{q^2} & \leq \frac{2}{f_{\min}\bigparenth{1-\opnorm{\ParameterMatrix}}^2} \sump \Expectation_{\rvbx_{-t}}\Bigbrackets{\Ent{\rvx_t | \rvbx_{-t}}{q^2}} \\
& \sless{(a)} \frac{2\sigma^2}{f_{\min}\bigparenth{1-\opnorm{\ParameterMatrix}}^2} \sump \Expectation_{\rvbx_{-t}}\Bigbrackets{ \Expectation_{\rvx_t | \rvbx_{-t}}\Bigbrackets{\twonorm{\nabla_{\rvx_t} q(\rvx_t; \rvbx_{-t})}^2}} \\
& \sequal{(b)} \frac{2\sigma^2}{f_{\min}\bigparenth{1-\opnorm{\ParameterMatrix}}^2} \Expectation_{\rvbx_{-t}}\Bigbrackets{ \Expectation_{\rvx_t | \rvbx_{-t}}\Bigbrackets{ \sump \twonorm{\nabla_{\rvx_t} q(\rvx_t; \rvbx_{-t})}^2 } } \\
& \sequal{(c)} \frac{2\sigma^2}{f_{\min}\bigparenth{1-\opnorm{\ParameterMatrix}}^2} \Expectation_{\rvbx}\Bigbrackets{\twonorm{\nabla_{\rvbx} q(\rvbx)}^2},
\end{align}
where $(a)$ follows because $\rvx_t | \rvbx_{-t}$ satisfies $\LSI{\rvx_t | \rvbx_{-t} = \svbx_{-t}}{\sigma^2}$ for all $t \in [p]$ and $\svbx_{-t} \in \cX^{p-1}$, $(b)$ follows by the linearity of expectation and $(b)$ follows by the law of total expectation. The claim follows.

\subsubsection{Proof of \cref{lemma_tenorization_kld}: \approxtensorofKLresultname}
\label{subsec_proof_lemma_tenorization_kld}
We start by establishing a reverse-Pinsker style inequality for distributions with compact support to bound their KL divergence by their total variation distance. We provide a proof at the end.

\begin{lemma}[{Reverse-Pinsker inequality}]\label{lemma_reverse_pinsker}
	For any distributions $f$ and $g$ supported on $\cX \subset \Reals$ such that $\min_{x \in \cX} f(x) > 0$, we have $\KLD{g}{f} \leq \frac{4}{\min_{x \in \cX} f(x)} \TV{g}{f}^2.$
\end{lemma}

\noindent Given \cref{lemma_reverse_pinsker}, we proceed to prove \cref{lemma_tenorization_kld}. 

\paragraph{Proof of bound~\cref{eq_tensorization_kld}}
To prove \cref{eq_tensorization_kld}, we show that the following inequality holds using the technique of mathematical induction on $p$:
\begin{align}
\KLD{g_{\rvby}}{f_{\rvbx}} \leq \frac{4C}{f_{\min}} \sump \Expectation\Bigbrackets{\TV{g_{\rvy_t | \rvby_{-t} = \svby_{-t}}}{f_{\rvx_t | \rvbx_{-t} = \svby_{-t}}}^2 }. \label{eq_tensor_kl_tv}
\end{align}
Then, \cref{eq_tensorization_kld} follows by using Pinsker's inequality to bound the right hand side of \cref{eq_tensor_kl_tv}.

\paragraph{Base case: $p = 1$} For the base case, we need to establish that the claim holds for all distributions supported on $\cX$ that satisfy the required conditions. In other words, we need to show that
\begin{align}
\KLD{g_{\rvy}}{f_{\rvx}} \leq \frac{4C}{f_{\min}} \TV{g_{\rvy}}{f_{\rvx}}^2 \qtext{for every} t \in [p],
\end{align}
for all random variables $\rvx$ and $\rvy$ supported on $\cX$ such that $f_{\min} = \min_{x \in \cX} f_{\rvx}(x) > 0$.  This follows from \cref{lemma_reverse_pinsker} by observing that $C \geq 1$.

\paragraph{Inductive step} Now, we assume that the claim holds for all  distributions supported on $ \cX^{p-1}$ that satisfy the required conditions, and establish it for distributions supported on $\cX^{p}$. From the chain rule of KL divergence, we have
\begin{align}
\KLD{g_{\rvby}}{f_{\rvbx}} = \KLD{g_{\rvy_t}}{f_{\rvx_t}} + \Expectation \bigbrackets{ \KLD{g_{\rvby_{-t} | \rvy_t}}{f_{\rvbx_{-t} | \rvx_t}}  } \qtext{for every} t \in [p]. \label{eq_kl_chain_rule}
\end{align}
Taking an average over all $t \in [p]$, we have
\begin{align}
\KLD{g_{\rvby}}{f_{\rvbx}} = \frac{1}{p} \sump \KLD{g_{\rvy_t}}{f_{\rvx_t}} + \frac{1}{p} \sump \Expectation \bigbrackets{ \KLD{g_{\rvby_{-t} | \rvy_t}}{f_{\rvbx_{-t} | \rvx_t}}  }. \label{eq_avg_kl_chain_rule}
\end{align}
Now, we bound the first term in \cref{eq_avg_kl_chain_rule}. Let $\pi^*$ be the coupling between $\rvbx$ and $\rvby$ that achieves $W_2(g_{\rvby}, f_{\rvbx})$ i.e.,\footnote{The minimum is achieved by using arguments similar to the ones used to show that the Wasserstein distance attains its minimum \cite[Chapter 4]{villani2009optimal}.}
\begin{align}
\pi^* = \argmin_{\pi: \pi(\rvbx) = f(\rvbx), \pi(\rvby) = g(\rvby)} \sump \Bigbrackets{\Probability_{\pi}(\rvx_t \neq \rvy_t)}^2. \label{eq_opt_coupling}
\end{align}
Then, we have
\begin{align}
\frac{1}{p} \sump \KLD{g_{\rvy_t}}{f_{\rvx_t}} & \sless{(a)} \frac{1}{p} \sump \frac{4}{f_{\min}} \TV{g_{\rvy_t}}{f_{\rvx_t}}^2  \\
& \sless{(b)} \frac{4}{pf_{\min}} \sump \Bigbrackets{\Probability_{\pi^*}(\rvx_t \neq \rvy_t)}^2\\
& \sequal{(c)} \frac{4}{pf_{\min}} W_2^2(g_{\rvby}, f_{\rvbx}) \\
& \sless{\cref{eq_w2_distance_bounded_assumption}} \frac{4C}{pf_{\min}} \sump \Expectation\Bigbrackets{\TV{g_{\rvy_t | \rvby_{-t} = \svby_{-t}}}{f_{\rvx_t | \rvbx_{-t} = \svby_{-t}}}^2 }, \label{eq_avg_kl_bound_2}
\end{align}
where $(a)$ follows from \cref{lemma_reverse_pinsker} because lower bound on conditional implies lower bound on marginals, i.e., $\min_{t \in [p], x_t \in \cX} f_{\rvx_t}(x_t) \!=\! \min_{t \in [p], x_t \in \cX} \int_{\svbx_{-t} \in \cX^{p-1}} f_{\rvx_t | \rvbx_{-t}}(x_t | \svbx_{-t}) f_{\rvbx_{-t}}(\svbx_{-t}) d\svbx_{-t}$ $> f_{\min}$, $(b)$ follows from the connections of total variation distance to optimal transportation cost, i.e., $\TV{g_\rvy}{f_{\rvx}} = \inf_{\pi: \pi(\rvx) = f(\rvx), \pi(\rvy) = g(\rvy)} \Probability_{\pi}(\rvx \neq \rvy)$, and $(c)$ follows from \cref{def_w2_distance,eq_opt_coupling}.\\

\noindent Next, we bound the second term in \cref{eq_avg_kl_chain_rule}. We have
\begin{align}
& \frac{1}{p} \sump \Expectation \bigbrackets{ \KLD{g_{\rvby_{-t} | \rvy_t}}{f_{\rvbx_{-t} | \rvx_t}}} \\
& \sless{(a)} \frac{1}{p} \sump \Expectation \biggbrackets{\frac{4C}{f_{\min}} \sum_{u \in [p] \setminus \{t\} }   \Expectation\Bigbrackets{\TV{g_{\rvy_u | \rvby_{-u} = \svby_{-u}}}{f_{\rvx_u | \rvbx_{-u} = \svby_{-u}}}^2  \Big| \rvy_t = y_t}}\\
& \sequal{(b)} \frac{4C}{pf_{\min}} \sump \sum_{u \in [p] \setminus \{t\}} \Expectation\Bigbrackets{\TV{g_{\rvy_u | \rvby_{-u} = \svby_{-u}}}{f_{\rvx_u | \rvbx_{-u} = \svby_{-u}}}^2 } \\
& = \frac{4C(p-1)}{pf_{\min}} \sum_{u \in [p]} \Expectation\Bigbrackets{\TV{g_{\rvy_u | \rvby_{-u} = \svby_{-u}}}{f_{\rvx_u | \rvbx_{-u} = \svby_{-u}}}^2 }, \label{eq_avg_kl_bound_1}
\end{align}
where $(a)$ follows from the inductive hypothesis and $(b)$ follows from the law of total expectation. Then, \cref{eq_tensor_kl_tv} follows by putting \cref{eq_avg_kl_bound_1,eq_avg_kl_bound_2,eq_avg_kl_chain_rule} together.\\

\paragraph{Proof of bound~\cref{eq_tensorization_entropy}}
To prove \cref{eq_tensorization_entropy}, we note that \cref{eq_tensorization_kld} holds for any random vector $\rvby$ supported on $\cX^p$.
Consider $\rvby$ to be such that $q(\rvbx) / \Expectation_{\rvbx}[q(\rvbx)]$ is the Radon-Nikodym derivative of $g_{\rvby}$ with respect to $f_{\rvbx}$. For any $\cA^p \subseteq \cX^p$, we have
\begin{align}
\int_{\rvby \in \cA^p} g_{\rvby} d\rvby = \int_{\rvbx \in \cA^p} \frac{q(\rvbx)}{ \Expectation_{\rvbx}[q(\rvbx)]} f_{\rvbx} d\rvbx.
\end{align}
Integrating out $\rvy_t$ and $\rvx_t$ for $t \in [p]$, we have
\begin{align}
\int_{\rvby_{-t} \in \cA^{p-1}} g_{\rvby_{-t}} d\rvby_{-t} = \int_{\rvbx_{-t} \in \cA^{p-1}} \frac{\Expectation_{\rvx_t | \rvbx_{-t}}\bigbrackets{q(\rvbx)}}{\Expectation_{\rvbx}\bigbrackets{q(\rvbx)}} f_{\rvbx_{-t}} d\rvbx_{-t},
\end{align}
implying
\begin{align}
\frac{dg_{\rvby_{-t}}}{df_{\rvbx_{-t}}} =\frac{\Expectation_{\rvx_t | \rvbx_{-t}}\bigbrackets{q(\rvbx)}}{\Expectation_{\rvbx}\bigbrackets{q(\rvbx)}} \qtext{and} \frac{dg_{\rvy_t | \rvby_{-t}}}{df_{\rvx_t | \rvbx_{-t}}} = \frac{q(\rvbx)}{\Expectation_{\rvx_t | \rvbx_{-t}}\bigbrackets{q(\rvbx)}} \label{eq_radon_niko_marginal} \qtext{for all} t \in [p].
\end{align}
We have
\begin{align}
\KLD{g_{\rvby}}{f_{\rvbx}} & \sequal{(a)} \Expectation_{\rvbx} \biggbrackets{\frac{dg_{\rvby}}{df_{\rvbx}} \log \frac{dg_{\rvby}}{df_{\rvbx}}}\\
& \sequal{(b)} \Expectation_{\rvbx} \biggbrackets{\frac{q(\rvbx)}{\Expectation_{\rvbx}\bigbrackets{q(\rvbx)}} \log \frac{q(\rvbx)}{\Expectation_{\rvbx}\bigbrackets{q(\rvbx)}}} \\
& = \frac{1}{\Expectation_{\rvbx}\bigbrackets{q(\rvbx)}} \Bigparenth{\Expectation_{\rvbx} \bigbrackets{q(\rvbx) \log q(\rvbx)} - \Expectation_{\rvbx} \bigbrackets{q(\rvbx)} \log \Expectation_{\rvbx} \bigbrackets{q(\rvbx)}} = \frac{\Ent{\rvbx}{q}}{\Expectation_{\rvbx}\bigbrackets{q(\rvbx)}}, \label{eq_kl_entropy_mapping}
\end{align}
where $(a)$ follows from the definition of KL divergence and $(b)$ follows from the choice of $\rvby$. Similarly, for every $t \in [p]$, we have
\begin{align}
& \Expectation_{\rvby_{-t}}\Bigbrackets{\KLD{g_{\rvy_t | \rvby_{-t} = \svby_{-t}}}{f_{\rvx_t | \rvbx_{-t} = \svby_{-t}}}} \\
& \sequal{(a)} \Expectation_{\rvby_{-t}}\biggbrackets{ \Expectation_{\rvy_t | \rvby_{-t}} \biggbrackets{\log \frac{dg_{\rvy_t | \rvby_{-t}}}{df_{\rvx_t | \rvbx_{-t}}} }} \\
& \sequal{(b)} \Expectation_{\rvby} \biggbrackets{\log \frac{dg_{\rvy_t | \rvby_{-t}}}{df_{\rvx_t | \rvbx_{-t}}} }\\
& \sequal{(c)} \Expectation_{\rvbx} \biggbrackets{\frac{dg_{\rvby}}{df_{\rvbx}} \log \frac{dg_{\rvy_t | \rvby_{-t}}}{df_{\rvx_t | \rvbx_{-t}}} } \\
& \sequal{(d)} \Expectation_{\rvbx} \biggbrackets{\frac{q(\rvbx)}{\Expectation_{\rvbx}\bigbrackets{q(\rvbx)}} \log \frac{q(\rvbx)}{\Expectation_{\rvx_t | \rvbx_{-t}}\bigbrackets{q(\rvbx)}}} \\
& \sequal{(e)} \frac{\Expectation_{\rvbx_{-t}}\bigbrackets{\Expectation_{\rvx_t | \rvbx_{-t}} \bigbrackets{q(\rvbx) \log q(\rvbx)} - \Expectation_{\rvx_t | \rvbx_{-t}} \bigbrackets{q(\rvbx) \log \Expectation_{\rvx_t | \rvbx_{-t}}\bigbrackets{q(\rvbx)}} }}{\Expectation_{\rvbx}\bigbrackets{q(\rvbx)}} \\
& \sequal{(f)} \frac{\Expectation_{\rvbx_{-t}}\bigbrackets{\Ent{\rvx_t | \rvbx_{-t}}{q}}}{\Expectation\bigbrackets{q(\rvbx)}}, \label{eq_kl_entropy_conditional_mapping}
\end{align}
where $(a)$ follows from the definition of KL divergence, $(b)$ follows from the law of total expectation, $(c)$ follows from the definition of Radon-Nikodym derivative, $(d)$ follows from the choice of $\rvby$ and \cref{eq_radon_niko_marginal}, $(e)$ follows from the law of total expectation, $(f)$ follows from the definition of entropy. Then, \cref{eq_tensorization_entropy} follows by putting \cref{eq_tensorization_kld,eq_kl_entropy_mapping,eq_kl_entropy_conditional_mapping} together.

\paragraph{Proof of \cref{lemma_reverse_pinsker}: {Reverse-Pinsker inequality}}
\label{subsubsec_proof_reverse_pinsker}
Using the facts (a) $\log a \geq 1 - \frac{1}{a} $ for all $a>0$, and (b) $\min_{x \in \cX} f(x)>0$, we find that
\begin{align}
\log \frac{f(x)}{g(x)} \geq 1 - \frac{g(x)}{f(x)} \qtext{for every} x \in \cX. \label{eq_rp_1}
\end{align}
Multiplying both sides of \cref{eq_rp_1} by $g(x) \geq 0$ and rearranging terms yields that
\begin{align}
g(x) \log  \frac{g(x)}{f(x)} \leq  \frac{g^2(x)}{f(x)} - g(x) \qtext{for every} x \in \cX. \label{eq_rp_2}
\end{align}
Now, we have
\begin{align}
\KLD{g}{f}  = \int_{x \in \cX}\!\!  g(x) \log  \frac{g(x)}{f(x)} dx & \sless{\cref{eq_rp_2}} \int_{x \in \cX} \biggparenth{\frac{g^2(x)}{f(x)} - g(x)} dx 
\\
& \sequal{(a)} \int_{x \in \cX} \frac{\bigparenth{g(x) -f(x)}^2}{f(x)}dx \\
& \leq \frac{1}{\min_{x \in \cX} f(x)} \int_{x \in \cX} \bigparenth{g(x) -f(x)}^2 dx\\
& \sless{(b)} \frac{1}{\min_{x \in \cX} f(x)} \Bigparenth{\int_{x \in \cX} \bigabs{g(x) -f(x)} dx}^2\\
& \sequal{(c)} \frac{1}{\min_{x \in \cX} f(x)} \Bigparenth{2\TV{g}{f}}^2 \\
& =\frac{4}{\min_{x \in \cX} f(x)} \TV{g}{f}^2,
\end{align}
where $(a)$ follows by simple manipulations, $(b)$ follows by using the order of norms on Euclidean space, and $(c)$ follows by the definition of the total variation distance.

\subsubsection{Proof of \cref{lemma_dobrushin_implies_tensorization}: \dobimpliesapproxtensorresultname}
\label{subsec_proof_lemma_dobrushin_implies_tensorization}

We start by defining the notion of Gibbs sampler which is useful in the proof. 

\begin{definition}\cite[{Gibbs Sampler}]{Marton2015}\label{def_gibbs_sampler}
	For a random vector $\rvbx$ with distribution $f$, define the Markov kernels and the Gibbs sampler as follows
	\begin{align}
	\Gamma_t(\svbx | \svbx') \defn \Indicator\normalparenth{\svbx_{-t} = \svbx'_{-t}} f_{\rvx_t | \rvbx_{-t}}(x_t | \svbx'_{-t}) 
	\qtext{and}
	\Gamma(\svbx | \svbx') \defn p\inv \sump \Gamma_t(\svbx | \svbx'),
	\label{eq_gibbs_sampler}
	\end{align}
	for all $t\in[p]$ and $x, x' \in \cX^p$.
	That is, the kernel $\Gamma_t$ leaves all but the $t^{th}$ coordinate unchanged, and updates the $t^{th}$ coordinate according to $f_{\rvx_t | \rvbx_{-t}}$, and the sampler
	$\Gamma$ selects an index $t \in [p]$ at random, and applies $\Gamma_t$. Further, for a random vector $\rvby$ with distribution $g$ supported on $\cX^p$, we also define
	\begin{align}
	g_{\rvby} \Gamma_t(\svby) & \defn \int g_{\rvby}(\svby') \Gamma_t(\svby | \svby') d\svby' \stext{for} t \in [p], \stext{and} \\
    g_{\rvby} \Gamma(\svby) & \defn \int g_{\rvby}(\svby') \Gamma(\svby | \svby') d\svby'
	\qtext{for all} \svby \in \cX^p.
	\label{eq_gibbs_sampler_expander}
	\end{align}
\end{definition}
\noindent We now proceed to prove \cref{lemma_dobrushin_implies_tensorization} and split it in two cases: (i) $\set = [p]$, and (ii) $\set \subset [p]$. 

\paragraph{Case~(i) ($\set=[p]$)} Let $\Gamma$ be the Gibbs sampler associated with the distribution $f$. Then,
\begin{align}
W_2\bigparenth{g_{\rvby_{\set} | \rvby_{\setC}}, f_{\rvbx_{\set} | \rvbx_{\setC}}} & = W_2(g_{\rvby}, f_{\rvbx}) \sless{(a)} W_2(g_{\rvby}, g_{\rvby} \Gamma) + W_2(g_{\rvby} \Gamma, f_{\rvbx}), \label{eq_triangle_w2}
\end{align}
where $(a)$ follows from the triangle inequality. We claim that
\begin{align}
W_2(g_{\rvby}, g_{\rvby} \Gamma) &\leq \frac{1}{p} \sqrt{\sump \Expectation_{\rvby_{-t}} \Bigbrackets{\TV{g_{\rvy_t | \rvby_{-t} = \svby_{-t}}}{f_{\rvx_t | \rvbx_{-t} = \svby_{-t}}}^2}}, \label{eq_coupling_bound_1} \qtext{and} \\
W_2(g_{\rvby} \Gamma, f_{\rvbx}) &\leq \biggparenth{1 - \frac{(1 - \opnorm{\ParameterMatrix})}{p}} W_2(g_{\rvby}, f_{\rvbx}). \label{eq_coupling_bound_2}
\end{align}
Putting \cref{eq_triangle_w2,eq_coupling_bound_1,eq_coupling_bound_2} together, we have
\begin{align}
W_2(g_{\rvby}, f_{\rvbx}) & \leq \frac{1}{p} \sqrt{\sump \Expectation_{\rvby_{-t}} \Bigbrackets{\TV{g_{\rvy_t | \rvby_{-t} = \svby_{-t}}}{f_{\rvx_t | \rvbx_{-t} = \svby_{-t}}}^2}} \\
& \qquad\qquad\qquad\qquad + \biggparenth{1 - \frac{(1 - \opnorm{\ParameterMatrix})}{p}} W_2(g_{\rvby}, f_{\rvbx}). \label{eq_w2_triangle_final}
\end{align}
Rearranging \cref{eq_w2_triangle_final} results in
\cref{eq_w2_distance_bounded} for $S = [p]$ as desired. It remains to prove our earlier claims~\cref{eq_coupling_bound_1,eq_coupling_bound_2} which we now do one-by-one.

\paragraph{Proof of bound~\cref{eq_coupling_bound_1} on $ W_2(g_{\rvby}, g_{\rvby} \Gamma)$} To bound $W_2(g_{\rvby}, g_{\rvby} \Gamma)$, we construct a random vector $\rvby^{\Gamma}$ such that it is coupled with the random vector $\rvby$.
We select an index $b \in [p]$ at random, and define
\begin{align}
y_v^{\Gamma} \defn y_v  \qtext{for all} v \in [p] \setminus \{b\}.
\end{align}
Then, given $b$ and $\rvby_{-b} = \svby_{-b}$, we define the joint distribution of $(\rvy_b, \rvy_b^{\Gamma})$ to be the maximal coupling of $g_{\rvy_b | \rvby_{-b} = \svby_{-b}}$ and $f_{\rvx_b | \rvbx_{-b} = \svby_{-b}}$ that achieves $\TV{g_{\rvy_b | \rvby_{-b} = \svby_{-b}}}{f_{\rvx_b | \rvbx_{-b}=\svby_{-b}}}$. It is easy to see that the marginal distribution of $\rvby$ is $g_{\rvby}$ and the marginal distribution of $\rvby^{\Gamma}$ is $g_{\rvby} \Gamma$ (see \cref{def_gibbs_sampler}). Then, we have
\begin{align}
W_2^2(g_{\rvby}, g_{\rvby} \Gamma) 
& \sless{(a)} \sump \biggbrackets{\Probability(b = t) \Probability(\rvy_t \neq \rvy_t^{\Gamma} | b = t) + \Probability(b \neq t) \Probability(\rvy_t \neq \rvy_t^{\Gamma} | b \neq t)}^2\\
& \sequal{(b)} \sump \biggbrackets{\frac{1}{p} \Probability(\rvy_t \neq \rvy_t^{\Gamma} | b = t)}^2\\
& \sequal{(c)} \frac{1}{p^2} \sump \biggbrackets{ \int\limits_{\svby_{-t} \in \cX^{p-1}}\Probability(\rvy_t \neq \rvy_t^{\Gamma} | b = t, \rvby_{-t} = \svby_{-t}) g_{\rvby_{-t} | b= t}(\svby_{-t} | b = t) d\svby_{-t}}^2\\
& \sequal{(d)} \frac{1}{p^2} \sump \biggbrackets{ \int\limits_{\svby_{-t} \in \cX^{p-1}}\TV{g_{\rvy_t | \rvby_{-t} = \svby_{-t}}}{f_{\rvx_t | \rvbx_{-t} = \svby_{-t}}} g_{\rvby_{-t}}(\svby_{-t}) d\svby_{-t}}^2\\
& = \frac{1}{p^2} \sump \biggbrackets{\Expectation_{\rvby_{-t}} \Bigbrackets{\TV{g_{\rvy_t | \rvby_{-t} = \svby_{-t}}}{f_{\rvx_t | \rvbx_{-t} = \svby_{-t}}}}}^2, \label{eq_coupling_bound_1_interim}
\end{align}
where $(a)$ follows from \cref{def_w2_distance} and the Bayes rule, $(b)$ follows because $\Probability(b = t) = \frac{1}{p}$ and $\Probability(\rvy_t \neq \rvy_t^{\Gamma} | b \neq t) = 0$, $(c)$ follows by the law of total probability, and $(d)$ follows because $g_{\rvby_{-t} | b= t}(\svby_{-t} | b = t) = g_{\rvby_{-t}}(\svby_{-t})$ and by the construction of the coupling between $\rvby$ and $\rvby^{\Gamma}$. Then, \cref{eq_coupling_bound_1} follows by using Jensen's inequality in \cref{eq_coupling_bound_1_interim}.

\paragraph{Proof of bound~\cref{eq_coupling_bound_2} on  $W_2(g_{\rvby} \Gamma, f_{\rvbx})$}
We first show that  $f_{\rvbx}$ is an invariant measure for $\Gamma$, i.e., $f_{\rvbx} = f_{\rvbx} \Gamma$, implying $W_2(g_{\rvby} \Gamma, f_{\rvbx}) = W_2(g_{\rvby} \Gamma, f_{\rvbx} \Gamma)$, and then $\Gamma$ is a contraction with respect to the $W_2$ distance with rate $1 - \frac{(1 - \opnorm{\ParameterMatrix})}{p}$, i.e., $W_2(g_{\rvby} \Gamma, f_{\rvbx} \Gamma) \leq \Bigparenth{1 - \frac{(1 - \opnorm{\ParameterMatrix})}{p}} W_2(g_{\rvby}, f_{\rvbx})$, implying \cref{eq_coupling_bound_2}.

\paragraph{Proof of $f_{\rvbx}$ being an invariant measure for $\Gamma$} We have
\begin{align}
f_{\rvbx} \Gamma (\svbx)  & \sequal{\cref{eq_gibbs_sampler_expander}} \int_{\svbx' \in \cX^p} f_{\rvbx}(\svbx') \Gamma(\svbx | \svbx') d\svbx' \\
& \sequal{\cref{eq_gibbs_sampler}} \int_{\svbx' \in \cX^p} f_{\rvbx}(\svbx') \biggparenth{\frac{1}{p} \sump \Gamma_t(\svbx | \svbx')} d\svbx' \\
& \sequal{\cref{eq_gibbs_sampler}} \frac{1}{p} \sump \int_{\svbx' \in \cX^p} f_{\rvbx}(\svbx') \Indicator\normalparenth{\svbx_{-t}  = \svbx'_{-t}} f_{\rvx_t | \rvbx_{-t}}(x_t | \svbx'_{-t}) d\svbx' \\
& = \frac{1}{p} \sump f_{\rvx_t | \rvbx_{-t}}(x_t | \svbx_{-t}) \int_{x'_t \in \cX} f_{\rvbx}(\svbx_{-t}, x'_t) dx'_t \\
& = \frac{1}{p} \sump f_{\rvx_t | \rvbx_{-t}}(x_t | \svbx_{-t}) f_{\rvbx_{-t}}(\svbx_{-t}) = f_{\rvbx}(\svbx).
\end{align}

\paragraph{Proof of $\Gamma$ being a contraction w.r.t the $W_2$ distance} Let $\pi^*$ be the coupling between $\rvbx$ and $\rvby$ that achieves $W_2(g_{\rvby}, f_{\rvbx})$ i.e.,\footnote{The minimum is achieved by using arguments similar to the ones used to show that the Wasserstein distance attains its minimum \cite[Chapter 4]{villani2009optimal}.}
\begin{align}
\pi^* = \argmin_{\pi: \pi(\rvbx) = f(\rvbx), \pi(\rvby) = g(\rvby)} \sqrt{\sump \Bigbrackets{\Probability_{\pi}(\rvx_t \neq \rvy_t)}^2}. \label{eq_opt_coupling_2}
\end{align}
We construct random variables $\rvbx'$ and $\rvby'$ as well as a coupling $\pi'$ between them such that the marginal distribution of $\rvbx'$ is $f_{\rvbx} \Gamma$ and the marginal distribution of $\rvby'$ is $g_{\rvby} \Gamma$. We start by selecting an index $b \in [p]$ at random, and defining
\begin{align}
y_v' \defn y_v \qtext{and} x_v' \defn x_v \qtext{for all} v \neq b. \label{eq_coupling_contraction}
\end{align}
Then, given $b$, $\rvby_{-b}' = \svby_{-b}$, and $\rvbx_{-b}' = \svbx_{-b}$, we define the joint distribution of $(\rvy_b', \rvx_b')$ to be the maximal coupling of $f_{\rvx_b | \rvbx_{-b}}(\cdot | \svby_{-b})$ and $f_{\rvx_b | \rvbx_{-b}}(\cdot | \svbx_{-b})$ that achieves $\TV{f_{\rvx_b | \rvbx_{-b} = \svby_{-b}}}{f_{\rvx_b | \rvbx_{-b} = \svbx_{-b}}}$. \\

\noindent Now, for every $t \in [p]$, we bound $\Probability_{\pi'}(\rvy_t' \neq \rvx_t')$ in terms of $\Probability_{\pi^*}(\rvy_t \neq \rvx_t)$. To that end, we have
\begin{align}
\Probability_{\pi'}(\rvy_t' \neq \rvx_t') & \sequal{(a)} \Probability(b = t) \Probability_{\pi'}(\rvy_t' \neq \rvx_t' | b = t) + \Probability(b \neq t) \Probability_{\pi'}(\rvy_t' \neq \rvx_t' | b \neq t) \\
& \sequal{(b)} 
\frac{1}{p} \Probability_{\pi'}(\rvy_t' \neq \rvx_t' | b = t) + \Bigparenth{1-\frac{1}{p}} \Probability_{\pi^*}(\rvy_t \neq \rvx_t), \label{eq_bayes_rule_step_1}
\end{align}
where $(a)$ follows from the Bayes rule and $(b)$ follows because $\Probability(b = t) = \frac{1}{p}$ and \cref{eq_coupling_contraction}. Focusing on $\Probability_{\pi'}(\rvy_t' \neq \rvx_t' | b = t)$ and using the law of total probability, we have
\begin{align}
& \Probability_{\pi'}(\rvy_t' \neq \rvx_t' | b = t) \\
& =   \int\limits_{\svby_{-t}, \svbx_{-t} \in \cX^{p-1}}   \Probability_{\pi'}(\rvy_t' \neq \rvx_t' | b \!=\! t, \rvby_{-t}' \!=\! \svby_{-t}, \rvbx_{-t}' \!=\! \svbx_{-t}) \pi'_{\rvby_{-t}', \rvbx_{-t}' | b= t}(\svby_{-t}, \svbx_{-t} | b \!=\! t) d\svby_{-t} d\svbx_{-t} \\
& \sequal{(a)}   \int\limits_{\svby_{-t}, \svbx_{-t} \in \cX^{p-1}}   \TV{f_{\rvx_t | \rvbx_{-t} = \svby_{-t}}}{f_{\rvx_t | \rvbx_{-t} = \svbx_{-t}}} \pi^*_{\rvby_{-t}, \rvbx_{-t}}(\svby_{-t}, \svbx_{-t}) d\svby_{-t} d\svbx_{-t} \\
&  =  \Expectation_{\pi^*_{\rvby_{-t}, \rvbx_{-t}}} \Bigbrackets{\TV{f_{\rvx_t | \rvbx_{-t} = \svby_{-t}}}{f_{\rvx_t | \rvbx_{-t} = \svbx_{-t}}} } \label{eq_ltp_0}
\end{align}
where $(a)$ follows by the construction of the coupling between $\rvby'$ and $\rvbx'$. Now, using the triangle inequality in \cref{eq_ltp_0}, we have
\begin{align}
\Probability_{\pi'}(\rvy_t' \neq \rvx_t' | b = t) & \leq \Expectation_{\pi^*_{\rvby_{-t}, \rvbx_{-t}}} \Bigbrackets{ \sum_{u \in [p] \setminus \{t\}} \!\!\!  \Indicator(r_v \!=\! s_v \!=\! y_v \forall v \!<\! u) \Indicator(r_v \!=\! s_v \!=\! x_v \forall v \!>\! u) ~~ \times \\
& \qquad \qquad \qquad \qquad \Indicator(r_u \!=\! y_u, x_u \!=\! s_u)  \TV{f_{\rvx_t | \rvbx_{-t} = \boldsymbol{r}_{-t}}}{f_{\rvx_t | \rvbx_{-t} = \boldsymbol{s}_{-t}}}} \\
& \sless{\cref{eq_dob_tv_bound}}  \Expectation_{\pi^*_{\rvby_{-t}, \rvbx_{-t}}} \Bigbrackets{ \sum_{u \in [p] \setminus \{t\}} \!\!\!\!\! \ParameterTU[tu] \Indicator(\rvy_u \neq \rvx_u)} = \sum_{u \in [p] \setminus \{t\}} \!\!\!\!\! \ParameterTU[tu] \Probability_{\pi^*}(\rvy_u \neq \rvx_u). \label{eq_bayes_rule_step_2}
\end{align}
Putting together \cref{eq_bayes_rule_step_1,eq_bayes_rule_step_2}, we have
\begin{align}
\Probability_{\pi'}(\rvy_t' \neq \rvx_t') \leq  
\frac{1}{p} \sum_{u \in [p] \setminus \{t\}} \ParameterTU[tu] \Probability_{\pi^*}(\rvy_u \neq \rvx_u) + \Bigparenth{1-\frac{1}{p}} \Probability_{\pi^*}(\rvy_t \neq \rvx_t). \label{eq_prob_coupled_bound}
\end{align}
Next, we use \cref{eq_prob_coupled_bound} to show contraction of $\Gamma$. To that end, we define $\diag \in \Reals^{p \times p}$ to be the matrix with diagonal same as $\ParameterMatrix$ and all non-diagonal entries equal to zeros. Then, we have
\begin{align}
W^2_2(g_{\rvby} \Gamma, f_{\rvbx} \Gamma)  & \sless{(a)} \sump \Bigbrackets{\Probability_{\pi'}(\rvy_t' \neq \rvx_t')}^2  \\
& \sless{\cref{eq_prob_coupled_bound}} \sump \biggbrackets{\frac{1}{p} \sum_{u \in [p] \setminus \{t\}} \ParameterTU[tu] \Probability_{\pi^*}(\rvy_u \neq \rvx_u) + \Bigparenth{1-\frac{1}{p}} \Probability_{\pi^*}(\rvy_t \neq \rvx_t)}^2\\
& \sless{(b)} \bopnorm{\Bigparenth{1-\frac{1}{p}} I + \frac{1}{p} \Bigparenth{\ParameterMatrix - \diag} }^2 \sump \Bigbrackets{\Probability_{\pi^*}(\rvy_t \neq \rvx_t)}^2 \\
& \sequal{(c)} \bopnorm{\Bigparenth{1-\frac{1}{p}} I + \frac{1}{p} \Bigparenth{\ParameterMatrix - \diag}}^2 W^2_2(g_{\rvby}, f_{\rvbx})\\
& \sless{(d)} \biggparenth{\Bigparenth{1-\frac{1}{p}} + \frac{1}{p} \opnorm{\ParameterMatrix - \diag}}^2 W^2_2(g_{\rvby}, f_{\rvbx}) \\
& \sless{(e)} \biggparenth{\Bigparenth{1-\frac{1}{p}} + \frac{1}{p} \opnorm{\ParameterMatrix}}^2 W^2_2(g_{\rvby}, f_{\rvbx}) ,\label{eq_coupling_bound_2_interim}
\end{align}
where $(a)$ follows from \cref{def_w2_distance}, $(b)$ follows by some linear algebraic manipulations, $(c)$ follows from \cref{def_w2_distance} and \cref{eq_opt_coupling_2},  $(d)$ follows from the triangle inequality, and $(e)$ follows because $\opnorm{\tbf{M}_1} \leq \opnorm{\tbf{M}_2}$ for any matrices $\tbf{M}_1$ and $\tbf{M}_2$ such that $0 \leq \tbf{M}_1 \leq \tbf{M}_2$ (component-wise). Then, contraction of $\Gamma$ follows by taking square root on both sides of \cref{eq_coupling_bound_2_interim}.

\paragraph{Case~(ii) ($\set\subset[p]$)} We can directly verify that the matrix $\ParameterMatrix_{\set} \defn \braces{\ParameterTU[tu]}_{t,u \in \set}$ is such that $\opnorm{\ParameterMatrix_{\set}} \leq \opnorm{\ParameterMatrix}$ This is true because the operator norm of any sub-matrix is no more than the operator norm of the matrix. Further, we note that for any $\svby_{\setC} \in \cX^{p-\normalabs{\set}}$, the random vector $\rvbx_{\set} | \rvbx_{\setC} = \svby_{\setC}$ with distribution $f_{\rvbx_{\set} | \rvbx_{\setC} = \svby_{\setC}}$ satisfies the Dobrushin's uniqueness condition (\cref{def_dobrushin_condition}) with coupling matrix $\ParameterMatrix_{\set}$. Then, by performing an analysis similar to the one above, we have
\begin{align}
W_2\bigparenth{g_{\rvby_{\set} | \rvby_{\setC}}, f_{\rvbx_{\set} | \rvbx_{\setC}}} & \leq \frac{1}{\bigparenth{1-\opnorm{\ParameterMatrix_{\set}}}} \sqrt{\sumset \Expectation\Bigbrackets{\TV{g_{\rvy_t | \rvby_{-t} = \svby_{-t}}}{f_{\rvx_t | \rvbx_{-t} = \svby_{-t}}}^2  \Big| \rvby_{\setC} = \svby_{\setC}}} \\
& \sless{(a)} \frac{1}{\bigparenth{1-\opnorm{\ParameterMatrix}}} \sqrt{\sumset \Expectation\Bigbrackets{\TV{g_{\rvy_t | \rvby_{-t} = \svby_{-t}}}{f_{\rvx_t | \rvbx_{-t} = \svby_{-t}}}^2  \Big| \rvby_{\setC} = \svby_{\setC}}},
\end{align}
where $(a)$ follows because $\frac{1}{\normalparenth{1-\opnorm{\ParameterMatrix_{\set}}}} \leq \frac{1}{\normalparenth{1-\opnorm{\ParameterMatrix}}}$. This completes the proof.

\subsection{Proof of \cref{thm_main_concentration}: \mainconcresultname}
\label{subsec_proof_main_concentration}
Fix a function $q : \cX^p \to \Reals$. Fix any pseudo derivative $\tnabla q$ for $q$ and any pseudo Hessian $\tnabla^2 q$ for $q$. To prove \cref{thm_main_concentration}, we bound the $p$-th moment of $q(\rvbx) - \Expectation\bigbrackets{q(\rvbx)}$ by certain norms of $\tnabla^2 q$ and $\Expectation_{\rvbx}\bigbrackets{\tnabla q(\rvbx)}$. To that end, first, we claim that in order to control the $p$-th moment of $q(\rvbx) - \Expectation\bigbrackets{q(\rvbx)}$, it is sufficient to control the $p$-th moment of $\twonorm{\nabla q(\rvbx)}$. Then, using \cref{eq:pseudo_Hessian}, we note that the $p$-th moment of $\twonorm{\nabla q(\rvbx)}$ is bounded by the $p$-th moment of $\stwonorm{\tnabla q(\rvbx)}$. Next, we claim that the $p$-th moment of $\stwonorm{\tnabla q(\rvbx)}$ is bounded by a linear combination of appropriate norms of $\tnabla^2 q$ and $\Expectation_{\rvbx}\bigbrackets{\tnabla q(\rvbx)}$. We formalize the claims below and divide the proof across \cref{subsec_proof_lemma_bounded_p_moment} and \cref{subsec_proof_eq_p_moment_of_2_norm_gradient}.

\begin{lemma}[{Bounded $p$-th moments of $q(\rvbx) - \Expectation\bigbrackets{q(\rvbx)}$ and $\stwonorm{\tnabla q(\rvbx)}$}]\label{lemma_bounded_p_moment}
	If a random vector $\rvbx$ satisfies $\LSI{\rvbx}{\sigma^2}$, then for any arbitrary function $q :\cX^p \to \Reals$,
	\begin{align}
	\moment{q(\rvbx) - \Expectation\bigbrackets{q(\rvbx)}}{p} \leq \sigma \sqrt{2p} \moment{\twonorm{\nabla q(\rvbx)}}{p} \qtext{for any $p \geq 2$.} \label{eq_moment_controlled_by_gradient}
	\end{align}
	Further, for any pseudo derivative $\tnabla q(\svbx)$ and any pseudo Hessian $\tnabla^2 q(\svbx)$ for $q$, and even $p \geq 2$,
	\begin{align}
	\smoment{\stwonorm{\tnabla q(\rvbx)}}{p} \!\!\leq\!  2c \sigma \bigparenth{\!\max_{\svbx \in \cX^p} \fronorm{\tnabla^2 q(\svbx)} \!+\! \sqrt{p}  \max_{\svbx \in \cX^p} \opnorm{\tnabla^2 q(\svbx)}} \!+\! 4 \stwonorm{\Expectation_{\rvbx}\bigbrackets{\tnabla q(\rvbx)}},
	\label{eq_p_moment_of_2_norm_gradient} 
	\end{align}
	where $c \geq 0$ is a universal constant.
\end{lemma}

\noindent Given these lemmas, we proceed to prove \cref{thm_main_concentration}. We let $q_c(\rvbx) = q(\rvbx) - \Expectation\bigbrackets{q(\rvbx)}$. Combining \cref{eq_moment_controlled_by_gradient,eq_p_moment_of_2_norm_gradient} for any even $p \geq 2$, there exists a universal constant $c'$ such that
\begin{align}
\moment{q_c(\rvbx)}{p} \!\leq \! c' \sigma^2  \Bigparenth{\! \sqrt{p}  \max_{\svbx \in \cX^p} \fronorm{\tnabla^2 q(\svbx)} \!+\! p \max_{\svbx \in \cX^p} \opnorm{\tnabla^2 q(\svbx)} \!+\! \sqrt{p}  \stwonorm{\Expectation_{\rvbx}\bigbrackets{\tnabla q(\rvbx)}}}. \label{eq_lemma_bounded_moments_combined}
\end{align}
Now, we complete the proof by using \cref{eq_lemma_bounded_moments_combined} along with Markov's inequality for a specific choice of $p$. For any even $p \geq 2$, we have 
\begin{align}
& \Probability\Bigbrackets{\bigabs{q_c(\rvbx)} > ec' \sigma^2  \Bigparenth{\sqrt{p}  \max_{\svbx \in \cX^p} \fronorm{\tnabla^2 q(\svbx)} + p \max_{\svbx \in \cX^p} \opnorm{\tnabla^2 q(\svbx)} + \sqrt{p}  \stwonorm{\Expectation_{\rvbx}\bigbrackets{\tnabla q(\rvbx)}}}} \\
& = \Probability\Bigbrackets{\bigabs{q_c(\rvbx)}^p \!>\! \bigparenth{ec' \sigma^2}^p  \bigparenth{\! \sqrt{p}  \max_{\svbx \in \cX^p} \fronorm{\tnabla^2 q(\svbx)} \!+\! p \max_{\svbx \in \cX^p} \opnorm{\tnabla^2 q(\svbx)} \!+\!\! \sqrt{p}  \stwonorm{\Expectation_{\rvbx}\bigbrackets{\tnabla q(\rvbx)}}}^p} \\
& \sless{(a)}  \frac{\Expectation{\bigbrackets{\bigabs{q_c(\rvbx)}^p}}}{\bigparenth{ec' \sigma^2}^p  \bigparenth{\sqrt{p}  \max_{\svbx \in \cX^p} \fronorm{\tnabla^2 q(\svbx)} + p \max_{\svbx \in \cX^p} \opnorm{\tnabla^2 q(\svbx)} + \sqrt{p}  \stwonorm{\Expectation_{\rvbx}\bigbrackets{\tnabla q(\rvbx)}}}^p}\\
& \sless{\cref{eq_lemma_bounded_moments_combined}} e^{-p},
\end{align}
where $(a)$ follows from Markov's inequality. The proof is complete by choosing an appropriate universal constant $c''$, and and performing basic algebraic manipulations after letting
\begin{align}
p = \frac{1}{c''\sigma^2}\min \Bigparenth{\dfrac{\varepsilon^2}{ \Expectation\bigbrackets{\stwonorm{\tnabla q(\rvbx)}}^2 + \max\limits_{\svbx \in \cX^p} \fronorm{\tnabla^2 q(\svbx)}^2}, \dfrac{\varepsilon}{\max\limits_{\svbx \in \cX^p} \opnorm{\tnabla^2 q(\svbx)}}}.
\end{align}
We note that an even $p \geq 2$ can be ensured by choosing appropriate $c''$.

\subsubsection{Proof of \cref{lemma_bounded_p_moment}\cref{eq_moment_controlled_by_gradient}: {Bounded $p$-th moment of $q(\rvbx) - \Expectation\bigbrackets{q(\rvbx)}$}}
\label{subsec_proof_lemma_bounded_p_moment}
Fix any $p \geq 2$. We start by using the following result from \cite[Theorem 3.4]{AidaS1994} since $\rvbx$ satisfies $\LSI{\rvbx}{\sigma^2}$:
\begin{align}
\moment{q(\rvbx) - \Expectation\bigbrackets{q(\rvbx)}}{p}^2 \leq & \moment{q(\rvbx) - \Expectation\bigbrackets{q(\rvbx)}}{2}^2 + 2\sigma^2 (p-2)   \moment{\twonorm{\nabla q(\rvbx)}}{p}^2. \label{eq_aida_stroock}
\end{align}
Then, we bound the first term in \cref{eq_aida_stroock} by using the fact that logarithmic Sobolev inequality implies Poincare inequality with the same constant:
\begin{align}
\moment{q(\rvbx) - \Expectation\bigbrackets{q(\rvbx)}}{2}^2 = \Variance(q(\rvbx)) \leq \sigma^2  \Expectation_{\rvbx}\Bigbrackets{\twonorm{\nabla q(\rvbx)}^2}. \label{eq_poincare}
\end{align}
Putting together \cref{eq_aida_stroock,eq_poincare}, we have
\begin{align}
\moment{q(\rvbx) - \Expectation\bigbrackets{q(\rvbx)}}{p}^2 & \leq  \sigma^2 \Expectation_{\rvbx}\Bigbrackets{\twonorm{\nabla q(\rvbx)}^2} + 2\sigma^2 (p-2)   \moment{\twonorm{\nabla q(\rvbx)}}{p}^2 \\
& \sless{(a)}  \sigma^2  \Bigparenth{\Expectation_{\rvbx}\Bigbrackets{\twonorm{\nabla q(\rvbx)}^p}}^{2/p} + 2\sigma^2 (p-2)   \moment{\twonorm{\nabla q(\rvbx)}}{p}^2 \\
& \sequal{(b)} \sigma^2  \moment{\twonorm{\nabla q(\rvbx)}}{p}^2 + 2\sigma^2 (p-2)   \moment{\twonorm{\nabla q(\rvbx)}}{p}^2 \\
& \leq 2\sigma^2 p \moment{\twonorm{\nabla q(\rvbx)}}{p}^2, \label{eq_aida_stroock_simplified}
\end{align}
where $(a)$ follows by Jensen's inequality and $(b)$ follows by the definition of $p$-th moment. Taking square root on both sides of \cref{eq_aida_stroock_simplified} completes the proof.

\subsubsection{Proof of \cref{lemma_bounded_p_moment}\cref{eq_p_moment_of_2_norm_gradient}: {Bounded $p$-th moment of $\stwonorm{\tnabla q(\rvbx)}$}}
\label{subsec_proof_eq_p_moment_of_2_norm_gradient}
Fix any even $p \geq 2$. Fix any pseudo derivative $\tnabla q$ and any pseudo Hessian $\tnabla^2 q$. We start by obtaining a convenient bound on $\stwonorm{\tnabla q(\svbx)}$ for every $\svbx \in \cX^p$ and then proceed to bound the $p$-th moment of $\stwonorm{\tnabla q(\rvbx)}$.\\

\noindent Consider a $p$-dimensional standard normal random vector $\rvb{g}$ independent of $\rvbx$. For a given $\rvbx = \svbx \in \cX^p$, the random variable $\frac{\tnabla q(\svbx)\tp \rvb{g}}{\stwonorm{\tnabla q(\svbx)}}$ is a standard normal random variable. 
Then, for every $\svbx \in \cX^p$, we have
\begin{align}
\moment{\frac{\tnabla q(\svbx)\tp \rvb{g}}{\stwonorm{\tnabla q(\svbx)}}}{p} \sequal{(a)} \biggparenth{\Expectation_{\rvb{g} | \rvbx = \svbx }\biggbrackets{\biggparenth{\frac{\tnabla q(\svbx)\tp \rvb{g}}{\stwonorm{\tnabla q(\svbx)}}}^p}}^{1/p} \sgreat{(b)} \frac{\sqrt{p}}{2}, \label{eq_normal_application}
\end{align}
where $(a)$ follows from the definition of $p$-th moment, and $(b)$ follows since $\moment{\rv{g}}{p} \geq \frac{\sqrt{p}}{2}$ for any standard normal random variable $\rv{g}$ and even $p \geq 2$. Rearranging \cref{eq_normal_application}, we have
\begin{align}
\stwonorm{\tnabla q(\svbx)} \leq \frac{2}{\sqrt{p}} \Bigparenth{\Expectation_{\rvb{g}| \rvbx = \svbx}\Bigbrackets{\bigparenth{\tnabla q(\svbx)\tp \rvb{g}}^p}}^{1/p}. \label{eq_bound_fixed_gradient}
\end{align}
Now, we proceed to bound the $p$-th moment of $\stwonorm{\tnabla q(\rvbx)}$ as follows
\begin{align}
\smoment{\stwonorm{\tnabla q(\rvbx)}}{p}  & \sequal{(a)} \Bigparenth{\Expectation_{\rvbx} \bigbrackets{\stwonorm{\tnabla q(\rvbx)}^p}}^{1/p} \\
& \sless{\cref{eq_bound_fixed_gradient}} \frac{2}{\sqrt{p}} \Bigparenth{\Expectation_{\rvbx, \rvb{g}}\Bigbrackets{\bigparenth{\tnabla q(\rvbx)\tp \rvb{g}}^p}}^{1/p} \\
& \sequal{(b)} \frac{2}{\sqrt{p}} \moment{\tnabla q(\rvbx)\tp \rvb{g}}{p} \\
& \sless{(c)} \frac{2}{\sqrt{p}} \Bigparenth{\!\moment{\tnabla q(\rvbx)\tp \rvb{g} \!-\! \Expectation_{\rvbx}\bigbrackets{\tnabla q(\rvbx)\tp \rvb{g}}}{p} \!+\! \moment{\Expectation_{\rvbx}\bigbrackets{\tnabla q(\rvbx)\tp \rvb{g}}}{p}\!}, \label{eq_bound_gradient_minkowski}
\end{align}
where $(a)$ and $(b)$ follow from the definition of $p$-th moment and $(c)$ follows by Minkowski's inequality. We claim that
\begin{align}
\moment{\tnabla q(\rvbx)\!\tp \! \rvb{g} \!-\! \Expectation_{\rvbx}\bigbrackets{\tnabla q(\rvbx)\!\tp \! \rvb{g}}}{p} & \!\!\!\!\! \leq \! c \sigma \Bigparenth{\!\! \sqrt{p} \max_{\svbx \in \cX^p} \fronorm{\tnabla^2 \! q(\svbx)} \!+\! p \max_{\svbx \in \cX^p} \opnorm{\tnabla^2 q(\svbx)\!}\!}, \text{ \&} \!\label{eq_bound_first_term_minkowski}\\
\moment{\Expectation_{\rvbx}\bigbrackets{\tnabla q(\rvbx)\tp \rvb{g}}}{p} & \!\!\!\!\! \leq \! 2\sqrt{p} \twonorm{\Expectation_{\rvbx}\bigbrackets{\tnabla q(\rvbx)}}, \label{eq_bound_second_term_minkowski}
\end{align}
where $c \geq 0$ is a universal constant.
Putting together \cref{eq_bound_gradient_minkowski,eq_bound_first_term_minkowski,eq_bound_second_term_minkowski} completes the proof. It remains to prove our claims \cref{eq_bound_first_term_minkowski,eq_bound_second_term_minkowski} which we now do one-by-one.

\paragraph{Proof of bound \cref{eq_bound_first_term_minkowski}} To start, we bound $\bigparenth{\Expectation_{\rvbx | \rvb{g} = \svb{g}} \bigbrackets{\bigparenth{\tnabla q(\rvbx)\tp \svb{g} - \Expectation_{\rvbx| \rvb{g} = \svb{g}}\bigbrackets{\tnabla q(\rvbx)\tp \svb{g}}}^p}}^{1/p}$ for every $\rvb{g} = \svb{g}$, and then proceed to bound $\smoment{\tnabla q(\rvbx)\tp \rvb{g} - \Expectation_{\rvbx}\bigbrackets{\tnabla q(\rvbx)\tp \rvb{g}}}{p}$.\\

\noindent To that end, we define $h_{\svb{g}}(\rvbx) \defn \tnabla q(\rvbx)\tp \svb{g} - \Expectation_{\rvbx| \rvb{g} = \svb{g}}\bigbrackets{\tnabla q(\rvbx)\tp \svb{g}}$ and observe that $\Expectation_{\rvbx| \rvb{g} = \svb{g}}\bigbrackets{h_{\svb{g}}(\rvbx)}$ $ = 0$. Now, applying \cref{lemma_bounded_p_moment}~\cref{eq_moment_controlled_by_gradient} to $h_{\svb{g}}(\cdot)$, we have
\begin{align}
\moment{h_{\svb{g}}(\rvbx)}{p} 
\leq \sigma \sqrt{2p} \Bigparenth{\Expectation_{\rvbx| \rvb{g} = \svb{g}} \Bigbrackets{\twonorm{\nabla h_{\svb{g}}(\rvbx)}^p}}^{1/p} & \sless{(a)} \sigma \sqrt{2p} \Bigparenth{\Expectation_{\rvbx| \rvb{g} = \svb{g}} \Bigbrackets{\twonorm{\nabla \bigbrackets{\svb{g} \tp \tnabla q(\rvbx)}}^p}}^{1/p} \\
& \sless{\cref{eq:pseudo_Hessian}} \sigma \sqrt{2p} \Bigparenth{\Expectation_{\rvbx| \rvb{g} = \svb{g}} \Bigbrackets{\twonorm{\svb{g}\tp \tnabla^2 q(\rvbx)}^p}}^{1/p}, \label{eq_first_term_minkowski_0}
\end{align}
where $(a)$ follows from the definition of $h_{\svb{g}}(\rvbx)$. Now, to obtain a bound on the RHS of \cref{eq_first_term_minkowski_0}, we further fix $\rvbx = \svbx$. Then, we let $\rvb{g}'$ be another $p$-dimensional standard normal vector and apply an inequality similar to \cref{eq_bound_fixed_gradient} to $\svb{g}\tp \tnabla^2 q(\svbx)$ obtaining
\begin{align}
\twonorm{\svb{g}\tp \tnabla^2 q(\svbx)} \leq \frac{2}{\sqrt{p}} \Bigparenth{\Expectation_{\rvb{g}'| \rvbx = \svbx, \rvb{g} = \svb{g}}\Bigbrackets{\Bigparenth{\svb{g}\tp \tnabla^2 q(\svbx) \rvb{g}'}^p}}^{1/p}, \label{eq_first_term_minkowski_1}
\end{align}
which implies
\begin{align}
\Bigparenth{\Expectation_{\rvbx | \rvb{g} = \svb{g}} \Bigbrackets{\twonorm{\svb{g}\tp \tnabla^2 q(\rvbx)}^p}}^{1/p} \leq \frac{2}{\sqrt{p}} \Bigparenth{\Expectation_{\rvbx, \rvb{g}' |\rvb{g} = \svb{g}}\Bigbrackets{\Bigparenth{\nabla \svb{g}\tp \tnabla^2 q(\rvbx) \rvb{g}'}^p}}^{1/p}. \label{eq_first_term_minkowski_2}
\end{align}
Putting together \cref{eq_first_term_minkowski_0,eq_first_term_minkowski_2}, and using the definition of $h_{\svb{g}}(\rvbx)$, we have
\begin{align}
\Expectation_{\rvbx|\rvb{g} = \svb{g}} \Bigbrackets{\Bigparenth{\tnabla q(\rvbx)\tp \svb{g} \!-\! \Expectation_{\rvbx|\rvb{g} = \svb{g}}\Bigbrackets{\tnabla q(\rvbx)\tp \svb{g}}}^p}
\!\! \leq (2\sqrt{2}\sigma)^p \Expectation_{\rvbx, \rvb{g}'|\rvb{g} = \svb{g}}\Bigbrackets{\Bigparenth{\svb{g}\!\tp \tnabla^2 q(\rvbx) \rvb{g}'}^p}. \label{eq_bound_first_term_fixed_g}
\end{align}
Now, we proceed to bound $\smoment{\tnabla q(\rvbx)\tp \rvb{g} - \Expectation_{\rvbx}\bigbrackets{\tnabla q(\rvbx)\tp \rvb{g}}}{p}$ as follows
\begin{align}
\moment{\tnabla q(\rvbx)\tp \rvb{g} - \Expectation_{\rvbx}\Bigbrackets{\tnabla q(\rvbx)\tp \rvb{g}}}{p} & \sequal{(a)} \Bigparenth{\Expectation_{\rvbx, \rvb{g}} \Bigbrackets{\Bigparenth{\tnabla q(\rvbx)\tp \rvb{g} - \Expectation_{\rvbx}\Bigbrackets{\tnabla q(\rvbx)\tp \rvb{g}}}^p}}^{1/p} \\
& \sless{\cref{eq_bound_first_term_fixed_g}} 2\sqrt{2}\sigma \Bigparenth{\Expectation_{\rvb{g}, \rvbx, \rvb{g}'}\Bigbrackets{\Bigparenth{\rvb{g}\tp \tnabla^2 q(\rvbx) \rvb{g}'}^p}}^{1/p}, \label{eq_bound_first_term_interim}
\end{align}
where $(a)$ follows from the definition of $p$-th moment. Finally, to bound the RHS of \cref{eq_bound_first_term_interim}, we fix $\rvbx = \svbx$ and bound the $p$-th norm of the quadratic form $\rvb{g}\tp \tnabla^2 q(\svbx) \rvb{g}'$ by the Hanson-Wright inequality resulting in
\begin{align}
\Bigparenth{\Expectation_{\svb{g},  \rvb{g}' | \rvbx = \svbx}\Bigbrackets{\Bigparenth{\rvb{g}\tp \tnabla^2 q(\svbx) \rvb{g}'}^p}}^{1/p} & \leq c \Bigparenth{\sqrt{p} \fronorm{\tnabla^2 q(\svbx)} + p  \opnorm{\tnabla^2 q(\svbx)}}\\
& \leq c \Bigparenth{\sqrt{p} \max_{\svbx \in \cX^p} \fronorm{\tnabla^2 q(\svbx)} + p \max_{\svbx \in \cX^p} \opnorm{\tnabla^2 q(\svbx)}}, \label{eq_hanson_wright}
\end{align}
where $c \geq 0$ is a universal constant. Then, \cref{eq_bound_first_term_minkowski} follows by putting together \cref{eq_bound_first_term_interim,eq_hanson_wright}.

\paragraph{Proof of bound \cref{eq_bound_second_term_minkowski}} By linearity of expectation, we have
\begin{align}
\smoment{\Expectation_{\rvbx}\bigbrackets{\tnabla q(\rvbx)\tp \rvb{g}}}{p} = \smoment{\bigparenth{\Expectation_{\rvbx}\bigbrackets{\tnabla q(\rvbx)}}\tp \rvb{g}}{p}. \label{eq_bound_second_term_minkowski_0}
\end{align}
We note that the random variable $\dfrac{\normalparenth{\Expectation_{\rvbx}\normalbrackets{\tnabla q(\rvbx)}}\tp \rvb{g}}{\stwonorm{\Expectation_{\rvbx}\normalbrackets{\tnabla q(\rvbx)}}}$ is a standard normal random variable. Therefore,
\begin{align}
\moment{\frac{\bigparenth{\Expectation_{\rvbx}\bigbrackets{\tnabla q(\rvbx)}}\tp \rvb{g}}{\stwonorm{\Expectation_{\rvbx}\bigbrackets{\tnabla q(\rvbx)}}}}{p} \sequal{(a)} \biggparenth{\Expectation_{\rvb{g}}\biggbrackets{\biggparenth{\frac{\bigparenth{\Expectation_{\rvbx}\bigbrackets{\tnabla q(\rvbx)}}\tp \rvb{g}}{\stwonorm{\Expectation_{\rvbx}\bigbrackets{\tnabla q(\rvbx)}}}}^p}}^{1/p} \sless{(b)} 2\sqrt{p}, \label{eq_normal_application_2}
\end{align}
where $(a)$ follows from the definition of $p$-th moment, and $(b)$ follows since $\moment{\rv{g}}{p} \leq 2\sqrt{p}$ for any standard normal variable $\rv{g}$. Then, \cref{eq_bound_second_term_minkowski} follows by using \cref{eq_normal_application_2} in \cref{eq_bound_second_term_minkowski_0}.

\section{Identifying weakly dependent random variables}
\label{sec_conditioning_trick}
In \cref{section_lsi_tail_bounds}, we derived (in \cref{thm_LSI_main}) that a random vector (supported on a compact set) satisfies the logarithmic Sobolev inequality if it satisfies the Dobrushin's uniqueness condition (in \cref{def_dobrushin_condition}). Further, we also derived (\cref{thm_main_concentration}) tail bounds for a random vector satisfying the logarithmic Sobolev inequality. Combining the two, we see that in order to use the tail bound, the random vector needs to satisfy the Dobrushin's uniqueness condition, i.e, the elements of the random vector should be weakly dependent. In this section, we show that any random vector (outside Dobrushin's regime) that is a $\dGM$-Sparse Graphical Model (to be defined) can be reduced to satisfy the Dobrushin's uniqueness condition. In particular, we show that
by conditioning on a subset of the random vector, the unconditioned subset of the random vector (in the conditional distribution) are only weakly dependent. We exploit this trick in \cref{lemma_concentration_psi} and \cref{lemma_concentration_bar_psi} to enable application of the tail bound in  \cref{section_lsi_tail_bounds}. The result below is a generalization of the result in \cite{DaganDDA2021} for discrete random vectors to continuous random vectors.\\

\noindent We start by defining the notion of $\dGM$-Sparse Graphical Model. 
\begin{definition}[{$\dGM$-Sparse Graphical Model}]\label{def:tau_sgm}
	A pair of random vectors $\braces{\rvbx, \rvbz}$ supported on $\cX^p \times \cZ^{p_z}$ is a $\dGM$-Sparse Graphical Model for model-parameters 
	$\dGM \defn (\aGM, \eGM, \xmax, \ParameterMatrix)$
	and denoted by $\tSGM$ if $\cX = \{-\xmax, \xmax\}$, and
	\begin{enumerate}
		\item for any realization $\svbz \in \cZ^{p_z}$, the conditional probability distribution of $\rvbx$ given $\rvbz=\svbz$ is given by $\JointDistfun$ in \cref{eq_conditional_distribution_vay} for a vector $\ExternalField(\svbz) \in \real^p$ depending on $\svbz$ and a symmetric matrix $\ParameterMatrix \in \real^{p\times p} $ (independent of $\svbz$),
		\item $\max\braces{\max_{\svbz \in \cZ^{p_z}} \infnorm{\ExternalField(\svbz)}, \maxmatnorm{\ParameterMatrix}} \leq \aGM$, and
		\item $\infmatnorm{\ParameterMatrix} \leq \eGM$.
	\end{enumerate}
\end{definition}

\noindent Now, we provide the main result of this section. 
\newcommand{\conditioningtrickresultname}{Identifying weakly dependent random variables}
\begin{proposition}[{\conditioningtrickresultname}]\label{lemma_conditioning_trick}
	Given a pair of random vectors $\braces{\rvbx, \rvbz}$ supported on $\cX^p \times \cZ^{p_z}$ that is a $\tSGM$ (\cref{def:tau_sgm}) with $\dGM \defn (\aGM, \eGM, \xmax, \ParameterMatrix)$, and a scalar $\lambda \in (0, \eGM]$, there exists $\numindsets \defn 32  \eGM^2 \log 4p / \lambda^2$ subsets $\sets \subseteq [p]$ that satisfy the following properties:
	\begin{enumerate}[label=(\alph*)]
		\item\label{item:cardinality_independence_set} For any $t \in [p]$, we have $\sum_{u=1}^{\numindsets} \Indicator(t \in \setU) = \ceils{{\lambda \numindsets }/({8\eGM})}$.
		\item \label{item:conditional_sgm_independence_set} For any $u \in [\numindsets]$, 
		\begin{enumerate}[label=(\roman*)]
			\item \label{item:conditional_sgm} the pair of random vectors $\braces{\rvbx_{\setU}, (\rvbx_{-\setU}, \rvbz)}$ correspond to a $\tSGM[1]$ with $\dGM_1 \defn (\aGM+2\xmax \eGM, \lambda, \xmax, \ParameterMatrix_{\setU})$  where $\ParameterMatrix_{\setU} \defn \braces{\ParameterTU[tv]}_{t,v \in \setU}$, and
			\item \label{item:conditional_sgm_dob} the random vector $\rvbx_{\setU}$ conditioned on $(\rvbx_{-\setU}, \rvbz)$ satisfies the Dobrushin's uniqueness condition (\cref{def_dobrushin_condition}) with coupling matrix $2\sqrt{2} \xmax^2 |\ParameterMatrix_{\setU}|$ whenever $\lambda \in \Big(0, \frac{1}{2\sqrt{2}\xmax^2}\Big]$ with $\opnorm{|\ParameterMatrix_{\setU}|} \leq \lambda$.
		\end{enumerate}
	\end{enumerate}
\end{proposition}

\begin{proof}[Proof of \lowercase{\Cref{lemma_conditioning_trick}}: \conditioningtrickresultname]
	~~ We prove each part one-by-one using a generalization of \citet[Lemma. 12]{DaganDDA2021}. 
	
	Recall \citet[Lemma. 12]{DaganDDA2021}: Let $A\in \Reals^{p \times p}$ be a matrix with zeros on the diagonal and $\infmatnorm{A} \leq 1$. Let $0 < \eta < 1$. Then, there exists subsets $\barsets \subseteq [p]$ with $\wbar{\numindsets} \defn  32 \log 4p / \eta^2$ such that 
	\begin{enumerate}[label=(\alph*)]
		\item\label{dagan_a} For any $t \in [p]$, we have $\sum_{u=1}^{\wbar{\numindsets}} \Indicator(t \in \barsetU) = \ceils{{\eta \wbar{\numindsets} }/{8}}$, and
		\item For any $u \in [\wbar{\numindsets}]$ and $t \in \barsetU$, $\sum_{v \in \barsetU} |A_{tv}| \leq \eta$.
	\end{enumerate}
	\noindent We claim that \citet[Lemma. 12]{DaganDDA2021} holds even when $A$ does not have zeros on the diagonal. The proof is exactly the same as the proof of \citet[Lemma. 12]{DaganDDA2021}.
	
	\paragraph{Proof of part~\cref{item:cardinality_independence_set}}
	From \cref{def:tau_sgm}, for any realization $\svbz \in \cZ^{p_z}$, the conditional probability distribution of $\rvbx$ given $\rvbz=\svbz$ is given by $\JointDistfun$ in \cref{eq_conditional_distribution_vay} where $\ExternalField(\svbz) \in \real^p$ is a vector and $\ParameterMatrix \in \real^{p\times p} $ is a symmetric matrix with $\infmatnorm{\ParameterMatrix} \leq \eGM$. Consider the matrix $A \defn \frac{1}{\eGM} \ParameterMatrix$. Since $\infmatnorm{A} \leq 1$, we can apply the generalization of \citet[Lemma. 12]{DaganDDA2021} on $A$ with $\eta = \frac{\lambda}{\eGM}$. Then part~\cref{item:cardinality_independence_set} follows directly from
	\citet[Lemma. 12.1]{DaganDDA2021}. 

	\paragraph{Proof of part~\cref{item:conditional_sgm_independence_set}\cref{item:conditional_sgm}}
	To prove this part, consider the distribution of $\rvbx_{\setU}$ conditioned on $\rvbx_{-\setU} = \svbx_{-\setU}$ and $\rvbz = \svbz$ for any $u \in [\numindsets]$, i.e., $f_{\rvbx_{\setU} | \rvbx_{-\setU},\rvbz} (\svbx_{\setU} | \svbx_{-\setU}, \svbz; \ExternalField(\svbz), \ParameterMatrix) \defn f (\svbx_{\setU} | \svbx_{-\setU}, \svbz; \ExternalField(\svbz), \ParameterMatrix)$ as follows
	\begin{align}
f (\svbx_{\setU} | \svbx_{-\setU}, \svbz; \ExternalField(\svbz), \ParameterMatrix) \propto \exp \biggparenth{ \! \sum_{t \in \setU} \!\!\Bigparenth{\!\ExternalFieldt(\svbz) \!+\! 2\!\sum_{v \notin \setU} \!\!\ParameterTU[tv] x_{v}} x_t \!+\! \sum_{t \in \setU} \!\! \sum_{~ v \in \setU} \!\ParameterTU[tv] x_t x_v}. 	
 \label{eq_conditional_dist_xIj}
	\end{align}
	We can re-parameterize $f(\svbx_{\setU} | \svbx_{-\setU}, \svbz; \ExternalField(\svbz), \ParameterMatrix)$ in \cref{eq_conditional_dist_xIj} as follows
	\begin{align}
	&f_{\rvbx_{\setU} | \rvbx_{-\setU},\rvbz} (\svbx_{\setU} | \svbx_{-\setU}, \svbz; \ConditioningField(\svbz, \svbx_{-\setU}), \ConditioningMatrix)  \propto  \exp \Bigparenth{ \normalbrackets{\ConditioningField(\svbz, \svbx_{-\setU})}\tp \svbx_{\setU} + \svbx_{\setU}\tp\ConditioningMatrix \svbx_{\setU}}
 \end{align}
 where
  \begin{align}
	&\ConditioningField(\svbz, \svbx_{-\setU}) \in \Reals^{|\setU| \times 1},
	\stext{with} \ConditioningFieldU(\svbz, \svbx_{-\setU}) \defn \ExternalFieldt(\svbz) + 2\sum_{k \notin \setU} \ParameterTU[tv] x_{k} 
	\stext{for} t \in \setU, \stext{and} \label{eq_parameter_mapping_0}\\
	&\ConditioningMatrix = \ConditioningMatrix\tp  \in \Reals^{|\setU| \times |\setU|}  
	\stext{with}
	\ConditioningParameter[tv] \defn \ParameterTU[tv],
	\stext{for all} t, v \in \setU.
	\label{eq_parameter_mapping}
	\end{align}
	Now, to show that the random vector $\rvbx_{\setU}$ conditioned on $\rvbx_{-\setU}$ and $\rvbz$ corresponds to an $\tSGM[1]$ with $\dGM_1 \defn (\aGM+2\xmax \eGM, \lambda,\xmax, \ParameterMatrix_{\setU})$, it suffices to establish that
	\begin{align}
	\label{eq:suff_steps_sgm}
	\max\braces{\max_{\svbz \in \cZ^{p_z}} \infnorm{\ConditioningField(\svbz, \svbx_{-\setU})}, \maxmatnorm{\ConditioningMatrix}} \sless{(i)} \aGM+2\xmax \eGM \qtext{and} \infmatnorm{\ConditioningMatrix} \sless{(ii)} \lambda.
	\end{align}
	To establish~(i) in \cref{eq:suff_steps_sgm}, we note that
	\begin{align}
	\maxmatnorm{\ConditioningMatrix} & \sless{\cref{eq_parameter_mapping}} \maxmatnorm{\ParameterMatrix} \sless{(a)}  \aGM \qtext{and} \label{eq:phitheta_bound_0}\\
	\infnorm{\ConditioningField(\svbz, \svbx_{-\setU})} & \sless{(b)} \infnorm{\ExternalField(\svbz)} + 2\max_{t \in \setU}\sonenorm{\ParameterRowt} \infnorm{\svbx} \sless{(c)} \infnorm{\ExternalField(\svbz)} + 2\xmax \infmatnorm{\ParameterMatrix} \\
 & \sless{(d)} \aGM+2\xmax \eGM, \label{eq:phitheta_bound_1}
	\end{align}
	where $(a)$ and $(d)$ follow from \cref{def:tau_sgm}, $(b)$ follows from \cref{eq_parameter_mapping_0} and the triangle inequality, and $(c)$ follows from the definition of $\infmatnorm{\cdot}$ and \cref{def:tau_sgm}. Then, from \cref{eq:phitheta_bound_0} and \cref{eq:phitheta_bound_1}, we have
	\begin{align}
	\max\braces{\max_{\svbz \in \cZ^{p_z}} \infnorm{\ConditioningField(\svbz, \svbx_{-\setU})}, \maxmatnorm{\ConditioningMatrix}} \leq \aGM+2\xmax \eGM,
	\end{align}
	as claimed. Next, to establish~(ii) in \cref{eq:suff_steps_sgm}, we again apply  the generalization of \citet[Lemma. 12]{DaganDDA2021} on the matrix $A = \frac{1}{\eGM} \ParameterMatrix$ with $\eta = \frac{\lambda}{\eGM}$. Then, we have
	\begin{align}
	\sum_{v \in \setU} \biggabs{\frac{\ParameterTU[tv]}{\eGM}} \leq \frac{\lambda}{\eGM} 
	\qtext{for all $t \in \setU$, $u \in [\numindsets]$.}
	\label{eq_lemma12_bound}
	\end{align}
	Therefore, we have
	\begin{align}
	\infmatnorm{\ConditioningMatrix} = \max_{t \in \setU} \Bigparenth{\sum_{v \in \setU} \bigabs{\ConditioningParameter}} \sequal{\cref{eq_parameter_mapping}} \max_{t \in \setU} \Bigparenth{\sum_{v \in \setU} \bigabs{\ParameterTU[tv]}} \sless{\cref{eq_lemma12_bound}} \lambda, \label{eq_bound_Upsilon_inf_norm}
	\end{align}   
	as desired. The proof for this part is now complete.
	
	\paragraph{Proof of part~\cref{item:conditional_sgm_independence_set}\cref{item:conditional_sgm_dob}}
	We start by noting that the operator norm of a symmetric matrix is bounded by the infinity norm of the matrix. Then, from the analysis in part~\cref{item:conditional_sgm_independence_set}~\cref{item:conditional_sgm}, for any $u \in \setU$, we have
	\begin{align}
	\opnorm{|\ParameterMatrix_{\setU}|} \leq \infmatnorm{|\ParameterMatrix_{\setU}|} \sequal{\cref{eq_parameter_mapping}} \infmatnorm{|\ConditioningMatrix|} \sless{\cref{eq_bound_Upsilon_inf_norm}} \lambda.
	\end{align}
	Therefore, $\infmatnorm{2\sqrt{2}\xmax^2 |\ParameterMatrix_{\setU}|} \leq 1$ whenever $\lambda \leq 1/2\sqrt{2}\xmax^2$. It remains to show %
	that for every $u \in [\numindsets]$, $t \in \setU, v \in \setU \!\setminus\! \{t\}$, $\rvbz = \svbz$, and $\svbx_{-t}, \tsvbx_{-t} \in \cX^{p-1}$ differing only in the $v^{th}$ coordinate,
	\begin{align}
	\TV{f_{\rvx_t | \rvbx_{-t} = \svbx_{-t}, \rvbz = \svbz}}{f_{\rvx_t | \rvbx_{-t} = \tsvbx_{-t}, \rvbz = \svbz}} \leq 2\sqrt{2} \xmax^2 |\ParameterTU[tv]|.
	\end{align}
	To that end, fix any $u \in [\numindsets]$, any $t \in \setU$, any $v \in \setU \!\setminus\! \{t\}$, any $\rvbz = \svbz$, and any $\svbx_{-t}, \tsvbx_{-t} \in \cX^{p-1}$ differing only in the $v^{th}$ coordinate. We have
	\begin{align}
	\TV{f_{\rvx_t | \rvbx_{-t} = \svbx_{-t}, \rvbz = \svbz}}{f_{\rvx_t | \rvbx_{-t} = \tsvbx_{-t}, \rvbz = \svbz}}^2 & \sless{(a)} \frac{1}{2}\KLD{f_{\rvx_t | \rvbx_{-t} = \svbx_{-t}, \rvbz = \svbz}}{f_{\rvx_t | \rvbx_{-t} = \tsvbx_{-t}, \rvbz = \svbz}}\\
	& \sequal{(b)} \frac{1}{2} (2\ParameterTU[tv] x_v - 2 \ParameterTU[tv] \tx_v)^2 \xmax^2 \sless{(c)} 8\xmax^4 \ParameterTU[tv]^2,
	\end{align}
	where $(a)$ follows from Pinsker's inequality, $(b)$ follows by (i) applying 
	\cite[Theorem 1]{BusaFSZ2019} to the exponential family parameterized as per $f_{\rvx_t | \rvbx_{-t}, \rvbz}$ in \cref{eq_conditional_dist}, (ii) noting that $f_{\rvx_t | \rvbx_{-t} = \svbx_{-t}, \rvbz =\svbz} \propto \exp\bigparenth{ \normalbrackets{\ExternalFieldt(\svbz) + 2\ParameterRowttt\tp \svbx_{-t}} x_t + \ParameterTU[tt]\cx_t}$ and $f_{\rvx_t | \rvbx_{-t} = \tsvbx_{-t}, \rvbz =\svbz} \propto \exp\bigparenth{ \normalbrackets{\ExternalFieldt(\svbz) + 2\ParameterRowttt\tp \tsvbx_{-t}} x_t + \ParameterTU[tt]\cx_t}$ where $\cx_t \defn x_t^2 - \xmax^2/3$, and (iii) noting that the Hessian of the log partition function for any regular exponential family is the covariance matrix of the associated sufficient statistic which is bounded by $\xmax^2$ when $\cX = \{-\xmax, \xmax\}$, and $(c)$ follows because $x_v, \tx_v \in \{-\xmax, \xmax\}$. This completes the proof.
\end{proof}

\section{Supporting concentration results}
In this section, we provide a corollary of \cref{thm_main_concentration} that is used to prove the concentration results in \cref{lemma_concentration_psi} and \cref{lemma_concentration_bar_psi}. To show any concentration result for the random vector $\rvbx$ conditioned on $\rvbz$ via \cref{thm_main_concentration}, we need $\rvbx | \rvbz$ to satisfy the logarithmic Sobolev inequality (defined in \cref{eq_LSI_definition}). From \cref{thm_LSI_main}, for this to be true, we need the random vector $\rvx_t$ conditioned on $(\rvbx_{-t}, \rvbz)$ to satisfy the logarithmic Sobolev inequality for all $t \in [p]$. In the result below, we show this holds with a proof in \cref{proof_lsi_one_dim}. We define a $\tau \defeq (\aGM, \eGM, \xmax, \ParameterMatrix)$-dependent constant:
\begin{align}
\cthree \defn  \exp{(\xmax(\aGM+ 2\eGM \xmax))}. \label{eq_constants_3_cont}
\end{align}

\newcommand{\lsionedimresultname}{Logarithmic Sobolev inequality for $\rvx_t | \rvbx_{-t}, \rvbz$}
\begin{lemma}[{\lsionedimresultname}]
	\label{lemma_lsi_one_dim}
	Given a pair of random vectors $\braces{\rvbx, \rvbz}$ supported on $\cX^p \times \cZ^{p_z}$ that is a $\tSGM$ (\cref{def:tau_sgm}) with $\dGM \defn (\aGM, \eGM, \xmax, \ParameterMatrix)$, $\rvx_t | \rvbx_{-t}, \rvbz$ satisfies $\mathrm{LSI}_{\rvx_t | \rvbx_{-t} = \svbx_{-t} , \rvbz = \svbz}\Big(\frac{8\xmax^2}{\pi^2} \cthree[2]\Big)$ for all $t \in [p]$, $\svbx_{-t} \in \cX^{p-1}$, and $\svbz \in \cZ^{p_z}$.

\end{lemma}

\noindent Now, we state the desired corollary of \cref{thm_main_concentration} with a proof in \cref{proof_of_coro}. The corollary makes use of some $\tau \defeq (\aGM, \eGM, \xmax, \ParameterMatrix)$-dependent constants:
\begin{align}
\cfour \defn 1 + \aGM \xmax + 4\xmax^2 \eGM \qtext{and} \cfive \defn \frac{32\xmax^3\cthree[4]}{\pi^2}.\label{eq_constants_4_cont}
\end{align}
\newcommand{\suppconcresultname}{Supporting concentration bounds}
\begin{corollary}[{\suppconcresultname}]\label{coro}
	Suppose a pair of random vectors $\braces{\rvbx, \rvbz}$ supported on $\cX^p \times \cZ^{p_z}$ corresponds to a $\tSGM$ (\cref{def:tau_sgm}) with $\dGM \defn (\aGM, \eGM, \xmax, \ParameterMatrix)$, and $\rvbx$ conditioned on $\rvbz$ satisfies the Dobrushin's uniqueness condition (\cref{def_dobrushin_condition}) with coupling matrix $\bParameterMatrix$. For any $\ExternalField, \bExternalField \in \ParameterSet_{\ExternalField}$ and $\ParameterMatrix \in \ParameterSet_{\ParameterMatrix}$, define the functions $q_1$ and $q_2$ as
	\begin{align}
	q_1(\rvbx) \defeq \sump \normalparenth{\omt \rvx_t}^2
	\qtext{and}
	q_2(\rvbx) \defeq \sump \omt \rvx_t \exp\Bigparenth{-\normalbrackets{\ExternalFieldt + 2 \ParameterRowttt\tp \rvbx_{-t}} \rvx_t - \ParameterTU[tt] \crx_t},
	\end{align}
	where $\om =  \bExternalField - \ExternalField$ and $\crx_t \defn \rvx_t^2 - \xmax^2/3$.
	Then, for any $\varepsilon > 0$
	\begin{align}
	\Probability\Bigbrackets{\bigabs{q_i(\rvbx) - \Expectation\bigbrackets{q_i(\rvbx) \big| \rvbz}} & \geq \varepsilon \Big| \rvbz} \leq \exp\biggparenth{ \dfrac{-c\bigparenth{1-\opnorm{\bParameterMatrix}}^4\varepsilon^2}{c_i \stwonorm{\om}^2}}
	\qtext{for} i = 1, 2, \label{eq_coro_combined}
	\end{align}
	where $c$ is a universal constant, $c_1 \defn 16 \aGM^2\xmax^2 \cfive[2]$, and $c_2 \defn \cthree[2] \cfour[2] \cfive[2]$ with $\cthree$ defined in \cref{eq_constants_3_cont} and $\cfour$ and $\cfive$ defined in \cref{eq_constants_4_cont}.
\end{corollary}

\subsection{Proof of \cref{lemma_lsi_one_dim}: \lsionedimresultname}
\label{proof_lsi_one_dim}
Let $\rvu$ be the uniform distribution on $\cX$. Then, $\rvu$ satisfies $\mathrm{LSI}_{\rvu}\Big(\frac{8\xmax^2}{\pi^2}\Big)$ (see \citet[Corollary. 2.4]{ghang2014sharp}). Then, using the Holley-Stroock perturbation principle (see \citet[Page. 31]{HS1986}, \citet[Lemma. 1.2]{Ledoux2001}), for every $t\in [p]$, $\svbx_{-t} \in \cX^{p-1}$, and $\svbz \in \cZ^{p_z}$, $\rvx_t | \rvbx_{-t} = \svbx_{-t}, \rvbz = \svbz$ satisfies the logarithmic Sobolev inequality with a 
constant bounded by $$\frac{8\xmax^2 \exp(\sup_{x_t \in \cX} \psi(x_t; \svbx_{-t}, \svbz) - \inf_{x_t \in \cX} \psi(x_t; \svbx_{-t}, \svbz))}{\pi^2},$$ where $\psi(x_t; \svbx_{-t}, \svbz) \defn - \normalbrackets{\ExternalFieldt(\svbz) + 2\ParameterRowttt\tp \svbx_{-t}} x_t - \ParameterTU[tt] \cx_t$ where $\cx_t = x_t^2 - \xmax^2/3$. We have
\begin{align}
\exp(\sup_{x_t \in \cX} \psi(x_t; \svbx_{-t}, \svbz) \!-\! \inf_{x_t \in \cX} \psi(x_t; \svbx_{-t}, \svbz))
& \sless{(a)} \!   \exp\bigparenth{ 2\bigabs{\ExternalFieldt(\svbz)\! +\! 2\ParameterRowttt\tp \svbx_{-t}} \xmax \!+\! \ParameterTU[tt] \xmax^2}\\ &\sless{(b)} \exp\big((2\aGM + 4\eGM\xmax)\xmax\big) \sequal{\cref{eq_constants_3_cont}} \cthree[2],
\end{align}
where $(a)$ follows from \cref{def:tau_sgm} and $(b)$ follows by using \cref{def:tau_sgm} along with triangle inequality and Cauchy–Schwarz inequality.

\subsection{Proof of \cref{coro}: \suppconcresultname}
\label{proof_of_coro}
To apply \cref{thm_main_concentration} to the random vector $\rvbx$ conditioned on $\rvbz$, we need $\rvbx | \rvbz$ to satisfy the logarithmic Sobolev inequality. From \cref{thm_LSI_main}, this is true if (i) $f_{\min} = \min_{t \in [p], \svbx \in \cX^p, \svbz \in \cX^{p_z}}$ $f_{\rvx_t | \rvbx_{-t}, \rvbz}(x_t | \svbx_{-t}, \svbz) > 0$ (see \cref{eq:smin}), (ii) $\rvbx | \rvbz$ satisfies the Dobrushin’s uniqueness condition, and (iii) $\rvx_t | \rvbx_{-t}, \rvbz$ satisfies the logarithmic Sobolev inequality for all $t \in [p]$. By assumption,  $\rvbx | \rvbz$ satisfies the Dobrushin’s uniqueness condition with coupling matrix $\bParameterMatrix$. From \cref{lemma_lsi_one_dim}, $\rvx_t | \rvbx_{-t}, \rvbz$ satisfies $\mathrm{LSI}_{\rvx_t | \rvbx_{-t} = \svbx_{-t} , \rvbz = \svbz}\Big(\frac{8\xmax^2\cthree[2]}{\pi^2}\Big)$. It remains to show that $f_{\min} > 0$. Consider any $t \in [p]$, any $\svbx \in \cX^p$, and any $\svbz \in \cX^{p_z}$. Let $\cx_t = x_t^2 - \xmax^2/3$. We have
\begin{align}
f_{\rvx_t | \rvbx_{-t}, \rvbz}(x_t | \svbx_{-t}, \svbz) 
& \sequal{(a)}  \frac{\exp\Bigparenth{\normalbrackets{\ExternalFieldt(\svbz) + 2\ParameterRowttt\tp \svbx_{-t} } x_t  + \ParameterTU[tt] \cx_t}}{
	\int_{\cX}
	\exp\Bigparenth{\normalbrackets{\ExternalFieldt(\svbz) + 2\ParameterRowttt\tp \svbx_{-t}} x_t + \ParameterTU[tt] \cx_t} d x_t} \\
& \sgreat{(b)}  \frac{\exp\Bigparenth{-\normalabs{\ExternalFieldt(\svbz) + 2\ParameterRowttt\tp \svbx_{-t}} \xmax - \ParameterTU[tt] \xmax^2}}{
	\int_{\cX}
	\exp\Bigparenth{\normalabs{\ExternalFieldt(\svbz) + 2\ParameterRowttt\tp \svbx_{-t}} \xmax + \ParameterTU[tt] \xmax^2} d x_t}\\
& \sgreat{(c)}  \frac{\exp\Bigparenth{-\bigparenth{\normalabs{\ExternalField(\svbz)} + 2\sonenorm{\ParameterRowttt} \sinfnorm{\svbx}} \xmax - \ParameterTU[tt] \xmax^2}}{
	\int_{\cX}
	\exp\Bigparenth{\bigparenth{\normalabs{\ExternalField(\svbz)} + 2\sonenorm{\ParameterRowttt} \sinfnorm{\svbx}} \xmax + \ParameterTU[tt] \xmax^2} d x_t}\\
& \sgreat{(d)} \frac{\exp\Bigparenth{-\normalparenth{\aGM + 2\eGM \xmax} \xmax}}{
	\int_{\cX}
	\exp\Bigparenth{\normalparenth{\aGM + 2\eGM \xmax} \xmax} d x_t} \sequal{(e)} \frac{1}{2\xmax \cthree[2]},
\end{align}
where $(a)$ follows from \cref{eq_conditional_dist}, $(b)$ and $(d)$ follow from \cref{def:tau_sgm}, $(c)$ follows by triangle inequality and Cauchy–Schwarz inequality, and $(e)$ follows because $\int_{\cX} dx_t = 2\xmax$. Therefore, $f_{\min} = \frac{1}{2\xmax \cthree[2]}$. Putting (i), (ii), and (iii) together, and using \cref{thm_LSI_main}, we see that $\rvbx | \rvbz$ satisfies $\mathrm{LSI}_{\rvbx}\Bigparenth{\frac{\cfive}{\normalparenth{1-\opnorm{\bParameterMatrix}}^2}}$ where $\cfive$ was defined in \cref{eq_constants_4_cont}. 

Now, we apply \cref{thm_main_concentration} to $q_1$ and $q_2$ one-by-one. The general strategy is to choose appropriate pseudo derivatives and pseudo Hessians for both $q_1$ and $q_2$, and evaluate the corresponding terms appearing in \cref{thm_main_concentration}.

\paragraph{Concentration for $q_1$}
Fix any $\svbx \in \cX^p$. We start by decomposing $q_1(\svbx)$ as follows
\begin{align}
q_1(\svbx) = \bom \tp r(\svbx), \label{eq_decomposition_cont_quad}
\end{align}
where $\bom \defn (\omt[1]^2, \cdots, \omt[p]^2)$ and $r(\svbx) \defn (r_1(\svbx), \cdots, r_p(\svbx))$ with $r_t(\svbx) = x_t^2$ for every $t \in [p]$. Next, we define $H : \cX^p \to \Reals^{p \times p}$ such that 
\begin{align}
H_{tu}(\svbx) = \frac{dr_u(\svbx)}{dx_t}  \qtext{for every $t,u \in [p]$.} \label{eq_H_matrix_cont_quad}
\end{align}

\paragraph{Pseudo derivative}
We bound the $\ell_2$ norm of the gradient of $q_1(\svbx)$ as follows
\begin{align}
\twonorm{\nabla q_1(\svbx)}^2 = \sump \Bigparenth{\frac{d q_1(\svbx)}{dx_t}}^2 & \sequal{\cref{eq_decomposition_cont_quad}} \sump \Bigparenth{\frac{\bom \tp d r(\svbx)}{dx_t}}^2 \\
& \sequal{\cref{eq_H_matrix_cont_quad}} \twonorm{H(\svbx) \bom}^2 \\
& \sless{(a)} \opnorm{H(\svbx)}^2 \twonorm{\bom}^2  \sless{(b)} \onematnorm{H(\svbx)}  \infmatnorm{H(\svbx)} \twonorm{\bom}^2, \label{eq:pseudo_derivative_cont_quad}
\end{align}
where $(a)$ follows because induced matrix norms are submultiplicative and $(b)$ follows because the matrix operator norm is bounded by square root of the product of matrix one norm and matrix infinity norm. Now, we claim that the one norm and the infinity norm of $H(\svbx)$ are bounded as follows
\begin{align}
\max\braces{\max_{\svbx \in \cX^p} \onematnorm{H(\svbx)}, \max_{\svbx \in \cX^p} \infmatnorm{H(\svbx)}} \leq 2\xmax. \label{eq_H_one_inf_bound_cont_quad}
\end{align}
Taking this claim as given at the moment, we continue with our proof. Combining \cref{eq:pseudo_derivative_cont_quad,eq_H_one_inf_bound_cont_quad}, we have 
\begin{align}
\max_{\svbx \in \cX^p} \twonorm{\nabla q_1(\svbx)}^2 \leq  4\xmax^2 \twonorm{\bom}^2 = 4\xmax^2 \!\sump\! \omt^4 \leq 4\xmax^2 \maxp[u] \omt[u]^2 \sump\! \omt^2 \!\sless{(a)}\! 16 \xmax^2 \aGM^2 \twonorm{\om}^2,
\end{align}
where $(a)$ follows because $\om \in 2\ParameterSet_{\ExternalField}$. Therefore, we choose the pseudo derivative (see \cref{def_pseudo_der_hes}) as follows
\begin{align}
\tnabla q_1(\svbx) = 4\xmax \aGM \twonorm{\om}. \label{eq_chosen_pseudo_der_cont_quad}
\end{align}

\paragraph{Pseudo Hessian}
Fix any $\rho \in \Reals$. We bound $\stwonorm{\nabla(\rho\tp \tnabla  q_1(\svbx))}^2$ (see \cref{def_pseudo_der_hes}) as follows
\begin{align}
\stwonorm{\nabla(\rho\tp \tnabla  q_1(\svbx))}^2 = \sump[u] \Bigparenth{\frac{d \rho\tp \tnabla  q_1(\svbx)}{dx_u}}^2 \sequal{\cref{eq_chosen_pseudo_der_cont_quad}} 0.
\end{align}
Therefore, we choose the pseudo Hessian (see \cref{def_pseudo_der_hes}) as follows
\begin{align}
\tnabla^2 q_1(\svbx) = 0. \label{eq_chosen_pseudo_Hess_cont_quad}
\end{align}
The concentration result in \cref{eq_coro_combined} for $q_1$ follows by applying \cref{thm_main_concentration} with the pseudo discrete derivative defined in \cref{eq_chosen_pseudo_der_cont_quad} and the pseudo discrete Hessian defined in \cref{eq_chosen_pseudo_Hess_cont_quad}.\\

\noindent It remains to show that the one-norm and the infinity-norm of $H(\svbx)$ are bounded as in \cref{eq_H_one_inf_bound_cont_quad}.
\paragraph{Bounds on the one-norm and the infinity-norm of $H(\svbx)$} We have 
\begin{align}\label{eq_matrix_H_quad}
H_{tu}(\svbx) = 
\begin{cases}
2x_t \qtext{if} t = u, \\
0 \qtext{otherwise.}
\end{cases}
\end{align}
Therefore,
\begin{align}
\onematnorm{H(\svbx)} 
& = \max_{u \in [p]} \sum_{t \in [p]} \normalabs{H_{tu}(\svbx)} \sless{\cref{eq_matrix_H_quad}} \max_{u \in [p]} 2 \normalabs{x_u} \sless{(a)} 2\xmax \qtext{and}\\
\infmatnorm{H(\svbx)} 
& = \max_{t \in [p]} \sum_{u \in [p]} |H_{tu}(\svbx)|  \sless{\cref{eq_matrix_H_quad}} \max_{t \in [p]} 2 \normalabs{x_t} \sless{(a)} 2\xmax,
\end{align}
where $(a)$ follows from \cref{def:tau_sgm}.

\paragraph{Concentration for $q_2$}
Fix any $\svbx \in \cX^p$. We start by decomposing $q_2(\svbx)$ as follows
\begin{align}
q_2(\svbx) = \om\tp r(\svbx), \label{eq_decomposition_cont}
\end{align}
where $r(\svbx) \defn (r_1(\svbx), \cdots, r_p(\svbx))$ with $r_t(\svbx) = x_t \exp\bigparenth{-\normalbrackets{\ExternalFieldt + 2 \ParameterRowttt\tp \svbx_{-t}} x_t - \ParameterTU[tt] \cx_t}$ for every $t \in [p]$. Next, we define $H : \cX^p \to \Reals^{p \times p}$ such that 
\begin{align}
H_{tu}(\svbx) = \frac{dr_u(\svbx)}{dx_t}  \qtext{for every $t,u \in [p]$.} \label{eq_H_matrix_cont}
\end{align}

\paragraph{Pseudo derivative}
We bound the $\ell_2$ norm of the gradient of $q_2(\svbx)$ as follows
\begin{align}
\twonorm{\nabla q_2(\svbx)}^2  = \sump \Bigparenth{\frac{d q_2(\svbx)}{dx_t}}^2 & \sequal{\cref{eq_decomposition_cont}} \sump \Bigparenth{\frac{\om\tp d r(\svbx)}{dx_t}}^2 \\ & \sequal{\cref{eq_H_matrix_cont}} \twonorm{H(\svbx) \om}^2 \\
& \sless{(a)} \opnorm{H(\svbx)}^2 \twonorm{\om}^2 \sless{(b)} \onematnorm{H(\svbx)}  \infmatnorm{H(\svbx)} \twonorm{\om}^2, \label{eq:pseudo_derivative_cont}
\end{align}
where $(a)$ follows because induced matrix norms are submultiplicative and $(b)$ follows because the matrix operator norm is bounded by square root of the product of matrix one norm and matrix infinity norm. Now, we claim that the one norm and the infinity norm of $H(\svbx)$ are bounded as follows
\begin{align}
\max\braces{\max_{\svbx \in \cX^p} \onematnorm{H(\svbx)}, \max_{\svbx \in \cX^p} \infmatnorm{H(\svbx)}} \leq \cthree \cfour. \label{eq_H_one_inf_bound_cont}
\end{align}
where $\cthree$ and $\cfour$ were defined in \cref{eq_constants_3_cont} and \cref{eq_constants_4_cont} respectively. Taking this claim as given at the moment, we continue with our proof. Combining \cref{eq:pseudo_derivative_cont,eq_H_one_inf_bound_cont}, we have 
\begin{align}
\max_{\svbx \in \cX^p} \twonorm{\nabla q_2(\svbx)}^2 \leq \cthree[2] \cfour[2] \twonorm{\om}^2.
\end{align}
Therefore, we choose the pseudo derivative (see \cref{def_pseudo_der_hes}) as follows
\begin{align}
\tnabla q_2(\svbx) = \cthree \cfour \twonorm{\om}. \label{eq_chosen_pseudo_der_cont}
\end{align}

\paragraph{Pseudo Hessian}
Fix any $\rho \in \Reals$. We bound $\stwonorm{\nabla(\rho\tp \tnabla  q_2(\svbx))}^2$ (see \cref{def_pseudo_der_hes}) as follows
\begin{align}
\stwonorm{\nabla(\rho\tp \tnabla  q_2(\svbx))}^2 = \sump[u] \Bigparenth{\frac{d \rho\tp \tnabla  q_2(\svbx)}{dx_u}}^2 \sequal{\cref{eq_chosen_pseudo_der_cont}} 0.
\end{align}
Therefore, we choose the pseudo Hessian (see \cref{def_pseudo_der_hes}) as follows
\begin{align}
\tnabla^2 q_2(\svbx) = 0. \label{eq_chosen_pseudo_Hess_cont}
\end{align}
The concentration result in \cref{eq_coro_combined} for $q_1$ follows by applying \cref{thm_main_concentration} with the pseudo discrete derivative defined in \cref{eq_chosen_pseudo_der_cont} and the pseudo discrete Hessian defined in \cref{eq_chosen_pseudo_Hess_cont}.\\

\noindent It remains to show that the one-norm and the infinity-norm of $H(\svbx)$ are bounded as in \cref{eq_H_one_inf_bound_cont}.
\paragraph{Bounds on the one-norm and the infinity-norm of $H$} We have 
\begin{align}\label{eq_matrix_H}
H_{tu}(\svbx) = 
\begin{cases}
\bigbrackets{1 - \normalbrackets{\ExternalFieldt[u] + 2\ParameterRowt[u]\tp \svbx} x_u}  \exp\bigparenth{-\normalbrackets{\ExternalFieldt[u] + 2\ParameterRowttt[u]\tp \svbx_{-u}} x_u - \ParameterTU[uu] \cx_u}  \qtext{if} t = u, \\
-2\ParameterTU[tu] x_u^2 \exp\bigparenth{-\normalbrackets{\ExternalFieldt[u] + 2\ParameterRowttt[u]\tp \svbx_{-u}} x_u - \ParameterTU[uu] \cx_u} \qquad\qquad\qquad \qtext{otherwise.}
\end{cases}
\end{align}
Therefore,
\begin{align}
\onematnorm{H(\svbx)} 
& = \max_{u \in [p]} \sum_{t \in [p]} \normalabs{H_{tu}(\svbx)} \\
& \sequal{\cref{eq_matrix_H}} \max_{u \in [p]} \bigabs{1 \!-\! \normalbrackets{\ExternalFieldt[u] \!+\! 2\ParameterRowt[u]\tp \svbx} x_u} \exp\bigparenth{-\normalbrackets{\ExternalFieldt[u] + 2\ParameterRowttt[u]\tp \svbx_{-u}} x_u - \ParameterTU[uu] \cx_u}  \\
& \qquad \qquad + 2\max_{u \in [p]} x_u^2  \exp\bigparenth{-\normalbrackets{\ExternalFieldt[u] + 2\ParameterRowttt[u]\tp \svbx_{-u}} x_u - \ParameterTU[uu] \cx_u}  \sum_{t \neq u}  \normalabs{\ParameterTU[tu]}\\
& \sless{(a)} (1 + \aGM \xmax + 4\xmax^2 \eGM) \exp{(\xmax(\aGM+ 2\eGM \xmax))} \sequal{(b)}  \cthree \cfour,
\end{align}
where $(a)$ follows from \cref{def:tau_sgm} along with triangle inequality and Cauchy–Schwarz inequality and $(b)$ follows from \cref{eq_constants_3_cont,eq_constants_4_cont}. Similarly, we have
\begin{align}
\infmatnorm{H(\svbx)} 
& = \max_{t \in [p]} \sum_{u \in [p]} |H_{tu}(\svbx)|  \\
& \sequal{\cref{eq_matrix_H}} \max_{t \in [p]} \bigabs{1 \!-\! \normalbrackets{\ExternalFieldt \!+\! 2\ParameterRowt\tp \svbx} x_t} \exp\bigparenth{-\normalbrackets{\ExternalFieldt[t] + 2\ParameterRowttt[t]\tp \svbx_{-t}} x_t - \ParameterTU[tt] \cx_t} \\
& \qquad \qquad + 2\max_{t \in [p]} \sum_{u \neq t}  \normalabs{\ParameterTU[tu]} x_u^2  \exp\bigparenth{-\normalbrackets{\ExternalFieldt[u] + 2\ParameterRowttt[u]\tp \svbx_{-u}} x_u - \ParameterTU[uu] \cx_u} \\
& \sless{(a)} (1 + \aGM \xmax + 4\xmax^2 \eGM) \exp{(\xmax(\aGM+ 2\eGM \xmax))} \sequal{(b)} \cthree \cfour,
\end{align}
where $(a)$ follows from \cref{def:tau_sgm} along with triangle inequality and Cauchy–Schwarz inequality and $(b)$ follows from \cref{eq_constants_3_cont,eq_constants_4_cont}.
\bibliographystyle{abbrvnat}
\bibliography{main}

\end{document}